\documentclass[a4paper,12pt]{article}

\usepackage[whole]{bxcjkjatype} 
\usepackage[unicode]{hyperref} 

\usepackage[utf8]{inputenc}
\usepackage{amsthm}
\usepackage{amsmath, amssymb}
\usepackage{type1cm}
\usepackage{bm}
\usepackage{comment}
\usepackage{url}
\usepackage[margin=31mm]{geometry}

\newtheoremstyle{jplain}% name
{}% space above
{}% space below
{\normalfont}% body font
{}% indent amount
{\bfseries}% theorem head font
{.}% punctuation after theorem head
{4pt}% space after theorem head (default: 5pt)
{\thmname{#1}\thmnumber{#2}\thmnote{\hspace{2pt}(#3)}}% theorem head spec

\theoremstyle{jplain}

\newtheorem{thm}{定理}
\numberwithin{thm}{subsection}
\newtheorem*{thm*}{定理} % 番号を付けない
\newtheorem{lem}[thm]{補題}
\newtheorem{cor}[thm]{系}
\newtheorem{prop}[thm]{命題}
\newtheorem{defn}[thm]{定義}
\newtheorem*{defn*}{定義}
\newtheorem{rem}[thm]{注意}
\newcommand{\R}{\mathbb{R}}
\newcommand{\N}{\mathbb{N}}
\newcommand{\A}{\mathbb{A}}
\newcommand{\T}{\mathrm{T}}
\newcommand{\B}{\mathcal{B}}
\newcommand{\M}{\mathbb{M}}
\newcommand{\F}{\mathcal{F}}
\newcommand{\Z}{\mathbb{Z}}
\newcommand{\Q}{\mathbb{Q}}
\def\ip<#1>{\langle #1 \rangle}
\def\l<#1>{\left\langle #1 \right\rangle}

\renewcommand\proofname{\bf 証明}

\begin{document}

\title{ニューラルネットワークの万能近似定理 \\ {\Large (Universal Approximation Theorem for Neural Networks)}}
\author{西島隆人(Takato Nishijima)\thanks{九州大学大学院数理学府}}
\date{}
\maketitle

\begin{abstract}
Is there any theoretical guarantee for the approximation ability of neural networks? The answer to this question is the "Universal Approximation Theorem for Neural Networks". This theorem states that a neural network is dense in a certain function space under an appropriate setting. This paper is a comprehensive explanation of the universal approximation theorem for feedforward neural networks, its approximation rate problem (the relation between the number of intermediate units and the approximation error), and Barron space in Japanese.
\end{abstract}

\renewcommand{\abstractname}{概要}
\begin{abstract}
今日, ニューラルネットワークは関数近似器として様々な分野で盛んに利用され, 多くの成功を収めている.
では, その成功は単なる偶然であろうか. 
すなわち, ニューラルネットワークの近似能力に理論的な保証はあるだろうか. 
その答えが「ニューラルネットワークの万能近似定理」である. 
この定理は, 適当な設定のもとで, ある関数空間においてニューラルネットワークが稠密であることを述べたものである. 
例えば, ある条件下で,  $\R^r$のコンパクト集合上の連続関数をsupノルムの意味で所望の精度で近似するニューラルネットワークが存在することを主張する. 
したがって, この定理は, 応用で現れるほとんどすべての関数をニューラルネットワークによって(適当な意味で)近似できる裏付けを与える. 
よって, これまでの応用での成功は単なる偶然ではなく, こうしたニューラルネットワークの普遍的な関数近似能力の賜物であると考えてよいだろう. feedforward型3層ニューラルネットワークの万能近似定理は1989年にCybenko, Funahashi, Hornikなど複数の研究者により独立に証明された. これらの結果は有界な活性化関数に対するものだったが, 1993年にはLeshno et al.により非有界の場合にも成立することが示された. 
本論文は, こういったfeedforward型ニューラルネットワークの万能近似定理や, その近似レートの問題(中間ユニット数と近似誤差の関係)およびBarron spaceについて総合的にまとめた解説論文である. 
\end{abstract}
\clearpage
\tableofcontents

\section{はじめに}
\subsection{記号}
本節では本論文全体で共通して用いる記号の定義を述べる. 

$\N = \{1,2,3,\ldots\}$を自然数全体の集合とし, $\Z$を整数全体の集合とし, $\mathbb{Q}$を有理数全体の集合とする.  また$\R$で実数全体の集合を表す. 
特に断らなければ, $\R^n$は通常の距離によって位相空間とみなす. 
\begin{defn}
  位相空間$X,Y$に対して, 
  \[
  C(X,Y) := \{f:X \rightarrow Y \mid f\mbox{は連続}\}
  \]
  と定義する. $C(X,\R)$は単に$C(X)$と表す. また, 関数$f:X \rightarrow \R$に対して, 
 \[
 \mathrm{supp}(f) := \mathrm{cl}(\{x \in X \mid f(x) \neq 0 \})
 \]
 と定義し, 開集合$\Omega \subset \R^n$と$m = 0,1,2,\ldots$に対して
 \[
 \begin{aligned}
 &C_0(X) := \{ f \in C(X) \mid \mathrm{supp}(f) \mbox{はコンパクト}\} \\
 &C^m(\Omega) := \{f \in C(\Omega) \mid f\mbox{は}\Omega \mbox{上で}m\mbox{階偏微分可能で全ての偏導関数が連続}\} \\
 &C^\infty(\Omega) := \bigcap_{m=0}^{\infty} C^m(\Omega)
 \end{aligned}
 \]
 と定義する. 
\end{defn}
\begin{defn}[生成される$\sigma$-加法族]　\\
  $X$を集合とし, $\mathcal{A}_0 \subset \mathcal{P}(X)$とする($\mathcal{P}(X)$は$X$の部分集合全体の集合). このとき, $\mathcal{A}_0$で生成される$X$上の$\sigma$-加法族$\sigma[\mathcal{A}_0]$を
  \[
  \sigma[\mathcal{A}_0] := \bigcap \{ \mathcal{A} \mid \mathcal{A} \supset \mathcal{A}_0, \mathcal{A}\mbox{は}X\mbox{上の}\sigma\mbox{-加法族} \}
  \]
  と定義する. 
\end{defn}
\begin{defn}[Borel $\sigma$-加法族, Borel測度]　\\
  $(X,\mathcal{O})$を位相空間とする. このとき, $X$上のBorel $\sigma$-加法族$\mathcal{B}_X$を$\mathcal{B}_X = \sigma[\mathcal{O}]$と定義する. そして, $X$上のBorel測度とは, 可測空間$(X,\mathcal{B}_X)$上の測度のことをいう.  $\R^n$は$(\R^n,\B_{\R^n})$として可測空間とみなす. 
\end{defn}
\begin{defn}[可測写像, $L^p$空間]　\\
  $(X,\mathcal{M}),(Y,\mathcal{N})$を可測空間とする.
  このとき, $X$から$Y$への可測写像全体の集合を$\M(X,Y)$で表す: 
  \[
  \M(X,Y) := \{f:X \rightarrow Y \mid \forall E \in \mathcal{N}, f^{-1}(E) \in \mathcal{M}\}
  \]
  また, $(X,\mathcal{M})$上の測度$\mu$と$1 \leq p < \infty$に対して, 
  \[
  L^p(X,\mathcal{M},\mu) := \{f \in \M(X,\R) \mid \int_{X} |f|^p d\mu < \infty\}
  \]
  と定義し, $f \in L^p(X,\mathcal{M},\mu)$に対して, 
  \[
  \lVert f \rVert_{L^p} = \left( \int_{X} |f|^p d\mu \right)^{1/p}
  \]
  と定義する. 文脈から推察できる場合には$L^p(X,\mathcal{M},\mu)$を単に$L^p(X)$と表す.  さらに, 
  \[
  L^{\infty}(X,\mathcal{M},\mu) := \{f \in \M(X, \R) \mid \mu\mbox{-a.e.で有界}\}
  \]
  とおき, $f \in L^{\infty}(X,\mathcal{M},\mu)$に対して, 
  \[
  \lVert f \rVert_{L^{\infty}} = \inf \{\lambda \geq 0 \mid \mu(|f|\geq \lambda) = 0\}
  \]
  と定義する. 文脈から推察できる場合には$L^\infty(X,\mathcal{M},\mu)$を単に$L^{\infty}(X)$と表す.
\end{defn}
なお, $L^p(X,\mathcal{M},\mu)$に属する関数はa.e.で一致するとき等しいとみなす. つまり, $f,g \in L^p(X,\mathcal{M},\mu)$に対して, $f=g$とは$\mu(\{x \in X \mid f(x) \neq g(x)\}) = 0$となることをいう. 
\begin{defn}[アフィン関数]
 $r \in \N$に対して, 
 \[
 \mathbb{A}^r := \{f:\mathbb{R}^r \rightarrow \mathbb{R} \mid f\mbox{はアフィン関数}\}
 \]
 とおく. ただし, $f:\mathbb{R}^r \rightarrow \mathbb{R}$がアフィン関数であるとは, ある$w \in \R^r$と$b \in \R$が存在して, 
 \[
 f(x) = w^{\T}x + b ~~~(x \in \R^r)
 \]
 となることをいう. 
\end{defn}

\subsection{万能近似定理の歴史}
本節では本論文で主に扱うfeedforward型3層ニューラルネットワークの定義を与え, 詳しく歴史を振り返る.
証明は章を変えて行う. 

feedforward型3層ニューラルネットワークを次により定義する.
\begin{defn}
 $r \in \N, W \subset \R^r, \Theta \subset \R$と関数$\Psi:\R \rightarrow \R$に対して, 
 \[
 \begin{aligned}
 &\mathcal{N}_r(\Psi,W,\Theta) := \mathrm{span} \{\R^r \ni x \mapsto \Psi(w^{\T} x + \theta) \in \R \mid w \in W, \theta \in \Theta \} \\
 &\Sigma^r(\Psi) := \mathcal{N}_r(\Psi,\R^r,\R) = \left\{ \sum_{j}^{q} \beta_j \Psi \circ A_j \mid  \beta_j \in \R, A_j \in \A^r , q \in \N  \right\} 
 \end{aligned}
 \]
 と定義する. これらを$\Psi$を活性化関数とするfeedforward型3層ニューラルネットワークと呼ぶ. 
\end{defn}
この定義において中間ユニット数$q$に制限を設けないことに注意する. この事情からこのようなニューラルネットワークはshallow wide networkと呼ばれることがある(なお, 逆に$q$に制限を設け, 層の数に制限を設けないニューラルネットワーク, deep narrow networkについても近年考察が進められている\cite{DeepNarrowNetworks}). 

明らかに, $\Psi \in \M(\R,\R)$, つまり$\Psi$が$\R$から$\R$へのBorel可測関数ならば$\Psi$を活性化関数とするfeedforward型3層ニューラルネットワークは$\M(\R^r,\R)$に含まれる. 
また, $\Psi \in C(\R,\R)$ならば$\Psi$を活性化関数とするfeedforward型3層ニューラルネットワークは$C(\R^r,\R)$に含まれる. 

次に, feedforward型3層ニューラルネットワークの万能近似定理の意味を正確に述べるために, 距離空間における稠密性の定義をする. 
\begin{defn}
 $(X,\rho)$を距離空間とし, $S,T \subset X$とする. このとき, $S$が$T$において$\rho$-稠密であるとは, 任意の$t \in T$と$\varepsilon > 0$に対して, ある$s \in S$が存在して, $\rho(s,t) < \varepsilon$が成り立つことをいう. 
\end{defn}
なお, しばしばこの概念を濫用して,  必ずしも$X$に含まれない集合$S,T$について$\rho$-稠密という表現をすることがある(その場合は$S,T$の元について$\rho$の値が定まることを前提とする). 
\begin{defn}
  $K \subset \R^r$を空でないコンパクト集合とする. $f,g \in C(K)$に対して, 
  \[
  \rho_K(f,g) := \sup_{x \in K} |f(x) - g(x)|
  \]
  と定義する. また, $S \subset C(\R^r)$に対して, 
  \[
  S_K := \{f|_K \mid f \in S\}
  \]
  と定義する. ここで$f|_K$は$f$の$K$への制限である. $(S_K,\rho_K)$は距離空間となる. 
\end{defn}
\begin{defn}
  $C(\R^r)$の元の列$(f_n)_{n=1}^{\infty}$が$f \in C(\R^r)$に広義一様収束するとは, 任意の空でないコンパクト集合$K \subset \R^r$に対して, $\rho_K(f_n,f) \rightarrow 0 ~(n \rightarrow \infty)$が成り立つことをいう. また, $S \subset C(\R^r)$が$C(\R^r)$において広義一様収束の意味で稠密であるとは, 任意の空でないコンパクト集合$K \subset \R^r$に対して, $S_K$が$C(K)$において$\rho_K$-稠密であることをいう. 
\end{defn}

1960年代以降のニューラルネットワークの応用における数多くの成功から, どのような条件下で$\mathcal{N}_r(\Psi,W,\Theta)$が$C(\R^r)$において広義一様収束の意味で稠密であるか？という問いが立てられた. Hornik\cite{Hornik}によれば, この問いを研究者が研究し始めたのは1980年代になってからである. 
歴史的にはまず$\Sigma^r(\Psi)$に関して考察が進められた. 
つまり, 活性化関数$\Psi$がどのようなものならば万能近似能力を持つかという問題が考えられた. 
この問題は1989年に複数の研究者により独立に異なる手法で解かれた.
Cybenko\cite{Cybenko}は$\Psi$が連続なsigmoidal関数であるときに万能近似能力を持つことをHahn-Banachの拡張定理やRiesz-Markov-角谷の表現定理などを用いて示した. 
ここで, sigmoidal関数の定義は次の通りである. 

\begin{defn}[sigmoidal関数]　\\
 関数 $\sigma : \mathbb{R} \rightarrow \mathbb{R}$ が次の条件をみたすとき, sigmoidal関数と呼ぶ:
 \[
 \sigma(t) \rightarrow
 \begin{cases}
 1 & (t \rightarrow +\infty)\\
 0 & (t \rightarrow -\infty)
 \end{cases}
 \]
\end{defn}

一方, Hornik\cite{Hornik}は$\Psi$がsquashing関数(単調非減少なsigmoidal関数)の場合をStone-Wierstrassの定理を使って証明した.
また, 船橋\cite{FUNAHASHI1989}は$\Psi$が非定数かつ有界かつ単調非減少な連続関数の場合を入江・三宅\cite{IrieMiyake1988}が示した事実を用いて証明した. 
そして, 1993年にはLeshno et al. \cite{Leshno} によって以下の定理が証明された.  

\begin{defn*}
  $\Omega \subset \R^n$を開集合とする. また,  $\mu$を$(\R^n,\B_{\R^n})$上のLebesgue測度とする. このとき, 
  \[
  \begin{aligned}
  L_{\mathrm{loc}}^{\infty}(\Omega) &:= \{f \mid \mbox{任意のコンパクト集合}K \subset \Omega \mbox{に対して}f|_{K} \in L^{\infty}(K) \} \\
  \mathcal{M}(\Omega) &:= \{f \in L_{\mathrm{loc}}^{\infty}(\Omega) \mid \mu(\mathrm{cl}\{f\mbox{の不連続点}\}) = 0\}
  \end{aligned}
  \]
  と定義する. 
\end{defn*}
\begin{thm*}[Leshno et al. 1993 \cite{Leshno}]
  $\mu$を$\R$上のLebesgue測度とする. $\Psi \in \mathcal{M}(\R)$に対して, $\Sigma^r(\Psi)$が$C(\R^r)$において広義一様収束の意味で稠密であるための必要十分条件は, いかなる一変数多項式関数$P:\R \rightarrow \R$についても$\Psi(x) = P(x) ~(\mu \mathrm{-a.e.}x)$とならないことである. 
\end{thm*}
さらに, この定理が発表されたのち, Hornik\cite{Hornik1993}により連続性に関する仮定を緩和した次の結果が得られた. 
\begin{thm*}[Hornik 1993 \cite{Hornik1993}]
  $\Psi:\R \rightarrow \R$は各閉区間上で有界かつRiemann積分可能であるとすると,  $\Sigma^r(\Psi)$が$C(\R^r)$において広義一様収束の意味で稠密であることと, $\Psi$が多項式関数とa.e.で一致しないことは同値である. 
\end{thm*}
閉区間上で有界なRiemann可積分関数はLebesgue可測かつその不連続点全体の集合の測度が$0$になることに注意する(猪狩\cite{igarizitukaiseki}定理$3$.$24$を参照されたい). 
Leshno et al.の結果では「不連続点全体の集合の閉包」の測度が$0$になることが仮定されていたので, 確かに連続性に関する仮定が緩和されている. 

Leshno et al.やHornikにより示されたこれらの定理の適用範囲は広く, 万能近似定理をめぐる一連の議論には一応の決着が付けられたと考えてよいだろう. ただし, 以上で述べた結果のうち, 1989年のCybenko, Hornikおよび1993年のLeshno et al., Hornikの結果は適用できる範囲に共通部分はあれど包含関係は無い. 

万能近似能力が明らかになってからは近似レートの問題(中間ユニットの数と近似精度の関係)に焦点が移り, MhaskarによるSobolev空間における近似レートの研究\cite{Mhaskar1995}やBarron\cite{Barron1993}に始まるBarron spaceと呼ばれる「ニューラルネットワークでよく近似できる空間」の研究など複数のアプローチでの結果が出ており, 現在でも研究が続いている(Weinan et al.\cite{Weinan2019}やSiegel, Xu\cite{Siegel2020}を参照されたい). 

本論文ではまず第3章で万能近似定理と関わりが深いリッジ関数の性質についていくつか述べる. 
そして第4章のはじめに1993年のLeshno et al., Hornikの結果を証明する. 
次いで, Cybenkoの結果を一般化したChuiとLiの結果\cite{ChuiAndLi}(下表参照)を解説する. 
さらにChuiとLiの結果をノルム空間へ一般化した, SunとCheneyによる結果\cite{SunAndCheney}を紹介し, その後, 1989年のHornikの結果の証明を与える. 

次に章を変え, 5章では$C(\R^r)$における万能近似定理の応用として以上の結果を多層・多出力にした場合や$L^p$空間における万能近似定理, 関数の補間について述べる. 
さらに, 6章ではfeedforward型ニューラルネットワークから一旦離れRBF(Radial-Basis-Function)ネットワークの万能近似定理を証明する. 
その後, 7章では話題を近似レートの問題に移し,  feedforward型3層ニューラルネットワークの近似誤差解析について解説する. 
  \begin{table}[htb]
    \caption{$\mathcal{N}_r(\Psi,W,B)$が$C(\R^r)$において万能近似能力を持つ条件}
    \small
    \begin{tabular}{|c|c|c|c|c|c|c|} \hline
     {\scriptsize 西暦}& 証明者 & $\Psi$ & $\Psi$の連続性 & $W$ & $B$ & \begin{tabular}{c} {\tiny 本論文} \\ {\tiny 定理番号} \end{tabular} \\ \hline 
     & Cybenko & sigmoidal & 連続 & $\R^r$ & $\R$ & {\scriptsize 4.2.1} \\ \cline{2-7}
    {\scriptsize 1989} & {\footnotesize Funahashi} & \begin{tabular}{c} 有界 \\ 単調非減少 \end{tabular} & 連続 & $\R^r$ & $\R$ & {\scriptsize 4.1.11} \\\cline{2-7}
     & Hornik & \begin{tabular}{c} sigmoidal \\ 単調非減少 \end{tabular} & 仮定なし & $\R^r$ & $\R$ & {\scriptsize 4.4.6} \\ \hline
    {\scriptsize 1992} & Chui, Li & sigmoidal & 連続 & $\Z^r$ & $\Z$ & {\scriptsize 4.2.1} \\ \hline
    {\scriptsize 1993} & {\scriptsize Leshno et al.} & {\footnotesize 非多項式, $L_{\mathrm{loc}}^{\infty}$} & { \footnotesize $\mu\left(\mathrm{cl}\{\mbox{不連続点}\}\right) = 0$ } & $\R^r$ & $\R$ & {\scriptsize 4.1.11} \\ \cline{2-7}
     & Hornik & \begin{tabular}{c} {\footnotesize $B$上非多項式}\\ 局所有界 \end{tabular} & {\footnotesize $\mu\left(\{\mbox{不連続点}\}\right)=0$ } & {\footnotesize $0 \in W^{\circ}$} & {\footnotesize 開区間} & {\scriptsize 4.1.12} \\ \hline
    \end{tabular}
  \end{table}

\section{リッジ関数とその万能近似定理への応用}
本章では万能近似定理と関わりが深いリッジ関数の性質について述べ, 万能近似定理を証明するためには$1$次元の場合のみ考えれば十分であることを示す. 
\subsection{リッジ関数の性質}

\begin{defn}[有界線形写像, 共役空間]　\\
  $X,Y$をノルム空間とする. $f:X \rightarrow Y$が有界線形写像であるとは, $f$が線形写像であって, かつ$M \geq 0$が存在して, 任意の$x \in X$に対して, 
  \[
  \lVert f(x) \rVert_Y \leq M \lVert x \rVert_X
  \]
  となることをいう. 有界線形写像はノルムに関して連続である. また, $X$の共役空間を 
  \[
  X^* := \{f:X \rightarrow \R \mid f\mbox{は有界線形写像}\}
  \]
  と定義する. そして, $f \in X^*$に対して, 
  \[
  \lVert f \rVert = \inf \{M \geq 0 \mid \forall x \in X, | f(x) | \leq M \lVert x \rVert_X \}
  \]
  と定義する. $\lVert \cdot \rVert$は$X^*$上のノルムであり, 作用素ノルムと呼ばれる. 
\end{defn}
\begin{defn}[リッジ関数]　\\
  $X$をノルム空間とする. 関数$f:X \rightarrow \R$がリッジ関数であるとは, $\varphi \in X^*$と$g:\R \rightarrow \R$が存在して, $f = g \circ \varphi$となることをいう. 特に$X$がHilbert空間であるときはRieszの表現定理より, $v \in X$と$g:\R \rightarrow \R$が存在して, $f(x) = g(\ip<{x,v}>)$となることをいう. 
\end{defn}
\begin{defn}
  ノルム空間$X$と$\mathcal{F} \subset X^*$に対して, 
  \[
  C(\R) \circ \mathcal{F} = \{g \circ f \mid g \in C(\R), f \in \mathcal{F}\}
  \]
  と定義する. $C(\R) \circ \mathcal{F} \subset C(X)$である. 
\end{defn}
\begin{lem}\label{VandermondeCor}
  $r,s \in \N \cup \{0\}$と相異なる$\beta_j \in \R \setminus \{0\} ~~(j=0,\ldots,r+s)$に対して, ある$c_j \in \R ~~(j=0,\ldots,r+s)$が存在して, 任意の$(x_1,x_2) \in \R^2$に対して, 
  \[
  \sum_{j=0}^{r+s} c_j (x_1 + \beta_j x_2)^{r+s} = x_1^r x_2^s
  \]
  となる. 
\end{lem}
\begin{proof}
  もし, そのような$c_j \in \R$が存在するなら, 任意の$(x_1,x_2) \in \R^2$に対して, 
  \[
  \begin{aligned}
  \sum_{j=0}^{r+s} c_j (x_1 + \beta_j x_2)^{r+s}
  &= \sum_{k=0}^{r+s} \binom{r+s}{k} \sum_{j=0}^{r+s}  c_j \beta_j^k x_1^{r+s-k} x_2^k \\
  &= \sum_{k=0}^{r+s} \delta_{k,s} x_1^{r+s-k} x_2^{k}
  \end{aligned}
  \]
  となるので, $c_j$は連立一次方程式
  \[
  \binom{r+s}{k} \sum_{j=0}^{r+s} c_j \beta_j^k = \delta_{k,s} ~~(k=0,\ldots,r+s)
  \]
  をみたす. ところが, この連立一次方程式はVandermonde行列の正則性から一意的に解ける($\beta_j$は相異なり, かつすべて$0$でないことに注意). そして, この解$c_j$は所望の性質を持つ. 
\end{proof}
\begin{thm}\label{densityRidgeFunc}
  $X$をノルム空間とするとき,  $X$上のリッジ関数全体の集合の線形包は$C(X)$において広義一様収束の意味で稠密である. 
\end{thm}
\begin{proof}
  コンパクト集合$K \subset X$に対して, 部分空間$P \subset C(K)$を次で定める. 
  \[
  P = \mathrm{span}\{K \ni x \mapsto (\varphi(x))^n \in \R \mid n \in \N \cup \{0\} , ~\varphi \in X^*\}
  \]
  すると, $P$は通常の和とスカラー倍, および積に関して代数をなす.
  そのことを示すには, 任意の$\varphi,\theta \in X^*$と$r,s \in \N \cup \{0\}$に対して, $x \mapsto \varphi(x)^r \theta(x)^s$が$P$に属することを言えばよいが, 補題\ref{VandermondeCor}より, $\beta_j,c_j \in \R ~~(j=0,\ldots,r+s)$が存在して, 任意の$x \in K$に対して, 
  \[
  \varphi(x)^r \theta(x)^s = \sum_{j=0}^{r+s} c_j (\varphi(x) + \beta_j \theta(x))^{r+s}
  \]
  となるので$x \mapsto \varphi(x)^r \theta(x)^s$は$P$に属する. そして, $x,y \in K$が$x \neq y$をみたすならば, Hahn-Banachの拡張定理より$\varphi(x) \neq \varphi(y)$なる$\varphi \in X^*$が取れる. さらに, $n$として$0$をとることで$1 \in P$がわかる. よって, Stone-Weierstrassの定理より,  $P$は$C(K)$において一様収束の意味で稠密である. これより$X$上のリッジ関数全体の集合の線形包は$C(X)$において広義一様収束の意味で稠密である. 
\end{proof}
\begin{lem}
  $V = \{\R^r \ni x \mapsto \sum_{i=1}^l f_i(a_i^{\T}x) \in \R \mid l \in \N , a_i \in \R^r , f_i \in C(\R) \}$とおくと, $V$は$C(\R^r)$において広義一様収束の意味で稠密である. 
\end{lem}
\begin{proof}
  定理\ref{densityRidgeFunc}による(連続関数を定数倍しても連続関数であることに注意). 
\end{proof}
次の命題により, 連続関数の空間でのニューラルネットワーク$\Sigma^r(\Psi)$の稠密性を考える上では$1$次元の場合を考えれば十分であることがわかる. 
\begin{prop}\label{DensityDim1Suff}
  任意の$\Psi:\R \rightarrow \R$に対して, $\Sigma^1(\Psi)$が$C(\R)$において広義一様収束の意味で稠密ならば, 任意の$r$について$\Sigma^r(\Psi)$は$C(\R^r)$において広義一様収束の意味で稠密である. 
\end{prop}
\begin{proof}
  $g \in C(\R^r)$とし, 任意に空でないコンパクト集合$K \subset \R^r$をとる. 任意の$\varepsilon>0$に対して, すぐ上の補題により$k \in \N$と$a_i \in \R^r, f_i \in C(\R) ~(i=1,\ldots,k)$が存在して, 
  \[
  \sup_{x \in K} |g(x) - \sum_{i=1}^{k} f_i(a_i^{\T}x)| < \varepsilon/2
  \]
  となる. $K$はコンパクトなので, 各$i$について$\{a_i^{\T}x \mid x \in K\} \subset [\alpha_i,\beta_i]$をみたす$\alpha_i,\beta_i \in \R$が取れる. そこで, $\Sigma^1(\Psi)$の稠密性から, $A_{i,j} \in \A^1$と$c_{i,j} \in \R$が存在して, 
  \[
  \sup_{y \in [\alpha_i,\beta_i]} |f_i(y) - \sum_{j=1}^{m_i} c_{i,j}\Psi(A_{i,j}(y))| < \varepsilon/2k
  \]
  となる. よって, 任意の$x \in K$に対して, 
  \[
  \begin{aligned}
  &|g(x) - \sum_{i=1}^k \sum_{j=1}^{m_i} c_{i,j}\Psi(A_{i,j}(a_i^{\T}x))| \\
  &\leq |g(x) - \sum_{i=1}^{k} f(a_i^{\T}x)| + \sum_{i=1}^k |f_i(a_i^{\T}x) - \sum_{j=1}^{m_i} c_{i,j}\Psi(A_{i,j}(a_i^{\T}x))| \\
  &< \varepsilon/2 + \varepsilon/2 = \varepsilon
  \end{aligned}
  \]
  となる. しかも$x \mapsto A_{i,j}(a_i^{\T}x)$は$\A^r$に属する. 
\end{proof}
さらにリッジ関数について次の定理が成り立つ. 
\begin{thm}[Lin and Pinkus  \cite{LinAndPinkus}]\label{LinAndPinkus}　\\
  $\mathcal{A} \subset \R^r$に対して, 
  \[
  R(\mathcal{A}) = \{ \R^r \ni x \mapsto f(w^{\T} x) \in \R \mid f \in C(\R),w \in A\}
  \]
  とおくと, $\mathrm{span}R(\mathcal{A})$が$C(\R^r)$において広義一様収束の意味で稠密であることと, $\mathcal{A}$上で$0$を取る非自明な$r$変数斉次多項式が存在しないことは同値である. 
\end{thm}
\begin{proof}
  次節参照. 
\end{proof}
この定理により次のことがわかる. 
\begin{prop}\label{MultiVariateDensity}
  $\Psi: \R \rightarrow \R$と$W,B \subset \R$に対して,  
  \[
  \mathcal{N}_1(\Psi,W,B) = \mathrm{span}\{x \mapsto \Psi(wx+b) \mid w \in W, b \in B\}
  \]
  は$C(\R)$において広義一様収束の意味で稠密であるとする. このとき, $\mathcal{A} \subset \R^r$上で$0$を取る非自明な$r$変数斉次多項式が存在しないならば, 
  \[
  \mathcal{N}_r(\Psi,W\mathcal{A},B) = \mathrm{span}\{\Psi(aw^{\T}x+b) \mid a\in \mathcal{A}, w \in W, b \in B\}
  \]
  は$C(\R^r)$において広義一様稠密の意味で稠密である. 
\end{prop}
\begin{proof}
  命題\ref{DensityDim1Suff}と同様に示せる. 詳しくは次の通りである. $g \in C(\R^r)$とし, 任意に空でないコンパクト集合$K \subset \R^r$をとる. 任意の$\varepsilon>0$に対して, すぐ上の定理により$k \in \N$と$a_i \in A, f_i \in C(\R) ~(i=1,\ldots,k)$が存在して, 
  \[
  \sup_{x \in K} |g(x) - \sum_{i=1}^{k} f_i(a_i^{\T}x)| < \varepsilon/2
  \]
  となる. $K$はコンパクトなので, 各$i$について$\{a_i^{\T}x \mid x \in K\} \subset [\alpha_i,\beta_i]$をみたす$\alpha_i,\beta_i \in \R$が取れる. そこで, $\mathcal{N}_1(\Psi,W,B)$の稠密性から, $m_i \in \N, w_{i,j} \in W, b_{i,j} \in B$と$c_{i,j} \in \R$が存在して, 
  \[
  \sup_{y \in [\alpha_i,\beta_i]} |f_i(y) - \sum_{j=1}^{m_i} c_{i,j}\Psi(w_{i,j}y + b_{i,j})| < \varepsilon/2k
  \]
  となる. よって, 任意の$x \in K$に対して, 
  \[
  \begin{aligned}
  &|g(x) - \sum_{i=1}^k \sum_{j=1}^{m_i} c_{i,j}\Psi(w_{i,j}(a_i^{\T}x)+b_{i,j})| \\
  &\leq |g(x) - \sum_{i=1}^{k} f(a_i^{\T}x)| + \sum_{i=1}^k |f_i(a_i^{\T}x) - \sum_{j=1}^{m_i} c_{i,j}\Psi(w_{i,j}(a_i^{\T}x) + b_{i,j})| \\
  &< \varepsilon/2 + \varepsilon/2 = \varepsilon
  \end{aligned}
  \]
  となる. 
\end{proof}

\subsection{LinとPinkusの結果の証明}
本節では前節に述べたLinとPinkusの結果(定理\ref{LinAndPinkus})の証明を述べる. 
\begin{defn}
  $\alpha = (\alpha_1,\ldots,\alpha_n) \in (\mathbb{Z}_{\geq 0})^n$に対して, 
  \[
  \begin{aligned}
  |\alpha| := \sum_{j=1}^n \alpha_j ,~~ \alpha ! := \alpha_1 ! \cdots \alpha_n !
  \end{aligned}
  \]
  と定義する. $x=(x_1,\ldots,x_n) \in \R^n$に対して, 
  \[
  x^{\alpha} := x_1^{\alpha_1} \cdots x_n^{\alpha_n}
  \]
  と定義する. 偏微分作用素$\partial :=(\partial_1,\ldots,\partial_n) =  (\partial/\partial x_1, \ldots, \partial/\partial x_n)$に対する$\partial^{\alpha}$も同様に定める. 
\end{defn}
\begin{defn}
  $n,k \in \N$に対して, $n$変数の$k$次斉次多項式関数の集合を
  \[
  \mathcal{H}_k(\R^n) := \left\{ \R^n \ni x \mapsto \sum_{|\alpha| = k} c_{\alpha} x^{\alpha} \in \R \mid c_{\alpha} \in \R \right\}
  \]
  と定義する. ただし, 和は$|\alpha| = k$なる$\alpha \in (\mathbb{Z}_{\geq 0})^n$全体にわたるものとする. また, 
  \[
  \begin{aligned}
  \mathcal{P}_k(\R^n) = \bigcup_{s=0}^{k} \mathcal{H}_s(\R^n) ,~~ \mathcal{P}(\R^n) = \bigcup_{s=0}^{\infty} \mathcal{H}_s(\R^n)
  \end{aligned}
  \]
  と定義する. 
\end{defn}
\begin{prop}
  $\mathcal{H}_k(\R^n)$は通常の和とスカラー倍に関して$\R$上の有限次元線形空間をなす. また, $\mathcal{H}_k(\R^n)$の元$p(x) = \sum c_{\alpha} x^{\alpha},q(x) = \sum c_{\alpha}' x^{\alpha} $に対して, 
  \[
  \ip<{p,q}> := p(\partial)q = \sum_{|\alpha| = k} \alpha! c_{\alpha} c_{\alpha}'
  \]
  と定義すると, これは$\mathcal{H}_k(\R^n)$上の内積であり, この内積に関して$\mathcal{H}_k(\R^n)$はHilbert空間をなす. 
\end{prop}
\begin{proof}
  $\mathcal{H}_k(\R^n)$が$\R$上の線形空間であること, $\ip<{\cdot,\cdot}>$が内積の公理をみたすことは明らかである. 有限次元であることは有限集合$\{ x^{\alpha} \mid |\alpha| = k\}$が$\mathcal{H}_k(\R^n)$を生成することからわかる. 完備性については, ノルム空間の有限次元部分空間が完備であることによる(\cite{FujitaKurodaIto}の{\S}1.5 (a)を参照されたい). 
\end{proof}
\begin{defn}
  自然数$m,n \geq 1$に対して, $m \times n$の実行列全体の集合を$\mathrm{Mat}(m,n;\R)$で表す. $A \in \mathrm{Mat}(m,n;\R)$に対して, 
  \[
  L(A) := \{ (y^{\T}A)^{\T} \mid y \in \R^m \} \subset \R^n
  \]
  と定義する. $\Omega \subset \mathrm{Mat}(m,n;\R)$に対して, 
  \[
  L(\Omega) := \bigcup_{A \in \Omega} L(A)
  \]
  と定義する. また, リッジ関数の集合$R(\Omega)$を
  \[
  R(\Omega) = \{\R^n \ni x \mapsto g(Ax) \in \R \mid A \in \Omega, g \in C(\R^m)\}
  \]
  と定義する. 
\end{defn}
\begin{thm}[射影定理]　\\
  $X$をHilbert空間とし, $L \subset X$を閉部分空間とする. このとき, 任意の$x \in X$は
  \[
  x = y + z ~~~(y \in L,~ z \in L^{\perp})
  \]
  の形に一意的に分解可能である. ただし, 
  \[
  L^{\perp} := \{x \in X \mid \ip<{x,y}> = 0 ~(\forall y \in L)\}
  \]
  である. 
\end{thm}
\begin{proof}
  \cite{FujitaKurodaIto}の定理$3$.$2$を参照されたい. 
\end{proof}
\begin{lem}\label{HomPolyReprensent}
  $k \geq 0$とし$\Omega \subset \mathrm{Mat}(n,m;\R)$とする. 
  また, $L(\Omega)$上で恒等的に$0$を取る$\mathcal{H}_k(\R^n)$の元は$0$に限るとする. 
  このとき, 
  \[
  \mathcal{H}_k(\R^n) = \mathrm{span}\{ \R^n \ni x \mapsto (d^{\T}x)^k \in \R \mid d \in L(\Omega)\} =: L
  \]
  となる. 
\end{lem}
\begin{proof}
  $\mathcal{H}_k(\R^n) \subset L$を示せばよい. $L$はノルム空間$\mathcal{H}_k(\R^n)$の有限次元部分空間なので閉である(\cite{FujitaKurodaIto}の{\S}1.5 (a)参照). したがって, 任意の$p \in \mathcal{H}_k(\R^n)$に対して, 射影定理より, $q \in L , r \in L^{\perp}$が存在して$p = q + r$となる. ところが, $r(x) = \sum c_{\alpha} x^{\alpha}$と書くことにすると, $r \in L^{\perp}$より, 任意の$d \in L(\Omega)$に対して, 
  \[
  \begin{aligned}
  0 
  &= \ip<r(x),(d^{\T} x)^k> \\
  &= r(\partial) (d^{\T} x)^k \\
  &= \sum_{|\alpha|=k} c_{\alpha} \partial^{\alpha} (d^{\T} x)^k \\
  &= \sum_{|\alpha|=k} c_{\alpha} k! d^{\alpha} \\
  &= k! r(d)
  \end{aligned}
  \]
  となる. ゆえに, $r(d) = 0 ~(\forall d \in L(\Omega))$であるから, 仮定より$r=0$である. よって, $p = q \in L$である. 
\end{proof}
\begin{thm}
  $\Omega \subset \mathrm{Mat}(n,m;\R)$とする. 
  このとき, $\mathrm{span}R(\Omega)$が$C(\R^n)$において広義一様収束の意味で稠密であることと, $L(\Omega)$上で恒等的に$0$を取る$\mathcal{P}(\R^n)$の元が$0$以外にないことは同値である. 
\end{thm}
\begin{proof}
  $L(\Omega)$上で$0$を取る$\mathcal{P}(\R^n)$の元は$0$に限るとする. このとき, 任意の$k \geq 0$に対して$\mathcal{H}_k(\R^n) \subset \mathrm{span}R(\Omega)$となることを示そう. これが示されれば, すべての多項式が$\mathrm{span}R(\Omega)$に含まれることがわかるので, Stone-Weierstrassの定理により$\mathrm{span}R(\Omega)$は$C(\R^n)$において広義一様収束の意味で稠密である. さて, すぐ上の補題により
  \[
  \mathcal{H}_k(\R^n) = \mathrm{span}\{ \R^n \ni x \mapsto (d^{\T}x)^k \in \R \mid d \in L(\Omega)\}
  \]
  となる. また, $d \in L(\Omega)$ならば$A \in \Omega$と$y \in \R^m$が存在して$d = A^{\T}y$となるので, $g(z) := (z^{T}y)^k$と定義すれば$g(Ax) = (A^{\T}x)y)^k = (x^{\T} Ay)^k = (x^{\T} d)^k$となる. よって, $x \mapsto (d^{\T} x)^k$は$R(\Omega)$に属するので, 
  \[
   \mathcal{H}_k(\R^n) = \mathrm{span}\{ \R^n \ni x \mapsto (d^{\T}x)^k \in \R \mid d \in L(\Omega)\} \subset \mathrm{span}R(\Omega)
  \]
  である. 逆の証明についてはLin, Pinkus, 1992 \cite{LinAndPinkus}を参照されたい. 
\end{proof}
最後に7章で使う命題を証明する. 証明には「斉次多項式の決定問題」に関する次の事実を使う. 
\begin{prop}\label{HomPolyDecis}
  $n$変数$d$次斉次多項式はある$r(n,d) := \binom{n-1+d}{d}$個の相異なる$0$でない点で決定される. つまり, ある$\xi^{(1)},\ldots,\xi^{(r(n,d))} \in \R^n\setminus\{0\}$が存在して, 任意の$p \in \mathcal{H}_{d}(\R^n)$に対して, 
  \[
  (\forall i \in \{1,\ldots,r(n,d) \},~ p(\xi^{(i)}) = 0) \Rightarrow p \equiv 0
  \]
  となる. さらに, このような$\xi^{(1)},\ldots,\xi^{(r(n,d))} \in \R^n \setminus \{0\}$として, 任意の$k = 1,\ldots,d$と$p \in \mathcal{H}_{k}(\R^n)$に対して, 
  \[
  (\forall i \in \{1,\ldots,r(n,d) \},~ p(\xi^{(i)}) = 0) \Rightarrow p \equiv 0
  \]
  となるようなものが取れる. 
\end{prop}
\begin{proof}
  \cite{HomPolyDecision}を参照されたい. 
\end{proof}
\begin{prop}\label{PolyRidgeRelation}
  $r = \mathrm{dim}\mathcal{H}_k(\R^n) = \binom{n-1+k}{k}$とおき, $n$変数の$k$次以下の多項式全体の集合を$\pi_k(\R^n)$と表すことにする. このとき, $a_1, \ldots, a_r \in \R^n$が存在して, 
  \[
  \begin{aligned}
  \pi_k(\R^n) = \left\{ x \mapsto \sum_{i=1}^r g_i(a_i^{\T} x )  \mid g_i \in \pi_k(\R) \right\}
  \end{aligned}
  \]
  となる. 
\end{prop}
\begin{proof}
  命題\ref{HomPolyDecis}より$a_1, \ldots, a_r \in \R^n$が存在して, 各$s = 0,\ldots, k$に対して, $A = \{a_1,\ldots,a_r\}$上で恒等的に$0$を取る$\mathcal{H}_s(\R^n)$の元は$0$に限る. したがって, 補題\ref{HomPolyReprensent}からわかるように
  \[
  \mathcal{H}_s(\R^n) = \mathrm{span}\{ \R^n \ni x \mapsto (d^{\T}x)^s \in \R \mid d \in A\}
  \]
  となる. このことと多項式を斉次多項式の和として表せることより
  \[
  \begin{aligned}
  \pi_k(\R^n) 
  &= \mathrm{span} \{x \mapsto (d^{\T}x)^s \mid d \in A, s = 0,\ldots, k\} \\
  &= \left\{ x \mapsto \sum_{i=1}^r g_i(a_i^{\T} x )  \mid g_i \in \pi_k(\R) \right\}
  \end{aligned}
  \]
  となる. 
\end{proof}

\section{ニューラルネットワークの$C(\R^r)$における万能近似定理}
本章ではまず1993年のLeshno et al., Hornikの結果を証明する. 
次いで, Cybenkoの結果を一般化したChuiとLiの結果を解説する. 
さらにChuiとLiの結果をノルム空間へ一般化した, SunとCheneyによる結果を紹介し, その後, 1989年のHornikの結果の証明を与える. 
\subsection{Leshno et al.の結果}
本節では, Leshno et al.の結果\cite{Leshno}およびその改良版について述べる. 本節では特に断らなければa.e.はLebesgue測度の意味で用いる. 
\begin{defn}
  $\Omega \subset \R^n$を開集合とする. また,  $\mu$を$(\R^n,\B_{\R^n})$上のLebesgue測度とする. このとき, 
  \[
  \begin{aligned}
  L_{\mathrm{loc}}^{\infty}(\Omega) 
  &:= \left\{f:\Omega \rightarrow \R ~\vline~ 
  \begin{aligned}
  &\mbox{任意のコンパクト集合}K \subset \Omega \mbox{に対して}\\
  &f|_{K} \in L^{\infty}(K)
  \end{aligned} 
  \right\} \\
  \mathcal{M}(\Omega) &:= \{f \in L_{\mathrm{loc}}^{\infty}(\Omega) \mid \mu(\mathrm{cl}\{f\mbox{の不連続点}\}) = 0\}
  \end{aligned}
  \]
  と定義する. 
\end{defn}
Leshno et al. \cite{Leshno}で示された主要な結果を再掲しよう. 
\begin{thm}[Leshno et al. 1993 \cite{Leshno}]\label{LeshnoMainResult1}　\\
  $\mu$を$\R$上のLebesgue測度とする. $\Psi \in \mathcal{M}(\R)$に対して, $\Sigma^r(\Psi)$が$C(\R^r)$において広義一様収束の意味で稠密であるための必要十分条件は, いかなる一変数多項式関数$P:\R \rightarrow \R$についても$\Psi = P ~(\mu \mathrm{-a.e.})$とならないことである. 
\end{thm}
命題\ref{DensityDim1Suff}よりこの定理を示すには$r=1$の場合を示せばよい. まず必要性を示そう. 必要性の証明は非常に単純である. 
\begin{lem}
  関数$\Psi:\R \rightarrow \R$がa.e.で多項式関数であるならば, $\Sigma^1(\Psi)$は$C(\R)$において広義一様収束の意味で稠密でない. 
\end{lem}
\begin{proof}
  $\mathrm{deg}(\Psi) = n$とすると, $\Sigma^r(\Psi)$は$n$次の$1$変数多項式とa.e.で一致する関数全体の集合に含まれる. したがって, $n+1$次の$1$変数多項式は広義一様収束の意味で近似できない. 
\end{proof}
十分性は次の定理に含まれるため本節では代わりに次の定理を証明する. 
\begin{thm}
  $W \subset \R$は孤立点を持たず$0 \in W$であるとする. 
  また, $\Psi:\R \rightarrow \R$は$\ \mathcal{M}(\R)$に属するとする. さらに, ある開区間$I$に対して$\Psi$が$I$上a.e.に多項式と一致しないとする. このとき, $\mathcal{A} \subset \R^r$上で$0$を取る非自明な$r$変数斉次多項式が存在しないならば, 
  \[
  \begin{aligned}
  \mathcal{N}_r(\Psi,W\mathcal{A},I) = \mathrm{span}\{ \R^r \ni x \mapsto \Psi(w a^{T} x + b) \in \R \mid w \in W, a \in \mathcal{A}, b \in I \}
  \end{aligned}
  \]
  は$C(\R^r)$において広義一様収束の意味で稠密である. 
\end{thm}
さらにこの定理は命題\ref{MultiVariateDensity}と次の定理より従うので次の定理を証明すればよい. 
\begin{thm*}
  $W \subset \R$は孤立点を持たず$0 \in W$であるとする. 
  また, $\Psi:\R \rightarrow \R$は$\mathcal{M}(\R)$に属するとする. このとき, 開区間$I$に対して$\Psi$が$I$上a.e.に多項式と一致しないなら, $\mathcal{N}_1(\Psi,W,I)$は$C(\R)$において広義一様収束の意味で稠密である. 
\end{thm*}
まず$\Psi$が$C^{\infty}$級である場合に成立することを証明しよう. 
\begin{lem}\label{NeuralNetApproPoly}
  $W \subset \R$は孤立点を持たず$0 \in W$であるとする. 
  また, $\Psi:\R \rightarrow \R$はある開区間$I$上で$C^{\infty}$かつ多項式と一致しないとする. 
  このとき, 各整数$n \geq 0$について多項式$x \mapsto x^n$は$\mathcal{N}_1(\Psi,W,I)$で広義一様近似される. したがって, 特に,  $\mathcal{N}_1(\Psi,W,I)$は$C(\R)$において広義一様収束の意味で稠密である. 
\end{lem}
\begin{proof}
  後半については$C(\R)$の元が多項式関数で広義一様近似できること, および前半より$\mathcal{N}_1(\Psi,W,I)$が全ての$1$変数の$\R$係数多項式を広義一様近似できることからわかる. 
  $K \subset \R$を空でないコンパクト集合とする. また, $b \in I$に対して, $f_b:\R \times \R \ni (x,w) \mapsto \Psi(wx + b) \in \R$と定義する. このとき, 任意の$k \in \N, w_0 \in W, b \in I$に対して, 関数$x \mapsto ((\partial w)^k f_b) (x,w_0)$が$\mathcal{N}_1(\Psi,W,I)$の元で$K$上一様近似できることを示せば十分である. なぜなら, もしそれが示されれば, $\Psi$が$I$上で多項式関数でないことから,  任意の$k$について$\Psi^{(k)}(b_k) \neq 0$なる$b_k \in I$が取れるので, $w_0=0, b = b_k$とすることで, 関数$x \mapsto ((\partial w)^k f_{b_k}) (x,w_0) = x^k \Psi^{(k)}(w_0 x + b_k) = x^k \Psi^{(k)}(b_k)$が$\mathcal{N}_1(\Psi,W,I)$の元で$K$上一様近似できることがわかる. さて, まず$k=1$の場合を示そう.  任意に$\delta_0,\varepsilon > 0$を取る. いま, 
  \[
  K \times [w_0-\delta_0,w_0+\delta_0] \ni (x,w) \mapsto ((\partial w) f_b)(x,w) \in \R
  \]
  は連続である. $K \times [w_0-\delta_0,w_0+\delta_0]$はコンパクトだから, $0 < \delta < \delta_0$が存在して, 任意の$(x_1,w_1),(x_2,w_2) \in K \times [w_0-\delta_0,w_0+\delta_0]$について, $|x_1 - x_2|,|w_1-w_2| < \delta$ならば, $|((\partial w) f_b)(x_1,w_1) - ((\partial w) f_b)(x_2,w_2)| < \varepsilon$となる. そこで, $-\delta < h < \delta$かつ$w_0+h \in W$となるように$h$をとると($W$は孤立点を持たないのでとれる),  任意の$x \in K$に対して, 平均値の定理より, 
  \[
  \begin{aligned}
  &\left|\frac{\Psi((w_0+h)x+b) - \Psi(w_0 x + b)}{h} -  ((\partial w) f_b)(x,w_0) \right| \\
  &\left|\frac{f_b(x,w_0+h) - f_b(x,w_0)}{h} -  ((\partial w) f_b)(x,w_0) \right| \\
  &= |((\partial w) f_b)(x,w_0+\theta h) - ((\partial w) f_b)(x,w_0)| ~~~(\exists \theta \in [0,1]) \\
  &\leq \varepsilon
  \end{aligned}
  \]
  となる. これで$k=1$の場合が証明された.
  次に$k$で成立するとするとして$k+1$での成立を示そう. 
  $k=1$の場合の証明と同様にして, 任意の$\varepsilon > 0$に対してある$h \in \R$が存在して, $w_0+h \in W$かつ任意の$x \in K$に対して, 
  \[
  \left| \frac{((\partial w)^k f_b)(x,w_0+h) - ((\partial w)^k f_b)(x,w_0)}{h} - ((\partial w)^{k+1} f_b)(x,w_0)\right| < \varepsilon
  \]
  となることがわかる. 帰納法の仮定から, 
  \[
  \begin{aligned}
  &x \mapsto ((\partial w)^k f_b)(x,w_0+h)\\
  &x \mapsto ((\partial w)^k f_b)(\cdot,w_0)
  \end{aligned}
  \]
  は$\mathcal{N}_1(\Psi,W,I)$で$K$上一様近似できる. ゆえに, 上式と合わせると\\
  $((\partial w)^{k+1} f_b)(\cdot,w_0)$は$\mathcal{N}_1(\Psi,W,I)$で$K$上一様近似される. 
  これで示せた. 
\end{proof}

次に一般の$\Psi$の場合を証明するためにいくつか補題を示す. 
\begin{lem}
  $\Psi \in \mathcal{M}(\R)$とし, $\mu$を$\R$上のLebesgue測度とする. このとき, $a \leq b$と$\delta > 0$に対して有限個の開区間$I_1, \ldots, I_n $が存在して, $U = \bigcup_{k=1}^n I_k$について, $\mu(U) < \delta$かつ$\Psi$は$[a,b] \setminus U$上で一様連続である. 
\end{lem}
\begin{proof}
  $A := [a,b] \cap \mathrm{cl}\{\Psi \mbox{の不連続点}\}$とおく. $\Psi \in \mathcal{M}(\R)$なので$\mu(A) = 0$である. すると, Lebesgue測度の外正則性より, 開集合$U \supset A$で$\mu(U) < \delta$なるものが取れる. $\R$の開集合は開区間の和集合で表せるので, $U = \bigcup_{j \in J}I_j$と表すことにすると, $A$のコンパクト性より$j_1,\ldots,j_m \in J$が存在して$A \subset \bigcup_{k=1}^{m} I_{j_k}$となる. $I = \bigcup_{k=1}^{m} I_{j_k}$とおくと$[a,b] \setminus I$は$\Psi$の不連続点を含まないコンパクト集合である.  コンパクト集合上の連続関数は一様連続であるので $I_{j_1},\ldots,I_{j_m}$が求めるものである. 
\end{proof}
\begin{lem}\label{CovolutionUniformAppro}
  $\Psi \in \mathcal{M}(\R)$とする. このとき, 任意の$\rho \in C_0^{\infty}(\R)$に対して, $\Psi$と$\rho$の畳み込み積
  \[
  (\Psi * \rho)(x) = \int_{\R} \Psi(x-y) \rho(y) dy
  \]
  は$\mathcal{N}_1(\Psi,\{1\},\R)$により広義一様近似される. 
\end{lem}
\begin{proof}
  コンパクト集合$K \subset \R$をとる. $\alpha> 0$を適当に取り$K$と$\mathrm{supp}(\rho)$が$ [-\alpha,\alpha]$に含まれるようにする.  $[-\alpha,\alpha]$上で一様近似できることを示せば十分である. そこで, $m \in \N$に対して$y_j = - \alpha + 2j\alpha/m, \Delta y_i = 2 \alpha/m ~(i=1,\ldots,m)$とおき, 
  \[
  x \mapsto \sum_{j=1}^{m} \Psi(x - y_j) \rho(y_j) \Delta y_j
  \]
  が$\Psi * \rho$に$[-\alpha,\alpha]$上一様収束することを示そう. $\rho \neq 0$としてよい. $\varepsilon > 0$を任意に取る. $0 < \delta < \alpha$を次をみたすように取る. 
  \[
  10 \delta \lVert \Psi \rVert_{L^{\infty}([-\alpha,\alpha])} \lVert \rho \rVert_{L^{\infty}} \leq \varepsilon.
  \]
  いま, $\Psi \in \mathcal{M}(\R)$なので, $r(\delta)$個の開区間が存在して, それらの和集合を$U$とすると, そのLebesgue測度$\mu(U)$は$\mu(U) < \delta$をみたし, かつ$\Psi$は$[-2\alpha,2\alpha] \setminus U$上で一様連続となる. そして, $m \in \N$を$m \delta > \alpha r(\delta)$かつ, 
  \[
  \begin{aligned}
    &|s-t| < 2\alpha/m \Rightarrow |\rho(s) - \rho(t)| \leq \frac{\varepsilon}{2\alpha \lVert \Psi \rVert_{L^{\infty}([-2\alpha,2\alpha])}} \\
    &s,t \in [-2\alpha,2\alpha] \setminus U, |s-t| < 2\alpha/m \Rightarrow |\Psi(s) - \Psi(t)| \leq \frac{\varepsilon}{\lVert \rho \rVert_{L^1}}
  \end{aligned}
  \]
  となるように取る($\rho$は一様連続であることに注意). ここで$x \in [-\alpha,\alpha]$を固定する. $\Delta_j = [y_{j-1},y_j]$とおく(ただし, $y_0 = \alpha$とする)と, $\mathrm{supp}(\rho) \subset [-\alpha,\alpha]$であるから, 
  \[
  \begin{aligned}
  &\left| \int_{\R} \Psi(x-y) \rho(y) dy - \sum_{j=1}^m \int_{\Delta_j} \Psi(x - y_j) \rho(y) dy \right| \\
  &\leq \sum_{j=1}^m \int_{\Delta_j} |\Psi(x-y) - \Psi(x-y_j)||\rho(y)| dy
  \end{aligned}
  \]
  となる. $y \in \Delta_j$について$|y - y_j| < 2\alpha/m$であり, $x - y \in [-2\alpha,2\alpha]$であるので, もし$(x - \Delta_j) \cap U = \emptyset$であるならば, 
  \[
  \int_{\Delta_j} |\Psi(x-y) - \Psi(x-y_j)||\rho(y)| dy \leq \frac{\varepsilon}{\lVert \rho \rVert_{L^1}} \int_{\Delta_j} |\rho(y)| dy
  \]
  となる. ゆえに, そのような$j$すべてについて和を取ると, その和は$\varepsilon$で抑えられる. また, $(x - \Delta_j) \cap U \neq \emptyset$となる$\Delta_j$を$\tilde{\Delta}_j$でかくことにすると, $U$のLebesgue測度が$\delta$未満であること及び $U$が$r(\delta)$個の開区間の和集合であることから,  $\tilde{\Delta}_j$のLebesgue測度の総和は$\delta + (4\alpha/m)r(\delta)$で抑えられる. したがって, $m$の選び方から$\tilde{\Delta}_j$のLebesgue測度の総和は$5 \delta$で抑えられる. よって, $\delta$の取り方から, 
  \[
  \sum \int_{\tilde{\Delta}_i} |\Psi(x-y) - \Psi(x-y_j)||\rho(y)| dy \leq 2\lVert \Psi \rVert_{L^{\infty}([-2\alpha,2\alpha])} \lVert \rho \rVert_{L^{\infty}} 5 \delta \leq \varepsilon
  \]
  となる. 最後に, $\rho$に対する$m$の取り方から, 
  \[
  \begin{aligned}
  &\left| \sum_{j=1}^m \int_{\Delta_j} \Psi(x - y_j) \rho(y) dy - \sum_{j=1}^m \Psi(x-y_j)\rho(y_j)\Delta y_j \right| \\
  &= \left| \sum_{j=1}^m \int_{\Delta_j} \Psi(x - y_j) ( \rho(y)  - \rho(y_j))dy \right| \\
  &\leq \sum_{j=1}^m  \int_{\Delta_j} |\Psi(x - y_j)|~| \rho(y)  - \rho(y_j)| dy \\
  &\leq \sum_{j=1}^m  \int_{\Delta_j} \lVert \Psi \rVert_{L^{\infty}([-2\alpha,2\alpha])} \frac{\varepsilon}{2\alpha \lVert \Psi \rVert_{L^{\infty}([-2\alpha,2\alpha])}} dy \\
  &= \varepsilon
  \end{aligned}
  \]
  となる. 以上から, 任意の$x \in [-\alpha,\alpha]$に対して, 三角不等式により, 
  \[
  \left| \int_{\R} \Psi(x-y) \rho(y) dy - \sum_{j=1}^m \Psi(x-y_j) \rho(y_j) \Delta y_j \right| \leq 3 \varepsilon.
  \]
\end{proof}
\begin{rem}
  上の定理で「$\mathcal{N}_1(\Psi,\{1\},\R)$により広義一様近似される」と述べたが, $y_j \notin \mathrm{supp}(\rho)$のとき$\Psi(x - y_j) \rho(y_j) = 0$であるので,  $\mathcal{N}_1(\Psi,\{1\},\mathrm{supp}(\rho))$により広義一様近似されることがわかる. 
\end{rem}
\begin{lem}\label{PointwiseConvergencePolyPreserve}
  $N \in \N$とする. $p_j:\R \rightarrow \R$を$\mathrm{deg}(p_j) \leq N$なる多項式関数の列とする. このとき, $f:\R \rightarrow \R$に対して, $p_j$が$f$に各点収束するならば, $f$はある$N$次以下の多項式と一致する. また, $p_j$が$f$に概収束する(つまり, a.e.$x \in \R$に対して$p_j(x) \rightarrow f(x)$となる)ならば, $f$はある$N$次以下の多項式とa.e.で一致する.
\end{lem}
\begin{proof}
  まず, 各点収束の場合を考える. $p_j(x) = \sum_{k=0}^N a_{j,k} x^k$と表すことにする. 相異なる$x_0, x_1,\ldots, x_N \in \R $を任意に取り, 
  \[
  \begin{aligned}
  &\mathbf{p}^{(j)}
  =
  \left(
  \begin{array}{c}
       p_j(x_0)  \\
       p_j(x_1) \\
       \vdots \\
       p_j(x_N)
  \end{array}
  \right)
  , ~~~
  \mathbf{X}
  =
  \left(
  \begin{array}{cccc}
    1 & x_0^1 & \cdots & x_0^N  \\
    1 & x_1^1 & \cdots & x_1^N \\
    \vdots & \vdots & \ddots & \vdots \\
    1 & x_N^1 & \cdots & x_N^N
  \end{array}
  \right)
  , \\
  &\mathbf{a}^{(j)}
  =
  \left(
  \begin{array}{c}
    a_{j,0}  \\
    a_{j,1} \\
    \vdots \\
    a_{j,N}
  \end{array}
  \right)
  , ~~~
  \mathbf{f}
  =
  \left(
  \begin{array}{c}
    f(x_0) \\
    f(x_1) \\
    \vdots \\
    f(x_N)
  \end{array}
  \right)
  \end{aligned}
  \]
  とおく. すると, Vandermonde行列$\mathbf{X}$は正則であり, 
  \[
  \mathbf{p}^{(j)}-\mathbf{f} = \mathbf{X} \mathbf{a}^{(j)} - \mathbf{f}
  \]
  である. したがって, $p_j(x_i) \rightarrow f(x_i)$であることと行列写像が連続であることから, 
  \[
  \mathbf{a}^{(j)} = \mathbf{X}^{-1} \mathbf{f} + \mathbf{X}^{-1} (\mathbf{p}^{(j)}-\mathbf{f}) \rightarrow \mathbf{X}^{-1} \mathbf{f} ~~(j \rightarrow \infty)
  \]
  となる. よって, $(a_1,\ldots,a_N)^{\T} = \mathbf{X}^{-1} \mathbf{f}$とおき, $p(x) = \sum_{k=0}^N a_k x^k$とおくと, 各点で$p_j \rightarrow p$となる. よって, $f = p$となる.  概収束の場合も同様に示せる(Lebesgue測度で考えているので収束点を$N+1$個取れることに注意). 
\end{proof}
\begin{lem}\label{ConvolutionAlmostEverywhereConvergence}
  $\Psi \in \mathcal{M}(\R)$とする. このとき, $\rho_j \in C_0^{\infty}(\R) ~(j=1,2,\ldots)$が存在して, $\mathrm{supp}(\rho_j) \subset [-1/j,1/j]$かつ$\Psi * \rho_j$は$\Psi$に$\R$上概収束する. 
\end{lem}
\begin{proof}
  $\rho \in C_0^{\infty}(\R)$で$\int_{\R} \rho(x)dx = 1$かつ$\mathrm{supp}(\rho) \subset [-1,1]$をみたすものをとる. 例えば, 
  \[
  k(t) := \left\{\begin{array}{cl}e^{-\frac{1}{t}}&(t>0)\\0&(t\leq0)\end{array}\right.
  \]
  とおき, $H(x) := k(1-|x|), ~\rho(x) := H(x)/\int_{\R} H(y) dy$とおけばこの条件を満たす($k \in C^\infty(\R)$の丁寧な証明が松本\cite{MatsumotoManifold}のp.176〜178にあるので参照されたい). そして, 
  \[
  \rho_j(x) := j \rho(jx) ~~(x \in \R, j \in \N)
  \]
  とおく. すると, $\mathrm{supp}(\rho_j) \subset [-1/j,1/j]$ゆえ$\rho_j \in C_0^{\infty}(\R)$であり, $\Psi$の連続点$x \in \R$に対して, 
  \[
  \begin{aligned}
  (\Psi*\rho_j)(x) 
  &= \int_{\R} \Psi(x-y)\rho_j(y) dy \\
  &= \int_{\R} \Psi(x-z/j) \rho(z)dz ~~~~(z = jy) \\
  &\rightarrow \int_{\R} \Psi(x) \rho(z)dz = \Psi(x) ~~~~(\mbox{優収束定理})
  \end{aligned}
  \]
  である. $\Psi$に関する仮定より$\Psi$の不連続点全体の集合の測度は$0$なのでこれで証明できた. 
\end{proof}
以上で準備が整ったので定理を証明していこう. 
\begin{thm}\label{GeneByPinkus}
  $W \subset \R$は孤立点を持たず$0 \in W$であるとする. 
  また, $\Psi:\R \rightarrow \R$は$\mathcal{M}(\R)$に属するとする. このとき, 開区間$I$に対して$\Psi$が$I$上a.e.に多項式と一致しないなら, $\mathcal{N}_1(\Psi,W,I)$は$C(\R)$において広義一様収束の意味で稠密である. 
\end{thm}
\begin{proof}
  開区間$J = (c,d)$と$\varepsilon$を$J^{\varepsilon}:=(c-\varepsilon,d+\varepsilon) \subset I$となるように取る. $\Psi$についての仮定より任意の$\rho \in C_0^{\infty}(\R)$に対して$\Psi * \rho$は$C^{\infty}$級関数である(付録の命題\ref{L1locandC0Convolution}を参照). また, 補題\ref{CovolutionUniformAppro}より, $\mathrm{supp}(\rho) \subset (-\varepsilon,\varepsilon)$なる任意の$\rho \in C_0^{\infty}(\R)$と任意の$w \in W, b \in I$に対して, $x$の関数
  \[
  (\Psi * \rho)(wx + b) = \int_{\R} \Psi(wx+b-y)\rho(y)dy
  \]
  は$y_i \in \mathrm{supp}(\rho) \subset (-\varepsilon,\varepsilon)$を適当に取れば$\sum_{i} \Psi(wx+b-y_i)\rho(y_i)\Delta y_i$で広義一様近似される. ゆえに, 
  \[
  \overline{\mathcal{N}_1(\Psi * \rho,W,I)} 
  \subset \overline{\mathcal{N}_1(\Psi,W,J^{\varepsilon})} 
  \subset \overline{\mathcal{N}_1(\Psi,W,I)} 
  \]
  がわかる(ただし, ここでの閉包は広義一様収束の意味である). 
  いま$\Psi * \rho$は$C^{\infty}$級だから,  補題\ref{NeuralNetApproPoly}の証明より,  任意の自然数$k$に対して$x \mapsto x^k (\Psi * \rho)^{(k)}(b)$は$\overline{\mathcal{N}_1((\Psi * \rho,W,I)}$に含まれることがわかる. Stone-Weierstrassの定理より多項式関数全体は$C(\R)$において広義一様収束の意味で稠密であるので, もし$\mathcal{N}_1(\Psi,W,I)$が$C(\R)$において広義一様収束の意味で稠密でないならば, ある$m$が存在して$\mathrm{supp}(\rho) \subset (-\varepsilon,\varepsilon)$なる任意の$\rho \in C_0^{\infty}(\R)$に対して$(\Psi * \rho)^{(m)}(b) = 0 ~(\forall b \in I)$とならなければならない. したがって, このとき, $\mathrm{supp}(\rho) \subset (-\varepsilon,\varepsilon)$なる任意の$\rho \in C_0^{\infty}(\R)$に対して$\Psi * \rho$は$I$上で$m$次以下の多項式となる. ところが, 補題\ref{ConvolutionAlmostEverywhereConvergence}より$\rho_j \in C_0^{\infty}(\R)$が存在して, $\mathrm{supp}(\rho_j) \subset (-\varepsilon,\varepsilon)$かつa.e.$x \in I$に対して$\Psi * \rho_j(x) \rightarrow \Psi(x)$となる. よって, 補題\ref{PointwiseConvergencePolyPreserve}より$\Psi$は$I$上である多項式とa.e.で一致する. しかし, これは我々の仮定に反する. 
\end{proof}
以上でLeshnoらの主結果が証明できた. 
なお, 第2章で述べたようにLeshnoらの結果における連続性の仮定はHornikの指摘\cite{Hornik1993}により次の形に緩和された. 
\begin{thm*}[Hornik 1993 \cite{Hornik1993}]
  $W \subset \R^r$は$0$を内点として持つとする. 
  また, $\Psi:\R \rightarrow \R$は各閉区間上で有界かつRiemann積分可能であるとする. さらに, ある開区間$I$に対して$\Psi$が$I$上で多項式と一致しないとする. このとき, $\mathcal{A} \subset \R^r$上で$0$を取る非自明な$r$変数斉次多項式が存在しないならば, $\mathcal{N}_r(\Psi,W,I)$は$C(\R^r)$において広義一様収束の意味で稠密である. 
\end{thm*}
この定理は次の定理に含まれる. なぜなら, $0$を内点として持つ$W \subset \R^r$に対して, $\mathcal{A} \subset \R^r$として原点中心の開球をとり十分小さい$\varepsilon > 0$に対して$W' = (-\varepsilon,\varepsilon)$とおけば$W' \mathcal{A} \subset W$となるからである(開球上で多項式は一意に決定されることに注意). 
\begin{thm}
  $W \subset \R$は孤立点を持たず$0 \in W$であるとする. 
  また, $\Psi:\R \rightarrow \R$は各閉区間上で有界かつRiemann積分可能であるとする. さらに, ある開区間$I$に対して$\Psi$が$I$上で多項式と一致しないとする. このとき, $\mathcal{A} \subset \R^r$上で$0$を取る非自明な$r$変数斉次多項式が存在しないならば, 
  \[
  \begin{aligned}
  \mathcal{N}_r(\Psi,W\mathcal{A},I) = \mathrm{span}\{ \R^r \ni x \mapsto \Psi(w a^{T} x + b) \in \R \mid w \in W, a \in \mathcal{A}, b \in I \}
  \end{aligned}
  \]
  は$C(\R^r)$において広義一様収束の意味で稠密である. 
\end{thm}
さらにこの定理を示すためには, Leshnoらの結果の証明と同様にして次が成り立つことを示せばよい. 
\begin{prop}
  各閉区間上で有界かつRiemann積分可能な$\Psi:\R \rightarrow \R$に対して, 任意の$\rho \in C_0^{\infty}(\R)$に対して, $\Psi$と$\rho$の畳み込み積
  \[
  (\Psi * \rho)(x) = \int_{\R} \Psi(x-y) \rho(y) dy
  \]
  は$\mathcal{N}_1(\Psi,\{1\},\mathrm{supp}(\rho))$により広義一様近似される
\end{prop}
\begin{proof}
  空でないコンパクト集合$K \subset \R$を任意に取る. $\alpha> 0$を適当に取り$K$と $\mathrm{supp}(\rho)$が$ [-\alpha,\alpha]$に含まれるようにする. $m \in \N$に対して
  \[
  y_j = - \alpha + 2j\alpha/m,~~ \Delta y_j = 2 \alpha/m ~~~~(j=1,\ldots,m)
  \]
  とおき, 
  \[
  x \mapsto \sum_{j=1}^{m} \Psi(x - y_j) \rho(y_j) \Delta y_j
  \]
  が$\Psi * \rho$に$[-\alpha,\alpha]$上一様収束することを示そう. 任意の$x \in [-\alpha,\alpha]$に対して, 
  \[
  \begin{aligned}
  &\left\lvert \int_{\R} \Psi(x-y) \rho(y) dy - \sum_{j=1}^{m} \Psi(x - y_j) \rho(y_j) \Delta y_j \right\rvert \\
  &\leq \left| \int_{\R} \Psi(x-y) \rho(y) dy - \sum_{j=1}^m \int_{\Delta_j} \Psi(x - y_j) \rho(y) dy \right| \\
  &~~~~ + \left| \sum_{j=1}^m \int_{\Delta_j} \Psi(x - y_j) \rho(y) dy - \sum_{j=1}^m \Psi(x-y_j)\rho(y_j)\Delta y_j \right|
  \end{aligned}
  \]
  である. 右辺の二つの項がそれぞれ$[-\alpha,\alpha]$上一様に$0$に収束することを示せばよい. 
  まず, 第一項について考える. $\mathrm{supp}(\rho) \subset [-\alpha,\alpha]$であるから, 任意の$x \in [-\alpha,\alpha]$に対して, 
  \[
  \begin{aligned}
  &\left| \int_{\R} \Psi(x-y) \rho(y) dy - \sum_{j=1}^m \int_{\Delta_j} \Psi(x - y_j) \rho(y) dy \right| \\
  &\leq \sum_{j=1}^m \int_{\Delta_j} |\Psi(x-y) - \Psi(x-y_j)||\rho(y)| dy \\
  &\leq \lVert \rho \rVert_{L^{\infty}} \sum_{j=1}^m \left(  \sup_{y \in \Delta_j} \Psi(x-y) - \inf_{y \in \Delta_j} \Psi(x-y) \right)\frac{2\alpha}{m}
  \end{aligned}
  \]
  となる. $[-2\alpha,2\alpha]$を$2m$等分する分割$\Delta'_j~(j=1,\ldots,2m)$について, 
  \[
  \begin{aligned}
  &\sum_{j=1}^m \left( \sup_{y \in \Delta_j} \Psi(x-y) - \inf_{y \in \Delta_j} \Psi(x-y) \right) \\
  &\leq \sum_{j=1}^{2m} \left( \sup_{z \in \Delta'_j} \Psi(z) - \inf_{z \in \Delta'_j} \Psi(z) \right) 
  \end{aligned}
  \]
  となるので, $\Psi$が$[-2\alpha,2\alpha]$上Riemann積分可能であることより, $[-\alpha,\alpha]$上一様に
  \[
  \begin{aligned}
  &\left| \int_{\R} \Psi(x-y) \rho(y) dy - \sum_{j=1}^m \int_{\Delta_j} \Psi(x - y_j) \rho(y) dy \right| \\
  &\leq \lVert \rho \rVert_{L^{\infty}} \sum_{j=1}^{2m} \left( \sup_{z \in \Delta'_j} \Psi(z) - \inf_{z \in \Delta'_j} \Psi(z) \right) \frac{2\alpha}{m} \\
  &\rightarrow \lVert \rho \rVert_{L^{\infty}} \left( \int_{-2\alpha}^{2\alpha} \Psi(z) dz - \int_{-2\alpha}^{2\alpha} \Psi(z) dz \right) = 0 ~~(m \rightarrow \infty)
  \end{aligned}
  \]
  となる. 
  次に第二項について考えよう. $\varepsilon > 0$を任意に取る. 任意の$x \in [-\alpha,\alpha]$に対して, 
  \[
  \begin{aligned}
  &\left| \sum_{j=1}^m \int_{\Delta_j} \Psi(x - y_j) \rho(y) dy - \sum_{j=1}^m \Psi(x-y_j)\rho(y_j)\Delta y_j \right| \\
  &= \left| \sum_{j=1}^m \int_{\Delta_j} \Psi(x - y_j) ( \rho(y)  - \rho(y_j))dy \right| \\
  &\leq \sum_{j=1}^m  \int_{\Delta_j} |\Psi(x - y_j)|~| \rho(y)  - \rho(y_j)| dy \\
  &\leq \lVert \Psi \rVert_{L^{\infty}([-2\alpha,2\alpha])} \sum_{j=1}^m  \int_{\Delta_j} | \rho(y)  - \rho(y_j)| dy 
  \end{aligned}
  \]
  である. そこで, $\rho$が$\R$上一様連続であることに注意して, $N \in \N$を次をみたすように取る.  
  \[
    |s-t| < 2\alpha/N \Rightarrow |\rho(s) - \rho(t)| \leq \frac{\varepsilon}{2\alpha \lVert \Psi \rVert_{L^{\infty}([-2\alpha,2\alpha])}}. 
  \]
  すると, 任意の$m \geq N$と$x \in [-\alpha,\alpha]$に対して, 
  \[
  \begin{aligned}
    &\left| \sum_{j=1}^m \int_{\Delta_j} \Psi(x - y_j) \rho(y) dy - \sum_{j=1}^m \Psi(x-y_j)\rho(y_j)\Delta y_j \right| \\
    &\leq \lVert \Psi \rVert_{L^{\infty}([-2\alpha,2\alpha])} \sum_{j=1}^m  \int_{\Delta_j} | \rho(y)  - \rho(y_j)| dy \leq \varepsilon
  \end{aligned}
  \]
  となる. 以上で証明できた. 
\end{proof}
\begin{cor}
  $\Psi:\R \rightarrow \R$は各閉区間上で有界かつRiemann積分可能であるとする. このとき, $\Sigma^r(\Psi)$が$C(\R^r)$において広義一様収束の意味で稠密であることと, $\Psi$がいかなる多項式関数ともa.e.で一致することはないことは同値である. 
\end{cor}

\subsection{ChuiとLiの結果}
本節ではCybenkoの結果\cite{Cybenko}を含むChuiとLiによる次の結果を証明する. 
\begin{thm}[Chui and Li 1992 \cite{ChuiAndLi}]\label{UATByChuiandLi}　\\
$\sigma : \mathbb{R} \rightarrow \mathbb{R}$ を連続なsigmoidal関数とすると,
\[
\mathcal{N}_r(\sigma,\mathbb{Z}^r,\mathbb{Z}) = \mathrm{span} \{\R^r \ni x \mapsto \sigma(m^{\T} x + k) \in \R \mid m \in \mathbb{Z}^r, k \in \Z \}
\]
は$C(\R^r)$において広義一様収束の意味で稠密である. 
\end{thm}
この定理をふたつの補題に分けて証明する. 補題の主張を述べるために言葉と記号の定義をしておく.

$X \subset \R^r$に対して, 位相空間$(X, \mathcal{O})$を$\mathbb{R}^r$の部分空間とする.  すなわち$\mathcal{O} = \{X \cap U \mid U \subset \mathbb{R}^r, \mathrm{open} \}$とする. また, $(X, \mathcal{O})$上の正則符号付きBorel測度全体の集合を$M(X)$で表すことにする(正則符号付き測度の定義については付録を参照されたい). 

\begin{defn}[discriminatory関数]　\\
関数 $\sigma : \mathbb{R} \rightarrow \mathbb{R}$ が次の条件をみたすとき, discriminatory関数と呼ぶ: 任意の空でないコンパクト集合$X \subset \R^r$に対して, 
\[
\forall \mu \in M(X), \left(\forall m \in \mathbb{Z}^r, k \in \mathbb{Z}, \int_{X}{\sigma \left(m^{\T} x + k \right)}{d\mu(x)} = 0 \right) \Rightarrow \mu = 0. 
\]
\end{defn}

この定義のもとで以下のふたつの命題が成り立つ. 

\begin{lem}\label{UATofDiscriminatoryFunc}　\\
$\sigma : \mathbb{R} \rightarrow \mathbb{R}$ を連続なdiscriminatory関数とすると, $\mathcal{N}_r(\sigma,\mathbb{Z}^r,\mathbb{Z})$は$C(\R^r)$において広義一様収束の意味で稠密である. 
\end{lem}

\begin{lem}\label{discriminatory}　\\
有界で可測なsigmoidal関数$\sigma$はdiscriminatory関数である. 特に連続なsigmoidal関数はdiscriminatory関数である. 
\end{lem}

上記ふたつの補題から定理\ref{UATByChuiandLi}が導かれることは明らかであろう. 以下でこれらの補題の証明をしていく. 

\renewcommand{\proofname}{\bf{補題}\ref{UATofDiscriminatoryFunc}\bf{の証明}}

\begin{proof}　\\
$K \subset \R^r$を空でないコンパクト集合とする.  $\mathcal{N}_r(\sigma,\mathbb{Z}^r,\mathbb{Z})$の元(関数)の$K$上への制限すべてを集めた集合を$S$とおく. $S$は $C(K)$ の部分線形空間である. 補題\ref{UATofDiscriminatoryFunc}の主張は$\overline{S} =  C(K)$が成り立つということである. そこで, $\overline{S} \neq C(K)$を仮定する(背理法). すると,  Hahn-Banachの拡張定理により, $C(K)$上の有界線形汎関数$L \neq 0$で, $L = 0 ~(\mathrm{on} ~\overline{S})$ なるものが取れる(付録の系\ref{HahnBanachCor1}参照). $K$はコンパクトHausdorff空間であるので, Riesz-Markov-角谷の表現定理から, $\mu \in M(K)$で, \[
\forall f \in C(K), \\ L(f) = \int_{K}{f}{d\mu}
\]
をみたすものが取れる. 任意の$m \in \mathbb{Z}^n, k \in \mathbb{Z}$に対して, 写像
\[
K \ni x \mapsto \sigma(m^{\T} x + k) \in \mathbb{R}
\]
は$S$に属するので, 
\[
\int_{K} {\sigma(m^{\T} x + k)}{d\mu(x)} = 0
\]
となる. ゆえに, $\sigma$がdiscriminatory関数であることから$\mu = 0$となり, したがって$L = 0$となるが, これは$L \neq 0$に反する. よって, $\overline{S} = C(K)$でなければならない. 
\end{proof}

\renewcommand{\proofname}{\bf{補題}\ref{discriminatory}\bf{の証明}}

\begin{proof}　\\
後半については連続性から可測性がわかり, sigmoidal関数であることと連続性から有界性がわかるので成立する. 前半を証明する. $X \subset \R^r$を空でないコンパクト集合とする. 
$\mu \in M(X)$を任意にとり固定する. そして, 
\[
\forall m \in \mathbb{Z}^r, k \in \mathbb{Z}, ~\int_{X}{\sigma \left(m^{\T} x + k \right)}{d\mu(x)} = 0 
\]
が成り立つと仮定する. 
このとき, $\mu = 0$を示すのが目標である. 
さて, いま, 任意の$m \in \mathbb{Z}^r, q \in \mathbb{Q}, k \in \mathbb{Z}$に対して, 
\[
\Pi_{m, q} = \left\{ x \in X \mid m^{\T} x + q = 0 \right\} , ~~
H_{m, q} = \left \{ x \in X \mid m^{\T} x + q > 0 \right \}
\]
とおき, 
\[
\gamma(x) :=
\begin{cases}
1 & x \in H_{m, q} \\
\sigma(k) & x \in \Pi_{m, q} \\
0 & \mathrm{otherwise}
\end{cases}
\]
と定めると,  $\sigma$はsigmoidal関数なので, $+\infty$に発散する自然数列$l_n$を任意に取ると, 任意の$x \in \R^r$に対して, 
\[
\lim_{n \rightarrow \infty} \sigma(l_n (m^{\T} x + q) + k) \rightarrow \gamma(x)
\]
となる. そこで, $q = q_1/q_2 , ~q_1 \in \Z, q_2 \in \N$と表すことにすると, 
$\sigma$の有界性と$\mu$の有限性および仮定より, 
\[
\begin{aligned}
0 &= \lim_{n \rightarrow +\infty}{\int_{X}{\sigma(n q_2 (m^{\T} x + q) + k)}{d\mu(x)}} \\
&= \int_{X}{\lim_{n \rightarrow +\infty}\sigma(n q_2 (m^{\T} x  + q) + k)}{d\mu(x)} \\
&= \int_{X}{\gamma(x)}{d\mu(x)}\\
&= \sigma(k)\mu\left( \Pi_{m, q} \right)+\mu\left( H_{m, q} \right)
\end{aligned} 
\]
となる(2つ目の等号は優収束定理による). よって, $\sigma$がsigmoidal関数であることと$k \in \Z$が任意であることに注意すると, $k \rightarrow -\infty$とすることで, $\mu\left( H_{m, q} \right) = 0$が得られ, さらに$k \rightarrow +\infty$とすることで$\mu\left( \Pi_{m, q} \right) = 0$が得られる. したがって,  $\mathbb{R}$上の有界な可測関数のなす線形空間上の線形汎関数$F$を
\[
F(h) = \int_{X}{h(m^{\T} x)}{d\mu(x)}
\]
により定めると, 定義関数
\[
h(x) = \chi_{[q, \infty)}(x) := 
\begin{cases}
1 & (x \in [q, \infty)) \\
0 & (x \notin [q, \infty))
\end{cases}
\]
に対して, 
\[
F(h) = \int_{X}{h(m^{\T} x)}{d\mu(x)} = \mu\left( \Pi_{m, -q} \right)+\mu\left( H_{m, -q} \right) = 0
\]
となる. したがって, $q \in \mathbb{Q}$は任意であったから,  $F$の線形性より有理数を端点とする任意の$\mathbb{R}$の区間$J$について$F(\chi_{J}) = 0$となる.  したがって, 有限個の$\mathbb{R}$の区間の定義関数の線形結合を$F$で写すと$0$になる. 以下, 一旦$m \in \mathbb{Z}^r$を固定して考える. このとき
$A = \{m^{\T} x \mid x \in X \}$は$\mathbb{R}$のコンパクト集合である. 何故なら$A$はコンパクト集合$X$の連続写像による像だからである. そこで有理数を端点とする有界閉区間$J$を$A \subset J$をみたすように取れる.  いま$J$上の任意の連続関数$g$は有理数を端点とする$\mathbb{R}$の区間の定義関数たちの線形結合$\sum_{k=1}^{N}{a_k \chi_{J_k}}$により一様近似できる.  すなわち任意の正数$\varepsilon$に対して, 適当な$a_k \in \R$と有理数を端点とする区間$J_k$を 取れば, 
\[
\forall x \in J,  ~ \left\lvert g(x) - \sum_{k=1}^{N}{a_k \chi_{J_k}(x)} \right\rvert < \varepsilon
\]
とできる. 実際,  コンパクト集合上の連続関数は一様連続なので, 任意の$\varepsilon > 0$に対して, $\delta>0$を適当に取れば, 
\[
\forall x, y \in J,  |x-y|<\delta \Rightarrow |g(x)-g(y)|<\varepsilon
\]
が成り立つ. そこで, 一つの区間の幅が$\delta$に収まるような分割$J = \sum_{k=1}^{N}{J_k}$であって, 各区間$J_k$の端点が有理数になるものを取る. そして各$k$に対して$x_k \in J_k$を任意に取れば, 
\[
\forall x \in J,~~ \left\lvert g(x) - \sum_{k=1}^{N}{g(x_k) \chi_{J_k}(x)} \right\rvert < \varepsilon
\]
となる. したがって, 
\[
\begin{aligned}
\left\lvert F(g)\right\rvert 
&= \left \lvert \int_{X}{g(m^{\T} x)}{d\mu(x)} \right \rvert \\
&= \left \lvert \int_{X}{g(m^{\T} x)}{d\mu(x)} - \int_{X}{\sum_{k=1}^{N}{a_k \chi_{J_k}(m^{\T} x)}}{d\mu(x)} \right \rvert \\
&\leq \int_{X}{\left \lvert g(m^{\T} x) - \sum_{k=1}^{N}{a_k \chi_{J_k}(m^{\T} x)}\right \lvert}{d\left\rvert\mu\right\rvert(x)} \\
&\leq \int_{X}{\varepsilon}{d\lvert\mu\rvert(x)} = \varepsilon \left\lvert \mu \right\rvert(X)
\end{aligned} 
\]
となる. ただし$\left\lvert\mu\right\rvert = \mu^{+}+\mu^{-}$である. よって, $\left\lvert\mu\right\rvert(X)$が$\varepsilon$に依存しない定数(有限値)であることと$\varepsilon$の任意性から, $F(g) = 0$がわかる. 特に$\cos(x), \sin(x)$についてもこれが成立するので, 
\[
0 = \int_{X}{\cos(m^{\T} x)+i\sin(m^{\T} x)}{d\mu(x)} = \int_{X}{\exp(im^{\T} x)}{d\mu(x)}
\]
がわかる. よって, $m \in \mathbb{Z}$は任意だったのでStone-Weierstrassの定理(付録の系\ref{CorforChuiandLi})から$\mu = 0$となる. これが示したいことであった. 
\end{proof}

\renewcommand{\proofname}{\bf{証明}}

\subsection{SunとCheney の結果}
本節では ChuiとLi の結果を自然にノルム空間$X$上の連続関数空間$C(X)$に拡張した SunとCheney による次の結果を証明する. 
また, mean-periodic関数などニューラルネットワークの万能近似定理に関連する事柄をいくつか紹介する. 
\begin{thm}[Sun and Cheney 1992 \cite{SunAndCheney}]　\\
  $X \neq \{0\}$をノルム空間とし, $\F \subset X^*$は
  \[
  \{ f/\lVert f \rVert \mid f \in \F, f \neq 0 \}
  \]
  が$X^*$の単位球面$\{f \in X^* \mid \lVert f \rVert = 1\}$において作用素ノルム$\lVert \cdot \rVert$に関して稠密であるものとする. このとき, 連続なsigmoidal関数$\sigma:\R \rightarrow \R$に対して, 
  \[
  \mathcal{G} = \{\R \ni t \mapsto \sigma(kt+j) \in \R \mid k,j \in \mathbb{Z} \}
  \]
  とおくと, 
  \[
  \mathcal{G} \circ \F = \{X \ni x \mapsto \sigma(kf(x)+j) \in \R \mid k,j \in \mathbb{Z}, f \in \F \}
  \]
  の線形包は$C(X)$において広義一様収束の意味で稠密である. 
\end{thm}
この定理を示すためにひとつ定理を示す.  なお本節では記述を簡潔にするために次の言葉を用いる. 
\begin{defn}
  $X$をノルム空間とするとき, 部分集合$S \subset C(X)$が基礎的であるとは, $S$の線形包が広義一様収束の意味で$C(X)$において稠密であることをいう. 
\end{defn}
\begin{thm}\label{MainResult2}
  $X \neq \{0\}$をノルム空間とし, $\mathcal{G} \subset C(\R)$は基礎的であるとする.  このとき, $\F \subset X^*$に対して, 
  \[
  \{ f/\lVert f \rVert \mid f \in \F, f \neq 0 \}
  \]
  が$X^*$の単位球面$\{f \in X^* \mid \lVert f \rVert = 1\}$において作用素ノルム$\lVert \cdot \rVert$に関して稠密であるならば, $\mathcal{G} \circ \F = \{ g \circ f \mid g \in \mathcal{G}, f \in \F \}$は$C(X)$において基礎的である. 
\end{thm}
これが示されれば, ChuiとLiの結果より$\mathcal{G}$は$C(\R)$において基礎的であるので本節のメインの定理の成立は明らかである. 
\begin{proof}
  $h \in C(X)$とし, 任意に空でないコンパクト集合$K \subset X$と$\varepsilon > 0$をとる. すると, 定理\ref{densityRidgeFunc}より,  $h_i \in C(\R)$と$\psi_i \in X^*$が存在して, 
  \[
  \sup_{x \in K} \left|h(x) - \sum_{i=0}^n h_i(\psi_i(x))\right| < \varepsilon/3
  \]
  となる. $h_i$たちを定数倍することで$\lVert \psi_i \rVert = 1$と仮定してよい($0$がある場合は$\R \ni t \mapsto h_i(0) \in \R$という定数関数を導入すればよい). ここで, $M = \sup_{x \in K} \lVert x \rVert$とおき, $\delta > 0$を次をみたすように取る($h_i$は連続なので取れる). 
  \[
  |h_i(s) - h_i(t)| < \varepsilon/3n ~~(|s|\leq M , |t| \leq M, |s-t|< M \delta, 1 \leq i \leq n).
  \]
  そして, $f_i \in \F$を$\lVert f_i / \lVert f_i \rVert - \psi_i \rVert < \delta$をみたすように選ぶ. 
  また, 各$i$について$\lambda_i = 1/ \lVert f_i \rVert$とおき, $a_{i,j} \in \R$と$g_{i,j} \in \mathcal{G}$を
  \[
  \sup_{|t| \leq M / \lambda_i} \left|h_i(\lambda_i t) - \sum_{j=0}^{N_i} a_{i,j} g_{i,j}(t)\right| < \varepsilon/3n 
  \]
  となるようにとる. すると, 任意の$x \in K$に対して, 
  \[
  |f_i(x)| \leq \lVert f_i \rVert \lVert x \rVert \leq M \lVert f_i \rVert = M/\lambda_i
  \]
  であり, また, $|\psi_i(x)| \leq \lVert \psi_i \rVert \lVert x \rVert \leq M$かつ
  \[
  |\lambda_i f_i(x) - \psi_i(x)| \leq \lVert f_i / \lVert f_i \rVert - \psi_i \rVert \lVert x \rVert < M \delta
  \]
  より$| h_i(\lambda_i f_i(x)) - h_i(\psi_i(x)) | < \varepsilon/3n$となる. よって, 三角不等式より, 
  \[
  \begin{aligned}
  &\left|h(x) - \sum_{i=0}^n \sum_{j=0}^{N_i} a_{i,j} g_{i,j}(f_i(x)) \right| \\
  &\leq \left|h(x) - \sum_{i=0}^n h_i(\psi_i(x))\right| + \left|\sum_{i=0}^n h_i(\psi_i(x)) - \sum_{i=0}^n h_i(\lambda_i f_i(x)\right| \\
  &~~~~ + \left| \sum_{i=0}^n h_i(\lambda_i f_i(x)) - \sum_{i=0}^n \sum_{j=0}^{N_i} a_{i,j} g_{i,j}(f_i(x)) \right| \\
  &\leq \varepsilon/3 + \varepsilon/3 + \varepsilon/3 = \varepsilon.
  \end{aligned}
  \]
\end{proof}
\begin{cor}
  上の定理は$\F \subset X^*$が単位球面において作用素ノルムに関して稠密であるという仮定のもとで成り立つ. 
\end{cor}
\begin{cor}
  $X \neq \{0\}$をノルム空間とし, $\F \subset X^*$を上の定理と同様のものとする. また, $g \in C(\R)$とし, $A \subset \R$とする. このとき, 
  \[
  \mathcal{G} = \{ g_{a}: \R \ni t \mapsto g(t + a) \in \R \mid a \in A\}
  \]
  が$C(\R)$において基礎的であるならば, $\mathcal{G} \circ \F = \{ g_a \circ f \mid a \in A, f \in \F\}$は$C(X)$において基礎的である. 
\end{cor}
この系に関連してmean-periodic関数について紹介する. 
\begin{defn}
   関数$f \in C(\R^r)$が mean-periodic であるとは, 
   \[
   \mathcal{N}_r(f,\{1\},\R) = \mathrm{span}\{x \mapsto f(x+a) \mid a \in \R^r \}
   \]
   が$C(\R^r)$において広義一様収束の意味で稠密でないことをいう. 
\end{defn}
\begin{thm}[Schwartz 1947 \cite{Schwartz1947} p.907]　\\
  $\Psi \in C(\R)$とする. このとき, ある$1\leq p < \infty$について$\Psi \in L^p(\R)$であるか, または$\Psi$は有界で非定数かつ$x \rightarrow \infty$か$x \rightarrow -\infty$としたときに極限値を持つならば, $\Psi$は mean-periodicでない. 
\end{thm}
これより上の定理の仮定をみたす$\Psi$については, 任意の実数$\lambda \neq 0$について, $\N(\Psi,\{\lambda\},\R)$は$C(\R)$において広義一様収束の意味で稠密である. さらに, LinとPinkusの結果(定理\ref{LinAndPinkus})を用いると次の命題の成立もわかる. 
\begin{prop}
  $\Psi$は上の定理の仮定をみたすものとする. このとき, $A \subset \R^r$上で$0$を取る非自明な$r$変数斉次多項式が存在しないならば, 
  \[
  \mathcal{N}_{r}(\Psi,A,\R) = \mathrm{span}\{\R^r \ni x \mapsto \Psi(\ip<{x,v}> + b) \mid v \in A, b \in \R\}
  \]
  は$C(\R^r)$において広義一様収束の意味で稠密である. 
\end{prop}
この命題を二段階に分けて証明しよう. 
\begin{lem}
  $X$をノルム空間とし, $E \subset C(\R)$は$C(\R)$において基礎的であるとする. また, $\F \subset X^*$は$C(\R) \circ \F$が$C(X)$において基礎的なものとする. このとき, $\{g \circ f \mid g \in E, f \in \F\}$は$C(X)$において基礎的である. 
\end{lem}
\begin{proof}
  $K \subset X$を空でないコンパクト集合とし, 任意に$\varphi \in C(X)$と$\varepsilon > 0$を取る. 
  仮定より, $\beta_i \in \R, g_i \in C(\R), f_i \in \F ~(i=1,\ldots,q)$が存在して, 
  \[
  \sup_{x \in K} |\sum_{i=1}^q \beta_i g_i(f_i(x)) - \varphi(x)| < \varepsilon/2
  \]
  となる. 各$i$について$f_i$は連続であるから$f_i(K) \subset \R$はコンパクトである. そこで, 仮定より各$i$について$\alpha_{i,j} \in \R, g_{i,j} \in E ~(j=1,\ldots,m_i)$が存在して
  \[
  \sup_{y \in f_i(K)} |\sum_{j=1}^{m_i} \alpha_{i,j} g_{i,j}(y) - \beta_i g_i(y)| < \varepsilon/(2q)
  \]
  となる. すると, 任意の$x \in K$に対して, 
  \[
  \begin{aligned}
  &\left| \sum_{i=1}^q \beta_i g_i(f_i(x)) - \sum_{i=1}^q \sum_{j=1}^{m_i} \alpha_{i,j}  g_{i,j}(f_i(x)) \right| \\
  &\leq \sum_{i=1}^q \left| \beta_i g_i(f_i(x)) - \sum_{j=1}^{m_i} \alpha_{i,j}  g_{i,j}(f_i(x)) \right| \\
  &\leq \sum_{i=1}^q \sup_{y \in f_i(K)} \left|\beta_i g_i(y) - \sum_{j=1}^{m_i} \alpha_{i,j} g_{i,j}(y)\right| \\
  &< \varepsilon/2
  \end{aligned}
  \]
  であるから, $\sum_{i=1}^q \sum_{j=1}^{m_i} \alpha_{i,j}  g_{i,j} \circ f_i$が求めるものである. 
\end{proof}
\begin{prop}
  $\Psi \in C(\R)$とし, $B \subset \R$とする. また, 集合
  \[
  S := \{ \R \ni t \mapsto \Psi(t+b) \in \R \mid b \in B \}
  \]
  は$C(\R)$において基礎的であるとする. このとき, $A \subset \R^r$上で$0$を取る非自明な$r$変数斉次多項式が存在しないならば, 
  \[
  \mathcal{N}_{r}(\Psi,A,B) = \mathrm{span}\{\R^r \ni x \mapsto \Psi(\ip<{x,v}> + b) \mid v \in A, b \in B\}
  \]
  は$C(\R^r)$において広義一様収束の意味で稠密である. 
\end{prop}
\begin{proof}
  仮定とLinとPinkus の結果より$C(\R) \circ \{x \mapsto \ip<{x,v}> \mid v \in A\}$が基礎的であるのですぐ上の補題から明らかである. 
\end{proof}

さて, 本節の冒頭で「ChuiとLi の結果を自然にノルム空間$X$上の連続関数空間$C(X)$に拡張した」と述べた. きちんと拡張になっていることを以下で確認しておく. 
つまり, $X=\R^s$とおくとき, $\F = \{x \mapsto \ip<{x,v}> \mid v \in \Z^s \}$とおくと, 
  \[
  \{ f/\lVert f \rVert \mid f \in \F, f \neq 0 \}
  \]
  が$X^*$の単位球面$\{f \in X^* \mid \lVert f \rVert = 1\}$において作用素ノルム$\lVert \cdot \rVert$に関して稠密であることを確かめておく. 
\begin{lem}
  $X,Y$を位相空間とし, $f:X \rightarrow Y$を全射連続写像とする. このとき, $A \subset X$が稠密であるならば, $f(A) \subset Y$は稠密である. 
\end{lem}
\begin{proof}
  $y \in Y$をとり, $y$の近傍$V$を任意に取る. $f$は全射なので$x \in X$で$y = f(x)$をみたすものが取れる. すると, $f^{-1}(V)$は$x$の近傍となる. そこで, 仮定より$a \in A$で$a \in f^{-1}(V)$なるものが取れる. したがって$f(a) \in V$であるから, $f(A) \cap V \neq \emptyset$である. これは$f(A)$が$Y$で稠密であることを意味する. 
\end{proof}
\begin{prop}
  $X=\R^s$とおくとき, $\F = \{x \mapsto \ip<{x,v}> \mid v \in \Z^s \}$とおくと, 
  \[
  \{ f/\lVert f \rVert \mid f \in \F, f \neq 0 \}
  \]
  は$X^*$の単位球面$\{f \in X^* \mid \lVert f \rVert = 1\}$において作用素ノルム$\lVert \cdot \rVert$に関して稠密である. 
\end{prop}
\begin{proof}
  シュワルツの不等式から$\R^s \ni x \mapsto \ip<{x,u}> \in \R$の作用素ノルムは$\lVert u \rVert$と一致するので, 
  \[
  \{f/ \lVert f \rVert \mid f \in \F , f \neq 0 \}
  \]
  が$\{\R^s \ni x \mapsto \ip<{x,u}> \in \R \mid u \in \R^s, \lVert u \rVert = 1 \}$において作用素ノルムの意味で稠密であることを示せばよい. そして, このことを示すには, 
  \[
  S = \{m/\lVert m \rVert \mid m \in \mathbb{Z}, z \neq 0\}
  \]
  が$U = \{ u \in \R^s \mid \lVert u \rVert = 1 \}$においてユークリッド距離に関して稠密であることを示せば十分である. 
  いま, 写像$\varphi$を
  \[
  \varphi:\R^s \setminus \{0\} \rightarrow U, ~ \varphi(x) = x/\lVert x \rVert
  \]
  と定義すると, これは連続かつ全射である. そして, $\mathbb{Q}^s \setminus \{0\}$は$\R^s \setminus \{0\}$において稠密であるから, すぐ上の補題より
  \[
  \varphi(\Q^s \setminus \{0\}) = \{ q/\lVert q \rVert \mid q \in \mathbb{Q}^s , q \neq 0 \}
  \]
  は$U$において稠密である. しかも, $\varphi(\Q^s \setminus \{0\}) = S$である. 実際, $q \in \Q^s \setminus \{0\}$を取ると, $n \in \mathbb{Z}^s$と$m \in \mathbb{Z}$を使って$q = n/m $と表せるので, $q/\lVert q \rVert = n/\lVert n \rVert \in S$となる. 逆向きの包含関係は明らかである. よって, $S$は$U$において稠密である. 以上で系は示された. 
\end{proof}
なお, この命題から次のことも直ちにわかる. 
\begin{prop}
  $\mathcal{G} \subset C(\R)$が基礎的であるとすると, 
  \[
  \{ \R^s \ni x \mapsto g(\ip<{x,v}>) \in \R \mid g \in \mathcal{G}, v \in \mathbb{Z}^s \}
  \]
  は$C(\R^s)$において基礎的である. 
\end{prop}

\subsection{Hornikの結果}
本節ではHornikの結果, つまり, 活性化関数$\Psi$がsquashing関数であるときに$\Sigma^r(\Psi)$が$C(\R^r)$において広義一様収束の意味で稠密であることの証明を与える. ただし, squashing関数の定義は次の通りである. 
\begin{defn}[squashing関数]　\\
 関数 $\Psi : \mathbb{R} \rightarrow [0,1]$がsquashing関数であるとは, $\Psi$が単調非減少であり, 
 \[
 \Psi(t) \rightarrow
 \begin{cases}
 1 & (t \rightarrow +\infty)\\
 0 & (t \rightarrow -\infty)
 \end{cases}
 \]
 が成り立つことをいう. つまり, 単調非減少なsigmoidal関数をsquashing関数という. 
\end{defn}
\begin{rem}
  squashing関数はBorel可測である. なぜなら,  単調非減少関数はBorel可測だからである. 実際, $f:\R \rightarrow \R$を単調非減少関数とするとき, 任意の$a \in \R$に対して, $x = \sup\{y \in \R \mid f(y) \leq a\}$とおくと, 
  \[
  f^{-1}((-\infty,a]) = (-\infty,x) ~\mbox{or}~ (-\infty,x] \in \mathcal{B}_{\R}
  \]
  となるから, $f$はBorel可測である. 
\end{rem}
squashing関数には連続性が要求されていないことに注意されたい. Hornikは連続性の仮定をなくすことができるかということに関心があったようである\cite{Hornik}.
Leshno et al.の結果の節で紹介したHornikの結果にもこの姿勢は表れている. 

さて, まず次の$\Sigma \Pi$-ネットワークについて考察していく. 
\begin{thm}\label{MainTheorem1}
  $G : \R \rightarrow \R$を定数でない連続関数とする. このとき, 
  \[
  \Sigma \Pi^r(G) 
  := \left\{ f:\R^r \rightarrow \R ~\vline~ 
  \begin{aligned}
  &f = \sum_{j}^{q} \beta_j \prod_{k=1}^{l_j}  G \circ A_{j,k} \\
  &\beta_j \in \R, A_{j,k} \in \A^r , l_j \in \N,  q \in \N
  \end{aligned}
  \right\}
  \]
  は$C(\R^r)$において広義一様収束の意味で稠密である. 
\end{thm}
\begin{proof}
  任意に空でないコンパクト集合$K \subset \R^r$を取る. すると$\Sigma \Pi^r(G)_K$は$C(K)$の部分代数である. 
  実際, 線形部分空間であることは明らかである. 
  また, $f,g \in \Sigma \Pi^r(G)_K$ならば$f g \in \Sigma \Pi^r(G)_K$であることも有限個の組$(i,j)$を一列化することで示される. 
  いま, $G$は定数でないので$G(a) \neq 0$なる$a \in \R$が取れる. 
  このとき, $A(x) := 0^{\T} x + a = a$と定義すると, $G \circ A / G(a) \in \Sigma \Pi^r(G)_K$となり, $G(A(x))/G(a) = 1 ~~(\forall x \in K)$となる. 
  次に, $x,y \in K , ~ x \neq y$をとる. 
  $G$は定数でないので, $G(a) \neq G(b)$なる$a,b \in \R$が取れる. 
  いま, $A(z) = w^{\T} z + \beta$で$A(x) = a, A(y) = b$なるものが取れる. 実際, $x \neq y$ゆえ, ある$i$で$x_i - y_i \neq 0$となることに注意して, $w_i = (a-b)/(x_i-y_i), w_j = 0 ~(j \neq i)$として, $\beta = (b x_i - a y_i)/(x_i - y_i)$とおけば良い. 
  このとき, $G(A(x)) = a \neq b = G(A(y))$である. 
  よって, Stone-Weierstrassの定理によって, $\Sigma \Pi^r(G)_K$は$C(K)$において一様収束の意味で稠密である. 
  以上から, $\Sigma \Pi^r(G)$は$C(\R^r)$において広義一様収束の意味で稠密である. 
\end{proof}
\begin{lem}\label{squashingFuncUniAppro}
  $\Psi:\R \rightarrow \R$をsquashing関数とする. このとき, 任意の連続なsquashing関数$F$と$\varepsilon > 0$に対して, ある$H \in \Sigma^1(\Psi)$が存在して, 
  \[
  \sup_{x \in \R} |F(x) - H(x)| < \varepsilon.
  \]
\end{lem}
\begin{proof}
  任意に$\varepsilon > $を取る. $\varepsilon < 1$としてよい. $\beta_j \in \R$と$A_j \in \A^1$で, 
  \[
  \sup_{x \in \R} |F(x) - \sum_{j=0}^{Q-1} \beta_j \Psi(A_j(x))| < \varepsilon
  \]
  なるものを見つけたい. そこで$Q \in \N$を$1/Q < \varepsilon/2$を満たすように取る. そして, $\beta_j = 1/Q$とおく. いま$\Psi$はsquashing関数なので$M > 0$で
  \[
  \Psi(-M) < \frac{1}{Q},~~\Psi(M) > 1 - \frac{1}{Q}
  \]
  なるものが取れる. また, $F$は連続なsquashing関数なので
  \[
  \begin{aligned}
  &r_j = \sup\{x \in \R \mid F(x) = j/Q \} ~~~(j=1,\ldots,Q-1) \\
  &r_Q = \sup \{x \in \R \mid F(x) = 1 - 1/(2Q) \}
  \end{aligned}
  \]
 が存在する. そして, $r_0 \in (-\infty,r_1)$を任意に取る. さて, $r < s$に対して$A_{r,s} \in \A^1$を$A_{r,s}(r) = -M, A_{r,s}(s) = M$なるものとする(これは一意に定まる). このとき, $A_j = A_{r_j,r_{j+1}}$とすると, 
 \[
 H = \sum_{j=0}^{Q-1} \beta_j \Psi \circ A_j ~\in \Sigma^1(\Psi)
 \]
 が求めるものであることがわかる. 実際, $x \in (-\infty,r_0]$の場合は, $0 \leq F(x) \leq 1/Q$であり, かつ$\Psi$がsquashing関数であることより, 
 \[
 0 \leq H(x) \leq \sum_{j=0}^{Q-1} \frac{1}{Q} \Psi(-M) \leq \frac{1}{Q}
 \]
 となる. よって, $|F(x) - H(x)| \leq 1/Q < \varepsilon$となる. 次に, $x \in (r_0,r_1]$の場合を考える. このとき, $0 \leq F(x) \leq 1/Q$であり, 
 \[
 \begin{aligned}
 0  \leq H(x) 
 &\leq \frac{1}{Q} + \sum_{j=1}^{Q-1} \frac{1}{Q} \Psi(-M) \\
 &\leq \frac{1}{Q} + \frac{Q-1}{Q}\frac{1}{Q} \\
 &\leq \frac{1}{Q} + \frac{1}{Q}
 \end{aligned}
 \]
 となるので, $|F(x)-H(x)| \leq 1/Q + 1/Q \leq \varepsilon$となる. 次に$x \in (r_j,r_{j+1}] ~~(j=1,\ldots,Q-2)$の場合を考える. このとき, $r_j$の定義より
 \[
 \frac{j}{Q} \leq F(x) \leq \frac{j+1}{Q}
 \]
 である. また, 
 \[
 \begin{aligned}
 H(x) \geq \sum_{k=0}^{j-1} \frac{1}{Q} \Psi(M) \geq \frac{j}{Q} - \frac{j}{Q^2} \geq \frac{j-1}{Q}
 \end{aligned}
 \]
 かつ
 \[
 \begin{aligned}
 H(x) 
 &\leq \sum_{k=0}^{j} \frac{1}{Q} + \sum_{k=j+1}^{Q-1} \frac{1}{Q} \Psi(-M) \\
 &\leq \frac{j+1}{Q} + \frac{1}{Q} = \frac{(j+2}{Q}
 \end{aligned}
 \]
 である. したがって, $|F(x)-H(x)| \leq 2/Q \leq \varepsilon$となる. 次に$x \in (r_{Q-1},r_Q]$の場合を考える. このとき, 
 \[
 1 - \frac{1}{Q} = \frac{Q-1}{Q} \leq F(x) \leq 1 - \frac{1}{2Q} 
 \]
 である. また, 
 \[
 H(x) \geq \sum_{j=0}^{Q-2} \frac{1}{Q} \Psi(M) \geq  \frac{Q-1}{Q} - \frac{Q-1}{Q^2} \geq \frac{Q-2}{Q} = 1 - \frac{2}{Q}
 \]
 かつ, 
 \[
 H(x) \leq \sum_{j=0}^{Q-1} \frac{1}{Q} = 1
 \]
 である. したがって, $|F(x)-H(x)| \leq \max\{1/Q,2/Q-1/(2Q)\} \leq \varepsilon$となる. 最後に, $x \in (r_Q,\infty)$の場合を考える. このとき, 
 $1 - 1/(2Q) \leq F(x) \leq 1$であり, 
 \[
 1 \geq H(x) \geq \sum_{j=0}^{Q-1} \frac{1}{Q} \Psi(M) \geq 1 - \frac{1}{Q}
 \]
 である. したがって, $|F(x)-H(x)| \leq 1/Q \leq \varepsilon$となる. 
\end{proof}
\begin{thm}\label{MainTheorem3}
  $\Psi:\R \rightarrow \R$をsquashing関数とする. このとき, $\Sigma \Pi^r(\Psi)$は$C(\R^r)$において広義一様収束の意味で稠密である. 
\end{thm}
\begin{proof}
  連続なsquashing関数$F$に対して, $\Sigma \Pi^r(\Psi)$が$\Sigma \Pi^r(F)$において広義一様収束の意味で稠密であることを示せば十分である. 
  なぜなら, それが示されれば定理\ref{MainTheorem1}により$\Sigma \Pi^r(F)$は$C(\R^r)$において広義一様収束の意味で稠密であるので, $\Sigma \Pi^r(\Psi)$は$C(\R^r)$において広義一様収束の意味で稠密となる. 
  
  さて, このことを示すためには, $\prod_{k=1}^{l} F \circ A_k$の形の関数を$\Sigma \Pi^r(\Psi)$の元で一様近似できることを示せばよい. 
  そこで, $\varepsilon > 0$を任意に取る. 積$\prod_{k=1}^{l} : [0,1]^l \ni (a_1,\ldots,a_l) \mapsto \prod_{k=1}^{l} a_l \in \R$は連続で$[0,1]^l$はコンパクトなので, ある$\delta$が存在して任意の$a_k,b_k \in [0,1]$に対して, $|a_k - b_k| < \delta ~~(k=1,\ldots,l)$ならば$|\prod_{k=1}^l a_k - \prod_{k=1}^l b_k| < \varepsilon$となる. 
  一方, 補題\ref{squashingFuncUniAppro}により, ある$\beta_j \in \R$と$A_j^1 \in \A^1$が存在して, $H = \sum_{j=1}^{m} \beta_j \Psi \circ A_j^1$について
  \[
  \sup_{x \in \R} |F(x) - H(x)| < \delta
  \]
  が成り立つ. したがって, 
  \[
  \sup_{x \in \R^r} \left| \prod_{k=1}^l F(A_k(x)) - \prod_{k=1}^l H(A_k(x)) \right| \leq \varepsilon
  \]
  となる. いま, $A_j^1 \circ A_k \in \A^r$であり, $\Sigma \Pi(\Psi)$は積と和に関して閉じているので, $\prod_{k=1}^l H \circ A_k \in \Sigma \Pi(\Psi)$である. これで証明は完了した. 
\end{proof}
次に本題の定理を示していく. 
\begin{thm}[Hornik 1989 \cite{Hornik}]\label{MainTheorem4}　\\
  $\Psi:\R \rightarrow \R$をsquashing関数とする. このとき, $\Sigma^r(\Psi)$は$C(\R^r)$において広義一様収束の意味で稠密である.  
\end{thm}
この定理を証明するために$2$つ補題を示す. 
\begin{lem}
  $\Psi:\R \rightarrow \R$をsquashing関数とする. このとき, 任意の$\varepsilon,M > 0$に対して, $F \in \Sigma^1(\Psi)$が存在して, 
  \[
  \sup_{x \in [-M,M]} |F(x) - \cos(x)| < \varepsilon
  \]
  が成立する. 便宜的にこのような$F$を任意に一つ取り$\cos_{M,\varepsilon}$で表すことにする. 
\end{lem}
\begin{proof}
  $n$を$M \leq (2n-1/2)\pi$となるように取る. このとき, $[-M,M] \subset [-(2n+1/2)\pi, (2n-1/2)\pi]$である. いま, 連続なsquashing関数$g,h:\R \rightarrow \R$を
  \[
  g(x) 
  = \left\{
  \begin{aligned}
  ~~~ 0 ~~~~~~~&(x \leq -\pi/2) \\
  \cos(x) ~~~&(-\pi/2 \leq x \leq 0) \\
  1 ~~~~~~~&(0 \leq x)
  \end{aligned}
  \right.
  \]
  および, 
  \[
  h(x) 
  = \left\{
  \begin{aligned}
  ~~~ 0 ~~~~~~~&(x \leq -\pi) \\
  1+\cos(x) ~~~&(-\pi \leq x \leq -\pi/2) \\
  1 ~~~~~~~&(-\pi/2 \leq x)
  \end{aligned}
  \right.
  \]
  で定義する. そして$g_{\alpha}(x) = g(x - \alpha\pi) ~~(\alpha \in \R)$と定義する. $h_{\alpha}$も同様に定義する. すると,  $g_{\alpha},h_{\alpha}$は連続なsquashing関数なので補題\ref{squashingFuncUniAppro}より, 任意の$\varepsilon > 0$に対して$G_{\alpha},H_{\alpha} \in \Sigma^1(\Psi)$が存在して, 
  \[
  \begin{aligned}
  &\sup_{x \in \R} |g_{\alpha}(x) - G_{\alpha}(x)| < \varepsilon \\
  &\sup_{x \in \R} |h_{\alpha}(x) - H_{\alpha}(x)| < \varepsilon
  \end{aligned}
  \]
  となる. 一方, $\tilde{h}_{\alpha} := h_{\alpha} - 1$とおくと, 任意の$x \in [-(2n+1/2)\pi, (2n-1/2)\pi]$に対して, 
  \[
  \cos(x) = \sum_{j=0}^{4n-1} \left((-1)^j g_{-2n + j}(x) + (-1)^{j+1} \tilde{h}_{-2n + 1 + j}(x)\right)
  \]
  である. また, $1 \geq \Psi(K) > 1 - \varepsilon$なる$K$を取り,  $\tilde{H}_{\alpha} :=H_{\alpha} - \Psi(K)$とおくと, $\Sigma^1(\Psi)$は和に関して閉じているので$\tilde{H}_{\alpha} \in \Sigma^1(\Psi)$であり, $\sup_{x \in \R} |\tilde{H}_{\alpha}(x) - \tilde{h}_{\alpha}| < 2\varepsilon$となる. よって, 
  \[
  F = \sum_{j=0}^{4n-1} \left((-1)^j G_{-2n + j}(x) + (-1)^{j+1} \tilde{H}_{-2n + 1 + j}(x)\right) ~\in \Sigma^1(\Psi)
  \]
  が求めるものである. 実際, 
 \[
 \begin{aligned}
 \sup_{x \in [-M,M]} |F(x) - \cos(x)|
 &\leq \sum_{j=0}^{4n-1} \sup_{x \in [-M,M]} |G_{-2n + j}(x) - g_{-2n + j}(x)| \\
 &~~~ + \sum_{j=0}^{4n-1} \sup_{x \in [-M,M]} |\tilde{H}_{-2n + 1 + j}(x)-\tilde{h}_{-2n + 1 + j}(x)| \\
 &\leq 4n\varepsilon + 8n\varepsilon = 12n\varepsilon. 
 \end{aligned}
 \]
\end{proof}
\begin{lem}\label{CompactUniApproTheorem}
  $g = \sum_{j=1}^Q \beta_j \cos \circ A_j , ~A_j \in \A^r$とする. また, $\Psi:\R \rightarrow \R$をsquashing関数とする. このとき, 任意の$\varepsilon > 0$と任意のコンパクト集合$C \subset \R^r$に対して, $f \in \Sigma^r(\Psi)$が存在して, $\sup_{x \in C}|f(x)-g(x)| < \varepsilon$となる. 
\end{lem}
\begin{proof}
  $C$がコンパクトで各$A_j$は連続なので$A_j(C) \subset \R$はコンパクトである. したがって, $M > 0$を適当に取れば$A_j(K) \subset [-M,M] ~(j=1,\ldots,Q)$となる. そこで$K = Q\sum_{j=1}^{Q} |\beta_j|$とおくと, 
  \[
  \sup_{x \in C} \left| \sum_{j=1}^Q \beta_j \cos_{M,\varepsilon/K} (A_j(x)) - g(x) \right| < \varepsilon
  \]
  となる. そして, $\cos_{M,\varepsilon/K} \in \Sigma^1(\Psi)$なので, $\cos_{M,\varepsilon/K} \circ A_j \in \Sigma^r(\Psi)$となる. よって, $\sum_{j=1}^Q \beta_j \cos_{M,\varepsilon/K} \circ A_j \in \Sigma^r(\Psi)$であるので, これで補題の成立が確かめられた. 
\end{proof}
\newtheorem*{proof1}{定理\ref{MainTheorem4}の証明.}
\begin{proof1}
  $\cos:\R \rightarrow \R$は定数でない連続関数なので, 定理\ref{MainTheorem1}により, 
  \[
  \left\{ \sum_{j=1}^Q \beta_j \prod_{k=1}^{l_j} \cos \circ A_{j,k} \mid Q,l_j \in \N , \beta_j \in \R, A_{j,k} \in \A^r  \right\}
  \]
  は$C(\R^r)$において広義一様収束の意味で稠密である. ところが, 任意の$a,b \in \R$に対して, 
  \[
  \cos(a)\cos(b) = \frac{\cos(a+b)-\cos(a-b)}{2}
  \]
  であり, $\A^r$は加法とスカラー倍に関して閉じているので, ある$\alpha_j \in \R, A_j \in \A^r$が存在して, 
  \[
  \sum_{j=1}^Q \beta_j \prod_{k=1}^{l_j} \cos \circ A_{j,k} = \sum_j \alpha_j \cos \circ A_j
  \]
  となる. よって, 補題\ref{CompactUniApproTheorem}より, $\Sigma^r(\Psi)$は$C(\R^r)$において広義一様収束の意味で稠密である. \qed
\end{proof1}

\section{$C(\R^r)$における万能近似定理の応用}

\subsection{多出力・多層への拡張}
前章で述べた結果はすべて, スカラー
を出力するネットワークに関するものであったが,  これらはベクトル値を出力するネットワークについても成り立つ. 本節ではこのことを正確に述べよう. 
\begin{defn}
 関数$G:\R \rightarrow \R$に対して, 
 \[
 \begin{aligned}
 &\Sigma^{r,s}(G) := \{ f:\R^r \rightarrow \R^s \mid f = \sum_{j}^{q} \beta_j G \circ A_j ,~ \beta_j \in \R^s, A_j \in \A^r , q \in \N  \} \\
 &\Sigma \Pi^{r,s}(G) := 
 \left\{ f:\R^r \rightarrow \R^s ~\vline~ 
 \begin{aligned}
 &f = \sum_{j}^{q} \beta_j \prod_{k=1}^{l_j}  G \circ A_{j,k} ~, \beta_j \in \R^s \\
 &A_{j,k} \in \A^r , l_j \in \N,  q \in \N 
 \end{aligned}
 \right\}
 \end{aligned}
 \]
 と定義する. 
\end{defn}
\begin{thm}\label{MainTheorem5}
  関数$\Psi:\R \rightarrow \R$は$\Sigma^r(\Psi)$が$C(\R^r,\R)$において広義一様収束の意味で稠密になるものとする. 
  このとき, $\Sigma^{r,s}(\Psi)$は$C(\R^r,\R^s)$において広義一様収束の意味で稠密である. すなわち, 任意に空でないコンパクト集合$K \subset \R^r$を取るとき, 任意の$g \in C(\R^r,\R^s)$と$\varepsilon>0$に対して, $f \in \Sigma^{r,s}(\Psi)$が存在して, 
  \[
  \sum_{j=1}^s \sup_{x \in K} | f_j(x) - g_j(x) | < \varepsilon
  \]
  となる. ここで, $f = (f_1,\ldots,f_s)^{\T}, g = (g_1,\ldots,g_s)^{\T}$である.  
\end{thm}
\begin{proof}
  各$g_i \in C(\R^r,\R)$を$\Sigma^r(\Psi)$の元
  \[
  f_i = \sum_{j=1}^{Q_i} \beta_{i,j} \Psi \circ A_{i,j}
  \]
  で近似する. すると, 
  \[
  f = \sum_{i=1}^{s} \sum_{j=1}^{Q_i} \beta_{i,j} (\Psi \circ A_{i,j}) \mathbf{e}_i
  \]
  が$g$の近似を与える.  ここに$\mathbf{e}_i$は第$i$成分が$1$で他の成分がすべて$0$である$\R^s$の元である. 
\end{proof}

さて, これまで述べてきた結果はすべて中間層が$1$つののみからなるネットワークに関するものであった. そこで次に, 多層かつ多出力のネットワークの万能近似定理について述べよう. ニューラルネットワークの層数を$l \geq 2$で表すことにする. ただし, ここでは入力層は層数に含めないことにする. つまり, これまで扱っていたのは$l = 2$の場合ということである. 

\begin{defn}[$\Sigma_l^{r,s}$ネットワーク]　\\
  関数$\Psi:\R \rightarrow \R$に対して, 次のように集合$I_1^{(q)},\ldots,I_{l-1}^{(q)}, I_l ~(q \in \N)$を帰納的に定義する. 
  \[
  \begin{aligned}
  I_1^{(q)} &= \{ (\Psi \circ A_1, \ldots , \Psi \circ A_q)^{\T} \mid A_1 \ldots, A_q \in \A^r \},  \\
  I_{k+1}^{(q)} &= \{ (\Psi \circ A_1, \ldots, \Psi \circ A_q )^{\T} \circ f \mid q' \in \N , f \in I_{k}^{(q')}, A_j \in \A^{(q')} \},  \\
  I_l &= \{  (A_1 , \ldots, A_s )^{\T} \circ f \mid q' \in \N , f \in I_{l-1}^{(q')}, A \in \A^{(q')} \}. 
  \end{aligned}
  \]
  このとき, $I_l$を$\Psi$を活性化関数とする$\Sigma_l^{r,s}$ネットワークと呼び, $\Sigma_l^{r,s}(\Psi)$で表す. 
\end{defn}
より具体的に書けば$\Sigma_l^{r,s}(\Psi)$の元は次のような$\R^r$から$\R^s$への関数である: 
$x \in \R^r$に対して, $a_0 = x$とおく. 
そして, $k = 1,\ldots,l$に対して, $q_0 = r, q_l = s, q_j \in \N$とし, $A_{k,i} \in \A^{q_{k-1}} ~~(i=1,\ldots,q_k)$とする. 
また, $G_1,\ldots,G_{l-1} = \Psi$とし, $G_l = \mathrm{id}$とする. 
このとき, 次の式
  \[
  a_{k,i} = G_k (A_{k,i}(a_{k-1})) ~~(i = 1,\ldots,q_k, k = 1,\ldots,l)
  \]
にしたがって計算を繰り返した結果$a_l=(a_{l,1},\ldots,a_{l,s})^{\T} \in \R^s$を出力する. 

この定義から$\Sigma^{r,s}(\Psi) \subset \Sigma_2^{r,s}(\Psi)$がわかる. 実際, $f = \sum_{j=1}^Q \beta_j \Psi \circ A_j ~~(\beta_j \in \R^s, A_j \in A^r)$とするとき, 各$k = 1,\ldots,Q$に対して$B_k \in \A^Q$を$B_k(x) = \sum_{j=1}^Q \beta_j^{(k)}x_j$と定義すれば, $f = (B_1,\ldots,B_Q)^{\T} \circ (\Psi \circ A_1,\ldots, \Psi \circ A_Q)^{\T}$となる. このことから, $\Sigma^{r,s}(\Psi)$が万能近似能力を持つとき, $\Sigma_2^{r,s}(\Psi)$も万能近似能力を持つことがわかる. より一般に次が成り立つ. 
\begin{thm}\label{MainTheorem6}
  関数$\Psi:\R \rightarrow \R$は$\Sigma^r(\Psi)$が$C(\R^r,\R)$において広義一様収束の意味で稠密になるものとする. このとき, $\Sigma_l^{r,s}(\Psi)$は$C(\R^r,\R^s)$において広義一様収束の意味で稠密である. 
\end{thm}
この定理を示すためにひとつ補題を証明しよう. 
\begin{lem}
  $\mathcal{F}$は「$\R$から$\R$への関数全体の集合」の部分集合で, $C(\R)$において広義一様収束の意味で稠密であるとし,  $\mathcal{G}$は「$\R^r$から$\R$への関数全体の集合」の部分集合で, $C(\R^r)$において広義一様収束の意味で稠密であるとする. このとき, 
  \[
  F \circ G = \{f \circ g \mid f \in \mathcal{F}, g \in \mathcal{G}\}
  \]
  は$C(\R^r)$において広義一様収束の意味で稠密である. 
\end{lem}
\begin{proof}
  空でないコンパクト集合$K \subset \R^r$を任意に取る. $h \in C(\R^r)$とし, $\varepsilon > 0$を任意に取ると, 仮定より$g \in \mathcal{G}$が存在して, $\sup_{x \in K} |g(x) - h(x)| < \varepsilon$となる. このとき, $h(K) \subset \R$はコンパクト, 特に有界であるので, $g(K)$も有界である. そこで, $g(K)$の閉包$S := \mathrm{cl}(g(K))$はコンパクトである. すると, 仮定より連続関数$\R \ni s \mapsto s \in \R$に対して, $f \in C(\R)$が存在して$\sup_{s \in S} |f(s) - s| < \varepsilon$となる. ゆえに, 任意の$x \in K$に対して, 
  \[
  \begin{aligned}
  |f(g(x)) - h(x)| 
  &\leq |f(g(x)) - g(x)| + |g(x) - h(x)| \\
  &\leq \sup_{s \in S} |f(s)-s| + \sup_{z \in K} |g(z)-h(z)| \\
  &< \varepsilon + \varepsilon = 2\varepsilon
  \end{aligned}
  \]
  である. よって, $\sup_{x \in K} |f(g(x)) - h(x)| \leq 2\varepsilon$となる. 
\end{proof}
\newtheorem*{mainresult6}{定理\ref{MainTheorem6}の証明.}
\begin{mainresult6}
  $s=1$の場合を証明すれば十分である. 
  次のように集合$J_1,\ldots,J_l$を帰納的に定義する. 
  \[
  \begin{aligned}
  J_1 &= \{ \sum_{j}^Q \beta_j \Psi \circ A_j \mid Q \in \N, A_j \in \A^r , \beta_j \in \R \},  \\
  J_{k+1} &= \{ \sum_{j}^Q \beta_j \Psi \circ A_j \circ f \mid Q \in \N , f \in J_{k}, A_j \in \A^1 , \beta_j \in \R \},  \\
  \end{aligned}
  \]
  このとき, $x \mapsto \sum_{j}\beta_j x_j$は線形写像なので$J_l \subset \Sigma_{l+1}^{r,s}(\Psi)$である. 
  そこで, $J_l$が$C(\R^r)$において広義一様収束の意味で稠密であることを示せばよい. このことを$l$に関する帰納法で示そう. まず$\Sigma^r(\Psi) = J_1$であるので$J_1$は$C(\R^r)$において広義一様収束の意味で稠密である. 次に, $J_k $が$C(\R^r)$において広義一様収束の意味で稠密であるとすると, $J_{k+1} = \Sigma^1(\Psi) \circ J_k$であり, $\Sigma^1(\Psi)$は$C(\R)$において広義一様収束の意味で稠密であるので, 前補題より$J_{k+1}$は$C(\R^r)$において広義一様収束の意味で稠密である. \qed
\end{mainresult6}

\subsection{$L^p$空間における万能近似定理}
本節では, 前章の結果を使って$\R^r$上のBorel可測関数全体の空間$\M^r := \M(\R^r,\R)$や$L^p$空間における$\Sigma^r(\Psi)$の万能近似定理を述べる. まず, $C(\R^r)$が$\M^r$において稠密であることについて述べる. 
\begin{defn}
  $\mu$を$(\R^r,\B^r)$上の有限測度とする. このとき, $f,g \in \M^r$が$\mu$-同値であるとは, $\mu(\{f \neq g\}) = 0$となることをいう. 以後, この同値関係による$\M^r$の商集合をそのまま$\M^r$で表す場合がある. 
\end{defn}
\begin{defn}
  $\mu$を$(\R^r,\B^r)$上の有限測度とする. このとき, $\rho_{\mu}:\M^r \times \M^r \rightarrow [0,\infty)$を, 
  \[
  \rho_{\mu}(f,g) := \inf \{ \varepsilon > 0 \mid \mu({|f-g|>\varepsilon}) < \varepsilon \}
  \]
  により定義する(有限測度なので$\rho_{\mu}(f,g) \leq \mu(\R^r) < \infty$となることに注意). 
\end{defn}
\begin{lem}\label{ProbabilityConvergenceEq}
  $\mu$を$(\R^r,\B^r)$上の有限測度とする. また, $(f_n)_{n=1}^{\infty}$を$\M^r$の元の列とし, $f \in \M^r$とする. このとき
  , 次の条件は同値である. 
  \begin{itemize}
      \item[(1):]$\rho_{\mu}(f_n,f) \rightarrow 0.$ 
      \item[(2):]$\forall \varepsilon > 0 , ~\mu(\{|f_n - f| > \varepsilon\}) \rightarrow 0.$ 
      \item[(3):]$\int_{\R^r} \min(|f_n-f|, 1) d\mu \rightarrow 0.$
  \end{itemize}
\end{lem}
\begin{proof}　\\
  $(1) \Rightarrow (2)$: 任意に$\varepsilon > 0$を取る. 
  いま, 仮定よりある$N \in \N$が存在して任意の$n \geq N$に対して$\rho_{\mu}(f_n,f) < \varepsilon$となる. 
  すると, 任意の$n \geq N$に対して, $\rho_{\mu}$の定義より$\mu(|f_n - f| > \varepsilon_n) < \varepsilon_n$なる$0 < \varepsilon_n < \varepsilon$が取れる. 
  ゆえに, 任意の$n \geq N$に対して, $\mu(|f_n - f| > \varepsilon) \leq \mu(|f_n - f| > \varepsilon_n) < \varepsilon_n < \varepsilon$となる.
  したがって, 任意に$\delta > 0$を取ると, ある$M \in \N$が存在して任意の$n \geq M$に対して$\mu(|f_n - f| > \delta) < \delta$となる. そこで, $\varepsilon \leq \delta$の場合は, 任意の$n \geq N$に対して, $\mu(|f_n - f| > \varepsilon) < \varepsilon \leq \delta$となる. $\varepsilon > \delta$の場合は, 任意の$n \geq M$に対して$\mu(|f_n - f| > \varepsilon) \leq \mu(|f_n - f| > \delta) < \delta$となる. これは$\mu(\{|f_n - f| > \varepsilon\}) \rightarrow 0$を意味する. \\
  $(2) \Rightarrow (1)$: 任意に$\varepsilon > 0$を取る. いま, 仮定よりある$N \in \N$が存在して任意の$n \geq N$に対して$\mu(|f_n - f| > \varepsilon) < \varepsilon$となる. ゆえに, 任意の$n \geq N$に対して$\rho_{\mu}(f_n,f) \leq \varepsilon$となる. \\
  $(2) \Rightarrow (3)$: $\varepsilon > 0$を任意に取る. このとき, 仮定よりある$N \in \N$が存在して任意の$n \geq N$に対して$\mu(|f_n - f| > \varepsilon) < \varepsilon$となる. そこで, 任意の$n \geq N$に対して, 
  \[
  \begin{aligned}
  \int_{\R^r} \min(|f_n - f|,1) d\mu
  &\leq \int_{|f_n-f|>\varepsilon} 1 d\mu + \int_{|f_n-f|\leq \varepsilon} |f_n - f| d\mu \\
  &\leq \varepsilon + \varepsilon \mu(\R^r)
  \end{aligned}
  \]
  となる. $\mu$は有限測度なのでこれで$(3)$の成立が確かめられた. \\
  $(3) \Rightarrow (2)$: 任意に$1 > \varepsilon > 0$を取る. このとき, 任意の$\delta > 0$に対して, ある$N \in \N$が存在して任意の$n \geq N$に対して, $\int_{\R^r} \min(|f_n-f|,1)d\mu < \varepsilon \delta$となる. そこで, 任意の$n \geq N$に対して, 
  \[
  \begin{aligned}
  \mu(|f_n - f| > \varepsilon)
  &\leq \frac{1}{\varepsilon} \int_{|f_n - f| > \varepsilon} \min(|f_n-f|,1) d\mu \\
  &\leq \frac{1}{\varepsilon} \int_{\R^r} \min(|f_n-f|,1) d\mu \\
  &\leq \delta
  \end{aligned}
  \]
  となる. よって, $\mu(\{|f_n - f| > \varepsilon\}) \rightarrow 0$である. $\varepsilon \geq 1$の場合も$0 < \varepsilon' < 1$に対して, 
  \[
  \mu(|f_n - f| > \varepsilon) \leq \mu(|f_n - f| > \varepsilon') \rightarrow 0.
  \]
  となるので成立する. 
\end{proof}
\begin{lem}\label{DensityOfCinM}
  $\mu$を$(\R^r,\B^r)$上の有限測度とするとき, $C(\R^r)$は$\M^r$において$\rho_{\mu}$-稠密である. 
\end{lem}
\begin{proof}
  $f \in \M^r$とし, $\varepsilon > 0$を任意に取る. 優収束定理により, 十分大きな$M \geq 0$に対して, 
  \[
  \int_{\R^r} \min(|f\cdot\chi_{\{|f| \geq M \}} - f|, 1)d\mu < \varepsilon
  \]
  が成立する. ここで, $f\cdot\chi_{\{|f| \geq M \}} \in L^1(\R^r,\B^r,\mu)$であるから, $C_0(\R^r)$が$L^1(\R^r,\B^r,\mu)$において$L^1$ノルムの意味で稠密であること
  (付録の補題\ref{DensityOfC0InLp})より, ある$g \in C(\R^r)$で$\int_{\R^r} |f\cdot\chi_{\{|f| \geq M \}} - g| d\mu < \varepsilon$となる. よって, 
  \[
  \begin{aligned}
  \int_{\R^r} \min(|f-g|,1) d\mu
  &\leq \int_{R^r} \min(|f - f\cdot\chi_{\{|f| \geq M \}}|,1) d\mu \\
  &~~~ + \int_{R^r} \min(|f\cdot\chi_{\{|f| \geq M \}} - g|,1) d\mu \\
  &\leq \varepsilon + \int_{R^r} |f\cdot\chi_{\{|f| \geq M \}} - g| d\mu \\
  &\leq \varepsilon + \varepsilon = 2\varepsilon
  \end{aligned}
  \]
  となる. これより任意の$\varepsilon > 0$に対して$g \in C(\R^r)$が存在して$\rho_{\mu}(f,g) < \varepsilon$となることがわかる(補題\ref{ProbabilityConvergenceEq}の$(3) \Rightarrow (2)$の証明参照).
\end{proof}
次に$\Sigma^r(\Psi)$が$\M^r$において稠密であることの証明に使う補題を示す. 
\begin{lem}\label{CompactUniConverThenProbConver}
  $(f_n)_{n=1}^{\infty}$を$\M^r$の元の列とし, $f \in \M^r$とする. このとき, $(f_n)_{n=1}^{\infty}$が$f$に広義一様収束するならば, 任意の$(\R^r,\B^r)$上の有限測度$\mu$に対して, $\rho_{\mu}(f_n,f) \rightarrow 0$となる. 
\end{lem}
\begin{proof}
  任意に$\varepsilon > 0$を取る. 
  $\mu$は有限測度なので付録の系\ref{FiniteBorelMeasOnRnIsRegular}より正則である.
  したがって, あるコンパクト集合$K \in \R^r$が存在して$\mu(\R^r \setminus K) < \varepsilon$となる. 
  そこで仮定より, ある$N \in \N$が存在して, 任意の$n \geq N$に対して, 
  \[
  \sup_{x \in K} |f_n(x) - f(x)| < \varepsilon
  \]
  となる. ゆえに, 任意の$n \geq N$に対して, 
  \[
  \begin{aligned}
  \int_{\R^r} \min(|f_n-f|,1)d\mu 
  &= \int_{K} \min(|f_n-f|,1)d\mu + \int_{\R^r \setminus K} \min(|f_n-f|,1)d\mu \\
  &\leq \int_{K} |f_n-f| d\mu + \int_{\R^r \setminus K} 1 d\mu \\
  &\leq \varepsilon \mu(\R^r) + \varepsilon \\
  \end{aligned}
  \]
  となる. よって, $\int_{\R^r} \min(|f_n-f|,1)d\mu  \rightarrow 0$なので, 補題\ref{ProbabilityConvergenceEq}より$\rho_{\mu}(f_n,f) \rightarrow 0$である. 
\end{proof}

\begin{thm}\label{UATinMr}
  関数$\Psi:\R \rightarrow \R$は$\Sigma^r(\Psi)$が$C(\R^r)$において広義一様収束の意味で稠密になるものとする. このとき, 任意の$(\R^r,\B^r)$上の有限測度$\mu$に対して, $\Sigma^r(\Psi)$は$\M^r$において$\rho_{\mu}$-稠密である.
\end{thm}
\begin{proof}
  $\Sigma^r(\Psi)$が$C(\R^r)$において広義一様収束の意味で稠密であることと補題\ref{CompactUniConverThenProbConver}より$\Sigma^r(\Psi)$は$C(\R^r)$において$\rho_{\mu}$-稠密である. このことと 補題\ref{DensityOfCinM}を合わせれば$\Sigma^r(\Psi)$は$\M^r$において$\rho_{\mu}$-稠密であることがわかる. 
\end{proof}
次の定理は$\M^r$をsupノルムの意味で$\Sigma^r(\Psi)$により十分よく近似できることを示すものである. 
\begin{thm}\label{DensityOfNNinMr}
  関数$\Psi:\R \rightarrow \R$は$\Sigma^r(\Psi)$が$C(\R^r)$において広義一様収束の意味で稠密になるものとする.  また
  $\mu$を$(\R^r,\B^r)$上の有限測度とする. このとき, 任意の$g \in \M^r$と$\varepsilon > 0$に対して, $f \in \Sigma^r(\Psi)$とコンパクト集合$K \subset \R^r$が存在して, $\mu(K) > \mu(\R^r) - \varepsilon$かつ$\forall x \in K , |f(x) - g(x)| < \varepsilon$となる. 
\end{thm}
証明には次の事実を使う(宮島\cite{miyajima}の第7章2節などを参照されたい). 
\begin{thm}[Tietzeの拡張定理]\label{TietzeTheorem}　\\
  $X$はHausdorff空間で, 任意の交わらない閉集合$F_1,F_2 \subset X$に対して$F_1 \subset U_1, F_2 \subset U_2$かつ$U_1 \cap U_2 = \emptyset$となる開集合$U_1,U_2 \subset X$が存在するとする. このとき, 閉集合$A \subset X$上の任意の実数値連続関数$f$に対して$X$上の実数値連続関数$f'$が存在して$f = f'|_A$が成立する. つまり, $f$を$X$上の実数値連続関数に拡張できる. 
\end{thm}
\renewcommand{\proofname}{\bf 定理\ref{DensityOfNNinMr}の証明.}
\begin{proof}
  Lusinの定理(付録の命題\ref{LusinTheorem})より, コンパクト集合$K \subset \R^r$が存在して, $\mu(K) > \mu(\R^r) - \varepsilon$かつ$K$上の関数$g|K$は連続となる.
  そこでTietzeの拡張定理より, $g' \in C(\R^r)$が存在して$g|K = g'|K$となる. いま $\Sigma^r(\Psi)$は$C(\R^r)$において広義一様収束の意味で稠密であるので, $f \in \Sigma^r(\Psi)$が存在して, 
  \[
  \sup_{x \in K} |f(x) - g(x)| = \sup_{x \in K} |f(x) - g'(x)| < \varepsilon. 
  \]
\end{proof}
\renewcommand{\proofname}{\bf 証明.}
さて, $L^p$空間における万能近似定理に話を移そう. 
\begin{defn}
  $f,g \in L^p(\R,\B^r,\mu)$に対して, 
  \[
  d_{\mu}^p(f,g) = \left( \int_{\R^r} |f-g|^p d\mu \right)^{1/p}
  \]
  と定義する. $d_{\mu}^p$は$L^p(\R^r,\B^r,\mu)$上の距離である. 
\end{defn}
\begin{thm}
  関数$\Psi:\R \rightarrow \R$は$\Sigma^r(\Psi)$が$C(\R^r)$において広義一様収束の意味で稠密になるものとする. また$\mu$を$(\R^r,\B^r)$上の有限測度とする. 
  このとき, $\mu(K)=\mu(\R^r)$なるコンパクト集合$K \subset \R^r$が存在すれば, 任意の$p \in [1,\infty)$について,  $\Sigma^r(\Psi)$は$L^p(\R^r,\B^r,\mu)$において$d_{\mu}^p$-稠密である. 
\end{thm}
\begin{proof}
  $g \in L^p(\R^r,\B^r,\mu)$と$\varepsilon > 0$を任意に取る. 
  付録の補題\ref{DensityOfC0InLp}より $C_0(\R^r)$は\\
  $L^p(\R^r,\B^r,\mu)$において$d_{\mu}^p$-稠密であるので, $h \in C(\R^r)$で$d_{\mu}^p(g,h) < \varepsilon$なるものが取れる. いま, $\Sigma^r(\Psi)$は$C(\R^r)$において広義一様収束の意味で稠密であるので, $f \in \Sigma^r(\Psi)$が存在して$\sup_{x \in K} |h(x) - f(x)| < \varepsilon$となる. よって, $\mu(K) = \mu(\R^r)$であることに注意すると, 
  \[
  \begin{aligned}
  d_{\mu}^p(g,f) \leq d_{\mu}^p(g,h) + d_{\mu}^p(h,f) 
  \leq \varepsilon + (\int_{K} \varepsilon^p d\mu)^{1/p} \leq \varepsilon + \varepsilon (\mu(\R^r))^{1/p}
  \end{aligned}
  \]
  となる. 
\end{proof}
\begin{cor}
  関数$\Psi:\R \rightarrow \R$は$\Sigma^r(\Psi)$が$C(\R^r)$において広義一様収束の意味で稠密になるものとする. また, $K \subset \R^r$をコンパクト集合とし$\mu$を$(K,\B_{K})$上の有限測度とする. このとき, 任意の$p \in [1,\infty)$について,  $\Sigma^r(\Psi)$(を$K$に制限したもの)は$L^p(K,\B_{K},\mu)$において$d_{\mu}^p$-稠密である. ただし, $\B_{K}$は$K$上のBorel $\sigma$-加法族であり, $\B_{K} = \{K \cap A \mid A \in \B^r\}$を満たす. 
\end{cor}
\begin{proof}
  $(\R^r,\B^r)$上の有限測度$\tilde{\mu}$を$\tilde{\mu}(A) = \mu(A \cap K)$で定め, $f \in L^p(K,\B_{K},\mu)$に対し$\tilde{f} \in L^o(\R^r,\B^r,\tilde{\mu})$を
  \[
  \tilde{f}(x) = 
  \begin{cases}
  f(x) ~~~(x \in K)\\
  0 ~~~~~~~~(\mathrm{otherwise})
  \end{cases}
  \]
  で定める. このとき, すぐ上の定理より, 任意の$\varepsilon > 0$に対して. $g \in \Sigma^r(\Psi)$が存在して, 
  \[
  \begin{aligned}
  \varepsilon > d_{\tilde{\mu}}^p(\tilde{f},g) = \left( \int_{\R^r} |\tilde{f}-g|^p d_{\tilde{\mu}} \right)^{1/p} = \left( \int_{K} |f-g|^p d\mu  \right)^{1/p} = d_{\mu}^p(f,g|_{K})
  \end{aligned}
  \]
  となる. 
\end{proof}

以下はLeshno et al.の結果の系である. 
\begin{cor}
  $\mu$を$(\R^r,\mathcal{B}_{\R^r})$上の有限測度で, $\mu(K)=\mu(\R^r)$なるコンパクト集合$K \subset \R^r$が存在するとする. このとき, $\Psi \in \mathcal{M}(\R)$に対して,  $\Psi$がどんな多項式関数ともa.e.で一致することがないことと $\Sigma^r(\Psi)$が$L^p(\R^r,\mathcal{B}_{\R^r},\mu) ~(p \geq 1)$で稠密であることは同値である. 
\end{cor}
\begin{proof}
  $\Psi$がある$m$次多項式関数と$a.e.$で一致するとすると, $\Sigma^r(\Psi)$は$r$変数の$m$次以下の多項式関数とa.e.で一致する関数全体の集合になるので$L^p(\R^r,\mathcal{B}_{\R^r},\mu)$において$m+1$次の多項式関数を近似することはできない($\mu(K)=\mu(\R^r)$ゆえ多項式関数は$L^p(\mu)$に属することに注意). 逆に$\Psi$がどんな多項式関数ともa.e.で一致することがないとし, $f \in L^p(\R^r,\mathcal{B}_{\R^r},\mu)$をとる. いま,  付録の命題\ref{DensityOfC0InLp}より$C(K)$は$L^p(K,\mu)$で稠密であるので, $g \in C(K)$が存在して, 
  \[
  \lVert f-g \rVert_{L^p}^p = \int_{\R^r} |f-g|^p d\mu = \int_{K} |f-g| d\mu < (\varepsilon/2)^p
  \]
  となる. 定理\ref{LeshnoMainResult1}より$\Sigma^r(\Psi)$は$C(K)$において稠密であるので, $h \in \Sigma^r(\Psi))$が存在して, 
  \[
  \sup_{x \in K} |g(x)-h(x)| < \varepsilon/{\mu(K)^{1/p}}
  \]
  となる. よって, 
  \[
  \lVert f - h \rVert_{L^p} \leq \lVert f - g \rVert_{L^p} + \lVert g - h \rVert_{L^p} \leq \varepsilon.
  \]
\end{proof}
以上の結果は多層・多出力の場合も成立する. 証明は$C(\R^r,\R^s)$の場合と同様にできる. 
\begin{defn}
  次のように定める. 
  \[
  \begin{aligned}
  &M^{r,s} = M(\R^r,\R^s) := \{f:\R^r \rightarrow \R^s \mid f\mbox{は}(\B^r,\B^s)\mbox{可測}\} \\
  &L^p(r,s,\mu) := \{ f = (f_1,\ldots,f_s) \in M^{r,s} \mid \int_{\R^r} |f_j|^p d\mu < \infty ~~(j=1,\ldots,s) \}
  \end{aligned}
  \]
\end{defn}
\begin{thm}
  関数$\Psi:\R \rightarrow \R$は$\Sigma^r(\Psi)$が$C(\R^r,\R)$において広義一様収束の意味で稠密になるものとする. このとき, $(\R^r,\B^r)$上の有限測度$\mu$に対して, $\Sigma^{r,s}(\Psi)$は$\M^{r,s} := \M(\R^r,\R^s)$において$\rho_{\mu}^s$-稠密である. 
  ただし, $f,g \in \M^{r,s}$に対して, 
  \[
  \rho_{\mu}^s(f,g) = \sum_{j=1}^s \inf \left\{ \varepsilon>0 ~\vline~ \mu(|f_j-g_j| > \varepsilon ) < \varepsilon \right\}
  \]
  である. さらに, $\mu(K) = \mu(\R^r)$なるコンパクト集合$K \subset \R^r$が存在すれば, 任意の$p \geq 1$に対して, $\Sigma^{r,s}(\Psi)$は$L^p(r,s,\mu)$において$d_{\mu}^{p,s}$-稠密である. ただし, $f,g \in L^p(r,s,\mu)$に対して, 
  \[
  d_{\mu}^{p,s}(f,g) = \sum_{j=1}^s \left( \int_{\R^r} |f_j-g_j|^p d\mu \right)^{1/p}
  \]
  である. 
\end{thm}
\begin{thm}
  関数$\Psi:\R \rightarrow \R$は$\Sigma^r(\Psi)$が$C(\R^r,\R)$において広義一様収束の意味で稠密になるものとする. このとき, 任意の$(\R^r,\B^r)$上の有限測度$\mu$に対して, $\Sigma_l^{r,s}(\Psi)$は$\M(\R^r,\R^s)$において$\rho_{\mu}^s$-稠密である.
  さらに, $\mu(K) = \mu(\R^r)$なるコンパクト集合$K \subset \R^r$が存在すれば, 任意の$p \geq 1$に対して, $\Sigma_l^{r,s}(\Psi)$は$L^p(r,s,\mu)$
  において$d_{\mu}^{p,s}$-稠密である. 
\end{thm}

\subsection{ニューラルネットワークによる補間}
本節ではfeedforward型ニューラルネットワークによる補間について述べる. 
つまり, 与えられた有限個の点を(近似的または真に)通るようなfeedforward型ニューラルネットワークの存在について述べる. 
\begin{thm}
  関数$\Psi:\R \rightarrow \R$は$\Sigma^r(\Psi)$が$C(\R^r,\R)$において広義一様収束の意味で稠密になるものとする.  また$(\R^r,\B^r)$上の有限測度$\mu$は有限個の点の上でのみ値を取るものとする. つまり, 相異なる$x_1,\ldots,x_n \in \R^r$が存在して, 
  \[
  \sum_{j=1}^n \mu(\{x_j\}) = \mu(\R^r),~~ \mu(\{x_j\}) > 0 ~~(j=1,\ldots,n)
  \]
  となるとする. このとき, 任意の$g \in \M^r$と$\varepsilon > 0$に対して, $f \in \Sigma^r(\Psi)$が存在して, $\mu(|f-g|>\varepsilon) = 0$となる. 
\end{thm}
\begin{proof}
  $\tilde{\varepsilon} := \min_{1\leq j \leq n} \mu(\{x_j\})$とおく. $\varepsilon < \tilde{\varepsilon}$としてよい. 定理\ref{UATinMr}より, $\Sigma^r(\Psi)$は$\M^r$において$\rho_{\mu}$-稠密であるので, $f \in \Sigma^r(\Psi)$が存在して, 
  \[
  \rho_{\mu}(f,g) = \inf\{\varepsilon' > 0 \mid \mu(|f-g|> \varepsilon') < \varepsilon'\} < \varepsilon
  \]
  となる. このとき, ある$\varepsilon' > 0$で, $\mu(|f-g|> \varepsilon') < \varepsilon'$かつ$\varepsilon' < \varepsilon$なるものが取れる. すると, 
  \[
  \mu(|f-g|> \varepsilon) \leq \mu(|f-g|> \varepsilon') < \varepsilon' < \varepsilon < \tilde{\varepsilon}
  \]
  となるので, $\mu(|f-g|> \varepsilon) = 0$でなければならない. 
\end{proof}
\begin{cor}
  関数$\Psi:\R \rightarrow \R$は$\Sigma^r(\Psi)$が$C(\R^r,\R)$において広義一様収束の意味で稠密になるものとする. このとき, 任意の関数$g:\{0,1\}^r \rightarrow \{0,1\}$と$\varepsilon > 0$に対して, $f \in \Sigma^r(\Psi)$が存在して, 
  \[
  \sup_{x \in \{0,1\}^r} |f(x) - g(x)| < \varepsilon
  \]
  となる. 
\end{cor}
\begin{proof}
  $\mu$を$\{0,1\}^r$の各点で$1/2^r$を値に取る$(\R^r,\B^r)$上の測度とする. $g$(正確には$g$を$\R^r$に拡張したもの)は$\M^r$に属するので, すぐ上の定理より$f \in \Sigma^r(\Psi)$が存在して, $\mu(|f-g| \geq \varepsilon) = 0$となる. このとき, もし$x \in \{0,1\}^r$が存在して, $|f(x)-g(x)| \geq \varepsilon$となるなら, $\mu((|f-g| \geq \varepsilon) \geq \mu(\{x\}) = 1/2^r > 0$となり不合理である. よって, $\sup_{x \in \{0,1\}^r} |f(x) - g(x)| < \varepsilon$である. 
\end{proof}
さらに, $\Psi$がある条件を満たせば, 任意の関数$g:\R^r \rightarrow \R$と任意に与えられた有限個の点に対して, その点の上で$g$と値が一致するような$f \in \Sigma^r(\Psi)$が存在する. このことを示すために次の補題を示す. 
\begin{lem}
  $n \in \N$とし, $K$を$|K| \geq n+1$なる体とする. このとき, 任意の$K$上の線形空間$V$は, $k \leq n$個の真線形部分空間の和集合で表せない. つまり, どんな線形部分空間$V_1 ,\ldots,V_k \subsetneq V$についても$V = \bigcup_{j=1}^k V_j$とはならない. 
  特に, $\R$上の線形空間$V$は真線形部分空間の有限和集合で表せない. 
\end{lem}
\begin{proof}
  $n$に関する帰納法による. $n = 1$のときは明らかに成立する(線形空間と真線形部分空間は一致しない). $n$で成立すると仮定する. 
  このとき, もしある$|K| \geq n+2$なる体$K$と$K$上の線形空間$V$および$k \leq n+1$について, 線形部分空間$V_1 ,\ldots,V_k \subsetneq V$が存在して$V = \bigcup_{j=1}^k V_j$となったとする. 
  ここですべての$V_j$は$\{0\}$でないとしてよい. 
  いま, $x \in V_k \setminus \{0\}$をとり, $y \in V \setminus V_k$をとると, 任意の$\alpha \in K \setminus \{0\}$に対して$x + \alpha y \in V \setminus V_k = \bigcup_{j=1}^{k-1} V_j$である. したがって, $|K| \geq n+2 $に注意して相異なる$\alpha_1,\ldots,\alpha_{k} \in K \setminus \{0\}$を取ると, 鳩の巣原理により, ある$l \neq k$と$i,j$について$x + \alpha_{i}y , x + \alpha_{j}y \in V_l$となる. すると, $V_l$は線形部分空間であるので, 
  \[
  (\alpha_i - \alpha_j)y = (x + \alpha_i y) - (x + \alpha_j y) \in V_l
  \]
  となり, $\alpha_i - \alpha_j \neq 0$ゆえ, この逆元をかけることで$y \in V_l$がわかる. よって, $x = (x + \alpha_i y)-\alpha_i y \in V_l$である. 以上から, $V = \bigcup_{j=1}^{k-1} V_j$となる. これは帰納法の仮定に反する. ゆえに$n+1$のときも成立しなければならない. 
\end{proof}
\begin{thm}
  $\Psi:\R \rightarrow \R$をsquashing関数とする. もし$\Psi(z_0) = 0, \Psi(z_1) = 1$なる$z_0,z_1 \in \R^r$が存在するならば, 相異なる任意の点$x_1,\ldots,x_n \in \R^r$と任意の関数$g:\R^r \rightarrow \R$に対して, $f \in \Sigma^r(\Psi)$が存在して$f(x_j)=g(x_j) ~(j=1,\ldots,n)$となる. 
\end{thm}
\begin{proof}
  二段階に分けて証明する. Step1で$r=1$の場合を証明し, Step2で$r \geq 2$の場合を証明する. \\
  {\bf Step1}: 添字を付け替えることで$x_1 < \cdots < x_n$としてよい. いま, $\Psi$は単調非減少なので$M > 0$を適当に取れば任意の$x \in \R$に対して, $x \leq -M$ならば$\Psi(x)=0$となり, $x \geq M$ならば$\Psi(x)=1$となる. そこで, $A_1(x) = M ~(x \in \R)$とし, $\beta_1 = g(x_1)$とおく. また, $A_j \in \A^1 ~~(j=2,\ldots,n)$を$A_j(x_{j-1}) = -M$かつ$A_j(x_j) = M$となるものとし, $\beta_j = g(x_j) - g(x_{j-1})$とする. すると, 
  \[
  f = \sum_{j=1}^{n} \beta_j \Psi \circ A_j ~\in \Sigma^1(\Psi)
  \]
  が求めるものである. \\
  {\bf Step2}: $r \geq 2$とする. このとき, ある$p \in \R^r$が存在して, $i \neq j$ならば$p^{\T}(x_i-x_j) \neq 0$となる. 実際, $\bigcup_{i \neq j} \{ q \in \R^r \mid q^{\T}(x_i - x_j) = 0 \}$は$\R^r$の超平面の有限和であるのですぐ上の補題より$\R^r$を覆い尽くすことはできない. ここで,  添字を付け替えることで$p^{\T}x_1 < \cdots < p^{\T} x_n$としてよい. すると, Step1と同等にして, $\beta_j \in \R, ~A_j \in \A^1$が存在して, 
  \[
  \sum_{j=1}^n \beta_j \Psi(A_j(p^{\T}x_k)) = g(x_k) ~~~(k=1,\ldots,n)
  \]
  となる. $A_j'(x) = A_j(p^{\T}x) ~(x \in \R^r)$とおくと$A_j' \in \A^r$であるので, 
  \[
  f = \sum_{j=1}^n \beta_j \Psi \circ A_j'
  \]
  が求めるものである. 
\end{proof}
次の定理はPinkus\cite{Pinkus}による. 
\begin{thm}
  $\Psi \in C(\R)$はある開区間上で多項式でないとする. このとき, 任意の$\alpha_1,\ldots,\alpha_k \in \R$と相異なる任意の点$x_1,\ldots,x_n \in \R^r$に対して, $w_1,\ldots,w_k \in \R^r$と$c_1,\ldots,c_k$,$b_1,\ldots,b_k \in \R$が存在して, 
  \[
  \sum_{j=1}^k c_j \Psi(w_j^{\T}x_i + b_j) = \alpha_i ~~~~(i=1,\ldots,k)
  \]
  となる. 
\end{thm}
\begin{proof}
  $K := \{ x_1,\ldots,x_k\} \subset \R^r$とおく. いま, すべてが$0$とはならない$d_1,\ldots,d_k \in \R$が存在して
  \[
  \sum_{i=1}^k d_i \Psi(w^{\T} x_i + b) = 0 ~~~~(\forall w \in \R^r, b \in \R)
  \]
  となると仮定しよう. このとき, 
  \[
  G: C(K) \ni f \mapsto \sum_{i=1}^k d_i f(x_i) \in \R
  \]
  とおくと, ある$d_i$は$0$でないので$G$は$0$でない有界線形汎関数であり, 
  \[
  G(f) = 0 ~~~~(\forall f \in \mathrm{span}\{K \ni x \mapsto \Psi(w^{\T}x+b) \mid w \in \R^r, b \in \R\})
  \]
  となる. よって, $\mathrm{span}\{K \ni x \mapsto \Psi(w^{\T}x+b) \mid w \in \R^r, b \in \R\}$は$C(K)$において稠密でない(稠密なら$G=0$でければならない). これは万能近似定理(定理\ref{LeshnoMainResult1})に反する. 
  よって, $w,b$の連続関数$\Psi(w^{\T} x_1 + b), \ldots, \Psi(w^{\T} x_k + b)$は一次独立である. ゆえに, $w_1,\ldots,w_k \in \R^r$及び$b_1,\ldots,b_k$が存在して
  \[
  \mathrm{det}((\Psi(w_j x_i + b_j))_{i,j=1,\ldots,k}) \neq 0
  \]
  となる. ゆえに, $c_1,\ldots,c_k \in \R$に関する連立一次方程式
  \[
  \sum_{j=1}^k c_j \Psi(w_j^{\T}x_i + b_j) = \alpha_i ~~~~(i=1,\ldots,k)
  \]
  は解を持つ. 
\end{proof}

\numberwithin{thm}{section}
\section{RBFネットワークの万能近似定理}
本章では以下で定義するRBF(Radial-Basis-Function)ネットワークの万能近似能力について述べる. すなわち, 関数$K$がある条件をみたすときに$K$を基底関数とするRBFネットワークが$L^p(\R^r)$で稠密であることや, $C(\R^r)$において広義一様収束の意味で稠密であることを示す. 

\begin{defn}
  関数$K:\R^r \rightarrow \R$に対して, 
  \[
  q(x) = \sum_{i=1}^M w_i K\left(\frac{x-z_i}{\sigma}\right) ~~~(M \in \N, w_i \in \R, \sigma > 0 , z_i \in \R^r)
  \]
  という形の関数$q:\R^r \rightarrow \R$全体の集合を$S_K$と表し, $S_K$を$K$を基底関数とするRBFネットワークという. 
\end{defn}
 まず次の補題を用意しておく. 
\begin{lem}
  $f \in L^p(\R^r), ~p \in [1,\infty)$とし, $\phi \in L^1(\R^r)$は$\int_{\R^r} \phi(x) dx = 1$をみたすとする. このとき, $\varepsilon > 0$に対して, 
  \[
  \phi_{\varepsilon}(x) := \frac{1}{\varepsilon^r}\phi\left(\frac{x}{\varepsilon}\right) ~~(x \in \R^r)
  \]
  と定義すると, $\lVert \phi_{\varepsilon} * f - f \rVert_{L^p} \rightarrow 0 ~~(\varepsilon \rightarrow 0)$となる. 
\end{lem}
\begin{proof}
  $\phi_{\varepsilon} \in L^1(\R^r) ~~(\varepsilon > 0)$であるので, 付録の命題\ref{L1andLpConvolution}より$\phi_{\varepsilon}*f \in L^p(\R^r)$である. また, $\mathrm{a.e.}y \in \R^r$に対して, 変数変換により, 
  \[
  (\phi_{\varepsilon}*f)(y) = \int_{\R^r} \phi_{\varepsilon}(x) f(y-x) dx = \int_{\R^r} f(y-\varepsilon x) \phi(x) dx
  \]
  である. したがって, $\int_{\R^r} \phi(x) dx = 1$より, 
  \[
  (\phi_{\varepsilon} * f)(y) - f(y) = \int_{\R^r} ( f(y-\varepsilon x) - f(y) ) \phi(x) dx
  \]
  である. そこで, $q \in [1,\infty]$を$1/p + 1/q = 1$をみたすものとすると, H\"{o}lderの不等式(付録の命題\ref{HolderInequality})より, 
  \[
  \begin{aligned}
  \left\lvert (\phi_{\varepsilon} * f)(y) - f(y) \right\rvert
  &\leq \int_{\R^r} \left\lvert f(y-\varepsilon x) - f(y) \right\rvert \lvert \phi(x) \rvert dx  \\
  &= \int_{\R^r} \left\lvert f(y-\varepsilon x) - f(y) \right\rvert \lvert \phi(x) \rvert^{1-1/q} \lvert \phi(x) \rvert^{1/q} dx \\
  &\leq \lVert |\phi|^{1/q} \rVert_{L^q} \left(\int_{\R^r} |f(y-\varepsilon x) - f(y)|^p |\phi(x)| dx \right)^{1/p}
  \end{aligned}
  \]
  となる($p=1$のときは$\lVert |\phi|^{1/q} \rVert_{L^q} = 1$と解釈する). ゆえに, Fubiniの定理より, 
  \[
  \begin{aligned}
  \lVert \phi_{\varepsilon}*f - f \rVert_{L^p}^p 
  &\leq \lVert |\phi|^{1/q} \rVert_{L^q}^p \int_{\R^r} \int_{\R^r} \lvert f(y-\varepsilon x) - f(y) \rvert^p \lvert \phi(x) \rvert dx dy \\
  &= \lVert |\phi|^{1/q} \rVert_{L^q}^p \int_{\R^r} \lvert \phi(x) \rvert \lVert f(\cdot - \varepsilon x) - f(\cdot) \rVert_{L^p}^p dx
  \end{aligned}
  \]
  となる. 任意の$x \in \R^r, \varepsilon > 0$に対して, 三角不等式より
  \[
  \begin{aligned}
  \lvert \phi(x) \rvert \lVert f(\cdot - \varepsilon x) - f(\cdot) \rVert_{L^p} \leq 2 \lvert  \phi(x) \rvert  \lVert f \rVert_{L^p}
  \end{aligned}
  \]
  であり, $L^p$ノルムの平行移動に関する連続性(付録の命題\ref{ContinuousOfTranslationAboutLpNorm})より, 
  \[
  \lVert f(\cdot - \varepsilon x) - f(\cdot) \rVert_{L^p} \rightarrow 0 ~~(\varepsilon \rightarrow 0)
  \]
  であるから, 優収束定理より
  \[
  \lVert \phi_{\varepsilon}*f - f \rVert_{L^p} \rightarrow 0 ~~(\varepsilon \rightarrow 0)
  \]
  である. 
\end{proof}
\begin{thm}[J. Parkn and I. W. Sandberg, 1991 \cite{Park}]　\\
  $\mu$を$\R^r$上のLebesgue測度とする. 関数$K:\R^r \rightarrow \R$が以下を満たすとする. 
  \begin{itemize}
      \item[(1)] $K$は有界. 
      \item[(2)] $K \in L^1(\R^r,\mu)$かつ$\int_{\R^r} K d\mu \neq 0$.
      \item[(3)] $\mu(\{K\mbox{の不連続点}\}) = 0$
  \end{itemize}
  このとき, 任意の$p \geq 1$に対して, $S_K$は$L^p(\R^r,\mu)$において$L^p$ノルムについて稠密である. 
\end{thm}
\begin{proof}
  任意に$f \in L^p(\R^r,\mu)$と$\varepsilon > 0$をとる. $C_0(\R^r)$は$L^p(\R^r,\,u)$において$L^p$ノルムについて稠密であるから, $f_c \in C_0(\R^r)$が存在して, $\lVert f - f_c \rVert_{L^p} < \varepsilon/3$となる. $0 \in S_K$であるから$f_c \neq 0$と仮定してよい. ここで関数$\phi:\R^r \rightarrow \R$を
  \[
  \phi(x) = \frac{1}{\int_{\R^r} K d\mu} K(x) 
  \]
  と定義する. そして, $\sigma > 0$に対して, $\phi_{\sigma}(x) = \phi(x/\sigma)/\sigma^r$と定義する. このとき, $\lVert \phi_{\sigma}*f_c - f_c \rVert_{L^p} \rightarrow 0 ~(\sigma \rightarrow 0)$となる. そこで, $\sigma > 0$を$\lVert \phi_{\sigma}*f_c - f_c \rVert_{L^p} < \varepsilon/3$となるようにとっておく. 次に$\mathrm{supp}f_c \subset [-T,T]^r$なる$T > 0$をとる. 各$\alpha \in \R^r$に対して$\phi_{\sigma}(\alpha - \cdot) f_c(\cdot)$は有界かつほとんど至る所連続であるから$[-T,T]^r$上Riemann積分可能である. したがって, 各$[-T,T]$を$n$等分することで得られる$[-T,T]^r$の分割を$\Delta_n$とおき, $n^r$個ある$\Delta_n$の各小領域から任意に$\alpha_i$を選ぶと, 
  \[
  v_n(\alpha) := \sum_{i=1}^{n^r} \phi_{\sigma}(\alpha - \alpha_i)f_c(\alpha_i)\left(\frac{2T}{n}\right)^r
  \]
  は$\int_{[-T,T]^r} \phi_{\sigma}(\alpha - x)f_c(x)dx$に収束する. 一方, $\mathrm{supp}f_c \subset [-T,T]^r$であるので, 
  \[
  \int_{[-T,T]^r} \phi_{\sigma}(\alpha - x)f_c(x)dx = \int_{\R^r} \phi_{\sigma}(\alpha - x)f_c(x)dx = (\phi_{\sigma} * f_c)(\alpha)
  \]
  である. よって, 各点$\alpha \in \R^r$について$v_n(\alpha) \rightarrow (\phi_{\sigma} * f_c)(\alpha)$となる. $\phi_{\sigma}*f_c \in L^p(\R^r,\mu)$に注意すると, ある$T_1 > 0$について, 
  \[
  \int_{\R^r \setminus [-T_1,T_1]^r} |\phi_{\sigma}*f_c|^p d\mu < (\varepsilon/9)^p
  \]
  となる. また, $\phi_{\sigma} \in L^1(\R^r,\mu)$であり, かつ$\phi_{\sigma}$は有界であることより$|\phi_{\sigma}|^p \in L^1$となるから, ある$T_2 > 0$について, 
  \[
  \int_{\R^r \setminus [-T_2,T_2]^r} |\phi_{\sigma}|^p d\mu < \left( \frac{\varepsilon}{9\lVert f_c \rVert_{L^{\infty}} (2T)^r} \right)^p
  \]
  となる. $[0,\infty) \ni y \rightarrow y^p \in \R$が凸関数であることより, 任意の$\alpha \in \R^r$に対して, 
  \[
  \left(\frac{1}{n^r} \sum_{i=1}^{n^r} |\phi_{\sigma}(\alpha - \alpha_i)| \right)^p \leq \frac{1}{n^r} \sum_{i=1}^{n^r} |\phi_{\sigma}(\alpha - \alpha_i)|^p
  \]
  となることに注意すると, 
  \[
  |v_n(\alpha)|^p \leq (\lVert f_c \rVert_{L^{\infty}} (2T)^r)^p \frac{1}{n^r} \sum_{i=1}^{n^r} |\phi_{\sigma}(\alpha - \alpha_i)|^p
  \]
  である. ここで, $T_0 = \max\{T_1,T_2+T\}$とおくと, $\alpha_i$の各成分$\alpha_{i,j} $は$|\alpha_{i,j}| \leq T$をみたすので, 
  \[
  \int_{[-T_2,T_2]^r} |\phi_{\sigma}(\alpha)|^p d\alpha \leq \int_{[-T_0,T_0]^r} |\phi_{\sigma}(\alpha - \alpha_i)|^p d\alpha
  \]
  となる. したがって, 
  \[
  \int_{\R^r \setminus [-T_0,T_0]^r} |\phi_{\sigma}(\alpha - \alpha_i)|^p d\alpha 
  \leq \int_{\R^r \setminus [-T_2,T_2]^r} |\phi_{\sigma}(\alpha)|^p d\alpha
  \]
  であるので, 
  \[
  \begin{aligned}
  &\int_{\R^r \setminus [-T_0,T_0]^r} |v_n(\alpha)|^p d\alpha \\
  &\leq (\lVert f_c \rVert_{L^{\infty}} (2T)^r)^p \frac{1}{n^r} \sum_{i=1}^{n^r} \int_{\R^r \setminus [-T_0,T_0]^r}|\phi_{\sigma}(\alpha - \alpha_i)|^p d\alpha \\
  &\leq (\lVert f_c \rVert_{L^{\infty}} (2T)^r)^p \frac{1}{n^r} \sum_{i=1}^{n^r} \int_{\R^r \setminus [-T_2,T_2]^r} |\phi_{\sigma}(\alpha)|^p d\alpha \\
  &< (\varepsilon/9)^p
  \end{aligned}
  \]
  となる. また, $T_0 \geq T_1$より, 
  \[
  \int_{\R^r \setminus [-T_0,T_0]^r} |\phi_{\sigma}*f_c|^p d\mu < (\varepsilon/9)^p
  \]
  である. いま, $\phi_{\sigma}*f_c \in L^p$と$v_n$の有界性より優収束定理が使えて, 
  \[
  \int_{[-T_0,T_0]^r} |\phi_{\sigma}*f_c - v_n|^p d\mu \rightarrow 0 ~~~(n \rightarrow \infty)
  \]
  となる. ゆえに, 適当に$N \in \N$を取ると, 
  \[
  \int_{[-T_0,T_0]^r} |\phi_{\sigma}*f_c - v_N|^p d\mu < (\varepsilon/9)^p
  \]
  となる. 以上から, 
  \[
  \begin{aligned}
  \lVert v_N - \phi_{\sigma}*f_c \rVert_{L^p}
  &\leq \lVert v_N\cdot \chi_{\R^r \setminus [-T_0,T_0]^r} \rVert_{L^p} + \lVert (v_N - \phi_{\sigma}*f_c)\cdot \chi_{[-T_0,T_0]^r} \rVert_{L^p} \\
  &~~~ + \lVert (\phi_{\sigma}*f_c) \cdot \chi_{\R^r \setminus [-T_0,T_0]^r} \rVert_{L^p} \\
  &< \varepsilon/9 + \varepsilon/9 + \varepsilon/9 = \varepsilon/3
  \end{aligned}
  \]
  となる. よって, 
  \[
  \lVert f - v_N \rVert_{L^p} \leq \lVert f - f_c \rVert_{L^p} + \lVert f_c - \phi_{\sigma}*f_c \rVert_{L^p} + \lVert \phi_{\sigma}*f_c - v_N \rVert_{L^p} < \varepsilon
  \]
  である. そして, 
  \[
  v_N = \sum_{i=1}^{N^r} \phi_{\sigma}(\cdot - \alpha_i)f_c(\alpha_i)\left(\frac{2T}{N}\right)^r = \sum_{i=1}^{N^r} w_i K\left(\frac{\cdot - \alpha_i}{\sigma}\right) \in S_K
  \]
  である. ただし, 
  \[
  w_i = \frac{1}{\sigma^r}f_c(\alpha_i)\left(\frac{2T}{N}\right)^r \frac{1}{\int_{\R^r} K d\mu}
  \]
  である. これで示せた. 
\end{proof}
\begin{thm}
  $\mu$を$\R^r$上のLebesgue測度とする. 関数$K:\R^r \rightarrow \R$が以下を満たすとする. 
  \begin{itemize}
      \item[(1)] $K$は有界. 
      \item[(2)] $K \in L^1(\R^r,\mu)$かつ$\int_{\R^r} K d\mu \neq 0$.
      \item[(3)] $K$は連続関数である. 
  \end{itemize}
  このとき, $S_K$は$C(\R^r,\R)$において広義一様収束の意味で稠密である. 
\end{thm}
\begin{proof}
  まず, 先と同様に, 関数$\phi:\R^r \rightarrow \R$を
  \[
  \phi(x) = \frac{1}{\int_{\R^r} K d\mu} K(x) 
  \]
  と定義し, $\sigma > 0$に対して, $\phi_{\sigma}(x) = \phi(x/\sigma)/\sigma^r$と定義する. 任意に$M > 0$をとる. $[-M,M]^r$上の関数$f$と$\varepsilon > 0$に対して, $g \in S_K$が存在して, \\
  $\sup_{x \in [-M,M]^r} |f(x) - g(x)| < \varepsilon$となることを示せば十分である. $T > M$をとると, Tietzeの拡張定理より
  \[
  \tilde{f}(x) = f(x) ~(x \in [-M,M]^r), ~~\tilde{f}(x) = 0 ~(x \in \R^r \setminus [-T,T])
  \]
  なる連続関数$\tilde{f}:\R^r \rightarrow \R$がとれる. 明らかに$\tilde{f}$は一様連続であるので, ある$\delta > 0$が存在して, 
  \[
  \forall x,y \in \R^r , |x-y| < \delta \Rightarrow |\tilde{f}(x) - \tilde{f}(y)| < \frac{\varepsilon}{4\lVert \phi \rVert_{L^1}} 
  \]
  となる. 一方, $\phi \in L^1(\R^r,\mu)$なので, ある$T_0 > 0$について, 
  \[
  \int_{\R^r \setminus [-T_0,T_0]^r} |\phi| d\mu <  \frac{\varepsilon}{8\lVert \tilde{f} \rVert_{L^{\infty}}}
  \]
  となる. そこで, 適当に$\sigma > 0$をとり, $\forall x \in [-T_0,T_0]^r,~|\sigma x | < \delta$となるようにすると, 任意の$\alpha \in [-M,M]^r$について, 
  \[
  \begin{aligned}
  |(\phi_{\sigma} * \tilde{f})(\alpha) - \tilde{f}(\alpha)|
  &= \left| \int_{\R^r} \frac{1}{\sigma^r}\phi\left( \frac{x}{\sigma} \right) \tilde{f}(\alpha - x) dx - \int_{\R^r} \tilde{f}(\alpha)\phi(x) dx \right| \\
  &= \left| \int_{\R^r} (\tilde{f}(\alpha - \sigma x) - \tilde{f}(\alpha))\phi(x) dx \right| \\
  &\leq \int_{\R^r} |\tilde{f}(\alpha - \sigma x) - \tilde{f}(\alpha)|\cdot|\phi(x)| dx \\
  &\leq \int_{[-T_0,T_0]^r} |\tilde{f}(\alpha - \sigma x) - \tilde{f}(\alpha)|\cdot|\phi(x)| dx \\
  &~~~ + \int_{\R^r \setminus [-T_0,T_0]^r} 2\lVert \tilde{f} \rVert_{L^{\infty}} |\phi(x)| dx  \\
  &< \varepsilon/2
  \end{aligned}
  \]
  となる. ここで, 前定理の証明と同様に, $[-T,T]^r$を等分割し, 各$\alpha \in \R^r$に対して, 
  \[
  \tilde{v}_n(\alpha) = \sum_{i=1}^{n^r} \phi_{\sigma}(\alpha - \alpha_i)\tilde{f}(\alpha_i)\left(\frac{2T}{n}\right)^r
  \]
  と定義する. 関数$(s,x) \mapsto \phi_{\sigma}(s-x)\tilde{f}(x)$は$[-M,M]^r \times [-T,T]^r$上で一様連続なので, ある$\delta_0 > 0$が存在して, 任意の$s \in [-M,M]^r, x,y \in [-T,T]^r$について, $|x-y| < \delta_0$ならば, 
  \[
  |\phi_{\sigma}(s-x)\tilde{f}(x) - \phi_{\sigma}(s-y)\tilde{f}(y)| < \frac{\varepsilon}{2(2T)^r}
  \]
  となる. $N$を十分大きく取ると, $\Delta_N$の各小領域の任意の$2$点間の距離は$\delta_0$未満になる. ゆえに, $[-T,T]^r = \bigcup_{i=1}^{N^r} J_i$と表すことにすると, 任意の$\alpha \in [-M,M]^r$に対して, 
  \[
  \begin{aligned}
  &|v_N(\alpha) - \int_{[-T,T]^r} \phi_{\sigma}(\alpha - x) \tilde{f}(x) dx| \\
  &\leq \sum_{i=1}^{N^r} \int_{J_i} |\phi_{\sigma}(\alpha-\alpha_i) \tilde{f}(\alpha_i) -  \phi_{\sigma}(\alpha-x) \tilde{f}(x)| dx \\
  &\leq \sum_{i=1}^{N^r} \int_{J_i} \frac{\varepsilon}{2(2T)^r} dx \leq \varepsilon/2
  \end{aligned}
  \]
  となる. 任意の$\alpha \in [-M,M]^r$に対して, 
  \[
  \int_{[-T,T]^r} \phi_{\sigma}(\alpha - x) \tilde{f}(x) dx　= (\phi_{\sigma} * \tilde{f})(\alpha)
  \]
  であることに注意すると, 
  \[
  \begin{aligned}
  |f(\alpha) - v_N(\alpha)| 
  &= |\tilde{f}(\alpha) -  v_N(\alpha)| \\
  &\leq |\tilde{f}(\alpha) - (\phi_{\sigma} * \tilde{f})(\alpha)| + |(\phi_{\sigma} * \tilde{f})(\alpha) - v_N(\alpha)| \\
  &< \varepsilon
  \end{aligned}
  \]
  となる. よって, $\sup_{x \in [-M,M]^r} |f(x) - v_N(x)| \leq \varepsilon$である. これで示せた. 
\end{proof}
なお, 以上の結果に関連して, Pinkusにより次の結果が得られている. 
\begin{thm}[Pinkus 1996 \cite{Pinkus1996} Theorem12]
  $\sigma \in C(\R_{+})$に対して, 
  \[
  \mathrm{span}\{x \mapsto \sigma(\rho \lVert x - a \rVert) \mid \rho >0, a \in \R^r\}
  \]
  が$C(\R^r)$において広義一様収束の意味で稠密であることと, $\sigma$が偶多項式と一致しないことは同値である. 
\end{thm}

\numberwithin{thm}{subsection}
\section{近似レートの評価}
本章では, 与えられた関数空間の部分集合を$3$層ニューラルネットワークにより近似することを考えるとき, 中間層のユニット数に関してどれくらいの効率で近似誤差が減少していくかという問題を考える. つまり, $X$を
\[
\mathcal{N}_r^{n}(\Psi) := \left\{x \mapsto \sum_{i=1}^n c_i \Psi(w_i^{\T} x + b_i) ~\vline~ w_i \in \R^r, c_i,b_i \in \R \right\}
\]
を含む関数空間とするとき, 部分集合$\mathcal{F} \subset X$に対して, 近似誤差
\[
\sup_{f \in \mathcal{F}} \inf_{g \in \mathcal{N}_r^{n}(\Psi)} \lVert f- g \rVert_X
\]
は$n$を大きくしていくときどのようなレートで減少するか？という問題を考える. 
\subsection{Sobolev空間における近似レート}
本節ではPinkus,1999 \cite{Pinkus}で紹介されている, Sobolev空間における近似レートの結果を見ていく.  
まず近似誤差を一般的な形で定義しておく. 
\begin{defn}[近似誤差]　\\
  $X$をノルム空間とする. $f \in X$と空でない$Y \subset X$に対して, 
  \[
  \lVert f - Y \rVert_X := \inf_{g \in Y} \lVert f - g \rVert_X
  \]
  と定義する. 付録の補題\ref{PropOfDist}より$X \ni f \mapsto \lVert f - Y \rVert_X \in \R$は連続である. さらに, 空でない$Z \subset X$に対して, 
  \[
  E(Z,Y,X) := \sup_{z \in Z} \lVert z - Y \rVert_X  = \sup_{z \in Z} \inf_{y \in Y} \lVert z - y \rVert_X
  \]
  と定義する. 
\end{defn}
次にSobolev空間の定義を述べよう. Sobolev空間についての詳しい説明は宮島\cite{MiyajimaSobolev}などを参照されたい. 
\begin{defn}[Sobolev空間]　\\
  開集合$\Omega \subset \R^r$と$p \in [1,\infty]$および整数$m \geq 1$に対して, 
  \[
  W^{m,p}(\Omega) 
  := \left\{f \in L^p(\Omega) ~\vline~ 
  \begin{aligned}
  &|\alpha| \leq m \mbox{なる任意の}\alpha \in \mathbb{Z}_{\geq 0}^r \mbox{に対して} \\
  &f \mbox{の弱導関数} \partial^{\alpha}f \mbox{が存在して} \partial^{\alpha}f \in L^p({\Omega})
  \end{aligned}
  \right\}
  \]
  と定義する. ここで, $h:\Omega \rightarrow \R$が $\partial^{\alpha}$に関する$f$の弱導関数であるとは, 任意の$\varphi \in C_0^{\infty}(\Omega)$に対して, 
  \[
  \int_{\Omega} f(x) (\partial^{\alpha}\varphi)(x) dx = (-1)^{|\alpha|} \int_{\Omega} h(x)\varphi(x)dx
  \]
  が成り立つことをいう. さらに, $f \in W^{m,p}(\Omega)$に対して, 
  \[
  \lVert f \rVert_{W^{m,p}(\Omega)} := \sum_{|\alpha| \leq m} \lVert \partial^{\alpha} f \rVert_{L^p(\Omega)}
  \]
  と定義する. また, 
  \[
  \mathcal{B}^{m,p}(\Omega) := \{f \in W^{m,p}(\Omega) \mid \lVert f \rVert_{W^{m,p}(\Omega)} \leq 1\}
  \]
  と定義する. 
\end{defn}

Pinkus\cite{Pinkus}では近似誤差$E(\mathcal{B}^{m,p},\mathcal{N}_r^n(\Psi),L^p)$の$n$に関する減少率について$1999$年までの結果がいくつか紹介されている. 

さて, まず多項式によるSobolev空間の近似について述べるために多項式の空間の定義を思い出しておく. 
\begin{defn}
  $r,k \in \N$に対して, $r$変数の$k$次斉次多項式関数の集合を
  \[
  \mathcal{H}_k(\R^r) := \left\{ \R^r \ni x \mapsto \sum_{|\alpha| = k} c_{\alpha} x^{\alpha} \in \R \mid c_{\alpha} \in \R \right\}
  \]
  と定義する. ただし, 和は$|\alpha| = k$なる$\alpha \in (\mathbb{Z}_{\geq 0})^r$全体にわたるものとする. また, 
  \[
  \begin{aligned}
  \mathcal{P}_k(\R^r) := \bigcup_{s=0}^{k} \mathcal{H}_s(\R^r) ,~~ \mathcal{P}(\R^r) := \bigcup_{s=0}^{\infty} \mathcal{H}_s(\R^r)
  \end{aligned}
  \]
  と定義する. 
  さらに$r$変数の$k$次以下の$\R$係数多項式全体の集合を
  \[
  \pi_k(\R^r) := \left\{ \R^r \ni x \mapsto \sum_{|\alpha| \leq k} c_{\alpha} x^{\alpha} \in \R \mid c_{\alpha} \in \R \right\}
  \]
  と定義する. 
\end{defn}
多項式によるSobolev空間の近似について次が成り立つことが知られている. 
\begin{thm}[Jackson評価]\label{JacksonRate}　\\
  $I = (-1,1)$とおく. $p \in [1,\infty], m,r \geq 1$とするとき, ある$C > 0$が存在して, 
  \[
  E(\mathcal{B}^{m,p}(I^r),\pi_k(\R^r), L^p(I^r)) \leq C k^{-m} ~~~~(k=1,2,3,\ldots)
  \]
  となる. 
\end{thm}
\begin{proof}
  付録の定理\ref{JacksonEstimate4}を参照されたい. 
\end{proof}
次に活性化関数が$C^{\infty}$級で多項式と一致しない場合のニューラルネットワークによる近似レートについて述べる. そのためにまず補題\ref{NeuralNetApproPoly}を精密化しよう. 
\begin{lem}
  $n \in \N$とすると, 次の式が成り立つ. 
  \[
  \sum_{i=0}^n (-1)^i \binom{n}{i} (n-2i)^j 
  = 
  \left\{
  \begin{aligned}
  &~~0 ~~~~~~(j=0,\ldots,n-1)\\
  &~2^n n! ~~~(j=n)
  \end{aligned}
  \right.
  \]
\end{lem}
\begin{proof}
  \[
  \begin{aligned}
  \sum_{i=0}^n (-1)^i \binom{n}{i} (n-2i)^j 
  &= \sum_{i=0}^n (-1)^i \binom{n}{i} \left.\left(\frac{d}{dx}\right)^j \exp((n-2i)x)\right|_{x=0} \\
  &= \left.\left(\frac{d}{dx}\right)^j \left(\sum_{i=0}^n (-1)^i \binom{n}{i} \exp((n-2i)x) \right)\right|_{x=0} \\
  &= \left.\left(\frac{d}{dx}\right)^j \left(\sum_{i=0}^n \binom{n}{i} \exp(x)^{n-i} (-\exp(-x))^i \right)\right|_{x=0} \\
  &= \left.\left(\frac{d}{dx}\right)^j \left(\exp(x)-\exp(-x) \right)^n \right|_{x=0} \\
  &= 2^n \left.\left(\frac{d}{dx}\right)^j \left(\sinh(x) \right)^n \right|_{x=0}
  \end{aligned}
  \]
  である. ここに
  \[
  \sinh(x) = \frac{\exp(x)-\exp(-x)}{2}
  \]
  である. よって
  \[
  \left.\left(\frac{d}{dx}\right)^j \left(\sinh(x) \right)^n \right|_{x=0}
  =
  \left\{
  \begin{aligned}
  &~~0 ~~~~(j=0,\ldots,n-1)\\
  &~~ n! ~~~(j=n)
  \end{aligned}
  \right.
  \]
  を示せばよい. これを$n$についての帰納法で示そう. $n=1$での成立は明らか. $n$で成立すると仮定すると, $j=0,\ldots,n-1$に対しては, 積の微分に関するLeibniz則より
  \[
  \begin{aligned}
  &\left.\left(\frac{d}{dx}\right)^j \left(\sinh(x) \right)^{n+1} \right|_{x=0} \\
  &= \left. \sum_{k=0}^{j} \binom{j}{k} \left( \left( \frac{d}{dx} \right)^k (\sinh(x))^n \right) \left( \left( \frac{d}{dx} \right)^{j-k} \sinh(x) \right)  \right|_{x=0} = 0.
  \end{aligned}
  \]
  また, 
  \[
  \left.\left(\frac{d}{dx}\right)^n \left(\sinh(x) \right)^{n+1} \right|_{x=0} 
  = \left.\left(\left(\frac{d}{dx}\right)^n \left(\sinh(x) \right)^n \right) \sinh(x) \right|_{x=0} = 0
  \]
  であり, 
  \[
  \begin{aligned}
  \left.\left(\frac{d}{dx}\right)^{n+1} \left(\sinh(x) \right)^{n+1} \right|_{x=0} 
  &= \left.\binom{n+1}{n}\left(\left(\frac{d}{dx}\right)^n \left(\sinh(x) \right)^n \right) \left(\frac{d}{dx} \sinh(x) \right)\right|_{x=0} \\
  &~~~ + \left.\left(\left(\frac{d}{dx}\right)^{n+1} \sinh(x)\right) \sinh(x) \right|_{x=0} \\
  &= (n+1)(n!) + 0 = (n+1)!
  \end{aligned}
  \]
  である. これで示せた. 
\end{proof}
\begin{lem}\label{PolyApproByNNRefine}
  $W \subset \R$は$0$を内点として持つとする. 
  また, $\Psi:\R \rightarrow \R$はある開区間$I$上で$C^{\infty}$かつ多項式と一致しないとする. 
  このとき, 各$n$について多項式$x \mapsto x^n$は
  \[
  \mathcal{N}_1^{n+1}(\Psi,W,I) = \left\{ x \mapsto \sum_{i=1}^{n+1} c_i \Psi(w_i x + b_i) \mid c_i \in \R,w_i \in W, b_i \in I \right\}
  \]
  で広義一様近似される. さらに, $1$変数の$n$次以下の$\R$係数多項式は$\mathcal{N}_1^{2n+1}(\Psi,W,I)$で広義一様近似される.  
\end{lem}
\begin{proof}
  $K \subset \R$を空でないコンパクト集合とする. 
  $\Psi$は$I$上$C^{\infty}$級で多項式と一致しないので付録の定理\ref{ConditionOfSmoothFunctionIsPoly}よりある$b \in I$が存在して任意の整数$n \geq 0$について$\Psi^{(n)}(b) \neq 0$となる. 
  $f:\R \times \R \ni (x,w) \mapsto \Psi(wx + b) \in \R$と定義する. 前半を示すには, 任意の$n \in \N$に対して, 関数$x \mapsto ((\partial w)^n f) (x,0)$が$\mathcal{N}_1^{n+1}(\Psi,W,I)$の元で$K$上一様近似できることを示せば十分である. 実際, それが示されれば
  \[
  x \mapsto ((\partial w)^n f) (x,0) = x^n \Psi^{(n)}(0 x + b) = x^n \Psi^{(n)}(b)
  \]
  なので$x \mapsto x^n$は$\mathcal{N}_1^{n+1}(\Psi,W,I)$の元で$K$上一様近似される. いま, $h > 0$に対して, 
  \[
  F_x(h) := \sum_{i=0}^n (-1)^i \binom{n}{i} f(x,(n-2i)h) = \sum_{i=0}^n (-1)^i \binom{n}{i} \Psi((n-2i)hx + b)
  \]
  とおくと, Taylorの定理より
  \[
  F_x(h) = F_x(0) + hF_x'(0)+\cdots+\frac{h^n}{n!}F_x^{(n)}(0)+\frac{h^{n+1}}{(n+1)!}F_x^{(n+1)}(\xi_x) ~~(\exists \xi_x \in (0,h))
  \]
  となる. 前命題より
  \[
  F_x^{(j)}(0) = ((\partial w)^j f)(x,0) \sum_{i=0}^n (-1)^i \binom{n}{i} (n-2i)^j 
  =
  \left\{
  \begin{aligned}
  &~~~~~ 0 ~~~~~~~~~~(j=0,\ldots,n-1)\\
  &~ n! 2^n ((\partial w)^n f)(x,0) ~~~~(j=n)
  \end{aligned}
  \right.
  \]
  となる. 従って, $M_1 = \sup \{((\partial w)^{n+1} f)(x,w) \mid x \in K, w \in [-n,n] \}$とおき, 
  \[
  M_2 = M_1 \left| \sum_{i=0}^n (-1)^i \binom{n}{i} (n-2i)^{n+1} \right|
  \]
  とおけば, 任意の$\varepsilon>0$に対して$\delta > 0$を$\delta < \min\{\varepsilon/M_2,1\}$となるよう取ることで, 任意の$0<h<\delta$と$x \in K$に対して
  \[
  \left| \frac{F_x(h)}{(2h)^n} - ((\partial w)^n f)(x,0) \right|
  \leq \frac{h}{(n+1)!}|F^{(n+1)}(\xi_x)|
  \leq h M_2 < \varepsilon
  \]
  となる. よって, $0<h<\delta$を満たし且つ$[-nh,nh] \subset W$を満たす$h$を取れば
  \[
  \R \ni x \rightarrow \frac{F_x(h)}{(2h)^n} = \frac{1}{(2h)^n}\sum_{i=0}^n (-1)^i \binom{n}{i} \Psi((n-2i)hx + b) \in \R
  \]
  は$\mathcal{N}_1^{n+1}(\Psi,W,I)$に属し, $((\partial w)^n f)(\cdot,0)$を$K$上一様近似する. これで前半が示せた. 次に$k \leq n$に対して$n$次以下の多項式$p(x) = a_nx^n + \cdots + a_0$を任意に取る. このとき, 上で示したことにより任意の$\varepsilon > 0$に対して$h > 0$が存在して$[-nh,nh] \subset W$且つ
  \[
  \sup_{x \in K} \left| x^j - \frac{\sum_{i=0}^j (-1)^i \binom{j}{i} \Psi((j-2i)hx + b) }{\Psi^{(k)}(b)(2h)^j} \right| < \varepsilon ~~~~(j=0,\ldots,n)
  \]
  となる. よって, 
  \[
  G(x) = \sum_{j=0}^n a_j \frac{\sum_{i=0}^j (-1)^i \binom{j}{i} \Psi((j-2i)hx + b) }{\Psi^{(k)}(b)(2h)^j}
  \]
  とおくと, 
  \[
  \sup_{x \in K} |p(x)-G(x)| < \sum_{i=0}^n |a_i| \varepsilon
  \]
  となる. そして, 
  \[
  S = \{ (j-2i)h \mid ~j =0,\ldots,n, ~i = 0,\ldots,j  \}
  \]
  の要素数は$2n+1$であるので, 
  \[
  G(x) = \sum_{j=1}^{2n+1} c_j \Psi(w_jx + b) ~~~(c_j \in \R, w_j \in W )
  \]
  と表される. これで後半も示せた. 
\end{proof}

\begin{thm}\label{NonPolyActFuncApproRate}
  $I = (-1,1)$とおく. $\Psi: \R \rightarrow \R$はある開区間$J$上で$C^{\infty}$級かつ多項式と一致しないとする. このとき, 各$p \in [1,\infty], m,r \geq 1$に対してある$C > 0$が存在して, 任意の整数$n \geq 1$に対して
  \[
  E(\mathcal{B}^{m,p}(I^r), \mathcal{N}_r^{n}(\Psi),L^p(I^r)) \leq C n^{-m/r}
  \]
  となる. 
\end{thm}
\begin{proof}
  $C$は$r$に依存してよいので$3r \leq n$としてよい. $\binom{r-1+k}{k}(2k+1) \leq n$なる最大の整数$k \geq 1$を取る. 
  命題\ref{PolyRidgeRelation}より, $l := \mathrm{dim}\mathcal{H}_k(\R^r) = \binom{r-1+k}{k}$とおくと, $a_1,\ldots,a_l \in \R^r$が存在して, 
  \[
  \pi_k(\R^r) = \left\{ x \mapsto \sum_{i=1}^l g_i(a_i^{\T} x )  \mid g_i \in \pi_k(\R) \right\}
  \]
  となる. また, 補題\ref{PolyApproByNNRefine}より, $\pi_k(\R)$の元は$\mathcal{N}_1^{2k+1}(\Psi)$により広義一様近似される. したがって, 
  \[
  \pi_k(\R^r) \subset \overline{\mathcal{N}_r^{l(2k+1)}(\Psi)}
  \]
  となる(ただし, 広義一様収束の位相に関する閉包をとっている). ゆえに, Jackson評価(定理\ref{JacksonRate})より$k$に依存しない定数$C > 0$が存在して, 
  \[
  \begin{aligned}
  E(\mathcal{B}^{m,p}(I^r), \overline{\mathcal{N}_r^{n}(\Psi)}, L^p(I^r)) 
  & E(\mathcal{B}^{m,p}(I^r), \overline{\mathcal{N}_r^{l(2k+1)}(\Psi)}, L^p(I^r)) \\
  &\leq E(\mathcal{B}^{m,p}(I^r), \pi_k(\R^r), L^p(I^r)) \leq C k^{-m}
  \end{aligned}
  \]
  となる. ここで, 
  \[
  E(\mathcal{B}^{m,p}(I^r), \mathcal{N}_r^n(\Psi), L^p(I^r)) \leq E(\mathcal{B}^{m,p}(I^r), \overline{\mathcal{N}_r^n(\Psi)}, L^p(I^r))
  \]
  を示す. 任意に$f \in \overline{\mathcal{N}_r^n(\Psi)}, g \in \mathcal{B}^{m,p}(I^r)$をとる. $I^r$上で$f$に一様収束する列$f_j \in \mathcal{N}_r^n(\Psi)$がとれる. すると, 
  \[
  \begin{aligned}
  \lVert g-f  \rVert_{L^p(I^r)} 
  &= \lim_{j \rightarrow \infty} \lVert g-f_j \rVert_{L^p(I^r)} \\ &\geq \inf_{h \in \mathcal{N}_r^n(\Psi)} \lVert g - h \rVert_{L^p(I^r)} = \lVert g - \mathcal{N}_r^n(\Psi)\rVert_{L^p(I^r)}
  \end{aligned}
  \]
  となる. $f \in \overline{\mathcal{N}_r^n(\Psi)}$は任意なので, 
  \[
  \lVert g - \overline{\mathcal{N}_r^n(\Psi)} \rVert_{L^p(I^r)} \geq \lVert g - \mathcal{N}_r^n(\Psi)\rVert_{L^p(I^r)}
  \]
  である. よって, 
  \[
  E(\mathcal{B}^{m,p}(I^r), \mathcal{N}_r^n(\Psi), L^p(I^r)) \leq E(\mathcal{B}^{m,p}(I^r), \overline{\mathcal{N}_r^n(\Psi)}, L^p(I^r))
  \]
  である. 以上より, 
  \[
  E(\mathcal{B}^{m,p}(I^r), \mathcal{N}_r^n(\Psi), L^p(I^r)) \leq C k^{-m} 
  \]
  となる. いま$k$の取り方から
  \[
  n \leq \binom{(r-1)+(k+1)}{k+1}(2k+3)
  \]
  となるが, $k$に依存しない定数$C_1 > 0$が存在して
  \[
  \begin{aligned}
  \binom{(r-1)+(k+1)}{k+1} 
  &= \frac{((k+1)+(r-1))((k+1)+(r-2))\cdots((k+1)+1)}{(r-1)(r-2) \cdots 1} \\
  &= \left( \frac{k+1}{r-1} + 1 \right)\left( \frac{k+1}{r-2} + 1 \right)\cdots\left( \frac{k+1}{1} + 1 \right) \\
  &\leq C_1 (k+1)^{r-1} 
  \end{aligned}
  \]
  となるので, $k$に依存しない定数$C_2 > 0$が存在して, 
  \[
  n \leq \binom{r+k}{k+1}(2k+3) \leq C_2 k^r
  \]
  となる. よって, $C C_2$を改めて$C$とおけば
  \[
  E(\mathcal{B}^{m,p}(I^r), \mathcal{N}_r^n(\Psi), L^p(I^r)) \leq C n^{-m/r}
  \]
  である. 
\end{proof}

次にリッジ関数による近似について以下が成り立つ. 
\begin{thm}
  $I = (-1,1)$とおく. また, $p \in [1,\infty], m \geq 1, r \geq 2$とする. このとき, ある$C > 0$が存在して, 任意の整数$n \geq 1$に対して
  \[
  E(\mathcal{B}^{m,p}(I^r),\mathcal{R}_r^{n},L^p(I^r)) \leq C n^{-m/(r-1)}
  \]
  が成り立つ. ここで
  \[
  \mathcal{R}_r^n := \left\{ x \mapsto \sum_{i=1}^n g_i(a_i^{\T} x) \mid a_i \in \R^r, g_i \in C(\R) \right\}
  \]
  である. 
\end{thm}
\begin{proof}　\\
  $n \geq r$としてよい. $k$を$\binom{r-1+k}{k} \leq n$なる最大の整数とする. 
  命題\ref{PolyRidgeRelation}より, $l := \mathrm{dim}\mathcal{H}_k(\R^r) = \binom{r-1+k}{k}$とおくと, $a_1,\ldots,a_l \in \R^r$が存在して, 
  \[
  \pi_k(\R^r) = \left\{ x \mapsto \sum_{i=1}^l g_i(a_i^{\T} x )  \mid g_i \in \pi_k(\R) \right\}
  \]
  となる. したがって, $\pi_k(\R^r) \subset \mathcal{R}_r^l$となる. ゆえに, Jackson評価(定理\ref{JacksonRate})より$k$に依存しない定数$C > 0$が存在して, 
  \[
  \begin{aligned}
  E(\mathcal{B}^{m,p}(I^r),\mathcal{R}_r^n, L^p(I^r)) 
  &\leq E(\mathcal{B}^{m,p}(I^r),\mathcal{R}_r^l, L^p(I^r)) \\
  \leq E(\mathcal{B}^{m,p}(I^r), \pi_k(\R^r), L^p(I^r)) \leq C k^{-m}
  \end{aligned}
  \]
  となる. いま$k$に依存しない$C' > 0$があって$n \leq C' k^{r-1}$であるから, 
  $CC'$を改めて$C$とおけば
  \[
  E(\mathcal{B}^{m,p}(I^r),\mathcal{R}_r^n, L^p(I^r)) \leq C n^{-m/(r-1)}
  \]
  となる. 
\end{proof}
上の定理と次の結果を合わせることでニューラルネットワークによる近似に関する結果が得られる. 
\begin{thm}[Maiorov and Pinkus 1999 \cite{MaiorovPinkus1999}]　\\
  狭義単調増加なsigmoidal関数$\Psi \in C^{\infty}(\R)$であって, 次の性質をみたすものが存在する: 任意の整数$n,r \geq 1$と$g \in \mathcal{R}_r^n, \varepsilon > 0$に対して, 
  \[
  \sup_{x \in B(0,1)} \left\lvert g(x) - \sum_{i=1}^{r+n+1} c_i \Psi(w_i^{\T} x + b_i) \right\rvert < \varepsilon
  \]
  なる$c_i,b_i \in \R, w_i \in \R^r$が存在する. 
\end{thm}
\begin{cor}
  $I=(-1,1)$とおく. 
  狭義単調増加なsigmoidal関数$\Psi \in C^{\infty}(\R)$であって, 次の性質をみたすものが存在する: $p \in [1,\infty], m \geq 1, r \geq 2$に対して, 
  \[
  E(\mathcal{B}^{m,p}(I^r), \mathcal{N}_r^{r+n+1}(\Psi), L^p(I^r)) \leq C r^{-m/(n-1)}.
  \]
\end{cor}
以上の結果は近似誤差の上からの評価であった. 一方で次の結果は近似誤差を下から評価するものである. 
\begin{thm}[Maiorov and Meir 1999 \cite{MaiorovMeir2000}, Theorem5]　\\
  $I = (-1,1)$とし$p \in [1,\infty], m \geq 1, r \geq 2$とする. 活性化関数$\Psi:\R \rightarrow \R$が
  \[
  \Psi(t) = \frac{1}{1+\exp(-t)}
  \]
  であるとき, ある定数$C > 0$が存在して, 任意の$n \geq 1$に対して
  \[
  E(\mathcal{B}^{m,p}(I^r), \mathcal{N}_r^{n}(\Psi),L^p(I^r)) \geq C (n\log(n))^{-m/r}
  \]
  が成り立つ. 
\end{thm}

\subsection{Barron spaceとBarron-Maurey-Jones評価}
前節の結果は, Sobolev空間をニューラルネットで近似する際のレートの評価であった. そして,  Sobolev空間の近似誤差の上からの評価については, 求められる精度$\varepsilon$に対して中間ユニット数$n$を$\varepsilon^{-r/m}$以上に取る必要があり, 入力の次元$r$に関して指数関数的に増大することがわかる. これは次元の呪いとしばしば呼ばれる現象である. 
では, 逆に, この次元の呪いから解放される, つまり目標の精度を達成するために必要な中間ユニット数が$r$に関してより緩やかに増大する, あるいは$r$と無関係であるような関数空間はどのようなものであろうか？ 
このような関数空間はBarron spaceと呼ばれ, 現在も研究が続いている(Weinan et al.\cite{Weinan2019}, Siegel,Xu\cite{Siegel2020}). 
本節では, 「G-variation(後述)に関して有界な空間」がこの問題のひとつの答えを与えることを示す. これはKainen et al.\cite{Kainen2013}で紹介されている事実である. 

まず, Barron-Maurey-Jones評価と呼ばれる次の結果がある. 
\begin{defn}
  $X$を$\R$上の線形空間とするとき, $n \in \N$と部分集合$G \subset X$に対して, 
  \[
  \mathrm{conv}_n(G) := \left\{\sum_{i=1}^n a_i g_i \mid a_i \geq 0,~ \sum_{i=1}^n a_i = 1,~ g_i \in G \right\}
  \]
  と定める. また, 
  \[
  \mathrm{conv}(G) := \bigcup_{n = 1}^{\infty} \mathrm{conv}_n(G)
  \]
  と定める. $\mathrm{conv}(G)$は$G$を含む最小の凸集合である. 
\end{defn}
\begin{thm}[Barron-Maurey-Jones \cite{Kurkova2003}]\label{MaureyRate}　\\
  $X$を内積空間とする. $G \subset X$を空でない有界集合とし, $s_G := \sup_{g \in G}\lVert g \rVert_X$とおくと, 任意の$f \in \mathrm{cl}(\mathrm{conv}(G))$と$n \in \N$に対して, 
  \[
  \lVert f - \mathrm{conv}_n(G) \rVert_X \leq \sqrt{\frac{s_G^2 - \lVert f \rVert_X^2}{n}}
  \]
  が成り立つ. 
\end{thm}
\begin{proof}
  $X \ni f \mapsto \lVert f - G \rVert_X \in \R$の連続性より, $f \in \mathrm{conv}(G)$の場合に証明すれば十分である. そこで, 
  \[
  f = \sum_{i=1}^m a_i h_i ~~~~(a_i \geq 0, \sum_{i=1}^m a_i = 1, h_i \in G)
  \]
  と表す. すると, 三角不等式より
  \[
  \lVert f \rVert \leq \sum_{i=1}^m a_i \lVert h_i \rVert \leq \sum_{i=1}^m s_G = s_G
  \]
  である. したがって, 
  \[
  c := s_G^2 - \lVert f \rVert^2 \geq 0
  \]
  である. ここで, 以下の性質を持つ列$g_i \in G$を帰納的に構成しよう. 
  \[
  k \in \N, f_k := \sum_{i=1}^k \frac{g_i}{k} \Rightarrow e_k^2 := \lVert f - f_n \rVert_X^2 \leq \frac{c}{k}.
  \]
  このようなものが構成できれば定理が成立することは明らかである. さて, いま, 
  \[
  \begin{aligned}
  \sum_{j=1}^m a_j \lVert f - h_j \rVert^2 
  &= \sum_{j=1}^m a_j \lVert f \rVert^2 - 2 \ip<{f,\sum_{j=1}^m a_j h_j}> + \sum_{j=1}^m a_j \lVert h_j \rVert^2 \\
  &= \lVert f \rVert^2 - 2 \lVert f \rVert^2 + \sum_{j=1}^m a_j \lVert h_j \rVert^2 \\
  &\leq - \lVert f \rVert^2 + \sum_{j=1}^m a_j s_G^2 \\
  &= s_G^2 - \lVert f \rVert^2 = c
  \end{aligned}
  \]
  であるので, ある$j$について$\lVert f - h_j \rVert^2 \leq c$である.  そこで, $g_1 = h_j$とおく. 次に$g_1,\ldots,g_k \in G$が上記条件を満たすとする. $g_{n+1} \in G$を選ぶために$e_{n+1}^2$を計算すると, 
  \[
  \begin{aligned}
  e_{n+1}^2 
  &= \lVert f - f_{n+1} \rVert^2 = \lVert \frac{n}{n+1}(f-f_n) - \frac{1}{n+1}(f-g_{n+1}) \rVert^2 \\
  &= \frac{n^2}{(n+1)^2}e_n^2 + \frac{2n}{(n+1)^2}\ip<{f-f_n,f-g_{n+1}}> + \frac{1}{(n+1)^2}\lVert f - g_{n+1} \rVert^2
  \end{aligned}
  \]
  となる. 一方, 
  \[
  \begin{aligned}
  &~~~\sum_{j=1}^m a_j \left( \frac{2n}{(n+1)^2}\ip<{f-f_n,f-h_j}> + \frac{1}{(n+1)^2}\lVert f - h_j \rVert^2 \right) \\
  &= \frac{1}{(n+1)^2} \left( 2n \ip<{f-f_n, \sum_{j=1}^m a_j (f - h_j)}> + \sum_{j=1}^m a_j \lVert f - h_j \rVert^2 \right) \\
  &= \frac{1}{(n+1)^2} \left( 2n \ip<{f-f_n, f - f}> + \sum_{j=1}^m a_j \lVert f - h_j \rVert^2 \right) \\
  &\leq \frac{c}{(n+1)^2}
  \end{aligned}
  \]
  であるから, ある$j$について
  \[
  \frac{2n}{(n+1)^2}\ip<{f-f_n,f-h_j}> + \frac{1}{(n+1)^2}\lVert f - h_j \rVert^2 \leq \frac{c}{(n+1)^2}
  \]
  である. そこで$g_{n+1} = h_j$とおけば, 
  \[
  \begin{aligned}
  e_{n+1}^2 
  &= \frac{n^2}{(n+1)^2}e_n^2 + \frac{2n}{(n+1)^2}\ip<{f-f_n,f-g_{n+1}}> + \frac{1}{(n+1)^2}\lVert f - g_{n+1} \rVert^2 \\
  &\leq \frac{n^2}{(n+1)^2} \frac{c}{n} + \frac{c}{(n+1)^2} = \frac{c}{n+1}
  \end{aligned}
  \]
  となる. これで示せた. 
\end{proof}

次に, この結果の$L^p$空間における類似であるDarken-Donahue-Gurvits-Sontagによる評価を述べる. 証明には次の不等式を使う. 
\begin{thm}[Clarksonの不等式]　\\
  $(X,\mathcal{M},\mu)$を測度空間とする. $p \in (1,\infty), q = p/(p-1), a = \min\{p,q\}$とすると, 任意の$f,g \in L^p(X,\mathcal{M},\mu)$に対して, 
  \[
  \lVert f+g \rVert_{L^p(X)}^a + \lVert f-g \rVert_{L^p(X)}^a \leq 2 (\lVert f \rVert_{L^p(X)}^a + \lVert g \rVert_{L^p(X)}^a).
  \]
\end{thm}
\begin{proof}
  付録参照. 
\end{proof}
\begin{defn}
  $X$を$\R$上の線形空間とするとき, $n \in \N$と部分集合$G \subset X$に対して, 
  \[
  \mathrm{span}_n(G) := \left\{\sum_{i=1}^n a_i g_i \mid a_i \in \R,~ g_i \in G \right\}
  \]
  と定める.
\end{defn}
\begin{thm}[Darken-Donahue-Gurvits-Sontag 1993 \cite{ddgs-rarmrnn-93}]\label{DarkenRate}　\\
  $(X,\mathcal{M},\mu)$を測度空間とする. $G \subset L^p(X,\mathcal{M},\mu),~ p \in (1,\infty)$とし, $f \in \mathrm{cl}(\mathrm{conv}G)$とする. このとき, $r > 0$が存在して, $G \subset B(f;r) = \{h \in L^p(X,\mathcal{M},\mu) \mid \lVert f - h \rVert_{L^p(X)}<r \}$となるならば, 任意の整数$n \geq 1$について, 
  \[
  \lVert f - \mathrm{span}_n G \rVert_{L^p(X)} \leq \frac{2^{1/a} r}{n^{1/b}}
  \]
  となる. ここで, $q := p/(p-1), a :=\min(p,q), b:= \max(p,q)$である. 
\end{thm}
\begin{proof}
  $L^p(X,\mathcal{M},\mu) \ni h \mapsto \lVert h - \mathrm{span} G \rVert_{L^p(X)} \in \R$は連続であるから, $f \in \mathrm{conv} G$の場合を証明すれば十分である. $f = \sum_{j=1}^m w_j h_j$と表す. このとき, $G$の元の列$(g_i)$であって, $f_n = \sum_{i=1}^n g_i/n$とおくと, 
  \[
  e_n := \lVert f - f_n \rVert_{L^p(X)} \leq \frac{2^{1/a} r}{n^{1/b}}
  \]
  となるものが存在することを帰納法により示す. まず, 任意に$g \in G$をとり$g_1 = g$とする. すると, $G \subset B(f;r)$および$a>0$より$\lVert f - g_1 \rVert_{L^p(X)} \subset r \leq 2^{1/a} r$である. $g_1,\ldots,g_n \in G$が得られているとする. $f - f_n = 0$ならば$g_i = g_n ~~(i \geq n+1)$とすればよい. $f-f_n \neq 0$ならばHahn-Banachの拡張定理より線形汎関数$\Pi_n : L^p(X,\mathcal{M},\mu) \rightarrow \R$で$\lVert \Pi_n \rVert = 1$かつ$\Pi_n(f-f_n) = \lVert f-f_n \rVert_X$なるものが取れる(付録の系\ref{HahnBanachCor2}). $\sum_{j=1}^m w_j (f-h_j) = 0$であるから, 
  \[
  0 = \Pi_n\left( \sum_{j=1}^m w_j (f-h_j) \right) = \sum_{j=1}^m w_j \Pi_n (f-h_j)
  \]
  である. したがって, ある$j$について$\Pi_n(f-h_j) \leq 0$である. このような$j$をひとつとり$g_{n+1} = h_j$とおく.  
  \[
  f_{n+1} = \frac{1}{n+1}\sum_{i=1}^{n+1} g_i = \frac{n}{n+1}f_n + \frac{1}{n+1}g_{n+1}
  \]
  であるから, Clarksonの不等式より, 
  \[
  \begin{aligned}
  (e_{n+1})^a 
  &= \left\lVert f-f_{n+1} \right\rVert_{L^p(X)}^a 
  = \left\lVert \frac{n}{n+1}(f-f_n) + \frac{1}{n+1}(f-g_{n+1}) \right\rVert_{L^p(X)}^a \\
  &= 2 \left( \left\lVert \frac{n}{n+1}(f-f_n) \right\rVert_{L^p(X)}^a + \left\lVert \frac{1}{n+1}(f-g_{n+1} \right\rVert_{L^p(X)}^a 
  \right)\\
  &~~~ - \left\lVert \frac{n}{n+1}(f-f_n) - \frac{1}{n+1}(f-g_{n+1}) \right\rVert_{L^p(X)}^a
  \end{aligned}
  \]
  となる. 一方, いま, $\lVert \Pi_n \rVert = 1$かつ$\Pi_n (f-g_{n+1}) \leq 0$であるから, 
  \[
  \begin{aligned}
  &\left\lVert \frac{n}{n+1}(f-f_n) - \frac{1}{n+1}(f-g_{n+1}) \right\rVert_{L^p(X)} \\
  &\geq \left\lvert \Pi_n \left( \frac{n}{n+1}(f-f_n) - \frac{1}{n+1}(f-g_{n+1}) \right) \right\rvert \\
  &=\left\lvert \frac{n}{n+1}\lVert f-f_n \rVert_{L^p(X)} - \frac{1}{n+1}\Pi_n(f-g_{n+1}) \right\rvert \\
  &\geq \frac{n}{n+1}\lVert f-f_n \rVert_{L^p(X)}
  \end{aligned}
  \]
  である. よって, このことと$g_{n+1} \in G$および帰納法の仮定より, 
  \[
  \begin{aligned}
  (e_{n+1})^a 
  &\leq 2 \left( \left\lVert \frac{n}{n+1}(f-f_n) \right\rVert_{L^p(X)}^a + \left\lVert \frac{1}{n+1}(f-g_{n+1} \right\rVert_{L^p(X)}^a 
  \right) \\
  &~~~ - \left( \frac{n}{n+1}\lVert f-f_n \rVert_{L^p(X)} \right)^a \\
  &= \left(\frac{n}{n+1}\right)^a \lVert f-f_n \rVert_{L^p(X)}^a + \frac{2}{(n+1)^a}\lVert f-g_{n+1} \rVert_{L^p(X)}^a \\
  &= \left(\frac{n}{n+1}\right)^a (e_n)^a + \frac{2}{(n+1)^a}\lVert f-g_{n+1} \rVert_{L^p(X)}^a \\
  &\leq \left(\frac{n}{n+1}\right)^a \left( \frac{2^{1/a} r}{n^{1/b}} \right)^a + \frac{2 r^a}{(n+1)^a} = \frac{2 r^a}{(n+1)^a}\left( 1+\frac{n^a}{n^{a/b}} \right)
  \end{aligned}
  \]
  となる. $1/a + 1/b = 1$より$a-a/b = 1$であるから, 
  \[
  (e_{n+1})^a \leq \frac{2 r^a}{(n+1)^a}( 1+n^{a-a/b} ) = \frac{2 r^a}{(n+1)^{a-1}} = \frac{2 r^a}{(n+1)^{a/b}}
  \]
  である. これで示せた. 
\end{proof}
以上のふたつの結果はいずれも$f \in \mathrm{cl}(\mathrm{conv} G)$に対する評価であり, 一般の$f \in X$については適用されない. そこで$X$の元すべてに対して適用できるような評価を与えよう. そのために次の概念を導入する. 
\begin{defn}[$G$-variation]　\\
  $X$をノルム空間とする. 空でない$G \subset X$と$f \in X$に対して, 
  \[
  \lVert f \rVert_{G} := \inf\{c > 0 \mid f/c \in \mathrm{cl}(\mathrm{conv}(G \cup -G))\}
  \]
  と定義する. ここで, $-G = \{-g \mid g \in G\}$である. $\lVert f \rVert_{G}$を$f$の$G$-variationという. 
\end{defn}
\begin{rem}
  $G \subset X$を空でない有界集合とし, $s_G := \inf_{g \in G} \lVert g \rVert_X$とおくと, $\lVert f \rVert_X \leq s_G \lVert f \rVert_G$となる. 実際, $b > 0$で$f/b \in \mathrm{cl}(\mathrm{conv}(G \cup -G))$なるものを任意に取ると, $f/b = \lim_{n \rightarrow \infty} h_n$なる$h_n \in \mathrm{conv}(G \cup -G)$が取れる. すると, $\lVert h_n \rVert_X \leq s_G$であるので$n \rightarrow \infty$として$\lVert f \rVert_X \leq b s_G $となる. よって, $b$の任意性から$\lVert f \rVert_X \leq s_G \lVert f \rVert_G$である. 
\end{rem}
\begin{lem}
  $X$を$\R$上のノルム空間とし, $C \subset X$は$0$を含む凸集合とし, $m$を$C$のMinkowski汎関数, つまり$x \in X$に対して$m(x) := \inf\{c > 0 \mid x/c \in C\}$とすると, 任意の$x,y \in X,~\lambda \geq 0$に対して$m(\lambda x) = \lambda m(x)$, $m(x+y) \leq m(x) + m(y)$であり, 
  \[
  \{ x \in X \mid m(x) < 1 \} \subset C
  \]
  となる. さらに$-C = C$であるならば, $\lambda \in \R$に対して$m(\lambda x) = |\lambda|m(x)$となる. 特に, $G$-variationは以上で述べた性質をすべて持ち, 
  \[
  \{f \in X \mid \lVert f \rVert_{G} < 1\} \subset \mathrm{cl}(\mathrm{conv}(G \cup -G))
  \]
  となる. 
\end{lem}
\begin{proof}
  まず$\lambda > 0$に対して, 
  \[
  \begin{aligned}
  \lambda m(x) 
  &= \inf\{\lambda c \mid c > 0,  x/c \in C\}  \\
  &= \inf\{ \lambda c \mid (\lambda c) > 0, \lambda x/(\lambda c) \in C \} \\
  &= \inf\{ b \mid b> 0, \lambda x/b \in X \} = m(\lambda x)
  \end{aligned}
  \]
  である. また, $\lambda = 0$のときは$m(0)=0$より$m(\lambda x) = \lambda m(x)$である. そして, $m(x) = +\infty$または$m(y) = +\infty$の場合は明らかに$m(x+y) \leq m(x) + m(y)$であり, そうでない場合は$x/c, x/b \in C$なる任意の$c,b \in (0,\infty)$に対し, $C$ が凸集合であることから, 
  \[
  \frac{1}{c+b}(x+y)=\frac{c}{c+b}\frac{1}{c}x+\frac{b}{c+b}\frac{1}{b} y\in C
  \]
  である. よって$m(x+y)\leq c+b$であり, $c,b$ の任意性より$m(x+y)\leq m(x)+m(y)$となる. 
  次に$m(x) < 1$ならば$m$の定義より$c \in (0,1)$で$x/c \in C$なるものが取れる. すると, $0 \in C$と$C$の凸性より, 
  \[
  x = c \frac{x}{c} + (1-c) 0 \in C
  \]
  である. 最後に, $-C=C$ならば$\lambda \in \R$と$b > 0$に対して, 
  \[
  \frac{\lambda x}{b} \in C \Longleftrightarrow \frac{|\lambda| x}{b} \in C
  \]
  であるので$m(\lambda x) = |\lambda|m(x)$となる. 
\end{proof}
\begin{rem}
  Minkowski汎関数という言葉を用いれば$G$-variationは$\mathrm{cl}(\mathrm{conv}(G \cup -G))$のMinkowski汎関数である. そして上の補題より$G$-variationが有限な$X$の元全体の集合は$X$の部分空間をなし, $G$-variationはその上の半ノルムとなる. また$G$が$X$のノルムに関して有界ならその上のノルムとなる. 
\end{rem}
\begin{lem}
  $X$をノルム空間とし, $\emptyset \neq G \subset X$とする. $c \geq 0$に対して$G(c) := \{wg \in \mid |w| \leq c, g \in G\}$とおくと, $\lVert f \rVert_G < \infty$なる$f \in X$に対して, 
  \[
  f \in \mathrm{cl}(\mathrm{conv}(G(\lVert f \rVert_{G})))
  \]
  となる. 
\end{lem}
\begin{proof}
  すぐ上の補題より, 任意の$\varepsilon > 0$に対して, 
  \[
  \frac{f}{\lVert f \rVert_V + \varepsilon} \in \mathrm{cl}(\mathrm{conv}(G \cup -G))
  \]
  である. ゆえに, $\varepsilon \rightarrow 0$とすると$f/\lVert f \rVert_G \in \mathrm{cl}(\mathrm{conv}(G \cup -G))$である. よって, 
  \[
  f \in \lVert f \rVert_G \mathrm{cl}(\mathrm{conv}(G \cup -G)) = \mathrm{cl} \left(\lVert f \rVert_G \mathrm{conv}(G \cup -G)\right) \subset \mathrm{cl} (\mathrm{conv} G(\lVert f \rVert_G)).
  \]
\end{proof}
\begin{thm}
  $X$をノルム空間とする. また, $G \subset X$を空でない有界集合とし, $s_G := \sup_{g \in G} \lVert g \rVert_X$とおく. このとき, $f \in X$と整数$n \geq 1$に対して, 
  \begin{itemize}
      \item[(1)] $X$が内積空間ならば, 
        \[
        \lVert f - \mathrm{span}_n G \rVert_X \leq \sqrt{\frac{s_G^2 \lVert f \rVert_G^2 - \lVert f \rVert_X^2}{n}}
        \]
        となる. 特に, $B_R(\lVert \cdot \rVert_G) := \{g \in X \mid \mid \lVert g \rVert_G \leq R\}$に対して, 
        \[
        E(B_R(\lVert \cdot \rVert_G), \mathrm{span}_n G, X) \leq s_G \frac{R}{n^{1/2}}.
        \]
      \item[(2)] 測度空間$(Y,\mathcal{N},\nu)$と$p \in (1,\infty)$に対して$X = L^p(Y,\mathcal{N},\nu)$であるとき, 
        \[
        \lVert f - \mathrm{span}_n G \rVert_X \leq \frac{2^{1+1/a} s_G \lVert f \rVert_G}{n^{1/b}}
        \]
        となる. ここで, $a := \min(p,p/(p-1)), b:= \max(p,p/(p-1))$である. 特に, 
        \[
        E(B_R(\lVert \cdot \rVert_G), \mathrm{span}_n G, L^p(Y,\nu)) \leq 2^{1+1/a} s_G \frac{R}{n^{1/b}}.
        \]
  \end{itemize}
\end{thm}
\begin{proof}
  $c \geq 0$に対して$G(c) := \{wg \in \mid |w| \leq c,~ g \in G\}$とおく. このとき, $\mathrm{conv}_n G(c) \subset \mathrm{span}_n G(c) = \mathrm{span}_n G$である. したがって, $f \in X$に対して, 
  \[
  \lVert f - \mathrm{span}_n G \rVert_X \leq \lVert f - \mathrm{conv}_n G(\lVert f \rVert_G ) \rVert_X
  \]
  である. 一方, すぐ上の補題より, $\lVert f \rVert_G < \infty$のとき$f \in \mathrm{cl}(\mathrm{conv}(G(\lVert f \rVert_{G})))$である. また, $\lVert f \rVert_X \leq s_G \lVert f \rVert_G$より$G(\lVert f \rVert_G ) \subset B(f,2s_G \lVert f \rVert_G)$である. よって, 定理\ref{MaureyRate}と定理\ref{DarkenRate}より主張は成り立つ. 
\end{proof}
以上の結果は直接的にニューラルネットによる近似レートを述べたものではないが, 活性化関数$\Psi:\R \rightarrow \R$および$W \subset \R^r, B \subset \R$に対して$G = G_r(\Psi,W,B) := \{ x \mapsto \Psi(w^{\T} x + b) \mid w \in W, b \in B \}$とおくとき, $\Psi, W, B$や問題とする関数空間に適当な条件を加えれば定理の仮定($G$が有界)をみたすようにできる. また, $b_j \in B,c_j \in \R, w_j \in W$に対して, 
\[
f(x) = \sum_{j=1}^m c_j \Psi(w_j^{\T} x + b_j) ~~~(x \in \R^r)
\]
とおくと, 
\[
\begin{aligned}
f(x) 
&= \sum_{c_j > 0} c_j \Psi(w_j^{\T} x + b_j) + \sum_{c_j \leq 0} c_j \Psi(w_j^{\T} x + b_j) \\
&= \sum_{c_j > 0} |c_j| \Psi(w_j^{\T} x + b_j) + \sum_{c_j \leq 0} |c_j| (-\Psi(w_j^{\T} x + b_j))
\end{aligned}
\]
であるので, 
\[
\lVert f \rVert_G = \inf\{c > 0 \mid f/c \in \mathrm{cl}(\mathrm{conv}(G \cup -G))\} \leq \sum_{j=1}^m |c_j|
\]
であることを注意しておく. 

次に近似誤差の下限に関する定理を紹介しよう. 
\begin{thm}[Makovoz 1996 \cite{Makovoz1996} theorem4]　\\
  sigmoidal関数$\Psi:\R \rightarrow \R$は$\Psi = \chi_{[0,\infty)}$であるか, Lipschitz連続且つある$\eta,C>0$が存在して
  \[
  |\Psi(t) - \chi_{[0,\infty)}(t)| \leq C|t|^{-\eta} ~~~~(\forall t \in \R\setminus\{0\})
  \]
  となるとする. このとき, $G := G_r(\Psi,\R^r,\R)$とおけば, 
  開凸集合$D \subset \R^r$と$\xi > 0$に対して定数$C(D,\xi) > 0$が存在して, 任意の整数$n \geq 1$に対して
  \[
  E(B_1(\lVert \cdot \rVert_G),\mathrm{span}_n G,L^2(D)) \geq C(D,\xi) n^{-1/2 - 1/r - \xi} .
  \]
\end{thm}

さて, 以上の結果は冪レートの結果であった. 最後に指数レートの結果を述べよう. 
\begin{thm}[K\r{u}rkov\'{a}-Sanguineti 2008 \cite{KurkovaSanguineti2008}]　\\
  $X$をHilbert空間とする. また, $G \subset X$を空でない有界集合とし, $s_G = \sup_{g\in G} \lVert g \rVert_X$とおく. このとき, 任意の$f \in \mathrm{conv} G$に対して, $\tau_f \in [0,1)$が存在して, 任意の整数$n \geq 1$について
  \[
  \lVert f - \mathrm{conv}_n G \rVert_X \leq \sqrt{\tau_f^{n-1}(s_G^2 - \lVert f \rVert_X^2)}
  \]
  が成り立つ. 
\end{thm}
\begin{proof}
  $\sum_{j=1}^m a_j = 1, ~a_j > 0, ~g_j \in G$により$f = \sum_{j=1}^m a_j g_j$と表す. $m=1$なら$\tau_f = 0$として成立するので$m > 1$としてよい. また, ある$j$について$f = g_j$となる場合も除いてよい. $G' := \{g_1,\ldots,g_m\}$とおく. 各$n = 1,\ldots,m$に対して$f_n \in \mathrm{conv}_n G', \rho_n \in (0,1)$が存在して
  \[
  \lVert f - f_n \rVert_X^2 \leq (1-\rho_n^2)^{n-1}(s_G^2 - \lVert f \rVert_X^2)
  \]
  が成り立つことを$n$に関する帰納法で示す. まず, $g_{j_1} \in G'$を
  \[
  \lVert f - g_{j_1} \rVert_X = \min_{g \in G'} \lVert f - g \rVert_X
  \]
  なるものとし$f = g_{j_1}$とおけば, 
  \[
  \begin{aligned}
  \lVert f-f_1 \rVert_X^2 
  &= \sum_{j=1}^m a_j \lVert f-f_1 \rVert_X^2 \leq \sum_{j=1}^m a_j \lVert f-g_j \rVert_X^2 \\
  &= \lVert f \rVert_X^2 - 2 \ip<{f,\sum_{j=1}^m a_j g_j}> + \sum_{j=1}^m a_j \lVert g_j \rVert_X^2 \\
  &= \lVert f \rVert_X^2 - 2 \lVert f \rVert_X^2 + \sum_{j=1}^m a_j \lVert g_j \rVert_X^2 \\
  &\leq s_G^2 - \lVert f \rVert_X^2
  \end{aligned}
  \]
  となる. よって, $\rho_1 \in (0,1)$を任意に取れば$n=1$の場合は成り立つ. 次に$n-1$で成り立つとする. $f_{n-1}=f$ならば$f_n = f_{n-1}, \rho_n = \rho_{n-1}$と取ればよいので$f_{n-1} \neq f$とする. このとき, 
  \[
  \sum_{j=1}^m a_j \ip<{f-f_{n-1},f-g_j}> = \ip<{f-f_{n-1},f-\sum_{j=1}^m a_j g_j}> = \ip<{f-f_{n-1},f-f}> = 0
  \]
  であるので, 次のいずれかが成り立つ($m>1$に注意). 
  \begin{itemize}
      \item[(1)] $\exists g \in G' , \ip<{f-f_{n-1},f-g}> < 0.$
      \item[(2)] $\forall g \in G', \ip<{f-f_{n-1},f-g}> < 0.$
  \end{itemize}
  (2)はあり得ないことを示そう. $f_{n-1} \in \mathrm{conv}_{n-1} G'$であるので, 
  \[
  f_{n-1} = \sum_{k=1}^{n-1} b_k g_{j_k} ~~~~(\sum_{k=1}^{n-1}b_k=1,~b_k > 0, g_{j_1} \in G')
  \]
  と表される. 従って, (2)が成り立つならば, 
  \[
  \begin{aligned}
  \lVert f - f_{n-1} \rVert_X^2
  &= \ip<{f-f_{n-1},f-\sum_{k=1}^{n-1}b_k g_{j_k}}> \\
  &= \sum_{k=1}^{n-1} b_k \ip<{f-f_{n-1},f-g_{j_k}}> = 0
  \end{aligned}
  \]
  であるので$f=f_{n-1}$となる. これは$f \neq f_{n-1}$に反する. よって, (2)はあり得ない. ゆえに(1)が成り立つ. そこで$g_{j_n} \in G'$を$\ip<{f-f_{n-1},f-g_{j_n}}> < 0$を満たすものとする. そして, $f \neq f_{n-1}, g_{j_n}$に注意して
  \[
  \rho_n := \min \left\{\rho_{n-1},- \frac{\ip<{f-f_{n-1},f-g_{j_n}}>}{\lVert f - f_{n-1} \rVert_X \lVert f - g_{j_n} \rVert_X} \right\} \in (0,1)
  \]
  とおく. また, $\alpha_n \in [0,1]$に対して
  \[
  f_n = \alpha_n f_{n-1} + (1-\alpha_n)g_{j_n} \in \mathrm{conv}_n G'
  \]
  とおく. 以下, $\alpha_n$をうまく取ることで$f_n,\rho_n$が条件を満たすようにできることを示す. $q_n = -\ip<{f-f_{n-1},f-g_{j_n}}> > 0$, $r_n =\lVert f-g_{j_n} \rVert_X > 0$とおき, 各$k = 1,\ldots,n$について$e_k := \lVert f - f_k \rVert_X$とおくと, 
  \[
  \begin{aligned}
  e_n^2 
  &= \lVert \alpha_n (f-f_{n-1}) + (1-\alpha_n)(f-g_{j_n}) \rVert_X^2 \\
  &= \alpha_n^2 e_{n-1}^2 + 2\alpha_n(1-\alpha_n)\ip<{f-f_{n-1},f-g_{j_n}}> + (1-\alpha_n)^2 \lVert f-g_{j_n} \rVert_X^2 \\
  &= \alpha_n^2 e_{n-1}^2 - 2\alpha_n(1-\alpha_n)q_n + (1-\alpha_n)^2 r_n^2 \\
  &=\alpha_n^2(e_{n-1}^2 + 2 q_n + r_n^2) - 2\alpha_n(q_n + r_n^2) + r_n^2
  \end{aligned}
  \]
  となる. 従って, 
  \[
  \frac{d e_n^2}{d\alpha_n} = 2\alpha_n(e_{n-1}^2 + 2 q_n + r_n^2) - 2(q_n + r_n^2)
  \]
  である. よって, $e_{n-1}^2 + 2 q_n + r_n^2 > 0$より$e_n^2$は
  \[
  \alpha_n = \frac{q_n + r_n^2}{e_{n-1}^2 + 2 q_n + r_n^2} \in [0,1]
  \]
  で最小値をとる. この$\alpha_n$が求めるものである. 実際, そのとき
  \[
  \begin{aligned}
  e_n^2 
  &= \frac{(q_n + r_n^2)^2}{e_{n-1}^2 + 2 q_n + r_n^2} -2\frac{(q_n + r_n^2)(q_n + r_n^2)}{e_{n-1}^2 + 2 q_n + r_n^2} + r_n^2 \\
  &= \frac{-q_n^2 -2q_n r_n^2 - r_n^4 + e_{n-1}^2 r_n^2 + 2 q_n r_n^2 + r_n^4}{e_{n-1}^2 + 2 q_n + r_n^2} \\
  &= \frac{-q_n^2 + e_{n-1}^2 r_n^2}{e_{n-1}^2 + 2 q_n + r_n^2}
  \end{aligned}
  \]
  であり, $\rho_n = \min \left\{\rho_{n-1}, q_n/(e_{n-1}r_n) \right\}
  $より$\rho_n e_{n-1} r_n \leq q_n$であるので, 
  \[
  \begin{aligned}
  e_n^2 
  &\leq \frac{-\rho_n^2 e_{n-1}^2 r_n^2 + e_{n-1}^2 r_n^2}{e_{n-1}^2 + 2 q_n + r_n^2} < \frac{(1-\rho_n^2) e_{n-1}^2 r_n^2}{r_n^2} = (1-\rho_n^2) e_{n-1}^2 \\
  &\leq (1-\rho_n^2) (1-\rho_{n-1}^2)^{n-2} (s_G^2 - \lVert f \rVert_X^2) \leq (1-\rho_n^2)^{n-1} (s_G^2 - \lVert f \rVert_X^2)
  \end{aligned}
  \]
  となる(最後の不等号は$\rho_n \leq \rho_{n-1}$による). よって$n$で成立する. \\
  さて, $\rho_f := \min\{\rho_n \mid n=1,\ldots,m\}$とおく. すると, $n=1,\ldots,m$に対して
  \[
  \lVert f - \mathrm{conv}_n G \rVert_X^2 \leq \lVert f - f_n \rVert_X^2 \leq (1-\rho_n)^{n-1} (s_G^2 - \lVert f \rVert_X^2) \leq (1-\rho_f)^{n-1} (s_G^2 - \lVert f \rVert_X^2)
  \]
  となり, $n > m$に対しては$\lVert f - \mathrm{conv}_n G \rVert_X = 0$となるので$\tau_f = 1- \rho_f^2$が求める定数である. 
\end{proof}

\section{まとめと今後の課題}
\subsubsection*{まとめ}
活性化関数が非多項式かつ局所有界でRiemann積分可能の場合にfeedforward型ニューラルネットワークが万能近似能力を持つなど, ニューラルネットワークの万能近似定理について, 複数の結果を記号を統一して解説した. 
また, その近似レートの問題において, 近似する空間としてSobolev空間を取った場合を扱った. 
さらに, その場合には次元の呪いという, 所望の精度を達成するために必要な中間ユニット数が入力の次元に関して指数関数的に増大する現象が発生することを説明した. 
そして, 次元の呪いから解放される空間としてBarron spaceの概念を導入し, Barron spaceのひとつとして「G-variationに関して有界な空間」を紹介した. 

\subsubsection*{今後の課題}
本論文で触れることができなかったことに, Deep narrow networkの万能近似定理\cite{DeepNarrowNetworks}や深層ニューラルネットワークの近似レート(Yarotsky 2016 \cite{Yarotsky}など)がある. 
また, Barron spaceに関しては, Barron自身の結果やごく最近のWeinan et al.\cite{Weinan2019},  Siegel, Xu\cite{Siegel2020}の結果をまとめることができなかった. 
今後は, こうした深層ニューラルネットワークやBarron spaceの近年の進展をまとめたい. 
さらに, 近年注目を集めている, 日本の研究者が創始したニューラルネットワークの積分表現理論(Murata 1996 \cite{NoboruMurata}やSonoda, Murata 2015 \cite{SonodaMurata}など)の解説を試みたい.

\section*{謝辞}
本論文の作成を支援してくださった九州大学マス・フォア・インダストリ研究所の溝口佳寛教授に深く感謝いたします. 
また, 本論文のいくつかの命題の証明の完成にご協力頂いた, 
片岡佑太さん, 九州大学大学院数理学研究院の高田了准教授, 慶應義塾大学経済学部の服部哲弥教授に厚く御礼申し上げます. 
最後になりますが, 家族や友人, 並びに論文作成をご支援いただいた全ての方々に感謝の意を表します.

\section{付録}

\subsection{測度の基本性質}
本節では, 測度の基本性質を証明する. まず, 可測空間, 測度の定義を述べよう. 
\begin{defn}[可測空間]　\\
  集合$X$と$X$の部分集合系$\mathcal{M}$の組$(X, \mathcal{M})$が可測空間であるとは, 
  \begin{itemize}
      \item[(1)] $\emptyset \in \mathcal{M}$.
      \item[(2)] $\forall E \in \mathcal{M}, X \setminus E \in \mathcal{M}$.
      \item[(3)] 任意の$\mathcal{M}$の元の列$(E_n)_{n\in \N}$について$\bigcup_{n \in\N} E_n \in \mathcal{M}$.
  \end{itemize}
  が成り立つことをいう. 
\end{defn}
\begin{rem}
  $(X, \mathcal{M})$を可測空間とすると,  $\mathcal{M}$は可算個までの共通部分をとる操作に関して閉じている. 実際, $(E_n)_{n\in \N}$を$\mathcal{M}$の元の列とすると, 
  \[
  \bigcup_{n=1}^{\infty} E_n = X \setminus \left(\bigcup_{n=1}^{\infty} (X \setminus E_n) \right) \in \mathcal{M}.
  \]
\end{rem}
\begin{defn}[測度, 測度空間]　\\
  $(X, \mathcal{M})$を可測空間とするとき, 写像$\mu:\mathcal{M} \rightarrow [0,\infty]$が$(X, \mathcal{M})$上の測度であるとは, 
  \begin{itemize}
      \item[(1)] $\mu(\emptyset) = 0$.
      \item[(2)] $\mathcal{M}$の元の列$(E_n)_{n\in \N}$が$E_n \cap E_m = \emptyset ~(n \neq m)$を満たすならば, 
      \[
      \mu\left( \bigcup_{n\in \N} E_n \right) = \sum_{n=1}^{\infty} \mu(E_n).
      \]
  \end{itemize}
  が成り立つことをいう. このとき, 三つ組$(X,\mathcal{M},\mu)$を測度空間という. 
\end{defn}
\begin{lem}
  $(X, \mathcal{M}, \mu)$を測度空間とするとき, 以下が成り立つ. 
  \begin{itemize}
      \item[(1)]$A,B \in \mathcal{M}$が$B \subset A$を満たすとき, $\mu(B) \leq \mu(A).$
      \item[(2)]$A,B \in \mathcal{M}$が$B \subset A$を満たすとき$\mu(B) < \infty$ならば$\mu(A \setminus B ) = \mu(A) - \mu(B).$
      \item[(3)]$\mathcal{M}$の元の列$(A_n)_{n=1}^{\infty}$に対して, 
      \[
      \mu\left( \bigcup_{n=1}^{\infty} A_n \right) \leq \sum_{n=1}^{\infty} \mu(A_n).
      \]
      \item[(4)]$\mathcal{M}$の元の列$(A_n)_{n=1}^{\infty}$が$A_n \subset A_{n+1} ~~(\forall n)$を満たすならば, 
      \[
      \lim_{n \rightarrow \infty} \mu(A_n) =  \mu\left(\bigcup_{n=1}^{\infty} A_n \right).
      \]
      \item[(5)]$\mathcal{M}$の元の列$(A_n)_{n=1}^{\infty}$が$A_n \supset A_{n+1} ~~(\forall n)$を満たすとき$\mu(A_1) < \infty$ならば
      \[
      \lim_{n \rightarrow \infty} \mu(A_n) =  \mu\left(\bigcap_{n=1}^{\infty} A_n \right).
      \]
  \end{itemize}
  $(1)$を測度の単調性, $(3)$を測度の劣加法性, $(4),(5)$を測度の連続性という. 
\end{lem}
\begin{proof}　\\
  $(1)$: $\mu(A) = \mu((A \setminus B) \cup B) = \mu(A \setminus B) + \mu(B) \geq \mu(B).$\\
  $(2)$: $(1)$の証明において$\mu(B)$を両辺から引けばよい. \\
  $(3)$: $B_1 = A_1$とし, $B_{n+1} = A_{n+1} \setminus \bigcup_{j=1}^{n} A_j$とする. このとき, 
  \[
  \bigcup_{n=1}^{\infty} B_n = \bigcup_{n=1}^{\infty} A_n
  \]
  であり, $n \neq m$ならば$B_n \cap B_m = \emptyset$である. したがって, 
  \[
  \begin{aligned}
  \mu\left( \bigcup_{n=1}^{\infty} A_n \right) 
  &= \mu\left( \bigcup_{n=1}^{\infty} B_n \right) \\
  &= \sum_{n=1}^{\infty} \mu(B_n) \\
  &\leq \sum_{n=1}^{\infty} \mu(A_n) ~~~(\because (1)).
  \end{aligned}
  \]
  \\
  $(4)$:$A_0 = \emptyset$とし, $B_n := A_n - A_{n-1}$とおく. すると, $B_ \in \mathcal{M}$で, $n \neq m$に対して$B_n \cap B_m = \emptyset$となる. さらに, $\bigcup_{n=1}^{\infty} A_n = \bigcup_{n=1}^{\infty} B_n$である. したがって, 測度の完全加法性より, 
  \[
  \begin{aligned}
  \mu(\bigcup_{n=1}^{\infty} A_n) 
  &= \mu(\bigcup_{n=1}^{\infty} B_n) \\
  &= \sum_{n=1}^{\infty} \mu(B_n) \\
  &= \sum_{n=1}^{\infty} (\mu(A_n) - \mu(A_{n-1})) ~~~(\because (2)) \\
  &= \lim_{n \rightarrow \infty} (\mu(A_n) - \mu(A_0)) \\
  &= \lim_{n \rightarrow \infty} \mu(A_n). 
  \end{aligned}
  \]
  \\
  $(5)$: $(\mu(A_n))_{n=1}^{\infty}$は非負な単調減少列なので極限が存在すること, また, $\mu(A_1) < \infty$より$\mu( \bigcap_{j=1}^{\infty} A_j) \leq \mu(A_n) < \infty ~(\forall n)$であることに注意する. 
  いま, $C_n := X \setminus A_n$とおくと, $C_n \in \mathcal{M}$であり, $C_n \subset C_{n+1}~(\forall n)$であるから, 
  \[
  \begin{aligned}
  \mu(X) - \mu\left( \bigcap_{n=1}^{\infty} A_n \right)
  &=\mu\left( X \setminus \bigcap_{n=1}^{\infty} A_n \right) ~~~(\because (2))\\
  &= \mu\left(\bigcup_{n=1}^{\infty} X \setminus A_n \right) \\
  &= \lim_{n \rightarrow \infty} \mu(X \setminus A_n) ~~~(\because (4))\\
  &= \lim_{n \rightarrow \infty} (\mu(X) - \mu(A_n)) ~~~(\because (2))\\
  &= \mu(X) - \lim_{n \rightarrow \infty} \mu(A_n)
  \end{aligned}
  \]
  となる. よって, 求める等式を得る. 
\end{proof}

\subsection{単関数近似定理}

本節では非負値可測関数が単調増加な非負単関数列で近似できることを示す. ここで可測空間$(X,\mathcal{M})$上の単関数とは, 
\[
f = \sum_{j=1}^m c_j \chi_{E_j} ~~~~(m \in \N, c_j \in \R, E_j \in \mathcal{M} )
\]
と表される関数のことをいう. ただし$\chi_E$で$E$の定義関数を表す. 
\begin{thm}[単関数近似定理]　\\
  $f \geq 0$を可測空間$(X,\mathcal{M})$上の可測関数とする. このとき, $n = 1,2,\ldots$に対して, 
  \[
  \begin{aligned}
  &A_k^n := f^{-1}([(k-1)/2^n, k/2^n)) ~~(k=1,\ldots,2^n n), \\
  &A_0^n := f^{-1}([n,\infty)), \\
  &\varphi_n := \sum_{k=1}^{2^n n} \frac{k-1}{2^n} \chi_{A_k^n} + n\chi_{A_0^n} 
  \end{aligned}
  \]
  と定義すると, 任意の$x \in X$に対して, 
  \[
  \begin{aligned}
  &0 \leq \varphi_1(x) \leq \varphi_2(x) \leq \cdots \leq \varphi_n(x) \leq \cdots \leq f(x),  \\
  &f(x) = \lim_{n \rightarrow \infty} \varphi_n(x)
  \end{aligned}
  \]
  が成り立つ. 
\end{thm}
\begin{proof}
  任意に$x \in X$を取る. また, 任意に$n$を取る. このとき, $\varphi_n(x) \leq \varphi_{n+1}(x)$を示そう. まず$f(x) \geq n+1$の場合は, $\varphi_n(x)=n \leq n+1 = \varphi_{n+1}(x)$で成立する. 次に$f(x) < n$のときは, $f(x) \in [(k-1)/2^n, k/2^n)$なる$k \in \{1,\ldots, 2^n n\}$が取れる. すると, 
  \[
  \frac{2k-2)}{2^{n+1}} \leq f(x) \leq \frac{2k-1}{2^{n+1}}
  \]
  または, 
  \[
  \frac{2k-1)}{2^{n+1}} \leq f(x) \leq \frac{2k}{2^{n+1}}
  \]  
  である. 前者なら$\varphi_n(x) = (k-1)/2^n = 2(k-1)/2^{n+1} = \varphi_{n+1}(x)$で成立する. 後者なら$\varphi_n(x) = (k-1)/2^n = 2(k-1)/2^{n+1} \leq (2k-1)/2^{n+1}$で成立する. 最後に$n \leq f(x) < n+1$の場合は, $(k-1)/2^{n+1} \leq f(x) < k/2^{n+1}$なる$k \in \{1+2^{n+1}n, \ldots, 2^{n+1}(n+1)\}$が取れる. したがって, 
  \[
  \varphi_n(x) = n \leq (k-1)/2^{n+1} = \varphi_{n+1}(x)
  \]
  となって成立する. 次に$\varphi_n(x) \rightarrow f(x)$を示す. 任意の$\varepsilon > 0$に対して, $f(x) > N$かつ$1/2^{N} < \varepsilon$なる$N$が取れる. このとき, $n \geq n$に対して, 
  \[
  \frac{k-1}{2^n} \leq f(x) < \frac{k}{2^n}
  \]
  なる$k \in \{1,\ldots, 2^n n\}$がとれ, $\varphi_n(x) = (k-1)/2^n$となるので, 
  \[
  0 \leq f(x) - \varphi_n(x) \leq \frac{1}{2^n} < \frac{1}{2^N} < \varepsilon
  \]
  となる. よって$\varphi_n(x) \rightarrow f(x)$である. 同時に$\varphi_n(x) \leq f(x)$も示された. 
\end{proof}

\subsection{Lebesgue積分の基本定理}
本節ではLebesgue積分論の基本定理を説明する. 詳しい解説や証明は猪狩\cite{igarizitukaiseki}などを参照されたい.  
まず積分の定義を述べる. 
\begin{defn}
  $(X,\mathcal{M},\mu)$を測度空間とする. 非負単関数$f:X \rightarrow [0,\infty)$, つまり$a_1,\ldots,a_k \geq 0$と$E_1,\ldots,E_k \in \mathcal{M}$により$f = \sum_{j=1}^k a_j \chi_{E_j}$と表される$f$に対しては
  \[
  \int_X f d\mu = \int_X f(x) d\mu(x) := \sum_{j=1}^k a_j \mu(E_j)
  \]
  と定義する. 非負の可測関数$f:X \rightarrow [0,\infty)$に対しては
  \[
  \int_X f(x) d\mu(x) := \sup\left\{ \int_X \varphi d\mu ~\vline~ 0 \leq \varphi \leq f , \varphi\mbox{は単関数} \right\}
  \]
  と定義する. 可測関数$f:X \rightarrow \R$に対しては$\int_X |f(x)| d\mu(x) < \infty$である場合に
  \[
  \int_X f(x) d\mu(x) := \int_X f^{+}(x) d\mu(x) - \int_X f^{-}(x) d\mu(x)
  \]
  と定義する. ここで$f^{+} = \max(0,f),~ f^{-} = \max(0,-f)$である. 
\end{defn}
次に収束定理について述べる. 
\begin{thm}[単調収束定理]　\\
  $(\Omega, \mathcal{F}, \mu)$を測度空間とする. $(f_n)$は$(\Omega, \mathcal{F}, \mu)$上の非負値可測関数列で関数$f:\Omega \rightarrow \R$に対して
  \[
  f_1(x) \leq f_2(x)\leq \cdots \leq f_n(x) \rightarrow f(x) ~~~~(\forall x \in \Omega)
  \]
  を満たすとすると, 
  \[
  \lim_{n \rightarrow \infty}{\int_{\Omega}{f_n}{d\mu}} = \int_{\Omega}{f}{d\mu}
  \]
  が成り立つ. 
\end{thm}
\begin{thm}[優収束定理]　\\
$(\Omega, \mathcal{F}, \mu)$を測度空間とする. $(f_n)$が$(\Omega, \mathcal{F}, \mu)$上の可測関数列で, 関数 $f$に各点収束し, かつある$(\Omega, \mathcal{F}, \mu)$上の可積分関数$g$が存在して任意の$n, x$に対して$\left\|f_n(x)\right\| \leq g(x)$が成り立つならば, $f$は可積分で, \[
\lim_{n \rightarrow \infty}{\int_{\Omega}{f_n}{d\mu}} = \int_{\Omega}{f}{d\mu}
\]
が成り立つ. 
\end{thm}
\begin{thm}[積分記号下の微分]\label{DiffinIntegral}　\\
  $(\Omega, \mathcal{F}, \mu)$を測度空間とし$a < b$とする. $f:\Omega \times (a,b) \rightarrow \R$は次の二条件を満たすとする. 
  \begin{itemize}
      \item[(1)] 各$t \in (a,b)$について$f(\cdot,t):\Omega \ni x \mapsto f(t,x) \in R$が$L^1(\Omega, \mathcal{F}, \mu)$に属し, 各点$(x,t)$において$(\partial_t f)(x,t)$が存在する. 
      \item[(2)] $0 \leq g \in L^1(\Omega, \mathcal{F}, \mu)$が存在して$|(\partial_t f)(x,t)| \leq g(x) ~~~(\forall x \in \Omega,t \in (a,b)).$
  \end{itemize}
  このとき, $F(t) := \int_{\Omega} f(x,t) d\mu(x)$は$(a,b)$上で微分可能であって, 
  \[
  F'(t) = \int_{\Omega} \frac{\partial f}{\partial t}(x,t) d \mu(x).
  \]
\end{thm}
\begin{proof}
  $t_0 \in (a,b)$を固定し$t_0$に収束する数列$t_n \in (a,b), ~t_n \neq t_0$を取り, $x \in \Omega$に対して
  \[
  h_n(x) := \frac{f(x,t_n) - f(x,t_0)}{t_n - t_0} ~~~~(n=1,2,\ldots,)
  \]
  とおく. このとき各$x$について$(\partial_t f )(x,t_0) = \lim_{n \rightarrow \infty} h_n(x)$であり, 平均値の定理より
  \[
  |h_n(x)| \leq \sup_{t \in (a,b)} \left| \frac{\partial f}{\partial t}(x,t) \right| \leq g(x)
  \]
  である. よって, 優収束定理より
  \[
  \lim_{n \rightarrow \infty} \frac{F(t_n)-F(t_0)}{t_n-t_0} = \lim_{n \rightarrow \infty} \int_{\Omega} h_n(x) d\mu(x) = \int_{\Omega} \frac{\partial f}{\partial t}(x,t_0) d\mu(x).
  \]
  ゆえに数列$(t_n)$の任意性から定理は示された. 
\end{proof}
最後にFubiniの定理について説明しよう. 
\begin{defn}[$\sigma$-有限測度空間]　\\
 測度空間$(X,\mathcal{M},\mu)$が$\sigma$-有限であるとは, $E_n \in \mathcal{M} ~~(n=1,2,\ldots)$が存在して$\mu(E_n) < \infty$且つ$X = \bigcup_{n \in \N} E_n$となることをいう. 
\end{defn}
\begin{defn}[直積可測空間]　\\
  $(X,\mathcal{M}),(Y,\mathcal{N})$を可測空間とする. $X \times Y$上の$\sigma$-加法族$\mathcal{M} \otimes \mathcal{N}$を次のように定める. \[
  \mathcal{M} \otimes \mathcal{N} := \sigma\left[ \{E \times F \mid E \in \mathcal{M}, F \in \mathcal{N}\} \right]
  \]
  このとき, 可測空間$(X \times Y, \mathcal{M} \otimes \mathcal{N})$を$(X,\mathcal{M}),(Y,\mathcal{N})$の直積可測空間という. 
\end{defn}
\begin{defn}[直積測度空間]　\\
  $(X,\mathcal{M},\mu),(Y,\mathcal{N},\nu)$を$\sigma$-有限な測度空間とする. このとき, 直積可測空間$(X \times Y, \mathcal{M} \otimes \mathcal{N})$上の測度$\mu \otimes \nu$で
  \[
  (\mu \otimes \nu)(E \times F) = \mu(E)\nu(F) ~~~~(E \in \mathcal{M}, F \in \mathcal{N})
  \]
  を満たすものが一意的に存在する. $\mu \otimes \nu$を$\mu,\nu$の直積測度といい, 測度空間$(X \times Y, \mathcal{M} \otimes \mathcal{N}, \mu \otimes \nu)$を$(X,\mathcal{M},\mu),(Y,\mathcal{N},\nu)$の測度空間という. 
\end{defn}
\begin{thm}[Tonelliの定理]　\\
  $(X,\mathcal{M},\mu),(Y,\mathcal{N},\nu)$を$\sigma$-有限な測度空間とする. このとき, $(X \times Y, \mathcal{M} \otimes \mathcal{N}$上の可測関数$f \geq 0$に対して次が成り立つ. 
  \begin{itemize}
      \item[(1)] $F_1(x) := \int_{Y} f(x,y) d\nu(y)$は$X$上の可測関数であり, \\
      $F_2(y) := \int_{X} f(x,y) d\mu(x)$は$Y$上の可測関数である. 
      \item[(2)] さらに, 
      \[
      \int_{X \times Y} f(x,y) d(\mu \otimes \nu)(x,y) = \int_{X} F_1(x) d\mu(x) = \int_{Y} F_2(y) d\nu(y).
      \]
  \end{itemize}
\end{thm}
\begin{thm}[Fubiniの定理]　\\
  $(X,\mathcal{M},\mu),(Y,\mathcal{N},\nu)$を$\sigma$-有限な測度空間とする. $f \in L^1(X \times Y, \mathcal{M} \otimes \mathcal{N}, \mu \otimes \nu)$に対して次が成り立つ. 
  \begin{itemize}
      \item[(1)] $\mu$-a.e.$x$で$f(x,\cdot) \in L^1(Y,\mathcal{N},\nu)$且つ$\nu$-a.e.$y$で$f(\cdot,y) \in L^1(X,\mathcal{M},\mu)$.
      \item[(2)] $\mu$-a.e.$x$で定義される$F_1(x) := \int_{Y} f(x,y) d\nu(y)$について$F \in L^1(X,\mathcal{M},\mu)$. また, $\nu$-a.e.$y$に対して定義される$F_2(y) := \int_{X} f(x,y) d\mu(x)$について$F_2 \in L^1(Y,\mathcal{N},\nu)$となる. 
      \item[(3)] さらに, 
      \[
      \int_{X \times Y} f(x,y) d(\mu \otimes \nu)(x,y) = \int_{X} F_1(x) d\mu(x) = \int_{Y} F_2(y) d\nu(y).
      \]
  \end{itemize}
\end{thm}

\subsection{Hahn-Banachの拡張定理}
本節ではHahn-Banachの定理の主張の説明をし, 本論文で用いたいくつかの事実をHahn-Banachの定理から導出する. 
\begin{defn}[劣線形汎関数]　\\
$V$を$\mathbb{R}$上の線形空間とする. 写像$p : V \rightarrow \mathbb{R}$が劣線形であるとは, 以下の2つの条件が成り立つことを言う:
\[
\begin{aligned}
&(1)~~\forall \lambda > 0, \forall x \in V, ~~ p(\lambda x) = \lambda p(x) \\
&(2)~~\forall x, y \in V, ~~ p(x+y) \leq p(x)+p(y)
\end{aligned}
\]
\end{defn}

例えば, ノルム空間においてノルムは劣線形汎関数である. Hahn-Banachの定理は線形部分空間上の線形汎関数は劣線形汎関数に支配されているのであればその支配を崩さずに全空間に拡張できるという主張である. 正確には以下の通りである. 

\begin{thm}[Hahn-Banachの拡張定理]　\\
$V$を$\mathbb{R}$上の線形空間, $W \subset V$を線形部分空間, $p : V \rightarrow \mathbb{R}$を劣線形汎関数とするとき, 線形汎関数$f : W \rightarrow \mathbb{R}$が$\forall x \in W, f(x) \leq p(x)$ をみたすならば, 線形汎関数$F : V \rightarrow \mathbb{R}$であって, $W$上で$F = f$となり, $V$上で$F \leq p$となるものが存在する. 
\end{thm}
\begin{proof}
  宮島\cite{miyajima}を参照されたい. 
\end{proof}
この定理から次の形のHahn-Banachの拡張定理が導ける. 
\begin{thm}[Hahn-Banachの拡張定理2]　\\
  $V$をノルム空間とし, $W \subset V$を部分線形空間とする. このとき, $W$上の有界線形汎関数$f:W \rightarrow \R$に対して, $V$上の有界線形汎関数$F$で
  \[
  F(x) = f(x) ~(x \in W),~ \lVert F \rVert = \lVert f \rVert
  \]
  を満たすものが存在する. 
\end{thm}
\begin{proof}
  $x \in V$に対して$p(x) = \lVert f \rVert \lVert x \rVert$とおく. 明らかに$p$は劣線形汎関数であり, $f(x) \leq p(x) ~(x \in W)$である. 従って, Hahn-Banachの拡張定理より$F(x) = f(x) ~(x \in W)$かつ$F(x) \leq p(x) ~(x \in V)$なる線形汎関数$F:V \rightarrow \R$が取れる. このとき
  \[
  -F(x) = F(-x) \leq p(-x) = p(x) = \lVert f \rVert \lVert x \rVert ~~(x \in V)
  \]
  であるので$F$は有界で$\lVert F \rVert \leq \lVert f \rVert$である. また, \[
  |f(x)| = |F(x)| \leq \lVert F \rVert \lVert x \rVert ~~~(x \in W)
  \]
  であるので$\lVert f \rVert \leq \lVert F \rVert$であるので$\lVert f \rVert = \lVert F \rVert$である. 
\end{proof}

本論文で用いるHahn-Banachの拡張定理の系を述べよう. 

\begin{cor}\label{HahnBanachCor}
  $V$を$\mathbb{R}$上のノルム空間, $W \subset V$を線形部分空間とする. 
  このとき, ある$x_0 \in V$が$d := d(x_0,W) = \inf_{x \in W} \lVert x - x_0 \rVert > 0$を満たすならば, 
  \[
  F(x) = 0 ~(x \in W),~ F(x_0) = 1,~ \lVert F \rVert = \frac{1}{d}
  \]
  を満たす有界線形汎関数$F:V \rightarrow \R$が存在する. 
\end{cor}
\begin{proof}
  $X := W + \R x_0$とおき, $f:X \in w+ax_0 \mapsto a \in \R$とおく. $f$はwell-definedである. 実際, $w_1+ax_0 = w_2 + bx_0$とすると, $(b-a)x_0 = w_1-w_2 \in W$であるので$d(x_0,W)>0$より$a=b$でなければならない.  明らかに$f$は線形汎関数であり$f(x) = 0 ~(x \in W)$かつ$f(x_0)=1$を満たす. そして, $x \in W, a \R \setminus \{0\}$に対して, 
  \[
  \begin{aligned}
  \lVert x + a x_0 \rVert =|a| \left\lVert (-\frac{1}{a}x) - x_0 \right\rVert \leq |a| d
  \end{aligned}
  \]
  であるので$|f(x+ax_0)| = |a| = \lVert x+ax_0 \rVert/d$である. よって, $f$は有界であり$\lVert f \rVert \leq 1/d$である. また, $d=\lim_{n\rightarrow\infty} \lVert x_0 - x_n \rVert$なる$W$の元の列$(x_n)$を取れば, 
  \[
  1=f(-x_n + x_0) \leq \lVert f \rVert \lVert x_0 - x_n \rVert \rightarrow \lVert f \rVert d
  \]
  であるので$1/d \leq \lVert f \rVert$である. よって$\lVert f \rVert = 1/d$であるので, Hahn-Banachの拡張定理$2$より, 有界線形汎関数$F:V \rightarrow \R$であって
  \[
  F(x) = f(x) ~(x \in X),~ \lVert F \rVert = \lVert f \rVert = \frac{1}{d}
  \]
  なるものが存在する. この$F$が求めるものである. 
\end{proof}

\begin{cor}\label{HahnBanachCor1}　\\
  $V$を$\mathbb{R}$上のノルム空間, $W \subset V$を線形部分空間とする. もし$W$の$V$における閉包$\overline{W}$が$V$と一致しなければ, $V$上の有界線形汎関数$f \neq 0$で, $f = 0 ~ (\mathrm{on} ~ \overline{W})$なるものが存在する. 
\end{cor}
\begin{proof}
  $x_0 \in V \setminus \overline{W}$とすると$d(x_0,W) > 0$である. 従って, すぐ上の系から$f(x_0)=1, f(x) = 0 ~(x\in W)$なる線形汎関数
  $f:V \rightarrow \R$が取れる. $x \in \overline{W}$に対して$x$に収束する列$x_n \in W$を取れば, 
  \[
  |f(x)|=|f(x) - f(x_n)| = |f(x-x_n)| \leq \lVert f \rVert \lVert x-x_n \rVert \rightarrow 0
  \]
  であるから$f(x) = 0 ~(x \in \overline{W})$である. よってこの$f$が求めるものである. 
\end{proof}

\begin{cor}\label{HahnBanachCor2}　\\
  $x \neq 0$ならば$f(x) = \lVert x \rVert_X$かつ$\lVert f \rVert = 1$なる有界線形汎関数$f:X \rightarrow \R$が存在する. 
\end{cor}
\begin{proof}
  $W = \{0\}$とおけば$x_0=x$として系\ref{HahnBanachCor}の仮定が満たされる. ゆえに, $d = d(x,W) = \lVert x \rVert$とおくと, 有界線形汎関数$F:V \rightarrow \R$で
  \[
  F(x_0) = 1,~ \lVert F \rVert = \frac{1}{d}
  \]
  を満たすものが取れる. よって, $f = \lVert x \rVert F$が求めるものである. 
\end{proof}

\subsection{Riesz-Markov-角谷の表現定理}
本節ではまず符号付き測度及びそれによる積分の定義を非負値の測度の知識を前提にして述べる. その後, 符号付きBorel測度の定義を述べ, Riesz-Markov-角谷の表現定理の主張を述べる. 証明については宮島\cite{miyajima}の第7章を参照されたい. 

\begin{defn}[符号付き測度]　\\
$(\Omega, \mathcal{F})$を可測空間とする. 写像$\varphi : \mathcal{F} \rightarrow \mathbb{R}$ が次の条件をみたすときを$\varphi$を$(\Omega, \mathcal{F})$上の符号付き測度と呼ぶ:
$\mathcal{F}$の元の列$(A_n)$が互いに素であるならば,  
\[
\varphi(\bigcup_{n=1}^{\infty}{A_n}) = \sum_{n=1}^{\infty}{\varphi(A_n)}
\]
となる. 
\end{defn}

以下の事実に基づき符号付き測度による積分を定義する. 
\begin{prop}　\\
$\varphi$を可測空間$(\Omega, \mathcal{F})$上の符号付き測度とするとき, $A \in \mathcal{F}$に対して, 
$$
\varphi^{+}(A) = \sup \\{\varphi(B) \mid B \in \mathcal{F}, B \subset A\\}, \\
\varphi^{-}(A) = -\inf \\{\varphi(B) \mid B \in \mathcal{F}, B \subset A\\}
$$とおくと, $\varphi^{+}, \varphi^{-}$は$(\Omega, \mathcal{F})$上の有限測度となり, $\varphi = \varphi^{+} - \varphi^{-}$ が成立する. 
\end{prop}

\begin{defn}　\\
$\varphi$を可測空間$(\Omega, \mathcal{F})$上の符号付き測度とするとき, $f \in L^1(\Omega, \mathcal{F}, \varphi^{+}) \cap L^1(\Omega, \mathcal{F}, \varphi^{-})$に対して, $\varphi$に関する$f$の積分を以下で定める.  $$\int_{\Omega}{f}{d\varphi} := \int_{\Omega}{f}{d\varphi^{+}} - \int_{\Omega}{f}{d\varphi^{-}}$$
\end{defn}

次に正則符号付きBorel測度の定義を述べる. 

\begin{defn}[正則符号付きBorel測度]　\\
位相空間$(\Omega, \mathcal{F})$上の符号付き測度$\varphi : \mathcal{F} \rightarrow \mathbb{R}$ は$\varphi^{+}, \varphi^{-}$がともに$(\Omega, \mathcal{F})$上の正則Borel測度になるとき, $(\Omega, \mathcal{F})$上の正則符号付き測度と呼ばれる. 
\end{defn}

なお, 文脈から推察可能な場合は$S$に入っている位相を明示せず単に$S$上の正則符号付き測度と呼ぶこともある. 以上の定義のもと以下が成り立つ. 

\begin{thm}[Riesz-Markov-角谷の表現定理]　\\
$S$をコンパクトHausdorff空間とし, $C(S)$を$S$上の実数値連続関数全体の集合とするとき, 任意の有界線形汎関数$\varphi : C(S) \rightarrow \mathbb{R}$に対して, $S$上の正則符号付きBorel測度$\mu$であって, $$\forall f \in C(S), \\ \varphi(f) = \int_{S}{f}{d\mu}$$をみたすものが一意的に存在する. 
\end{thm}

\subsection{Stone-Weierstrassの定理}
本節ではStone-Weierstrassの定理という, 連続関数空間における一様収束の意味での稠密性に関する定理の主張を述べる. 証明は宮島\cite{miyajima}または新井\cite{AraiFourier}の付録を参照されたい. 
なお, 本節では$S$をコンパクトHausdorff空間とし, $C(S,\mathbb{K})$で$S$上の$\mathbb{K}$-値連続関数全体の集合を表す($K = \R, \mathbb{C}$とする).
$C(S,\mathbb{K})$は$\|f\| = \sup \{ f(x) \mid x \in S\}$により$\mathbb{K}$上のノルム空間になる. 

\begin{defn}[$C(S,\mathbb{K})$の部分代数]　\\
$A \subset C(S,\mathbb{K})$が部分代数であるとは, $A$が$C(S,\mathbb{K})$の$\mathbb{K}$上の線形部分空間であり, かつ任意の$ f, g \in A$について積$fg$が$A$の元になることをいう. 
\end{defn}

$C(S,\mathbb{K})$の部分集合$A$が与えられたとき, $A$を含む(包含関係の意味で)最小の部分代数が存在する($A$を含む部分代数全体の共通部分を取ればよい). それを$A$で生成された部分代数と呼ぶ.  Stone-Weierstrassの定理は生成された部分代数の稠密性について述べたものである. 

\begin{thm}[Stone-Weierstrassの定理(実数形)]　\\
$A \subset C(S,\R)$が次の2つの条件をみたしているとする. 
\[
\begin{aligned}
&(1)\mbox{恒等的に1を取る関数} \mathbf{1} \mbox{は} A \mbox{に属する}.  \\
&(2)\mbox{任意の相異なる} x, y \in S \mbox{に対して, } f(x) \neq f(y) \mbox{なる} f \in A \mbox{が存在する}. 
\end{aligned}
\]
このとき$A$が生成する部分代数の閉包は$C(S,\R)$と一致する. すなわち, 任意の$f \in C(S,\R)$と$\varepsilon > 0$に対して, $A$が生成する部分代数の元$g$で, $|f-g| < \varepsilon$をみたすものが取れる. 
\end{thm}
\begin{cor}
  $K \subset \mathbb{R}$をコンパクト集合とするとき, 1変数の$\R$係数多項式関数を$K$上に制限したもの全体の集合を$A$とおくと, 明らかに$A$は上の2条件をみたすのでStone-Weierstrassの定理から$A$が生成する部分代数の閉包は$C(K,\R)$と一致する. ところが$A$は部分代数なので$A$が生成する部分代数は$A$になる. よって, $K$上の実数値連続関数は多項式により一様近似できる. 
\end{cor}
\begin{thm}[Stone-Weierstrassの定理(複素数形)]　\\
$A \subset C(S,\mathbb{C})$が次の2つの条件をみたしているとする. 
\[
\begin{aligned}
&(1)\mbox{恒等的に1を取る関数} \mathbf{1} \mbox{は} A \mbox{に属する}.  \\
&(2)\mbox{任意の相異なる} x, y \in S \mbox{に対して, } f(x) \neq f(y) \mbox{なる} f \in A \mbox{が存在する}. \\
&(3)f \in A \Rightarrow \overline{f} \in A
\end{aligned}
\]
このとき$A$が生成する部分代数の閉包は$C(S,\mathbb{C})$と一致する. すなわち, 任意の$f \in C(S,\mathbb{C})$と$\varepsilon > 0$に対して, $A$が生成する部分代数の元$g$で, $|f-g| < \varepsilon$をみたすものが取れる. 
\end{thm}
\begin{cor}
  $K \subset \mathbb{R}^n$をコンパクト集合とし, $m \in \mathbb{Z}^n$に対して$e_{m}:K \ni x \mapsto \exp(i\ip<{m,x}>) \in \mathbb{C}$とおく. このとき$A := \mathrm{span}\{e_m\}_{m \in \mathbb{Z}^n}$とおくと, 明らかに$A$は上の3条件をみたすのでStone-Weierstrassの定理から$A$が生成する部分代数の閉包は$C(S)$と一致する.  ところが$A$は部分代数なので$A$が生成する部分代数は$A$になる. よって,  $K$上の複素数値連続関数は$A$の元により一様近似できる. 
\end{cor}

Riesz-Markov-角谷の定理とStone-Weierstrassの定理から次のことがわかる. 

\begin{cor}\label{CorforChuiandLi}
$K \subset \mathbb{R}^n$をコンパクト集合とし,  $\mu$を部分空間$K$上の正則符号付きBorel測度とすると, 
\[
\left(\forall m \in \mathbb{Z}^n, ~\int_{K}{\exp(i \left<x, m\right>)}{d\mu(x)} = 0 \right) \Rightarrow \mu = 0
\]
が成り立つ. 
\end{cor}
\begin{proof}
  任意に$f \in C(K)$と$\varepsilon$をとる. すぐ上の系より$c_k \in \mathbb{C}, m_k \in \mathbb{Z}^n$ $(k=1,\ldots,N)$が存在して, 
  \[
  \forall x \in K, ~ \left|f(x) - \sum_{k=1}^N c_k e_{m_k}(x) \right| < \varepsilon
  \]
  となる. よって, 
  \[
  \begin{aligned}
  |\int_{K} f(x) d\mu(x)| 
  &= |\int_{K} f(x) - \sum_{k=1}^N c_k e_{m_k}(x) d\mu(x)| \\
  &\leq \int_K \left|f(x) - \sum_{k=1}^N c_k e_{m_k}(x) \right| d|\mu|(x)
  \leq |\mu|(K) \varepsilon
  \end{aligned}
  \]
  となる. $\varepsilon > 0$は任意なので$\int_{K} f(x) d\mu(x) = 0$である. ところが, $K$上の正則符号付きBorel測度$\nu = 0$もこれを満たす. よって, Riesz-Markov-角谷の定理の一意性の部分により$\mu = \nu = 0$である. 
\end{proof}

\subsection{畳み込み積の性質}
本節では, $C_0^{\infty}(\R)$と$L_{\mathrm{loc}}^1(\R)$の畳み込みが$C^{\infty}(\R)$になること, $L^1(\R^N)$と$L^p(\R^N)$の畳み込みが$L^p(\R^N)$になることを示す. 

\begin{prop}\label{L1andLpConvolution}
  $1 \leq p < \infty$とする. このとき, $\varphi \in L^1(\R^n)$とし$f \in L^p(\R^n)$とすると$\varphi * f \in L^p(\R^n)$であり, 
  \[
   \lVert \varphi * f \rVert_{L^p} \leq \lVert \varphi \rVert_{L^1} \lVert f \rVert_{L^p}.
  \]
\end{prop}
\begin{proof}
   $p=1$の場合をまず考える. このとき, 
   \[
   \begin{aligned}
   \int_{\R^n} |(\varphi*f)(x)| dx
   &= \int_{\R^n} \left| \int_{\R^n} \varphi(x-y)f(y) dy \right| dx \\
   &= \int_{\R^n} \int_{\R^n} |\varphi(x-y)| |f(y)| dy dx \\
   &= \int_{\R^n} \int_{\R^n} |\varphi(x-y)| |f(y)| dx dy ~~~~(\mbox{Fubiniの定理}) \\
   &= \int_{\R^n} \lVert \varphi \rVert_{L^1} |f(y) dy \\
   &= \lVert \varphi \rVert_{L^1} \lVert f \rVert_{L^1}
   \end{aligned}
   \]
   であるので成立する. 次に$p > 1$とし$1/p + 1/q = 1$とする. すると, 任意の$x \in \R^n$に対して, H\"{o}lderの不等式(命題\ref{HolderInequality})より
   \[
   \begin{aligned}
   |(\varphi*f)(x)|
   &= \left| \int_{\R^n} \varphi(x-y)f(y) dy \right| \\
   &\leq  \int_{\R^n} |\varphi(x-y)|^{1-1/p}|\varphi(x-y)|^{1/p}|f(y)| dy \\
   &\leq \left( \int_{\R^n} |\varphi(x-y)| d\mu(y)  \right)^{1/q} \left( \int_{\R^n} |\varphi(x-y)||f(y)|^p dy  \right)^{1/p} \\
   &= \lVert \varphi \rVert_{L^1}^{1/q} \left( \int_{\R^n} |\varphi(x-y) | |f(x-y)|^p dy \right)^{1/p}
   \end{aligned}
   \]
   であるから, 
   \[
   \begin{aligned}
   &\int_X \left| \int_{\R^n}  \varphi(x-y) f(y) dx \right|^p dx\\
   &\leq \int_{\R^n} \lVert \varphi \rVert_{L^1}^{p/q} \int_{\R^n} |\varphi(x-y)| |f(x-y)|^p dy dx \\
   &= \lVert \varphi \rVert_{L^1}^{p/q} \int_{\R^n} \int_{\R^n} |\varphi(x-y)||f(y)| dxdy ~~~~(\because ~\mbox{Fubiniの定理}) \\
   &= \lVert \varphi \rVert_{L^1}^{p/q} \int_{\R^n} |f(y)|^p \lVert \varphi \rVert_{L^1} dy \\
   &= \lVert \varphi \rVert_{L^1}^p \lVert f \rVert_{L^p}^p < \infty ~~~~(\because ~ 1+p/q = p)
   \end{aligned}
   \]
   となる. よって, $\varphi * f \in L^p(\R^n)$であり, 
   \[
   \lVert \varphi * f \rVert_{L^p} \leq \lVert \varphi \rVert_{L^1} \lVert f \rVert_{L^p}.
   \]
\end{proof}
\begin{prop}\label{L1locandC0Convolution}
  $\varphi \in C_0^{\infty}(\R)$とし$f \in L_{\mathrm{loc}}^1(\R)$とすると$\varphi * f \in C^{\infty}(\R)$であり, 任意の整数$k \geq 0$に対して$(\varphi * f)^{(k)} = \varphi^{(k)} * f$となる. 
\end{prop}
\begin{proof}
  $K = \mathrm{supp}\varphi$とおく. 任意の整数$k \geq 0$に対して$\mathrm{supp}(\varphi^{(k)}) \subset K$であることに注意する. 命題を$k$についての帰納法で示そう. $k=0$では自明. $k$で成立すると仮定する. 任意に$x_0 \in \R$を取り$a<b$を$x_0 \in [a,b]$となるように取る. $K' = [a,b] - K = \{ x-y \mid x \in [a,b],y \in K \}$とおくと$K'$はコンパクトであり, $x \in (a,b)$に対して
  \[
  (\varphi^{(k)} * f)(x) = \int_{\R} \varphi^{(k)}(x-y) f(y) dy = \int_{K'} \varphi^{(k)}(x-y) f(y) dy
  \]
  となる. いま, $g(x,y) = \varphi^{(k)}(x-y)f(y) ~~(x \in (a,b), y \in K')$とおけば, 
  \[
  \left| \frac{\partial g}{\partial x}(x,y) \right| = M |f(y)| ~~~~(x \in (a,b), y \in K')
  \]
  となる. ここで, $M = \sup_{t \in \R} |\varphi^{(k+1)}(t)| < \infty $である. $f$は$K'$上可積分であるので積分記号下の微分(付録の定理\ref{DiffinIntegral})より$(a,b)$上で
  \[
  F(x) := (\varphi * f)^{(k)}(x) = (\varphi^{(k)} * f)(x) = \int_{K'} \varphi^{(k)}(x-y) f(y) dy = \int_{K'} g(x,y) dy
  \]
  は微分可能で, 
  \[
  F'(x) = \int_{K'} \frac{\partial g}{\partial x}(x,y) dy = \int_{K'} \varphi^{(k+1)}(x-y) f(y) dy = (\varphi^{(k+1)} * f)(x)
  \]
  となる. よって, $(\varphi * f)^{(k+1)} = \varphi^{(k+1)} * f$である. これで示せた. 
\end{proof}

\subsection{$\R^n$上の有限Borel測度の正則性}
本節では$(\R^n,\B_{\R^n})$上の有限測度が次の意味で正則であることを示す. 
\begin{defn}[正則Borel測度]　\\
位相空間$(X, \mathcal{O})$上のBorel測度$\mu$は以下の条件をみたすとき正則であると言われる:
$$\forall A \in \mathcal{B}_X, ~~ \mu(A) = \inf \{\mu(U) \mid U \supset A, \mathrm{open}\} = \sup \{\mu(K) \mid K \subset A, \mathrm{cpt. }\}. $$
\end{defn}
ただし, 位相空間上のBorel $\sigma$-加法族, およびBorel測度の定義は以下の通りである. 
\begin{defn}[生成される$\sigma$-加法族]　\\
  $X$を集合とし, $\mathcal{A}_0 \subset \mathcal{P}(X)$とする($\mathcal{P}(X)$は$X$の部分集合全体の集合). このとき, $\mathcal{A}_0$で生成される$\sigma$-加法族$\sigma[\mathcal{A}_0]$を
  \[
  \sigma[\mathcal{A}_0] := \bigcap \{ \mathcal{A} \mid \mathcal{A} \supset \mathcal{A}_0, \mathcal{A}\mbox{は}X\mbox{上の}\sigma\mbox{-加法族} \}
  \]
  と定義する. 明らかに$\sigma[\mathcal{A}_0]$は$X$上の$\sigma$-加法族である. 
\end{defn}
\begin{defn}[Borel $\sigma$-加法族, Borel測度]　\\
  $(X,\mathcal{O})$を位相空間とする. このとき, $X$上のBorel $\sigma$-加法族$\mathcal{B}_X$を$\mathcal{B}_X = \sigma[\mathcal{O}]$と定義する. そして, $X$上のBorel測度とは, 可測空間$(X,\mathcal{B}_X)$上の測度のことをいう. 
\end{defn}
\begin{rem}
  Borel $\sigma$-加法族は閉集合全体の集合$\mathcal{D}$から生成される. 実際, $\mathcal{O} \subset \mathcal{A}$なる$\sigma$-加法族$\mathcal{A}$は$E \in \mathcal{A} \Rightarrow X \setminus E \in \mathcal{A}$をみたすので$\mathcal{D} \subset \mathcal{A}$をみたし, $\mathcal{D} \subset \mathcal{A}$なる$\sigma$-加法族$\mathcal{A}$は$E \in \mathcal{A} \Rightarrow X \setminus E \in \mathcal{A}$をみたすので$\mathcal{O} \subset \mathcal{A}$をみたす. 
\end{rem}
\begin{lem}\label{PropOfDist}
  $(X,d)$を距離空間とし, $A \subset X$は空でないとする. このとき, $x \in X$に対して
  \[
  \mathrm{dist}(x,A) := \inf \{d(x,y) \mid y \in A\}
  \]
  と定義すると, $X \ni x \mapsto \mathrm{dist}(x,A) \in \R$は連続である. また, $x \in X$に対して$\mathrm{dist}(x,A) = 0$と$x \in \mathrm{cl}(A)$は同値である. 特に,  $A$が閉集合であるとき, $\mathrm{dist}(x,A) = 0$と$x \in A$は同値である.
\end{lem}
\begin{proof}
  任意の$x_1,x_2 \in X, y \in A$に対して
  \[
  \mathrm{dist}(x_1,A) \leq d(x_1,y) \leq d(x_1,x_2)+d(x_2,y)
  \]
  であるので, $\mathrm{dist}(x_1,A) - d(x_1,x_2) \leq d(x_2,y)$となる. そこで, 左辺が$y$に依存していないことに注意して$y$に関して下限を取ると
  \[
  \mathrm{dist}(x_1,A) - d(x_1,x_2) \leq \mathrm{dist}(x_2,A)
  \]
  が得られる. よって, $\mathrm{dist}(x_1,A)- \mathrm{dist}(x_2,A) \leq d(x_1,x_2)$となる. $x_1$と$x_2$の役割を交換すると, $\mathrm{dist}(x_2,A)-\mathrm{dist}(x_1,A) \leq d(x_2,x_1)$が得られる. 以上から, 
  \[
  |\mathrm{dist}(x_1,A)-\mathrm{dist}(x_2,A)| \leq d(x_1,x_2)
  \]
  となり, これより連続性が従う. 次に$x \in X$に対して$\mathrm{dist}(x,A) = 0$とする. すると, 任意の$n$に対して, $d(x,y_n) < 1/n$なる$y_n \in A$が取れる. このとき, $y_n \rightarrow x$であるので$x \in \mathrm{cl}(A)$となる. $x \in \mathrm{cl}(A)$のとき$d(x,A) =  0$となることについては上の議論を逆向きにすればよい. 
\end{proof}
\begin{prop}
  距離空間$X$上の有限Borel測度$\mu$は次の意味で正則である: 任意の$A \in \mathcal{B}_X$に対して, 
  \[
  \begin{aligned}
  \mu(A) 
  &= \sup\{\mu(C) \mid C \subset A, C\mbox{は閉集合}\} \\
  &= \inf\{\mu(U) \mid U \supset A, U\mbox{は開集合} \}
  \end{aligned}
  \]
  が成立する. 
\end{prop}
\begin{proof}
  $A \in \mathcal{B}_X$かつ
  \[
  \begin{aligned}
  \mu(A) 
  &= \sup\{\mu(C) \mid C \subset A, C\mbox{は閉集合}\} \\
  &= \inf\{\mu(U) \mid U \supset A, U\mbox{は開集合} \}
  \end{aligned}
  \]
  なる集合$A$全体の集合を$\mathcal{R}$とおく. すると, $\mathcal{R}$は$X$上の$\sigma$-加法族である. まず, $\emptyset \in \mathcal{R}$は明らか. 次に, $A \in \mathcal{R}$とすると$X \setminus A \in \mathcal{R}$であることを示す. 任意に$\varepsilon > 0$をとる. 閉集合$C \subset A$と開集合$U \supset A$を
  \[
  \mu(A) < \mu(C) + \varepsilon, \mu(A) > \mu(U) - \varepsilon
  \]
  をみたすように取る. すると, $X \setminus C \supset X \setminus A$は開集合であるので, 
  \[
  \begin{aligned}
  \mu(X \setminus A) 
  &= \mu(X) - \mu(A) \\
  &> \mu(X) - \mu(C) - \varepsilon \\
  &= \mu(X \setminus C) - \varepsilon \\
  &\geq \inf\{\mu(E) \mid E \supset X \setminus A, E\mbox{は開集合}\} - \varepsilon
  \end{aligned}
  \]
  となり, $X \setminus U \subset X \setminus A$は閉集合なので, 
  \[
  \begin{aligned}
  \mu(X \setminus A)
  &= \mu(X) - \mu(A) \\
  &< \mu(X) - \mu(U) + \varepsilon \\
  &= \mu(X \setminus U) + \varepsilon \\
  &\leq \sup\{\mu(E) \mid E \subset X \setminus A, E\mbox{は閉集合}\} + \varepsilon
  \end{aligned}
  \]
  となる. よって, $\varepsilon > 0$の任意性によって$X \setminus A \in \mathcal{R}$となる. 次に$\mathcal{R}$の元の列$(A_n)_{n=1}^{\infty}$に対して$\bigcup_{n=1}^{\infty} A_n \in \mathcal{R}$となることを示す. そのために$\varepsilon > 0$を任意にとり, 各$n$に対して, 開集合$U_n$と閉集合$C_n$を$C_n \subset A_n \subset U_n$かつ
  \[
  \mu(U_n) - \mu(A_n) < \frac{\varepsilon}{2^n},~~ \mu(A_n) - \mu(C_n) < \frac{\varepsilon}{2^n}
  \]
  となるようにとる. すると, $U := \bigcup_{n=1}^{\infty} U_n \supset \bigcup_{n=1}^{\infty} A_n =: A$は開集合であり, 
  \[
  \begin{aligned}
  \mu(U) - \mu(A) 
  &= \mu\left((\bigcup U_n) \setminus (\bigcup A_n)\right) \\
  &\leq \mu\left(\bigcup (U_n \setminus A_n)\right) \\
  &\leq \sum \mu(U_n \setminus A_n) \\
  &= \sum (\mu(U_n) - \mu(A_n))  \\
  &\leq \sum \frac{\varepsilon}{2^n} = \varepsilon
  \end{aligned}
  \]
  となる. また, $C = \bigcup C_n$とおくと, 測度の連続性から
  \[
  \mu(\bigcup C_n) = \lim_{N \rightarrow \infty} \mu\left(\bigcup_{n=1}^{N} C_n\right)
  \]
  である. したがって, ある$N$について$D_N = \bigcup_{n=1}^N C_n$とおくと, $\mu(C) - \mu(D_N) < \varepsilon$となる. このとき, $D_N \subset A$は閉集合であり, 
  \[
  \begin{aligned}
  \mu(A) - \mu(D_N) 
  &< \mu(A) - \mu(C) + \varepsilon \\
  &\leq \mu(\bigcup (A_n \setminus C_n)) + \varepsilon \\
  &\leq \sum (\mu(A_n) - \mu(C_n)) + \varepsilon \\
  &\leq \sum \frac{\varepsilon}{2^n} + \varepsilon = 2\varepsilon
  \end{aligned}
  \]
  となる. 以上から$\bigcup A_n \in \mathcal{R}$となる. \\
  さて, $\mathcal{R}$が$\sigma$-加法族であるので, すべての$X$の閉集合が$\mathcal{R}$に含まれることを示せば$\mathcal{B}_X \subset \mathcal{R}$となることがわかる. そこで, $A \subset X$を空でない閉集合とする. このとき, $U_n = \{x \in X \mid \mathrm{dist}(x,A) < 1/n\}$とおけば, $x \mapsto \mathrm{dist}(x,A)$の連続性から$U_n$は開集合であって, $A$は閉集合であるから,
  \[
  \bigcap_{n=1}^{\infty} U_n = \{x \in X \mid d(x,A) = 0 \} = A
  \]
  となる. しかも, $U_1 \supset U_2 \supset \cdots \supset U_n \supset \cdots$であるから, 測度の連続性より$\mu(A) = \lim_{n \rightarrow \infty} \mu(U_n) = \inf_{n} \mu(U_n)$となる. さらに$U_n \supset A$であるから, 
  \[
  \mu(A) 
  \leq \inf\{\mu(U) \mid U \supset A, U\mbox{は開集合}\} 
  \leq \inf_{n} \mu(U_n) = \mu(A)
  \]
  となる. また, $A$は閉集合なので$\mu(A) = \sup \{\mu(C) \mid C \subset A, C\mbox{は閉集合}\}$となる. よって, $A \in \mathcal{R}$である. 
\end{proof}
\begin{cor}\label{FiniteBorelMeasOnRnIsRegular}
  $\R^n$上の有限Borel測度$\mu$は正則である. つまり任意の$A \in \mathcal{B}_{\R}$に対して, 
  \[
  \begin{aligned}
  \mu(A) 
  &= \sup\{\mu(K) \mid K \subset A, C\mbox{はコンパクト}\} \\
  &= \inf\{\mu(U) \mid U \supset A, U\mbox{は開集合} \}
  \end{aligned}
  \]
  が成立する. 
\end{cor}
\begin{proof}
  $\R^n$において有界閉集合とコンパクト集合は一致することに注意する. $\R^n$は距離空間であるのですぐ上の命題より
  \[
  \mu(A) = \sup\{\mu(C) \mid C \subset A, C\mbox{は閉集合}\}
  \]
  である. そこで, 任意の$\varepsilon > 0$に対して, 閉集合$C \subset A$で$\mu(A) - \varepsilon < \mu(C)$なるものが取れる. このとき, $0 \in \R^n$中心の半径$n$の閉球を$B(0;n)$と表すと, $C_n = C \cap B(0;n)$は有界閉集合であるのでコンパクトである. そして, $C_1 \subset C_2 \subset \cdots$であり, $C = \bigcup C_n$であるので, 測度の連続性より$\mu(C)= \lim_{n \rightarrow \infty  \mu(C_n)}$となる. したがって, $\mu(C)-\mu(C_N) < \varepsilon$なる$N$が取れる. ゆえに, 
  \[
  \mu(A) < \mu(C) + \varepsilon < \mu(C_N) + 2\varepsilon \leq \sup\{\mu(K) \mid K \subset A, C\mbox{はコンパクト}\} + 2\varepsilon
  \]
  となる. よって, $\varepsilon > 0$の任意性より
  \[
  \mu(A) \leq \sup\{\mu(K) \mid K \subset A, C\mbox{はコンパクト}\}
  \]
  となる. これで示せた. 
\end{proof}

\subsection{$C_0(\R^n)$の$L^p$空間における稠密性}
本節では, $\R^n$上の正則Borel測度$\mu$であって任意のコンパクト集合$K \subset \R^n$に対して$\mu(K) < \infty$となるものについて, $C_0(\R^n)$が$L^p(\R^n,\B_{\R^n},\mu)$において稠密であることを証明する. 

\begin{lem}
  $U \subset \R^n$を開集合とし, $\emptyset \neq K \subset U$をコンパクト集合とする. このとき, 開集合$V \subset \R^n$で$K \subset V \subset \mathrm{cl}(V) \subset U$かつ$\mathrm{cl}(V)$がコンパクトになるものが存在する. 
\end{lem}
\begin{proof}
  各$x \in U$に対して, $B(x;2\varepsilon_x) \subset U$なる$\varepsilon_x > 0$をとる. このとき, $K \subset U = \bigcup_{x \in U} B(x;\varepsilon_x)$となる. $K$はコンパクトなのである$x_1,\ldots,x_k$について$K \subset \bigcup_{j=1}^k B(x_j;\varepsilon_{x_j})$となる. そこで, $V = \bigcup_{j=1}^k B(x_j;\varepsilon_{x_j})$とおけば, $V$は開集合で, $K \subset V$かつ, 
  \[
  \mathrm{cl}(V) \subset \bigcup_{j=1}^k \mathrm{cl}(B(x_j;\varepsilon_{x_j})) \subset \bigcup_{j=1}^k B(x_j;2\varepsilon_{x_j}) \subset U
  \]
  である. $\mathrm{cl}(V)$のコンパクト性も上の関係からわかる. 
\end{proof}
\begin{defn}
  次のように定義する. 
  \[
  C_0(\R^n) := \{f \in C(\R^n) \mid \mathrm{supp}(f)\mbox{はコンパクト}\}.
  \]
\end{defn}
\begin{lem}
  $U \subset \R^n$を開集合とし, $\emptyset \neq K \subset U$をコンパクト集合とする. このとき, ある$\varphi \in C_0(\R^n)$が存在して次の$(1),(2),(3)$をみたす. 
  \begin{itemize}
      \item[(1)] $\forall x \in K, ~ \varphi(x) = 1$
      \item[(2)] $0 \leq \varphi \leq1$
      \item[(3)] $\mathrm{supp}\varphi \subset U$
  \end{itemize}
\end{lem}
\begin{proof}
  $K \subset V \subset \mathrm{cl}(V) \subset U$かつ$\mathrm{cl}(V)$がコンパクトになるような開集合$V \subset \R^n$をとる. このとき, $K \cap (\R^n \setminus V)$であるので, 前節の補題\ref{PropOfDist}より, 任意の$x \in \R^n$に対して, $\mathrm{dist}(x,K) = 0$かつ$\mathrm{dist}(x,\R^n \setminus V) = 0$となることはない. したがって, $\mathrm{dist}(x,K)+\mathrm(x,\R^n \setminus V) > 0$である. このことに注意して, 
  \[
  \varphi(x) = \frac{\mathrm{dist}(x,\R^n \setminus V)}{\mathrm{dist}(x,K)+ \mathrm{dist}(x,\R^n \setminus V)}
  \]
  と定義する. すると, 明らかに$0 \varphi \leq 1$であり, $\mathrm{dist}$の連続性から$\varphi \in C(\R^n)$である. また, $x \in K$に対して, $\mathrm{dist}(x,K) = 0$だから$\varphi(x) = 1$である. そして, $\{x \in \R^n \mid \varphi(x) = 0\} = \R^n \setminus V$であるので, $\{x \in \R^n \mid \varphi(x) \neq 0\} = V$である. よって, $\mathrm{supp}\varphi = \mathrm{cl}(V) \subset U$である. そして, この関係から$\mathrm{supp}\varphi$のコンパクト性がわかるので$\varphi \in C_0(\R^n)$である. 
\end{proof}

\begin{prop}\label{DensityOfC0InLp}
  $1 \leq p < \infty$とし, $\mu$を$(\R^n,\B_{\R^n})$上の正則測度で任意のコンパクト集合$K \subset \R^n$に対して$\mu(K) < \infty$となるものとする. 
  このとき, $L^p(\R^n,\B_{\R^n},\mu)$において$C_0(\R^n)$は$L^p$距離で稠密である. 
  すなわち, 任意の$f \in L^p(\R^n,\B_{\R^n},\mu)$と
  $\varepsilon > 0$に対して, $g \in C_0(\R^n)$が存在して, 
  \[
  d_{\mu}^p(f,g) = \lVert f-g \rVert_{L^p} =\left(\int_{\R^r} |f - g|^p d\mu \right)^{1/p} < \varepsilon
  \]
  となる. 
\end{prop}
\begin{proof}
  まず, $f = \chi_E ~(E \in \B_{\R^r})$の場合を考える. $f^p = f \in L^1$であるので, $\mu(E) < \infty$である. すると, $\mu$の正則性より$\mu(U) < \mu(E)+\varepsilon$なる開集合$U \supset E$と$\mu(E)-\varepsilon < \mu(K)$なるコンパクト集合$\emptyset \neq K \subset E$が取れる. いま, $K \subset U$に対して, すぐ上の補題を適用して$\varphi \in C_0(\R^n)$を取る. このとき, $\mu(K) < \infty$より$\mu(U \setminus K) = \mu(U) - \mu(K) =\mu(U) - \mu(E) + \mu(E) - \mu(K) < 2\varepsilon$であることに注意すると, 
  \[
  \begin{aligned}
  \int_{\R^n} |\chi_E - \varphi|^p d\mu 
  &= \int_{U} |\chi_E - \varphi|^p d\mu ~~~(\because \mathrm{supp}\varphi \subset U, E \subset U) \\
  &= \int_{U \setminus K} |\chi_E - \varphi|^p d\mu ~~~(\because K \subset E, \varphi \equiv 1 ~(\mathrm{on} ~K)) \\
  &\leq \int_{U \setminus K} 2^p d\mu = 2^p \mu(U \setminus K) < 2^{p+1} \varepsilon 
  \end{aligned}
  \]
  となる. よって, この場合は成立する. 
  次に$f \geq 0$の場合を考える. このとき,  単関数近似定理より非負単関数列$\varphi_n$が存在して, 各点$x$で$0 \leq \varphi_1(x) \leq \varphi_2(x) \leq \cdots \leq \varphi_n(x) \rightarrow f(x)$となる. すると, $|f - \varphi_n|^p \leq 2^p f^p \in L^1$であるので, 優収束定理より
  \[
  d_{\mu}^p(f,\varphi_n)^p = \int (f - \varphi_n)^p d\mu \rightarrow 0
  \]
  となる. そこで, ある非負単関数$\varphi = \sum_{j=1}^n c_j \chi_{E_j}$について$d_{\mu}^p(f,\varphi) < \varepsilon$となる. 
  ここで, すぐ上で示したことより, 各$\chi_{E_j}$に対して適当に$g_i \in C_0(\R^r)$を選ぶと, $d_{\mu}^p(\chi_{E_j},g_i) < \varepsilon/n\max_{j}\{|c_j|\}$とできる. ゆえに, $g= \sum_{j=1}^n c_j g_j$とおくと, $g \in C_0(\R^r)$であり, 
  \[
  \lVert \varphi - g \rVert_{L^p} \leq \sum_{j=1}^n |c_j| \lVert \chi_{E_j} - g_j \rVert_{L^p} < \sum_{j=1}^n |c_j| \varepsilon/n\max_{1 \leq j \leq n}\{|c_j|\} \leq \varepsilon
  \]
  となる. よって, $\lVert f - g \rVert_{L^p} \leq \lVert f - \varphi  \rVert_{L^p} + \lVert \varphi - g \rVert_{L^p} < 2\varepsilon$となる. 一般の$f \in L^p$については$f = f^{+} - f^{-}, f^{\pm} \geq 0$と分解することで$f \geq 0$の場合に帰着される. 
\end{proof}

この定理を使うと次のことが示せる. 
\begin{prop}\label{ContinuousOfTranslationAboutLpNorm}
  $1 \leq p < \infty$とし, $\mu$を$(\R^n,\B_{\R^n})$上のLebesgue測度とする. 
  このとき, $f \in L^p(\R^n,\B_{\R^n},\mu)$に対して, 
  \[
  \lVert f(\cdot - h) - f(\cdot) \rVert_{L^p(\mu)} \rightarrow 0 ~~~(h \rightarrow 0).
  \]
\end{prop}
\begin{proof}
  $\varepsilon > 0$を任意に取る. 定理\ref{DensityOfC0InLp}より$\varphi \in C_0(\R^n)$で$\lVert f - \varphi \rVert_{L^p} < \varepsilon$なるものが取れる. 
  $\varphi$は$\R^n$上一様連続なのである$\delta > 0$があって
  \[
  \forall x,y \in \R^n , |x-y|<\delta \Rightarrow |\varphi(x)-\varphi(y)|<  \frac{\varepsilon}{(2\mu(\mathrm{supp}(\varphi))^{1/p}} 
  \]
  となる. ゆえに, 任意の$0 < h < \delta$に対して, Lebesgue測度の平行移動不変性($\mu(E+h) = \mu(E)$)より, 
  \[
  \begin{aligned}
  \int_{\R^n} |\varphi(x-h)-\varphi(x)|^p d\mu(x) 
  &= \int_{(\mathrm{supp}(\varphi)+h) \cup \mathrm{supp}(\varphi)} |\varphi(x-h)-\varphi(x)|^p d\mu(x) \\
  &\leq \frac{\varepsilon^p}{2\mu(\mathrm{supp}(\varphi))} (\mu((\mathrm{supp}(\varphi)+h)) +  \mu(\mathrm{supp}(\varphi))) \\
  &= \frac{\varepsilon^p}{2\mu(\mathrm{supp}(\varphi))} 2\mu(\mathrm{supp}(\varphi)) = \varepsilon^p
  \end{aligned}
  \]
  となる. ゆえに, 任意の$0 < h < \delta$に対して, Lebesgue測度の平行移動不変性(すぐ下の注意)より
  \[
  \begin{aligned}
  &\lVert f(\cdot - h) - f(\cdot) \rVert_{L^p(\mu)} \\
  &\leq 
  \lVert f(\cdot - h) - \varphi(\cdot - h) \rVert_{L^p(\mu)}
  + \lVert \varphi(\cdot - h) - \varphi(\cdot) \rVert_{L^p(\mu)}
  + \lVert \varphi - f \rVert_{L^p(\mu)} \\
  &\leq 2\lVert \varphi - f \rVert_{L^p(\mu)} + \varepsilon \leq 2\varepsilon
  \end{aligned}
  \]
  となる. 
\end{proof}
\begin{rem}
  上の証明で可測関数$f \geq 0$に対して
  \[
  \int_{\R^n} f(x-h) d\mu(x) = \int f(x) d\mu(x)
  \]
  となることを使った. これを示すには$f$に収束する単調増加単関数列$\varphi_j$を取る(単関数近似定理). $\varphi_j = \sum_{k=1}^{m_j} c_{j,k} \chi_{E_{j,k}}$と表すと
  \[
  \varphi_j(x-h) = \sum_{k=1}^{m_j} c_{j,k} \chi_{E_{j,k}+h}(x)
  \]
  となるので, Lebesgue測度の平行移動不変性より
  \[
  \int_{\R^n} \varphi_j(x-h) d\mu(x) = \sum_{k=1}^{m_j} c_{j,k} \mu(E_{j,k}+h) = \sum_{k=1}^{m_j} c_{j,k} \mu(E_{j,k}) = \int_{\R^n}  \varphi_j(x) d\mu(x)
  \]
  となる. ゆえに単調収束定理より
  \[
  \begin{aligned}
  \int_{\R^n} f(x) d\mu(x) 
  &= \lim_{j \rightarrow \infty} \int_{\R^n} \varphi_j(x) d\mu(x) \\
  &= \lim_{j \rightarrow \infty} \int_{\R^n} \varphi_j(x-h) d\mu(x) \\
  &= \int_{\R^n} f(x-h) d\mu(x).
  \end{aligned}
  \]
\end{rem}

\subsection{Lusinの定理}
本節ではLusinの定理を証明する. そのためにまずEgorovの定理を証明しよう. 
\begin{thm}[Egorovの定理]　\\
  $(X,\mathcal{M},\mu)$を有限測度空間とする. 可測関数列$(f_n:X \rightarrow \R)_{n \in \N}$が可測関数$f:X \rightarrow \R$に概収束する(つまり, $f_n \rightarrow f $ ($\mu$-a.e.))ならば, 任意の$\varepsilon > 0$に対して$E \in \mathcal{M}$が存在して$\mu(E) < \varepsilon$かつ$X \setminus E$上で一様に$f_n \rightarrow f$となる. 
\end{thm}
\begin{proof}
  はじめに$f_n$が$f$に各点収束する場合を考える. $k \in \N$に対して
  \[
  E_n(k) := \{ x \in X \mid |f_n(x)-f(x)| \geq \frac{1}{k} \} ~~(n \in \N)
  \]
  とおき, $F_n(k) := \bigcup_{j=n}^{\infty} E_j(k)$とおくと, $F_n(k) \supset F_{n+1}(k) ~(\forall n)$であり, $f_n \rightarrow f$より$\bigcap_{n=1}^{\infty} F_n(k) = \emptyset$である. 従って, $\mu(X) < \infty$に注意すると, 測度の連続性より$\mu(F_n(k)) \rightarrow \mu(\emptyset) = 0 ~(n \rightarrow \infty)$である. ゆえに, 任意の$\varepsilon > 0$に対して, $k \in \N$ごとに適当な$n_k \in \N$が存在して$\mu(F_{n_k}(k)) < \varepsilon/2^{k+1}$となる. そこで$E = \bigcup_{k=1}^{\infty} F_{n_k}(k)$とおくと, 
  \[
  \mu\left( E \right) = \sum_{k=1}^{\infty} \mu\left( F_{n_k}(k) \right) \leq \sum_{k=1}^{\infty} \frac{\varepsilon}{2^{k+1}} < \varepsilon
  \]
  となる. また, 任意の$\delta > 0$に対して$1/k_0 < \delta$なる$k_0 \in \N$を取ると, 任意の$n \geq n_{k_0}$に対して
  \[
  \forall x \in X \setminus E = \bigcap_{k=1}^{\infty} \bigcap_{j=n_k}^{\infty} \left\{x \in X \mid |f_j(x) - f(x)| < \frac{1}{k} \right\},~ |f_n(x) - f(x) | < \frac{1}{k_0} < \delta
  \]
  となる. よって, $X \setminus E$上で一様に$f_n \rightarrow f$となる. 次に概収束する場合であるが, $\mu(X \setminus F) = 0$なる$F \in \mathcal{M}$上で各点収束するとすれば, 上で示したことから任意の$\varepsilon > 0$に対して$A \in \mathcal{M}$が存在して$\mu(F\cap A) < \varepsilon$かつ$F \setminus (F\cap A)$上で一様収束する. よって, $E = (X \setminus F) \cup (F \cap A)$とおけば$\mu(E) \leq \mu(X \setminus F) + \mu(F \cap A) < \varepsilon$かつ$X \setminus E \subset F \setminus (F \cap A)$上で一様収束する. 
\end{proof}
\begin{thm}[Lusinの定理]\label{LusinTheorem}　\\
  $\mu$を位相空間$X$上の有限正則Borel測度とする. このとき, 任意の可測関数$f:X \rightarrow \R$と$\varepsilon > 0$に対して, コンパクト集合$K \subset X$が存在して, $\mu(K) > \mu(X) - \varepsilon$かつ$f$は$K$上連続となる. 
\end{thm}
\begin{proof}
  $f \geq 0$として証明すれば十分である. 実際, それが示されれば, $f = f^{+} - f^{-}$と分解するとき, コンパクト集合$K_{+},K_{-} \subset X$が存在して$\mu(X \setminus K_{\pm}) < \varepsilon/2$かつ$f^{\pm}$は$K_{\pm}$上連続となるから, $K = K_{+} \cap K_{-}$が所望のものになる. \\
  さて, $f \geq 0$とする. まず$f$が単関数である場合を考える. この場合, $a_1,\ldots,a_k \geq 0$と互いに素な可測集合$E_1,\ldots,E_k$により$f = \sum_{j=1}^k a_j \chi_{E_j}$と表される. すると,  $\mu$が有限正則測度であることよりコンパクト集合$K_j \subset X$が存在して$\mu(E_j \setminus K_j) < \varepsilon/k$となる.  $K_1,\ldots,K_k$は互いに素で$f$は$K_j$上定数であるから$K = \bigcup_{j=1}^k K_j$が求めるものである. 
  次に一般の$f \geq 0$について考える. 単関数近似定理より$f$に収束する非負値単関数列$(\varphi_n)_{n \in \N}$が取れる. 上で示したことから各$n \in \N$に対してコンパクト集合$K_n \subset X$が存在して$\mu(X \setminus K_n) < \varepsilon/2^{n+1}$かつ$K_n$上$\varphi_n$は連続となる. そこで$K = \bigcap_{n=1}^{\infty} K_n$とおくと
  \[
  \mu(X \setminus K) \leq \sum_{n=1}^{\infty} \mu(X \setminus K_n) < \varepsilon/2
  \]
  となる. Egorovの定理よりコンパクト集合$K_0 \subset K$で$\mu(K \setminus K_0 ) < \varepsilon/2$かつ$K_0$上で一様に$\varphi_n \rightarrow f$となるものが取れる. 従って,  任意の$n$について$\varphi_n$は$K_0$上連続であるので, $f$は$K_0$上連続である. そして, 
  \[
  \mu(X \setminus K_0) \leq \mu(X \setminus K) + \mu(K \setminus K_0) < \varepsilon
  \]
  となるので$K_0$が求めるものである. 
\end{proof}
\begin{cor}
  $\mu$を$\R^r$上の有限Borel測度とする. このとき, 任意の$f \in \M^r$と$\varepsilon > 0$に対して, コンパクト集合$K \subset \R^r$が存在して, $\mu(K) > \mu(\R^r) - \varepsilon$かつ$K$上の関数$f|K$は連続となる. 
\end{cor}
\begin{proof}
  $\R^r$上の有限Borel測度は正則であるから上の定理より従う. 
\end{proof}

\subsection{Clarksonの不等式}
本節ではClarksonの不等式について述べる.
参考文献は\cite{Clarkson1936}, \cite{MizuguchiHoshi},  \cite{RealandAbstractAnalysis}である. 
歴史的にはClarksonの不等式は$L^p$空間が一様凸であることを示すために$1936$年のClarkson\cite{Clarkson1936}により初めて証明された. まず, $0<p<\infty$のときの$L^p$空間における基本的な不等式について述べよう. 本節では特に断らなければ$(X,\mathcal{M},\mu)$で測度空間を表す. 
\begin{lem}\label{LemmaForLpInequality}
  \begin{itemize}
      \item[(1)] $1 \leq p < \infty$とすると, 
      \[
      (a+b)^p \leq 2^{p-1}(a^p + b^p) ~~~(\forall a,b \geq 0).
      \]
      \item[(2)] $0 < p < 1$とすると, 
      \[
      (a+b)^p \leq a^p + b^p ~~~(\forall a,b \geq 0).
      \]
      \item[(3)] $1 \leq p < \infty, q = p/(p-1)$とすると, 
      \[
      ab \leq \frac{a^p}{p} + \frac{b^q}{q} ~~~(\forall a,b \geq 0).
      \]
      \item[(4)] $0 < p < 1, q = p/(p-1)$とすると, 
      \[
      \frac{a^p}{p} + \frac{b^q}{q} \leq ab ~~~(\forall a \geq 0, b > 0).
      \]
  \end{itemize}
\end{lem}
\begin{proof}
  $(1)$: $t \in [0,1]$に対して$f(t) = (1+t)^p/(1+t^p)$とおくと, 
  \[
  \begin{aligned}
  f'(t) 
  &= p\frac{(1+t)^{p-1}}{(1+t^p)} - pt^{p-1}\frac{(1+t)^p}{(1+t^p)^2} \\
  &= \frac{p(1+t)^{p-1}}{(1+t^p)^2}((1+t^p) - t^{p-1}(1+t))) \\
  &= \frac{p(1+t)^{p-1}}{(1+t^p)^2}(1-t^{p-1}) \geq 0
  \end{aligned}
  \]
  である. ゆえに$f(1) = 2^{p-1}$が$f$の最大値である. よって, 任意の$t \in [0,1]$に対して
  \[
  (1+t)^p \leq 2^{p-1}(1+t^p)
  \]
  である. これより所望の不等式を得る($a/b$または$b/a$を考えればよい). \\
  $(2)$: 上の$f$について$0<p<1$より$f'(t) \leq 0$となるから$f(0) = 1$が最大値になる. これより従う. 
  \\
  $(3)$ $t \in [0,\infty)$に対して$f(t) = t^{p}/p + b^{q}/q - bt$とおくと, 
  \[
  f'(t) = t^{p-1} - b
  \]
  であるので$t = b^{1/(p-1)}$で$f$は最小値をとる. 
  \[
  f(b^{1/(p-1)}) = \frac{b^{p/(p-1)}}{p} + \frac{b^q}{q} - b^{p/(p-1)} = \frac{b^q}{p} + \frac{b^q}{q} - b^q = 0
  \]
  であるので, 
  \[
  0 = f(b^{1/(p-1)}) \leq f(a) = \frac{a^p}{p} + \frac{b^q}{q} - ab
  \]
  である. よって成立する. 
  \\
  $(4)$: $t \in [0,\infty)$に対して$f(t) = t^{1/p} - t/p - 1/q$とおくと, 
  \[
  f'(t) = \frac{1}{p}t^{1/p -1} - \frac{1}{p} = \frac{1}{p}(t^{1/p-1}-1) 
  \]
  である. ゆえに, $1/p-1 > 0$より$0 \leq t < 1$のとき$f'(t) \leq 0$であり
  , $1<t<\infty$のとき$f'(t) \geq 0$である. そして$f'(1)=0$であるので, 
  \[
  f(1) = 1 - 1/p -1/q = 0
  \]
  が$f$の最小値である($1/p + 1/q = 1$に注意). よって, 
  \[
  0 = f(1) \leq f(a^p/b^q) = \frac{a}{b^{q/p}} - \frac{a^p}{pb^q} - \frac{1}{q}
  \]
  であるから, 両辺に$b^q$をかけると, 
  \[
  0 \leq ab^{q-q/p} - \frac{a^p}{p} - \frac{b^q}{q}
  \]
  となる. $q/p = q-1$であるのでこれで所望の不等式が得られた. 
\end{proof}
\begin{prop}[H\"{o}lderの不等式]\label{HolderInequality}　\\
  $(1)$: $1 \leq p < \infty$とし, $q = p/(p-1)$とする. このとき, 可測関数$f, g: X \rightarrow \R$に対して, 
  \[
  \lVert fg \rVert_1 \leq \lVert f \rVert_p \lVert g \rVert_q. 
  \]
  \\
  $(2)$: $0 < p < 1$とし, $q = p/(p-1)$とする. このとき, $f \in L^p(X,\mathcal{M},\mu), g \in L^q(X,\mathcal{M},\mu)$に対して, 
  \[
  \lVert f \rVert_p \lVert g \rVert_q \leq \lVert fg \rVert_1
  \]
  が成り立つ. ただし, $r < 0$に対して$|0|^r = +\infty$と解釈し, $h \in L^r(X,\mathcal{M},\mu)$であるとは, $h$が可測関数であり, 
  \[
  \lVert h \rVert_r^r := \int_X |h|^{r} d\mu < \infty
  \]
  であることと定義する(このとき特に$\mu$-a.e.で$|h| > 0$であり, $\lVert h \rVert_r > 0$である). 
\end{prop}
\begin{proof}　\\
  $(1)$: $\lVert f \rVert_p = 0$ならば$f=0 ~(\mathrm{a.e.})$となるので$\lVert f \rVert_p > 0$としてよい. 同様に$\lVert g \rVert_q > 0$としてよい. さらに, このとき$f \in L^p(X,\mathcal{M},\mu), g \in L^q(X,\mathcal{M},\mu)$としてよい. 
  まず$1 < p < \infty$の場合を考える. $x \in X$に対して$a = |f(x)|/\lVert f \rVert_p$,$b = |g(x)|/\lVert g \rVert_q$とおいて, すぐ上の補題の$(3)$を適用すると, 
  \[
  \frac{|f(x)||g(x)|}{\lVert f \rVert_p \lVert g \rVert_q} \leq \frac{|f(x)|^p}{p \lVert f \rVert_p^p} + \frac{|g(x)|^q}{q \lVert g \rVert_q}
  \]
  となる. よって, 両辺を積分して$1/p + 1/q = 1$を使えば所望の不等式が得られる. 
  次に$p=1$の場合を考える. このとき, $\mu$-a.e.$x$に対して$|f(x)g(x)| \leq |f(x)|\lVert g \rVert_{\infty}$が成り立つので両辺を積分すれば所望の不等式を得る. 
  \\
  $(2)$: $\lVert f \rVert_p > 0$としてよい. このとき, $a = |f(x)|/\lVert f \rVert_p, b = |g(x)|/\lVert g \rVert_q$とおき, すぐ上の補題の$(4)$を適用すると, $\mu$-a.e.$x \in X$に対して
  \[
  \frac{|f(x)|^p}{p\lVert f \rVert_p^p} + \frac{|g(x)|^q}{q\lVert g \rVert_q^q} 
  \leq \frac{|f(x)||g(x)|}{\lVert f \rVert_p \lVert g \rVert_q}
  \]
  となる. したがって, 両辺を積分すれば所望の不等式を得る. 
\end{proof}
\begin{prop}[Minkowskiの不等式]　\\
  $0 < p < 1$とする. このとき, $f,g \in L^p(X,\mathcal{M},\mu)$に対して, 
  \[
  \begin{aligned}
  \lVert f + g &\rVert_p^p \leq \lVert |f| + |g| \rVert_p^p \leq \lVert f \rVert_p^p + \lVert g \rVert_p^p, \\
  &\lVert f \rVert_p + \lVert g \rVert_p \leq \lVert |f| + |g| \rVert_p.
  \end{aligned}
  \]
\end{prop}
\begin{proof}
  補題\ref{LemmaForLpInequality}$(2)$より, 
  \[
  |f(x)+g(x)|^p \leq (|f(x)|+|g(x)|)^p \leq |f(x)|^p + |g(x)|^p
  \]
  であるので積分することで第一の不等式を得る. 次に第二の不等式についてはa.e.$x$で$f(x),g(x) \neq 0$と仮定してよいので, a.e.$x$に対して
  \[
  (|f(x)| + |g(x)|)^p = (|f(x)|+|g(x)|)^{p-1}|f(x)| + (|f(x)|+|g(x)|)^{p-1}|g(x)|
  \]
  であるから, H\"{o}lderの不等式より
  \[
  \begin{aligned}
  \lVert |f| + |g| \rVert_p^p 
  &= \lVert (|f|+|g|)^{p-1}f \rVert_1 + \lVert (|f|+|g|)^{p-1}g \rVert_1 \\
  &\geq \lVert (|f|+|g|)^{p-1} \rVert_{q} (\lVert f \rVert_p + \lVert g \rVert_p) \\
  &= \lVert |f| + |g| \rVert_{p}^{p/q} (\lVert f \rVert_p + \lVert g \rVert_p)
  \end{aligned}
  \]
  となる. なお, $(|f|+|g|)^{p-1} \in L^q(X,\mathcal{M},\mu)$については第一の不等式からわかる. よって, 
  \[
  \lVert |f| + |g| \rVert_p = \lVert |f| + |g| \rVert_p^{p-p/q} \geq \lVert f \rVert_p + \lVert g \rVert_p. 
  \]
\end{proof}
\begin{lem}
  $1 < p \leq 2$とし, $q = p/(p-1)$とすると, 
    \[
    ( 1+x )^q + (1-x)^q  \leq 2 (1 + x^p )^{q-1} ~~~(\forall x \in [0,1]).
    \]
\end{lem}
\begin{proof}
  任意に$0 < x < 1$を取り, 
  \[
  f_x(s) := (1+s^{1-q}x)(1+sx)^{q-1} + (1-s^{1-q} x)(1-sx)^{q-1} ~~~(s \in [0,1])
  \]
  とおく. $1< p \leq 2$より$0 < x^{p-1} < 1$であり, $(q-1)(p-1) = 1$より
  \[
  f_x(x^{p-1}) = (1+x^{0})(1+x^p)^{q-1} + (1-x^0)(1-x^p)^{q-1} = 2(1+x^p)^{q-1}
  \]
  である. また, $f_x(1) = (1+x)^q + (1-x)^q$である. よって, $s \in (0,1)$に対して$f_x'(s) \leq 0$が成り立つことを示せば, 
  \[
  (1+x)^q + (1-x)^q = f_x(1) \leq f_x(x^{p-1}) = 2(1+x^p)^{q-1}
  \]
  がわかる. さて, $s \in (0,1)$に対して, 
  \[
  \begin{aligned}
  f_x'(s) 
  &= (1-q)s^{-q}x(1+sx)^{q-1} + (q-1)x(1+s^{1-q}x)(1+sx)^{q-2} \\
  &~~~ -(1-q)s^{-q}x(1-sx)^{q-1} - (q-1)x(1-s^{1-q}x)(1-sx)^{q-2} \\
  &= (q-1)x \left[
  \begin{aligned}
  &(1+sx)^{q-2}( -s^{-q}(1+sx) +(1+s^{1-q}x) ) \\
  &+ (1-sx)^{q-2} (s^{-q}(1-sx) - (1-s^{1-q}x))
  \end{aligned}
  \right] \\
  &= (q-1)x(1-s^{-q})[(1+sx)^{q-2} - (1-sx)^{q-2}]
  \end{aligned}
  \]
  である. $1 < p \leq 2$より$2 \leq q < \infty$であるので, $s \in (0,1)$に対して, $1-s^{-q} < 0$かつ$(1+sx)^{q-2} - (1-sx)^{q-2} > 0$である. よって, $f_x'(s) \leq 0$である. 
\end{proof}
\begin{lem}\label{ClarksonLemma}
  以下の不等式が成り立つ. 
  \begin{itemize}
      \item[(1)] $1 < p \leq 2$とし, $q = p/(p-1)$とすると, 
      \[
       | z+w |^q + |z-w|^q  \leq 2 (|z|^p + |w|^p )^{q-1} ~~~(\forall z,w \in \mathbb{C}).
      \]
      \item[(2)] $2 \leq p < \infty$とし, $q = p/(p-1)$とすると, 
      \[
      |z+w|^q + |z-w|^q \geq 2(|z|^p + |w|^p)^{q-1} ~~~(\forall z,w \in \mathbb{C}).
      \]
  \end{itemize}
\end{lem}
\begin{proof}
  $(1)$: $z,w=0$のときは成り立つので$z,w \neq 0$とする. $0 < |z| \leq |w|$としてよい. このとき, $z/w = r \exp(i\theta) , 0 < r \leq 1, 0 \leq \theta \leq 2\pi$と表せる. ゆえに, $q/(q-1) = p$に注意すると, 所望の不等式を得るためには任意の$\theta \in [0,2\pi]$に対して
  \[
  | 1 + r\exp(i\theta) |^q + |-1 + r\exp(i\theta) |^q
  \leq 2(r^p + 1)^{q-1}
  \]
  が成り立つことを示せば十分である. $\theta = 0$の場合はすぐ上の補題から成り立つ. そのほかの$\theta$について見るために, 
  \[
  \begin{aligned}
  f(\theta) 
  :&= | 1 + r\exp(i\theta) |^q + |-1 + r\exp(i\theta) |^q \\
  &= ((1+r\cos(\theta))^2 + r^2\sin^2(\theta))^{q/2} + ((-1+r\cos(\theta))^2 + r^2\sin^2(\theta))^{q/2} \\
  &= (1+2r\cos(\theta) + r^2)^{q/2} + (1-2r\cos(\theta) + r^2)^{q/2}
  \end{aligned}
  \]
  とおく. $f(0)$が$f$の最大値であることを示せば証明は完了する. 対称性から$\theta \in [0,\pi/2]$の場合を示せば十分である. $\theta \in [0,\pi/2]$に対して
  \[
  \begin{aligned}
  f'(\theta) 
  &= \frac{q}{2}(-2r \sin(\theta))(1+2r\cos(\theta) + r^2)^{\frac{q}{2}-1} \\
  &~~~ + \frac{q}{2}(2r\sin(\theta))(1-2r\cos(\theta) + r^2)^{\frac{q}{2}-1} \\
  &= qr\sin(\theta)\left( (1-2r\cos(\theta) + r^2)^{\frac{q}{2}-1} - (1+2r\cos(\theta) + r^2)^{\frac{q}{2}-1} \right)
  \end{aligned}
  \]
  である. さらに, $1 < p \leq 2$より$q \geq 2$であるから$\cos(\theta) \geq 0$より
  \[
  (1-2r\cos(\theta) + r^2)^{\frac{q}{2}-1} - (1+2r\cos(\theta) + r^2)^{\frac{q}{2}-1} \leq 0
  \]
  である. よって, $f'(\theta) \leq 0$である. これより$f$は$\theta = 0$で最大値をとる. \\
  $(2)$: $1 < q \leq 2$かつ$p=q/(q-1)$である. したがって, $(1)$より, 
  \[
  | z+w |^p + |z-w|^p  \leq 2 (|z|^q + |w|^q )^{p-1} ~~~(\forall z,w \in \mathbb{C}).
  \]
  ゆえに, $q(p-1) = p$に注意すると, 任意の$u,v \in \mathbb{C}$に対して$z=(u+v)/2, w=(u-v)/2$とおくことで, 
  \[
  |u|^p + |v|^p \leq 2 \frac{1}{2^{p}} (|u+v|^q+|u-v|^q)^{p-1} = 2^{1-p}(|u+v|^q+|u-v|^q)^{p-1}
  \]
  となる. $p-1=1/(q-1)$であるから, 
  \[
  2^{1/(q-1)} (|u|^p + |v|^p) \leq (|u+v|^q+|u-v|^q)^{1/(q-1)}
  \]
  となる. よって, 両辺を$q-1$乗することで所望の不等式を得る. 
\end{proof}
\begin{thm}[Clarksonの不等式1]\label{ClarksonInequality1}　\\
  $(X,\mathcal{M},\mu)$を測度空間とする. $p \in (1,\infty), q = p/(p-1)$とする. このとき, 任意の$f,g \in L^p(X,\mathcal{M},\mu)$に対して, 
  \begin{itemize}
      \item[(1)] $1 < p \leq 2$ならば, 
      \[
      \left\lVert f+g \right\rVert_{p}^q + \left\lVert f-g \right\rVert_{p}^q 
      \leq 2 \left( \left\lVert f \right\rVert_{p}^p + \left\lVert g \right\rVert_{p}^p \right)^{q-1}.
      \]
      \item[(2)] $2 \leq p < \infty$ならば, 
      \[
      \left\lVert f+g \right\rVert_{p}^q + \left\lVert f-g \right\rVert_{p}^q 
      \geq 2 \left( \left\lVert f \right\rVert_{p}^p + \left\lVert g \right\rVert_{p}^p \right)^{q-1}.
      \]
  \end{itemize}
\end{thm}
\begin{rem}
  上の定理で$p=2$の場合は内積空間の中線定理に他ならない. これらの不等式は$L^p$空間がHilbert空間に近い性質を持つことを含意する\cite{Takahashi2004}. 
\end{rem}
\begin{proof}
  $(1)$: $p < 2$としてよい. $q(p-1)=p$に注意する. $0 < p-1 < 1$であるから Minkowskiの不等式より, 
  \[
  \begin{aligned}
  \lVert |f+g|^q + |f-g|^q \rVert_{p-1}
  &\geq \lVert |f+g|^q \rVert_{p-1} + \lVert |f-g|^q \rVert_{p-1} \\
  &=\lVert f + g \rVert_p^{p/(p-1)} + \lVert f - g \rVert_p^{p/(p-1)} \\
  &=\lVert f + g \rVert_p^{q} + \lVert f + g \rVert_p^{q}
  \end{aligned}
  \]
  である. 一方, $(q-1)(p-1)=1$および補題\ref{ClarksonLemma}$(1)$より, 
  \[
  \begin{aligned}
  \lVert |f+g|^q + |f-g|^q \rVert_{p-1}
  &\leq \lVert 2(|f|^p + |g|^p)^{q-1} \rVert_{p-1} \\
  &= 2 \left( \int_X (|f|^p + |g|^p)^{(q-1)(p-1)} d\mu \right)^{1/(p-1)} \\
  &= 2 \left( \lVert f \rVert_p^p + \lVert g \rVert_p^p \right)^{q-1}
  \end{aligned}
  \]
  である. よって, 所望の不等式を得る. \\
  $(2)$: $q(p-1)=p$に注意する. $p-1 \geq 1$であるから 三角不等式より, 
  \[
  \begin{aligned}
  \lVert |f+g|^q + |f-g|^q \rVert_{p-1}
  &\leq \lVert |f+g|^q \rVert_{p-1} + \lVert |f-g|^q \rVert_{p-1} \\
  &=\lVert f + g \rVert_p^{p/(p-1)} + \lVert f - g \rVert_p^{p/(p-1)} \\
  &=\lVert f + g \rVert_p^{q} + \lVert f + g \rVert_p^{q}
  \end{aligned}
  \]
  である. 一方, $(q-1)(p-1)=1$および補題\ref{ClarksonLemma}$(2)$より, 
  \[
  \begin{aligned}
  \lVert |f+g|^q + |f-g|^q \rVert_{p-1}
  &\geq \lVert 2(|f|^p + |g|^p)^{q-1} \rVert_{p-1} \\
  &= 2 \left( \int_X (|f|^p + |g|^p)^{(q-1)(p-1)} d\mu \right)^{1/(p-1)} \\
  &= 2 \left( \lVert f \rVert_p^p + \lVert g \rVert_p^p \right)^{q-1}
  \end{aligned}
  \]
  である. よって, 所望の不等式を得る. 
\end{proof}
次にClarksonにより示されたもうひとつの不等式を示していく.
\begin{lem}\label{ClarksonLemma2}
  \begin{itemize}
      \item[(1)] $1 < p \leq 2$とすると, 
      \[
       x^p + y^p \geq (x^2 + y^2)^{p/2} ~~~(\forall x,y \geq 0).
      \]
      \item[(2)] $2 \leq p < \infty$とすると, 
      \[
       x^p + y^p \leq (x^2 + y^2)^{p/2} ~~~(\forall x,y \geq 0).
      \]
  \end{itemize}
\end{lem}
\begin{proof}
  $(1)$: 関数$f:[0,1] \ni t \mapsto (t^p+1)-(1+t^2)^{p/2} \in \R$が非負であることを示せばよいが, $1 < p \leq 2$より$s \mapsto s^{p/2 -1}$は単調減少であるので, 
  \[
  f'(t) = pt^{p-1} - \frac{p}{2}(2t)(1+t^2)^{p/2 -1} \geq pt^{p-1} - ptt^{p-1} = 0
  \]
  となる. よって, $f(0) = 0$と合わせて$f \geq 0$である. \\
  $(2)$: $(1)$の証明中の$f$について, $2 \leq p < \infty$より$s \mapsto s^{p/2 -1}$が単調増加であることから, 
  \[
  f'(t) = pt^{p-1} - \frac{p}{2}(2t)(1+t^2)^{p/2 -1} \leq pt^{p-1} - ptt^{p-1} = 0
  \]
  となるので$f \leq 0$である. これよりわかる. 
\end{proof}
\begin{lem}
  以下の不等式が成り立つ. 
  \begin{itemize}
      \item[(1)] $1 < p \leq 2$とすると, 
      \[
       2^{p-1}(|z|^p + |w|^p) \leq | z+w |^p + |z-w|^p  \leq 2 (|z|^p + |w|^p ) ~~~(\forall z,w \in \mathbb{C}).
      \]
      \item[(2)] $2 \leq p < \infty$とすると, 
      \[
       2^{p-1}(|z|^p + |w|^p) \geq | z+w |^p + |z-w|^p  \geq 2 (|z|^p + |w|^p ) ~~~(\forall z,w \in \mathbb{C}).
      \]
  \end{itemize}
\end{lem}
\begin{proof}
  $(1)$: 第一の不等式については, $1<p \leq 2$より$s \mapsto s^{p/2}$が凹関数であることおよび $\mathbb{C}$における中線定理と補題\ref{ClarksonLemma2}$(1)$から
  \[
  \begin{aligned}
  |z+w|^p + |z-w|^p 
  &\leq ( |z+w|^2 + |z-w|^2 )^{p/2} \\
  &= (2 (|z|^2 + |w|^2) )^{p/2} \\
  &= 2^{p/2} 2^{p/2} \left( \frac{|z|^2}{2} + \frac{|w|^2}{2} \right) \\
  &\leq 2^{p} \frac{1}{2} (|z|^p + |w|^p)
  \end{aligned}
  \]
  となるので成り立つ. 第二の不等式は第一の不等式において$z = u+v, w = u-v$とすることで得られる. \\
  $(2)$: 第一の不等式については, $\mathbb{C}$における中線定理と補題\ref{ClarksonLemma2}$(2)$および$r \geq 1, a,b \geq 0$に対する不等式$(a+b)^r \leq 2^{r-1} (a^r+b^r)$より, 
  \[
  \begin{aligned}
    |z+w|^p + |z-w|^p 
    &\leq (|z+w|^2 + |z-w|^2)^{p/2} \\
    &= ( 2(|z|^2 + |w|^2) )^{p/2} \\
    &\leq 2^{p/2} 2^{p/2 -1} (|z|^p + |w|^p) \\
    &= 2^{p-1} (|z|^p + |w|^p)
  \end{aligned}
  \]
  となるので成り立つ. 第二の不等式は第一の不等式において$z = u+v, w = u-v$とすることで得られる. 
\end{proof}
\begin{thm}[Clarksonの不等式2]\label{ClarksonInequality2}　\\
  $(X,\mathcal{M},\mu)$を測度空間とする. $p \in (1,\infty), q = p/(p-1)$とする. このとき, 任意の$f,g \in L^p(X,\mathcal{M},\mu)$に対して, 
  \begin{itemize}
      \item[(1)] $1 < p \leq 2$ならば, 
      \[
      \left\lVert f+g \right\rVert_{p}^p + \left\lVert f-g \right\rVert_{p}^p 
      \leq 2  (\left\lVert f \right\rVert_{p}^p + \left\lVert g \right\rVert_{p}^p ).
      \]
      \item[(2)] $2 \leq p < \infty$ならば, 
      \[
      \left\lVert f+g \right\rVert_{p}^p + \left\lVert f-g \right\rVert_{p}^p 
      \geq 2 \left( \left\lVert f \right\rVert_{p}^p + \left\lVert g \right\rVert_{p}^p \right).
      \]
  \end{itemize}
\end{thm}
\begin{proof}
  前補題を$f(x),g(x)$に対して適用して全体を積分すればよい.
\end{proof}
次の不等式の証明は\cite{Sanguineti2005}による. 
\begin{thm}[Clarksonの不等式]　\\
  $(X,\mathcal{M},\mu)$を測度空間とする. $p \in (1,\infty), q = p/(p-1), a = \min\{p,q\}$とすると, 任意の$f,g \in L^p(X,\mathcal{M},\mu)$に対して, 
  \[
  \lVert f+g \rVert_p^a + \lVert f-g \rVert_p^a \leq 2 (\lVert f \rVert_p^a + \lVert g \rVert_p^a).
  \]
\end{thm}
\begin{proof}
  まず$1 < p \leq 2$の場合を考える. このとき, $p \leq 2 \leq q < \infty$であるので$a = \min\{p,q\} = p$である. したがって, 定理\ref{ClarksonInequality2}$(1)$より, 
  \[
  \begin{aligned}
  \lVert f+g \rVert_p^a + \lVert f-g \rVert_p^a
  &= \lVert f+g \rVert_p^p + \lVert f-g \rVert_p^p \\
  &\leq 2(\lVert f \rVert_p^p + \lVert g \rVert_p^p) \\
  &= 2 (\lVert f \rVert_p^a + \lVert g \rVert_p^a)
  \end{aligned}
  \]
  である. 次に$2 \leq p < \infty$の場合を考える. このとき, $1< q \leq 2 \leq p$であるので $a = \min\{p,q\} = q$である. 定理\ref{ClarksonInequality1}$(2)$において, $f=u+v, g=u-v$とおくことで, 
  \[
  2^{q-1}\left( \lVert u \rVert_p^q + \lVert v \rVert_p^q \right) \geq \left( \lVert u+v \rVert_p^p + \lVert u-v \rVert_p^p \right)^{q-1}
  \]
  を得る. $(p-1)(q-1)=1$であるから, 両辺を$p-1$乗して
  \[
  \lVert u+v \rVert_p^p + \lVert u-v \rVert_p^p \leq 2\left( \lVert u \rVert_p^q + \lVert v \rVert_p^q \right)^{p-1}
  \]
  を得る. 一方, 補題\ref{LemmaForLpInequality}$(1)$より任意の$r \in [1,\infty), x,y \geq 0$に対して, $(x+y)^r \leq 2^{r-1}(x^r + y^r)$であるから, $q(p-1)=p$に注意すると, 
  \[
  \begin{aligned}
  \lVert u+v \rVert_p^p + \lVert u-v \rVert_p^p
  &= \lVert u+v \rVert_p^{q(p-1)} + \lVert u-v \rVert_p^{q(p-1)} \\
  &\geq \frac{1}{2^{p-2}} \left( \lVert u+v \rVert_p^q + \lVert u-v \rVert_p^q \right)^{p-1}
  \end{aligned}
  \]
  である. よって, 
  \[
  \left( \lVert u+v \rVert_p^q + \lVert u-v \rVert_p^q \right)^{p-1} \leq 2^{p-1} \left( \lVert u+v \rVert_p^q + \lVert u-v \rVert_p^q \right)^{p-1}
  \]
  であるので, $a=q$より
  \[
  \lVert u+v \rVert_p^a + \lVert u-v \rVert_p^a \leq 2 \left( \lVert u+v \rVert_p^a + \lVert u-v \rVert_p^a \right).
  \]
\end{proof}

\subsection{Jackson評価}
本節ではJackson評価の証明を与える. 本節を書くにあたって参照した文献は, Mhaskar\cite{Mhaskar1995}, Mhaskar, Micchelli\cite{MhaskarandMicchelli1995}, Schultz\cite{Schultz}, 猪狩\cite{igarifourier}である.  
まず次の不等式を示す. 
\begin{thm}[積分形のMinkowskiの不等式]　\\
  $1 \leq p < \infty$とし$(X,\mathcal{M},\mu),(Y,\mathcal{N},\nu)$を$\sigma$-有限な測度空間とする. このとき, 直積可測空間$X \times Y$上の可測関数$f \geq 0$に対して, 
  \[
  \left\{ \int_X \left( \int_Y f(x,y) d\nu(y) \right)^p d\mu(x) \right\}^{1/p} \leq \int_Y \left( \int_X f(x,y)^p d\mu(x) \right)^{1/p} d\nu(y).
  \]
\end{thm}
\begin{proof}
  $p=1$の場合はFubini-Tonelliの定理より明らか. $1 < p < \infty$とする. 
  \[
  F(x) := \int_{Y} f(x,y) d\nu(y)
  \]
  とおくと, $F$は$X$上の可測関数である(Fubini-Tonelliの定理を参照されたい). そして, Fubini-Tonelliの定理とH\"{o}lderの不等式より
  \[
  \begin{aligned}
  \int_X F(x)^p d\mu(x)
  &= \int_X F(x)^{p-1} \int_Y f(x,y) d\nu(y) d\mu(x) \\
  &= \int_X \int_Y F(x)^{p-1} f(x,y) d\nu(y) d\mu(x) \\
  &= \int_Y \int_X F(x)^{p-1} f(x,y) d\mu(x) d\nu(y) \\
  &\leq \int_Y \left( \int_X F(x)^{q(p-1)} d\mu(x) \right)^{1/q} \left( \int_X f(x,y)^p d\mu(x) \right)^{1/p} d\nu(y) \\
  &= \left( \int_X F(x)^p d\mu(x) \right)^{1/q} \int_Y \left( \int_X f(x,y)^p d\mu(x) \right)^{1/p} d\nu(y)
  \end{aligned}
  \]
  となるので$\left( \int_X F(x)^p d\mu(x) \right)^{1/q}$で両辺を割れば所望の不等式を得る. 
\end{proof}
\begin{cor}\label{CorOfMinkowskiInequality}
  $1 \leq p < \infty$とし$(X,\mathcal{M},\mu),(Y,\mathcal{N},\nu)$を$\sigma$-有限な測度空間とする. このとき, $Y$上の可測関数$g \geq 0$と直積可測空間$X \times Y$上の可測関数$f \geq 0$に対して, 
  \[
  \left\{ \int_X \left( \int_Y g(y)f(x,y) d\nu(y) \right)^p d\mu(x) \right\}^{1/p} \leq \int_Y g(y) \left( \int_X f(x,y)^p d\mu(x) \right)^{1/p} d\nu(y).
  \]
\end{cor}
\subsubsection{一般Jackson核の性質}
\begin{defn}[三角多項式]
  $N = 0,1,2,\ldots$に対して, 
  \[
  \sum_{|\alpha| \leq N} c_{\alpha} \exp(i\alpha^{T}x) ~~~~(c_{\alpha} \in \mathbb{C},~ x \in \R^d)
  \]
  と表される関数を高々$N$次の$d$変数三角多項式という.  ここで$\alpha \in \mathbb{Z}^d$であり, 
  \[
  |\alpha| = |\alpha_1| + \cdots + |\alpha_d|
  \]
  である. また, $|\alpha|=N$を満たすある$\alpha \in \mathbb{Z}^d$に対して$c_{\alpha} \neq 0$となるならば$N$次の$d$変数三角多項式という. 
\end{defn}
\begin{defn}[Dirichlet核]　\\
  $N = 0,1,\ldots $に対して次の関数を$N$次のDirichlet核という. 
  \[
  D_N(x) := \frac{\sin\left(N+\frac{1}{2}\right)}{\sin\frac{x}{2}} ~~~~(x \in \R)
  \]
\end{defn}
\begin{prop}
  次の式が成り立つ. 
  \[
  D_N(x) = 1+2\sum_{n=1}^N \cos(nx) = \sum_{n=-N}^{N} \exp(inx) ~~~~(x \in \R).
  \]
  特に$N$次のDirichlet核は$N$次の$1$変数三角多項式である. 
\end{prop}
\begin{proof}
  複素数$z \neq 1$に対して, 
  \[
  1+2\sum_{n=1}^N z^n  = \frac{1+z-2z^{n+1}}{1-z}
  \]
  であるので$z = \exp(ix)$とおき右辺の分母分子に$\exp(-ix/2)$をかけると
  \[
  \frac{\cos(x/2) - \left[ \cos\left(N+\frac{1}{2}\right)x + i\sin\left(N+\frac{1}{2}\right)x \right]}{-i\sin(x/2)}
  \]
  となる. よって, 実部を比較することで所望の等式を得る. 
\end{proof}
\begin{defn}[Fej\'{e}r核]　\\
  $N = 0,1,2,\ldots $に対して次の関数を$N$次のFej\'{e}r核という. 
  \[
  F_N(x) := \frac{1}{N+1}\left[ \frac{\sin\frac{1}{2}(N+1)x}{\sin\frac{1}{2}x} \right]^2
  \]
\end{defn}
\begin{prop}
  次の式が成り立つ. 
  \[
  \begin{aligned}
  F_N(x) = 1 + 2 \sum_{n=1}^N \left( 1- \frac{n}{N+1} \right) \cos nx = \sum_{n=-N}^N \left(1 - \frac{|n|}{N+1} \right) \exp(i nx)
  \end{aligned}
  \]
  特に$N$次のFej\'{e}r核は$N$次の$1$変数三角多項式である. 
\end{prop}
\begin{proof}
  $D_N(x) = \sum_{n=-N}^{N} \exp(inx)$を使うと, 
  \[
  \sum_{n=-N}^N \left(1 - \frac{|n|}{N+1} \right) \exp(i nx) = \frac{1}{N+1}(D_0(x)+D_1(x)+\cdot+D_N(x))
  \]
  がわかる. したがって, $D_N$の定義より
  \[
  \begin{aligned}
  &\sum_{n=-N}^N \left(1 - \frac{|n|}{N+1} \right) \exp(i nx) \\
  &= \frac{1}{N+1}\frac{1}{2\left(\sin(x/2)\right)^2}\left( \sum_{n=0}^N 2\sin(n+\frac{1}{2})x\cdot \sin(x/2) \right) \\
  &= \frac{1}{N+1}\frac{1}{2\left(\sin(x/2)\right)^2}\left( \sum_{n=0}^N (\cos nx - \cos(n+1)x ) \right) \\
  &= \frac{1}{N+1}\frac{1}{2\left(\sin(x/2)\right)^2}\left( 1-\cos(N+1)x \right) \\
  &= \frac{1}{N+1}\frac{1}{2\left(\sin(x/2)\right)^2}2 \sin\left( \frac{N+1}{2}x \right) = F_N(x)
  \end{aligned}
  \]
  となる. 
\end{proof}
\begin{defn}[一般Jackson核]　\\
  $N = 1,2,3,\ldots $と$r = 2,3,4,\cdots$に対して次の関数を一般Jackson核という. 
  \[
  J_{N,r}(x) := \frac{1}{c_{N,r}}\left[ \frac{\sin\frac{N}{2}x}{\sin\frac{1}{2}x} \right]^{2r}.
  \]
  ここに定数$c_{N,r}$は
  \[
  \frac{1}{2\pi} \int_{-\pi}^{\pi} J_{N,r}(x) dx = 1
  \]
  となるように選ぶ. 
\end{defn}
\begin{rem}
  一般Jackson核は$N-1$次のFej\'{e}r核を$r$乗したものなので高々$r(N-1)$次の$1$変数三角多項式である. 
\end{rem}
一般Jackson核の定数$c_{N,r}$について次が成り立つ.
\begin{prop}
   $r = 2,3,4,\cdots$に対して, 
  \[
  \left( \frac{2}{\pi} \right)^{2r} N^{2r-1} \leq c_{N,r} \leq \pi^{2r-1} \left( \frac{1}{2^{2r}} + \frac{1}{2r-1} \right)  N^{2r-1} ~~~~(\forall N = 0,1,2,\ldots).
  \]
\end{prop}
\begin{proof}
  \[
  f(x) = \left[ \frac{\sin\frac{N}{2}x}{\sin\frac{1}{2}x} \right]^{2r}
  \]
  とおくと定義から
  \[
  c_{N,r} = \frac{1}{\pi} \int_0^\pi f(x) dx
  \]
  である. $2x/\pi \leq \sin x \leq x ~(0 \leq x \leq \pi/2)$より, 
  \[
  \begin{aligned}
  c_{N,r} 
  &= \frac{1}{\pi} \left( \int_0^{1/N} f(x) dx + \int_{1/N}^{\pi} f(x) dx \right) \\
  &\leq \frac{1}{\pi} \left( \int_0^{1/N} \left(\frac{Nx/2}{x/\pi} \right)^{2r} dx + \int_{1/N}^{\pi} \left( \frac{1}{x/\pi} \right)^{2r} dx \right) \\
  &\leq \frac{1}{\pi} \left( \left(\frac{\pi N}{2} \right)^{2r} \int_0^{1/N} 1 dx + \pi^{2r} \int_{1/N}^{\infty} \frac{1}{x^{2r}} dx \right) \\
  &= \pi^{2r-1} \left( \frac{N^{2r-1}}{2^{2r}} + \frac{N^{2r-1}}{2r-1} \right)
  \end{aligned}
  \]
  である. また, 
  \[
  \begin{aligned}
  c_{N,r} 
  &\geq \frac{1}{\pi}\int_{0}^{\pi/N} f(x) dx \geq \frac{1}{\pi} \int_0^{\pi/N} \left( \frac{\frac{2}{\pi}\frac{N}{2}x}{x/2} \right)^{2r} dx \\
  &= \frac{1}{\pi} \left( \frac{\frac{2}{\pi}\frac{N}{2}x}{x/2} \right)^{2r} \frac{\pi}{N} = \left( \frac{2}{\pi} \right)^{2r} N^{2r-1}.
  \end{aligned}
  \]
\end{proof}
この命題より一般Jackson核$J_{N,r}$について次の評価が成り立つことがわかる. 
\begin{prop}\label{BoundedJacksonKernel}
  $N=1,2,3,\ldots$と$r=2,3,4,\ldots$に対して次が成り立つ.
   \[
   \begin{aligned}
   & J_{N,r}(x) \leq \left( \frac{\pi}{2}\right)^{4r} N ~~~~(|x| \leq \pi), \\
   & J_{N,r}(x) \leq \left(\frac{\pi^2}{2x}\right)^{2r} N^{-2r+1} ~~~(|x| \leq \pi).
   \end{aligned}
   \]
\end{prop}
\begin{proof}
  第一の不等式は$c_{N,r}$の下からの評価と$2x/\pi \leq \sin x \leq x ~(0 \leq x \leq \pi/2)$より
  \[
  J_{N,r}(x) = \frac{1}{c_{N,r}} \left( \frac{\sin \frac{N}{2}x }{\sin \frac{x}{2}} \right)^{2r} \leq \frac{1}{N^{2r-1}} \left( \frac{\pi}{2} \right)^{2r} \left( \frac{\frac{N}{2}x }{ \frac{1}{\pi}x} \right)^{2r} = \left( \frac{\pi}{2}\right)^{4r} N
  \]
  となり成立する. 第二の不等式は同様に
  \[
  J_{N,r}(x) \leq \frac{1}{N^{2r-1}} \left( \frac{\pi}{2} \right)^{2r} \left( \frac{ 1 }{ \frac{1}{\pi}x} \right)^{2r} =  \left(\frac{\pi^2}{2x}\right)^{2r} N^{-2r+1}.
  \]
\end{proof}
\begin{lem}\label{BoundedJacksonKernelIntegral}
  $N = 2,3,\ldots$と$r=2,3,4,\ldots$に対して, $q=0,1,\ldots,2r-2$ごとに定数$C_q > 0$が存在して
  \[
  \frac{1}{\pi} \int_0^\pi x^q J_{N,r}(x) dx \leq C_q \frac{1}{N^q}
  \]
  となる. さらに定数$C_q$として$C_q \leq 2^{-4r+q+2}\pi^{4r-1}$を満たすものが取れる. 
\end{lem}
\begin{proof}
  命題\ref{BoundedJacksonKernel}より
  \[
  \begin{aligned}
  \frac{1}{\pi} \int_0^\pi x^q J_{N,r}(x) dx
  &= \frac{1}{\pi} \left( \int_0^{2/N} x^q J_{N,r}(x) dx + \int_{2/N}^\pi x^q J_{N,r}(x) dx \right) \\
  &\leq \frac{1}{\pi} \left( \int_0^{2/N} x^q \left( \frac{\pi}{2}\right)^{4r} N dx + \int_{2/N}^\infty x^q \left(\frac{\pi^2}{2x}\right)^{2r} N^{-2r+1} dx \right) \\
  &= \frac{\pi^{4r-1}}{2^{4r}} \left( N \int_0^{2/N} x^q dx + 2^{2r}N^{-2r+1} \int_{2/N}^\infty x^{q-2r}  dx \right) \\
  &= \frac{\pi^{4r-1}}{2^{4r}} \left(  \frac{N}{q+1} \left( \frac{2}{N} \right)^{q+1} + \frac{2^{2r}N^{-2r+1}}{2r-q-1} \left( \frac{2}{N} \right)^{1+q-2r} \right) \\
  &= \frac{\pi^{4r-1}}{2^{4r}} \left( \frac{2^{q+1}}{q+1} + \frac{2^{q+1}}{2r-q-1} \right) \frac{1}{N^q} \leq \pi^{4r-1} 2^{-4r+q+2} \frac{1}{N^q}.
  \end{aligned}
  \]
\end{proof}
\subsubsection{差分, 滑率の性質}
\begin{defn}[差分, 滑率]　\\
  関数$f:\R \rightarrow \R$と$r = 1,2,3,\ldots$および$h \in \R$に対して
  \[
  \Delta_h^r(f)(x) := \sum_{k=0}^r \binom{r}{k} (-1)^{r-k} f(x+kh) ~~~~(x \in \R)
  \]
  を$f$の$r$階の差分という. また, $f$が$2\pi$周期の周期関数で$I=(-\pi,\pi) \subset \R$と$p \in [1,\infty]$に対して$f \in L^p(I)$のとき
  \[
  \omega_r(f)_{p,I}(t) := \sup_{0 < h \leq t} \lVert \Delta_h^r(f) \rVert_{L^p(I)} ~~~~(t > 0)
  \]
  を$f$の$r$階の滑率(modulus of smoothness)という. 
\end{defn}
\begin{prop}
  関数$f:\R \rightarrow \R$と$h \in \R$に対して次が成り立つ. 
  \[
  \Delta_{h}^1 (\Delta_{h}^{r-1}(f)) =\Delta_h^r(f) = \Delta_h^{r-1}(\Delta_h^1(f)) ~~~~(r = 1,2,3,\ldots)
  \]
  となる. ただし, $\Delta_h^{0}(f)=f$とする. 
\end{prop}
\begin{proof}
  $r$に関する帰納法. $r=1$での成立は明らか. $r$で成立すると仮定すると, 
  \[
  \binom{r}{k} + \binom{r}{k-1} = \binom{r+1}{k}
  \]
  であるから, 
  \[
  \begin{aligned}
  \Delta_h^{r}(\Delta_h^1(f))(x)
  &= \sum_{k=0}^r \binom{r}{k} (-1)^{r-k} \Delta_h^1(f)(x+kh) \\
  &= \sum_{k=0}^r \binom{r}{k} (-1)^{r-k} [f(x+(k+1)h)-f(x+kh)] \\
  &= \sum_{k=1}^r (-1)^{r+1-k}\left( \binom{r}{k} + \binom{r}{k-1}  \right) f(x+kh) \\
  &~~~+(-1)^{r+1}f(x) + f(x+(r+1)h) \\
  &= \sum_{k=0}^{r+1} \binom{r+1}{k} (-1)^{r+1-k} f(x+kh) = \Delta_h^{r+1}(f)(x).
  \end{aligned}
  \]
  よって, 第二の等式は成立する. 同様に
  \[
  \begin{aligned}
  \Delta_h^1(\Delta_h^r(f))(x) 
  &= \Delta_h^r(f)(x+h) - \Delta_h^r(f)(x) \\
  &= \sum_{k=0}^r \binom{r}{k} (-1)^{r-k} f(x+(k+1)h) - \sum_{k=0}^r \binom{r}{k} (-1)^{r-k} f(x+kh) \\
  &= \Delta_h^{r+1}(f)(x)
  \end{aligned}
  \]
  であるから第一の等式も成立する. 
\end{proof}
\begin{prop}
  関数$f:\R \rightarrow \R$と整数$n \geq 1$, $h \in \R,~r=1,2,3,\ldots$に対して次の等式が成り立つ. 
  \[
  \Delta_{nh}^r(f)(x) = \sum_{k_1=0}^{n-1} \cdots \sum_{k_r=0}^{n-1} \Delta_h^r(f)(x+k_1 h + \cdots + k_r h) ~~~~(x \in \R).
  \]
\end{prop}
\begin{proof}
  $r$に関する帰納法による. まず, 
  \[
  \sum_{k_1=0}^{n-1} \Delta_h^1(f)(x+k_1 h) = \sum_{k=0}^{n-1} [f(x+(k+1)h - f(x+kh)] = f(x+nh) - f(x) = \Delta_{nh}^1(f)(x)
  \]
  であるから$r=1$で成立する. 次に$r$で成立すると仮定すると,  $\Delta_{nh}^{r+1} = \Delta_{nh}^1 \Delta_{nh}^r$であるから, 
  \[
  \begin{aligned}
  &\Delta_{nh}^{r+1}(f)(x)  \\
  &= \Delta_{nh}^1\left(\sum_{k_1=0}^{n-1} \cdots \sum_{k_r=0}^{n-1} \Delta_h^r(f)(\cdot +k_1 h + \cdots + k_r h)\right)(x) \\
  &= \sum_{k_1=0}^{n-1} \cdots \sum_{k_r=0}^{n-1} \left[\Delta_h^r(f)(x + nh +k_1 h + \cdots + k_r h) - \Delta_h^r(f)(x +k_1 h + \cdots + k_r h) \right] \\
  &= \sum_{k_1=0}^{n-1} \cdots \sum_{k_r=0}^{n-1} \sum_{k_{r+1}=0}^{n-1} 
  \left[
  \begin{aligned}
  &\Delta_h^r(f)(x +k_1 h + \cdots + k_r h + (k_{r+1}+1)h ) \\
  &- \Delta_h^r(f)(x +k_1 h + \cdots + k_r h + k_{r+1} h) 
  \end{aligned}
  \right] \\
  &= \sum_{k_1=0}^{n-1} \cdots \sum_{k_r=0}^{n-1} \sum_{k_{r+1}=0}^{n-1} \Delta_h^1\left(\Delta_h^r(f)(\cdot +k_1 h + \cdots + k_r h + k_{r+1}h )\right)(x) \\
  &= \sum_{k_1=0}^{n-1} \cdots \sum_{k_r=0}^{n-1}  \sum_{k_{r+1}=0}^{n-1} \Delta_h^{r+1}(f)(x +k_1 h + \cdots + k_r h + k_{r+1}h ).
  \end{aligned}
  \]
\end{proof}
\begin{prop}\label{ModuliOfSmoothnessEstimate}
  $I = (-\pi,\pi) \subset \R$とおく. 周期$2\pi$の周期関数$f:\R \rightarrow \R$であって, ある$p \in [1,\infty]$について$f \in L^p(I)$となるものに対して次の不等式が成り立つ. 
  \[
  \omega_r(f)_{p,I}(nt) \leq n^r \omega_r(f)_{p,I}(t), ~~~ \omega_r(f)_{p,I}(\lambda t) \leq (\lambda+1)^r \omega_r(f)_{p,I}(t) ~~~(\lambda > 0).
  \]
\end{prop}
\begin{proof}
  すぐ上の命題より
  \[
  \begin{aligned}
  w_r(f)_{p,I}(nt) 
  &= \sup_{0 < h \leq nt} \lVert \Delta_h^r(f) \rVert_{L^p(I)} 
  = \sup_{0 < h \leq t} \lVert \Delta_{nh}^r(f) \rVert_{L^p(I)} \\
  &\leq \sup_{0 < h \leq t} \sum_{k_1=0}^{n-1} \cdots \sum_{k_r=0}^{n-1}  \left\lVert \Delta_{h}^r(f)(\cdot + k_1 h+\cdots+k_r h) \right\rVert_{L^p(I)}
  \end{aligned}
  \]
  となる. ゆえに$p=\infty$の場合は$f$の周期性より成立する. 
  次に$p < \infty$とする. $f$の周期性(およびすぐ下の注意)より
  \[
  \begin{aligned}
  &\left\lVert \Delta_{h}^r(f)(\cdot + k_1 h+\cdots+k_r h) \right\rVert_{L^p(I)}^p \\
  &= \int_{-\pi}^\pi \left\lvert \sum_{k=0}^r \binom{r}{k} (-1)^{r-k} f(x+kh+k_1 h + \cdots + k_r h) \right\rvert^p dx \\
  &= \int_{-\pi + k_1 h + \cdots + k_r h}^{\pi + k_1 h + \cdots + k_r h} \left\lvert \sum_{k=0}^r \binom{r}{k} (-1)^{r-k} f(x+kh) \right\rvert^p dx \\
  &= \int_{-\pi}^{\pi} \left\lvert \sum_{k=0}^r \binom{r}{k} (-1)^{r-k} f(x+kh) \right\rvert^p dx = \left\lVert \Delta_h^r(f) \right\rVert_{L^p(I)}^p.
  \end{aligned}
  \]
  である. よって, 
  \[
  \begin{aligned}
  w_r(f)_{p,I}(nt) 
  &\leq \sup_{0 < h \leq t} \sum_{k_1=0}^{n-1} \cdots \sum_{k_r=0}^{n-1}  \left\lVert \Delta_{h}^r(f)(\cdot + k_1 h+\cdots+k_r h) \right\rVert_{L^p(I)} \\
  &= n^r \sup_{0 < h \leq t} \left\lVert \Delta_h^r(f) \right\rVert_{L^p(I)} = n^r \omega_r(f)_{p,I}(t).
  \end{aligned}
  \]
  第二式については$n < \lambda \leq n+1$なる整数$n$をとると, 前半より
  \[
  w_r(f)_{p,I}(\lambda t) \leq w_r(f)_{p,I}((n+1)t) \leq (n+1)^r \omega_r(f)_{p,I}(t) \leq (\lambda+1)^r \omega_r(f)_{p,I}(t).
  \]
\end{proof}
\begin{rem}
  周期$2a$の関数$f:\R \rightarrow \R$と$b \in \R$に対して
  \[
  \int_{-a+b}^{a+b} f(x) dx = \int_{-a}^a f(x) dx
  \]
  である. 実際, 
  \[
  \begin{aligned}
  \int_{-a+b}^{a+b} f(x) dx 
  &= \int_{-a+b}^{a} f(x) dx + \int_{a}^{a+b} f(x) dx \\
  &= \int_{-a+b}^{a} f(x) dx + \int_{-a}^{-a+b} f(x+2a) dx \\
  &= \int_{-a+b}^{a} f(x) dx + \int_{-a}^{-a+b} f(x) dx = \int_{-a}^a f(x) dx.
  \end{aligned}
  \]
\end{rem}
\begin{prop}
  $I= (-\pi,\pi)$とおく. 周期$2\pi$の周期関数$f:\R \rightarrow \R$であって, ある$p \in [1,\infty]$について$f \in L^p(I)$であるものに対して, 
  \[
  \lVert \Delta_t^r(f) \rVert_{L^p(I)} \leq w_r(f)_{p,I}(|t|) ~~~~(|t| < \pi). 
  \]
\end{prop}
\begin{proof}
  $t \geq 0$のときは定義より明らかである. $t \in (-\pi,0)$とする.
  $f$の周期性より
  \[
  \begin{aligned}
  \lVert \Delta_t^r(f) \rVert_{L^p(I)} 
  &= \left\lVert  \sum_{k=0}^r \binom{r}{k} (-1)^{r-k} f(\cdot+kt)  \right\rVert_{L^p(I)} \\
  &= \left\lVert  \sum_{k=0}^r \binom{r}{k} (-1)^{r-k} f(\cdot+(k-r)t) \right\rVert_{L^p(I)} \\
  &= \left\lVert \sum_{k=0}^r \binom{r}{r-k} (-1)^{k} f(x+k(-t)) \right\rVert_{L^p(I)} \\
  &= \left\lVert (-1)^{-r} \sum_{k=0}^r \binom{r}{k} (-1)^{r-k} f(\cdot+k(-t)) \right\rVert_{L^p(I)} \\
  &= \lVert \Delta_{-t}^r(f) \rVert_{L^p(I)} \leq w_r(f)_{p,I}(|t|).
  \end{aligned}
  \]
\end{proof}
\subsubsection{$L^p$かつ$2\pi$周期で$1$変数の場合}
\begin{lem}
  関数$g:\R \rightarrow \R$はある正の整数$k$について$2\pi/k$周期の関数であるとする. このとき, $k$で割り切れない任意の整数$l$について
  \[
  \int_{-\pi}^\pi g(t) \cos(lt) dt = \int_{-\pi}^\pi g(t) \sin(lt) dt = 0. 
  \]
\end{lem}
\begin{proof}
  $g$は$2\pi$周期関数でもあるので, 
  \[
  \begin{aligned}
  \int_{-\pi}^\pi g(t) \exp(ilt) dt 
  &= \int_{-\pi+2\pi/k}^{\pi+2\pi/k} g(s+2\pi/k) \exp(il(s+2\pi/k)) ds \\
  &= \exp(2\pi il/k) \int_{-\pi+2\pi/k}^{\pi+2\pi/k} g(s) \exp(ils) ds \\
  &= \exp(2\pi il/k) \int_{-\pi}^{\pi} g(s) \exp(ils) ds
  \end{aligned}
  \]
  である. $\exp(2\pi il/k) \neq 1$であるので所望の等式を得る. 
\end{proof}
\begin{lem}\label{DefApproOperater}
  $r \geq 2$とする. 整数$n \geq 1$に対して$m(n,r)= \lfloor n/r \rfloor + 1$とおき$K_{n,r} := J_{m(n,r), r}$とおく(ただし, $\lfloor x \rfloor$は$x$を超えない最大の整数を表し, $J_{N,r}$は一般Jackson核である). すると, 定数$C_r > 0$が存在して
  \[
  \int_0^\pi t^q K_{n,r}(t) dt \leq C_r \frac{1}{n^q} ~~~~(q=0,1,\ldots,2r-2)
  \]
  となる. また, 周期$2\pi$の周期関数$f:\R \rightarrow \R$で$f \in L^1(I)$なるものに対して
  \[
  S_{n,r}(f)(x) := \int_{-\pi}^\pi \left( \Delta_t^r(f)(x) + f(x) \right) K_{n,r}(t) dt
  \]
  と定めると, $S_{n,r}(f)$は高々$n$次の三角多項式である. 
  
\end{lem}
\begin{proof}
  (前半)：補題\ref{BoundedJacksonKernelIntegral}より定数$C_r > 0$が存在して$q = 0,1,\ldots,2r-2$に対して
  \[
  \int_0^\pi t^q K_{n,r}(t) dt = \int_0^{\pi} t^q J_{m(n,r),r}(t) dt \leq C_r \frac{1}{m(n,r)^q} \leq C_r \frac{r^q}{n^q}
  \]
  となる. したがって, 改めて$C_r r^{2r-2}$を$C_r$とおくことにすれば
  \[
  \int_0^\pi t^q K_{n,r}(t) dt \leq C_r \frac{1}{n^q} ~~~~(q=0,1,\ldots,2r-2)
  \]
  となる. \\
  (後半)：一般Jackson核の定義より$K_{n,r} = J_{m(n,r),r}$は高々$r ( m(n,r)-1) = r \lfloor n/r \rfloor \leq n$次の三角多項式であり, 偶関数であるので$K_{n,r}(t)$は$1,\cos(t),\cos(2t),\ldots,\cos(nt)$の線形結合で表される. ゆえに$S_{n,r}(f)(x)$は
  \[
  \int_{-\pi}^{\pi} f(x+kt) \cos(lt) dt ~~~~(k=0,\ldots,r,~l=1,\ldots,n)
  \]
  の線形結合で表される. $k=0$のとき上の式は定数である. 
  $k=1,\ldots,r$に対しては$t \mapsto f(x+kt)$は$2\pi/k$周期なのですぐ上の補題より$k$で割り切れない任意の$l=1,\ldots,n$について
  \[
  \int_{-\pi}^{\pi} f(x+kt) \cos(lt) dt = 0
  \]
  となる. また, $k$で割り切れる任意の$l=1,\ldots,n$については$m=l/k$とおくと
  \[
  \begin{aligned}
  &\int_{-\pi}^{\pi} f(x+kt) \cos(lt) dt \\
  &= \int_{x-k\pi}^{x+k\pi} f(u) \cos\left( \frac{l(u-x)}{k} \right) \frac{1}{k} du ~~~~(u = x+kt)\\
  &= \int_{x-k\pi}^{x+k\pi} \frac{1}{k} f(u) (\cos(mu)\cos(mx) + \sin(mu)\sin(mx)) du \\
  &= \left(\int_{x-k\pi}^{x+k\pi} \frac{1}{k} f(u) \cos(mu) du \right)\cos(mx)  + \left(\int_{x-k\pi}^{x+k\pi} \frac{1}{k} f(u) \sin(mu) du \right) \sin(mx) \\
  &= \left(\int_{-k\pi}^{k\pi} \frac{1}{k} f(u) \cos(mu) du \right)\cos(mx) + \left(\int_{-k\pi}^{k\pi} \frac{1}{k} f(u) \sin(mu) du \right) \sin(mx)
  \end{aligned}
  \]
  となる. よって, $m \leq n$に注意すれば$S_{n,r}(f)$が高々$n$次の三角多項式であることがわかる.
\end{proof}
\begin{thm}\label{JacksonEstimate0}
  $I= (-\pi,\pi)$とおき, $r \geq 2$, $p \in [1,\infty]$とする. このとき, 次をみたす定数$C$が存在する: 周期$2\pi$の周期関数$f:\R \rightarrow \R$であって, あるについて$f \in L^p(I)$であるものに対して
  \[
  \lVert S_{n,r}(f) - f \rVert_{L^p(I)} \leq C \omega_r(f)_{p,I}\left(\frac{1}{n}\right) ~~~~(n=1,2,3,\ldots)
  \]
  となる. 
\end{thm}
\begin{proof}
  すぐ上の補題の前半より, $n$に依存しない定数$C$が存在して
  \[
  \int_{-\pi}^{\pi} K_{n,r}(t) (n|t|+1)^r dt = 2 \int_0^\pi K_{n,r}(t) (n|t|+1)^r \leq C
  \]
  となる. まず$p=\infty$の場合を考える. 
  \[
  \int_{-\pi}^\pi K_{n,r}(t) dt = 1
  \]
  であることに注意すれば, 任意の$x \in I$に対して命題\ref{ModuliOfSmoothnessEstimate}より
  \[
  \begin{aligned}
  |S_{n,r}(f)(x) - f(x)|
  &= \left\lvert \int_{-\pi}^\pi \left[ \Delta_t^r(f)(x) + f(x) \right] K_{n,r}(t) dt -  \int_{-\pi}^\pi f(x) K_{n,r}(t) dt \right\rvert \\
  &= \left\lvert \int_{-\pi}^\pi \Delta_t^r(f)(x) K_{n,r}(t) dt \right\rvert \leq \int_{-\pi}^\pi \lvert \Delta_t^r(f)(x) K_{n,r}(t) \rvert dt \\
  &\leq \int_{-\pi}^\pi \lVert \Delta_t^r(f)(x) \rVert_{L^\infty(I)} K_{n,r}(t) dt \leq \int_{-\pi}^{\pi} K_{n,r}(t) \omega_r(f)_{\infty,I}(|t|) dt \\
  &\leq \int_{-\pi}^{\pi} K_{n,r}(t) \omega_r(f)_{\infty,I}\left(n|t|\frac{1}{n}\right) dt \\
  &\leq \omega_r(f)_{\infty,I}\left(\frac{1}{n}\right) \int_{-\pi}^{\pi} K_{n,r}(t) (n|t|+1)^r dt \leq C \omega_r(f)_{\infty,I}\left(\frac{1}{n}\right).
  \end{aligned}
  \]
  次に$1 \leq p < \infty$の場合を考える. 積分形のMinkowskiの不等式および命題\ref{ModuliOfSmoothnessEstimate}より
  \[
  \begin{aligned}
  &\lVert S_{n,r}(f) - f \rVert_{L^p(I)} \\
  &= \left( \int_{-\pi}^{\pi} \left|\int_{-\pi}^\pi  \Delta_t^r(f)(x) K_{n,r}(t) dt \right|^p dx \right)^{1/p} \\
  &\leq \int_{-\pi}^{\pi} K_{n,r}(t) \lVert \Delta_t^r(f) \rVert_{L^p(I)} dt \leq \int_{-\pi}^{\pi} K_{n,r}(t) \omega_r(f)_{p,I}(|t|) dt \\
  &\leq \omega_r(f)_{p,I}\left(\frac{1}{n}\right) \int_{-\pi}^{\pi} K_{n,r}(t) (n|t|+1)^r dt \leq C \omega_r(f)_{p,I}\left(\frac{1}{n}\right).
  \end{aligned} 
  \]
\end{proof}

\subsubsection{$L^p$かつ$2\pi$周期で多変数の場合}
この節では多変数の場合を扱う. 

\begin{lem}
  $I = (-\pi,\pi)$とおき, $1 \leq p \leq \infty$とする. また, $T$を$L^p(I)$から$L^p(I)$への有界線形写像とする. このとき, $f \in L^p(I^d)$に対して$f$の$i$番目の変数以外を固定して得られる$1$変数関数に$T$を作用させたものを$T_i f$とかくことにする. つまり, $x_1,\ldots,x_{i-1},x_i,x_{i+1},\ldots,x_d \in I$に対して関数$f_i$を
  \[
  f_i:I \ni x \rightarrow f(x_1,\ldots,x_{i-1},x,x_{i+1},\ldots,x_d) \in \R
  \]
  により定め$(T_i f)(x_1,\ldots,x_d) := (T f_i)(x_i)$と定義する. すると$T_i f \in L^p(I^d)$である. 
\end{lem}
\begin{proof}
  \[
  \lVert T_i f \rVert_{L^p(I^d)} = \lVert \lVert T_i f \rVert_{L^p(I)} \rVert_{L^p(I^{d-1})} \leq \lVert \lVert T \rVert \lVert f_i \rVert_{L^p(I)} \rVert_{L^p(I^{d-1})} = \lVert T \rVert \lVert f \rVert_{L^p(I^d)}.
  \]
\end{proof}
\begin{lem}\label{LemforMultiJacksonEst}
  $I = (-\pi,\pi)$とおき, $1 \leq p \leq \infty$とする. また, $T$を$L^p(I)$から$L^p(I)$への有界線形写像とする. このとき, $f \in L^p(I^d)$に対して, $\varepsilon_i := \lVert f - T_i f  \rVert_{L^p(I^d)}$とおき, $N = 1,\ldots,d$に対して$T^N = T_N T_{N-1} \cdots T_1$とおくと, 
  \[
  \lVert f - T^N f \rVert_{L^p(I^d)} \leq \sum_{i=1}^N  \varepsilon_i \lVert T \rVert^{N-i} .
  \]
\end{lem}
\begin{proof}
  $N$に関する帰納法. $N=1$での成立は明らか. $N$で成立するとすると, 
  \[
  \begin{aligned}
  \lVert f - T^{N+1} f \rVert_{L^p(I^d)}
  &\leq \lVert f -  T_{N+1} f \rVert_{L^p(I^d)} + \lVert T_{N+1} f - T_{N+1} T^N f \rVert_{L^p(I^d)} \\
  &= \lVert f -  T_{N+1} f \rVert_{L^p(I^d)} + \lVert \lVert T_{N+1} (f - T^N f) \rVert_{L^p(I)} \rVert_{L^p(I^{d-1})} \\
  &\leq \varepsilon_{N+1} + \lVert \lVert T \rVert \lVert  f - T^N f \rVert_{L^p(I)} \rVert_{L^p(I^{d-1})} \\
  &= \varepsilon_{N+1} + \lVert T \rVert \lVert  f - T^N f \rVert_{L^p(I^d)} \\ 
  &\leq \varepsilon_{N+1} + \lVert T \rVert \sum_{i=1}^N  \varepsilon_i \lVert T \rVert^{N-i} = \sum_{i=1}^{N+1}  \varepsilon_i \lVert T \rVert^{N+1-i}
  \end{aligned}
  \]
  となるので$N+1$でも成立する. 
\end{proof}

\begin{prop}
  $I= (-\pi,\pi)$とおき, $r,d \geq 1$とする. $n = 1,2,3,\ldots$に対して
  \[
  K_{n,r}^d := J_{\left\lfloor \frac{n}{rd}  \right\rfloor+1,r}
  \]
  とおく. このとき, 周期$2\pi$の周期関数$f:\R \rightarrow \R$であって, ある$p \in [1,\infty]$について$f \in L^p(I)$であるものに対して,  
  \[
  S_{n,r}^d (f)(x) := \int_{-\pi}^{\pi} \left( \Delta_t^r(f)(x) + f(x) \right) K_{n,r}^d(t) dt
  \]
  と定めると, $S_{n,r}^d (f)$は周期$2\pi$の周期関数であって$L^p(I)$に属し, 
  \[
  \lVert S_{n,r}^d(f) \rVert_{L^p(I)} \leq (2^r +1) \lVert f \rVert_{L^p(I)} ~~~~(n=1,2,3,\ldots).
  \]
\end{prop}
\begin{proof}
  補題\ref{DefApproOperater}と同様にして$S_{n,r}^d(f)$は三角多項式であることがわかるので周期$2\pi$の周期関数である. 不等式を示そう. 
  まず$p=\infty$の場合を考える. このとき, 任意の$x \in [-\pi,\pi]$に対して
  \[
  \begin{aligned}
  |S_{n,r}^d(f)(x)| 
  &\leq \int_{-\pi}^\pi |\Delta_t^r(f)(x)+f(x)| K_{n,r}^d(t) dt \\
  &\leq \left( \sum_{k=0}^r \binom{r}{k} + 1 \right)\lVert f \rVert_{L^\infty(I)} \int_{-\pi}^\pi K_{n,r}^d(t) dt =\left( 2^r + 1 \right)\lVert f \rVert_{L^\infty(I)}
  \end{aligned}
  \]
  となるのでよい. 次に$p < \infty$の場合を考える. このとき, 積分形のMinkowskiの不等式より
  \[
  \begin{aligned}
  \lVert S_{n,r}^d(f) \rVert_{L^p(I)} 
  &\leq \int_{-\pi}^\pi K_{n,r}^d(t) \left( \int_{-\pi}^\pi |\Delta_t^r(f)(x)+f(x)|^p dx \right)^{1/p} dt \\
  &\leq \int_{-\pi}^\pi K_{n,r}^d(t) \left( \sum_{k=0}^r \binom{r}{k} + 1 \right) \lVert f \rVert_{L^p(I)} dt = \left( 2^r + 1 \right)\lVert f \rVert_{L^p(I)}
  \end{aligned}
  \]
  となる. 
\end{proof}

\begin{prop}\label{JacksonEstimate1}
  $I = (-\pi,\pi)$とおき, $1 \leq p \leq \infty,~r \geq 2$とする. $f:\R^d \rightarrow \R$は各変数ごとに周期$2\pi$の周期関数で$f \in L^p(I)$であるとする. このとき, $x = (x_1,\ldots,x_d)$に対して
  \[
  S_{n,r,j}^d(f)(x) := \int_{-\pi}^\pi \left[ \sum_{k=0}^r \binom{r}{k} (-1)^{r-k} f(x_1,\ldots,x_j+kt,\ldots,x_d) + f(x) \right] K_{n,r}^d(t) dt
  \]
  とおき, $k=1,\ldots,d$に対して$S^k = S_{n,r,k}^d S_{n,r,k-1}^d \cdots S_{n,r,1}^d$とおくと, $S^d f$は高々$n$次の$d$変数三角多項式であり, 
  \[
  \lVert f - S^d f \rVert_{L^p(I^d)} \leq \sum_{j=1}^d \lVert f - S_{n,r,j}^d f \rVert_{L^p(I^d)} (2^r+1)^{d-j} .
  \]
\end{prop}
\begin{proof}
  $k=1,\ldots,d$に対し$S^k f$は各変数ごとに$2\pi$周期関数であり, $x_{k+1},\ldots,x_d$を固定するとき$S^k f$は$x_1,\ldots,x_k$の関数として高々$k \lfloor n/d \rfloor$次の$k$変数三角多項式であることを帰納法で示す. 
  まず$S^1 f$が各変数ごとに$2\pi$周期関数であることは$f$の周期性と$S_{n,r,1}$の定義より明らか.  そして$x_2,\ldots,x_d$を固定するとき$f$の周期性より補題\ref{DefApproOperater}と同様に$S_{n,r,1}(f)(x_1,\ldots,x_d)$は$x_1$の関数として高々$\lfloor n/d \rfloor$次の$1$変数三角多項式である. 次に$k$で成り立つと仮定する. $S^{k+1} f$の周期性は$S^k f$の周期性より従う. $x_1,\ldots,x_i,x_{k+2},\ldots,x_d$を固定する. $S^k f$の周期性より補題\ref{DefApproOperater}と同様に$S^{k+1} f$は$x_{k+1}$の関数として
  \[
  \cos(m x_{k+1}),~\sin(m x_{k+1}) ~~~~(m=0,\ldots,\lfloor n/d \rfloor)
  \]
  の線形結合で表されることがわかる.  そして, その係数のうち$x_1,\ldots,x_k$に依存する部分は補題\ref{DefApproOperater}の証明から
  \[
  \begin{aligned}
  &\int_{-m\pi}^{m\pi} \frac{1}{k} (S^k f)(x_1,\ldots,x_k,u,x_{k+2},\ldots,x_d)\cos(ju) du ~~~~(j=0,\ldots,\lfloor n/d \rfloor), \\
  &\int_{-m\pi}^{m\pi} \frac{1}{k} (S^k f)(x_1,\ldots,x_k,u,x_{k+2},\ldots,x_d)\cos(ju) du ~~~~(j=0,\ldots,\lfloor n/d \rfloor)
  \end{aligned}
  \]
  である. 一方, 帰納法の仮定より$S^k f$は$x_1,\ldots,x_k$の関数として高々$k \lfloor n/d \rfloor$次の$k$変数三角多項式であるので
  \[
  (S^k f)(x_1,\ldots,x_k,u,x_{k+2},\ldots,x_d) = \sum_{|\alpha|\leq k \lfloor n/d \rfloor} c_{u,x_{k+2},\ldots,x_d}^{(\alpha)} \exp(i(\alpha_1 x_1 + \cdots + \alpha_i x_k))
  \]
  と表される. よって, 
  \[
  \cos(mx) = \frac{\exp(imx) + \exp(-imx)}{2}, ~~ \sin(mx) = \frac{\exp(imx) - \exp(-imx)}{2}
  \]
  であることに注意すれば$S^{k+1} f$は$x_1,\ldots,x_{k+1}$の関数として
  \[
  \exp(i (\alpha_1 x_1 + \alpha_{k+1} x_{k+1})) ~~~~(|\alpha| \leq (k+1) \lfloor n/d \rfloor)
  \]
  の線形結合で表される. これで$k+1$の成立が示された. このことより前半の主張は成立する.  後半の主張については前命題と補題\ref{LemforMultiJacksonEst}より従う. 
\end{proof}

\subsubsection{Sobolev空間の元の場合}
本小節では, $r \geq 1$, $1 \leq p \leq \infty$に対して定数$C_{r,p} > 0$が存在して, 周期$2\pi$の周期関数$f:\R \rightarrow \R$で任意の有界開区間$J \subset \R$について$f \in W^{r,p}(J)$であるものに対して
\[
\omega_r(f)_{p,I}\left(\frac{1}{n}\right) \leq C_{r,p} \frac{1}{n^r} \lVert f^{(r)} \rVert_{L^p(I)} ~~~~(n = 1,2,3,\ldots) ~~~~\tag{1}
\]
となることを示していく(ただし, $I=(-\pi,\pi)$である). これが示されれば次のことがわかる. 
\begin{thm}\label{JacksonEstimate3}
  $I = (-\pi,\pi)$とおき, $1 \leq d < \infty$, $1 \leq p \leq \infty$, $r \geq 2$とする. このとき, 次をみたす定数$C$が存在する. 
  $f:\R^d \rightarrow \R$が各変数ごとに周期$2\pi$の周期関数かつ任意の有界開区間$J \subset \R$について$f \in W^{r,p}(J)$であるならば, 
  \[
  \begin{aligned}
  E_{n,p}(f) &:= \inf \{ \lVert T - f \rVert_{L^p(I^d)} \mid T\mbox{は高々}n\mbox{次の}d\mbox{変数三角多項式} \} \\
  &~\leq C \frac{1}{n^r} \sum_{j=1}^d \left\lVert \partial_j^r f \right\rVert_{L^p(I^d)}  ~~~~(n=1,2,3,\ldots).
  \end{aligned}
  \]
\end{thm}
\begin{proof}
  命題\ref{JacksonEstimate1}より$S^d f$は高々$n$次の$d$変数三角多項式であり, 
  \[
  \lVert f - S^d f \rVert_{L^p(I^d)} \leq \sum_{j=1}^d \lVert f - S_{n,r,j}^d f \rVert_{L^p(I^d)} (2^r+1)^{d-j} .
  \]
  となる. $f,S_{n,r,j}^d f$を$x_j$の関数とみなすとき定理\ref{JacksonEstimate0}と同様にして定数$C_j$が存在して
  \[
  \lVert f - S_{n,r,j}^d \rVert_{L^p(I)} \leq C \omega_r(f)_{p,I}\left( \frac{1}{n} \right) ~~~~(n=1,2,3,\ldots)
  \]
  となる. したがって, (1)より
  \[
  \begin{aligned}
  \lVert f - S^d f \rVert_{L^p(I^d)} 
  &\leq \sum_{j=1}^d \lVert f - S_{n,r,j}^d f \rVert_{L^p(I^d)} (2^r+1)^{d-j} \\
  &\leq (2^r+1)^d \sum_{j=1}^d \lVert \lVert f - S_{n,r,j}^d f \rVert_{L^p(I)} \rVert_{L^p(I^{d-1})} \\
  &\leq (2^r+1)^d C \sum_{j=1}^d \left\lVert \omega_r(f)_{p,I}\left( \frac{1}{n} \right) \right\rVert_{L^p(I^{d-1})} \\
  &\leq C' \frac{1}{n^r} \sum_{j=1}^d \left\lVert \left\lVert \partial_j^r f \right\rVert_{L^p(I)} \right\rVert_{L^p(I^{d-1})} = C' \frac{1}{n^r} \sum_{j=1}^d \left\lVert \partial_j^r f \right\rVert_{L^p(I^d)}
  \end{aligned}
  \]
  となる. これで示せた. 
\end{proof}
さらにこのことから次のこともわかる. 
\begin{thm}\label{JacksonEstimate4}
  $1 \leq d < \infty$, $1 \leq p \leq \infty$, $r \geq 2$とする. このとき, $n$に依存しない定数$C$が存在して, 周期$2\pi$の周期関数$f:\R \rightarrow \R$で任意の有界開区間$J \subset \R$に対して$f \in W^{r,p}(J)$であるものに対して
  \[
  \begin{aligned}
  E_{n,p}(f) &:= \inf \{ \lVert P - f \rVert_{L^p((-1,1)^d)} \mid T\mbox{は高々}n\mbox{次の}d\mbox{変数多項式} \} \\
  &~\leq C \frac{1}{n^r}  \left\lVert  f \right\rVert_{W^{r,p}((-1,1)^d)}  ~~~~(n=1,2,3,\ldots).
  \end{aligned}
  \]
\end{thm}
\begin{proof}
  次をみたす有界線形作用素が存在する(拡張作用素. 宮島\cite{MiyajimaSobolev}を参照). 
  \[
  T:W^{r,p}((-1,1)^d) \rightarrow W^{r,p}(\R^d),~~ (Tf)(x) = f(x) ~~~(\forall x \in (-1,1)^d).
  \]
  $0 \leq \varphi \in C_0^\infty(\R^d)$を$(-1,1)^d$上で$1$をとり$(-3/2,3/2)^d$の外で$0$をとるものとする. $x \in \R^d$に対して$g(x) := \varphi(x)(Tf)(x)$とおく. そして
  $x = (x_1,\ldots,x_d) \in \R^d$に対して
  \[
  h(x) := g(2\cos(x_1),\ldots,2\cos(x_d))
  \]
  とおく. すると, $(-3/2,3/2)^d$の外で$g=0$ゆえ, $h$は各変数ごとに$2\pi$周期の周期関数かつ任意の有界開区間$J \subset \R$に対して$h \in W^{r,p}(J)$である. したがって, 特に, 命題\ref{JacksonEstimate1}より$S^d h$は高々$n$次の$d$変数三角多項式である. また, $h$は各変数ごとに偶関数であるので命題\ref{JacksonEstimate1}の証明より$S^d h$は次のように表される:
  \[
  (S^d h)(x) = \sum_{|\alpha| \leq n} c_{\alpha} \prod_{j=1}^d \cos(\alpha_j x_j) ~~~~(c_{\alpha} \in \R).
  \]
  ここで整数$k \geq 0$に対して高々$k$次の$1$変数多項式$T_k$が存在して
  \[
  T_k(2\cos x) = \cos(kx) ~~~~(\forall x \in \R)
  \]
  となる(Chebyshev多項式). そこで高々$n$次の$d$変数多項式$P_f$を
  \[
  P_f(t) = \sum_{|\alpha| \leq n} c_\alpha \prod T_{\alpha_j}(t_j)
  \]
  で定め, $\Phi(x) = (2\cos(x_1),\ldots,2\cos(x_d))$とおくと, $P(\Phi(x)) = (S^d h)(x)$となる.  よって
  , 変数変換と定理\ref{JacksonEstimate3}の証明より
  \[
  \begin{aligned}
  \lVert f - P_f \rVert_{L^p((-1,1)^d)} 
  &\leq \lVert g - P_f \rVert_{L^p((-2,2)^d)} \leq C \lVert g \circ \Phi - P_f \circ \Phi \rVert_{L^p((-\pi,\pi)^d)} \\
  &= \lVert h - S^d h \rVert_{L^p((-\pi,\pi)^d)} \leq \frac{C'}{n^r} \sum_{j=1}^d \lVert \partial_j^r h \rVert_{L^p((-\pi,\pi)^d)} \\
  &\leq \frac{C''}{n^r} \lVert  g \rVert_{W^{r,p}((-2,2)^d)} ~~~~(\because g(x) = 0 ~(\forall x \notin (-3/2,3/2))) \\
  &\leq \frac{C''}{n^r} \lVert Tf \rVert_{W^{r,p}((-2,2)^d)} \leq \frac{C''}{n^r} \lVert T \rVert \lVert f \rVert_{W^{r,p}((-1,1)^d)} 
  \end{aligned}
  \]
  となる. 
\end{proof}
さて, 小節のはじめに述べたことを以下で示していこう. 
\begin{prop}
  $\R$上の関数$M_r ~~(r=1,2,3,\ldots)$を次のように帰納的に定める. 
  \[
  M_1 := \chi_{[0,1]}, ~~ M_r(x) := (M_{r-1}*M_1)(x) = \int_{\R} M_{r-1}(x-y)M_1(y) dy
  \]
  すると, $\mathrm{supp}M_r \subset [0,r]$, $0 \leq M_r \leq 1$となる. 
\end{prop}
\begin{proof}
  $r=1$での成立は明らか. $r-1$で成立すると仮定する. すると, 
  \[
  0 \leq M_r(x) = \int_{\R} M_{r-1}(x-y)M_1(y) dy \leq \int_{[0,1]} 1 dy = 1 ~~~~(\forall x)
  \]
  であり, $x \in (-\infty,0)\cup (r,\infty)$とすると$y \in [0,1]$に対して$x-y \notin [0,r-1]$であるので$M_r(x)=0$である. 
\end{proof}
\begin{prop}\label{IncDiffRepresent}
  $I = (-\pi,\pi) \subset \R$とおき, $r = 1,2,3,\ldots$に対して, 
  \[
  M_r(x,h) := \frac{1}{h}M_r\left( \frac{x}{h} \right) ~~~~(x \in \R,~ h > 0)
  \]
  と定める. このとき, 周期$2\pi$の周期関数$f:\R \rightarrow \R$で任意の有界開区間$J \subset \R$に対して$f \in W^{r,p}(J)$であるものと$x \in I$について
  \[
  h^{-r} \Delta_h^r(f)(x) = h^{-r} \sum_{k=0}^r \binom{r}{k} (-1)^{r-k} f(x+kh) = \int_{\R} f^{(r)}(x+t)M_r(t,h) dt
  \]
  が成り立つ. 
\end{prop}
\begin{proof}
  $r$に関する帰納法による. まず$r=1$のときは, 周期$2\pi$の周期関数$f:\R \rightarrow \R$で任意の有界開区間$J \subset \R$について$f \in W^{r,p}(J)$であるものに対して
  \[
  \begin{aligned}
  \int_{\R} f^{(1)}(x+t) M_1(t,h) dt 
  &= \int_{\R} f^{(1)}(x+t) h^{-1} \chi_[0,1](h^{-1}t) dt \\
  &= \int_{[0,h]} f^{(1)}(x+t) dt = f(x+h)-f(x)
  \end{aligned}
  \]
  だから成立する. 次に$r-1$で成立するとする. 周期$2\pi$の周期関数$f:\R \rightarrow \R$で任意の有界開区間$J \subset \R$について$f \in W^{r,p}(J)$であるものを任意に取る. このとき$\Delta_h^1(f)$は周期$2\pi$の周期関数で任意の有界開区間$J \subset \R$について$\Delta_h^1(f) \in W^{r,p}(J)$であるので帰納法の仮定より
  \[
  \begin{aligned}
  h^{-r} \Delta_h^r(f,x) 
  &= h^{-1} h^{-(r-1)} \Delta_h^{r-1}(\Delta_h^1(f))(x) \\
  &= h^{-1} \int_{\R} (\Delta_h^1(f))^{(r-1)}(x+t) M_{r-1}(t,h) dt \\
  &= h^{-1} \int_{\R} \left( f^{(r-1)}(x+h+t) - f^{(r-1)}(x+t) \right) M_{r-1}(t,h) dt \\
  &= h^{-1} \int_{\R} \left( \int_0^h f^{(r)}(x+v+t) dv \right) M_{r-1}(t,h) dt \\
  &= h^{-1} \int_{\R} \left( \int_0^h f^{(r)}(x+v+uh) dv \right) M_{r-1}(u) du \\
  &= \int_{\R} \left( \int_0^1 f^{(r)}(x+(v+u)h) dv \right) M_{r-1}(u) du \\
  &= \int_0^1 \left( \int_{\R} f^{(r)}(x+(v+u)h) M_{r-1}(u) du \right) dv \\
  &= \int_0^1 \left( \int_{\R} f^{(r)}(x+wh) M_{r-1}(w-v) dw \right) dv \\
  \end{aligned}
  \]
  \[
  \begin{aligned}
  &= \int_{\R} \left( f^{(r)}(x+wh) \int_0^1 M_{r-1}(w-v) dv \right) dw \\
  &= \int_{\R} f^{(r)}(x+wh) M_r(w) dw \\
  &= \int_{\R} f^{(r)}(x+t) M_r\left( \frac{t}{h}\right) \frac{1}{h} dt \\
  &= \int_{\R} f^{(r)} (x+t) M_r(t,h) dt.
  \end{aligned}
  \]
\end{proof}
\begin{rem}
  上の命題の証明でSobolev空間の元に対する微分積分学の基本定理を用いた. この証明は例えば宮島\cite{MiyajimaSobolev}の命題3.4を参照されたい. 
\end{rem}
\begin{thm}
  $I = (-\pi,\pi) \subset \R$とおく. $r \geq 1$, $1 \leq p \leq \infty$に対して定数$C_{r,p} > 0$が存在して, 周期$2\pi$の周期関数$f:\R \rightarrow \R$で任意の有界開区間$J \subset \R$について$f \in W^{r,p}(J)$であるものに対して
  \[
  \omega_{r}(f)_{p,I}(t) = \sup_{0 < h \leq t} \lVert \Delta_h^r(f) \rVert_{L^p(I)} \leq C_{r,p} t^r \lVert f^{(r)} \rVert_{L^p(I)} ~~~~(\forall t \in (0,2\pi/r)).
  \]
\end{thm}
\begin{proof}
  命題\ref{IncDiffRepresent}より$x \in I$に対して
  \[
  \Delta_h^r(f)(x) = h^r \int_{\R} f^{(r)}(x+t)M_r(t,h) dt
  \]
  であるので, $r,p$のみに依存する定数$C_{r,p} > 0$が存在して
  \[
  \left\lVert \int_{\R} f^{(r)}(\cdot+t)M_r(t,h) dt \right\rVert_{L^p(I)} \leq C_{r,p} \lVert f^{(r)} \rVert_{L^p(I)} ~~~~(\forall h \in (0,2\pi/r) )
  \]
  となることを示せばよい. まず$p=\infty$の場合は
  \[
  \begin{aligned}
  \int_{\R} f^{(r)} (x+t) M_r(t,h) dt 
  &= \int_{\R} f^{(r)}(x+t) M_r\left( \frac{t}{h}\right) \frac{1}{h} dt  \\
  &= \int_{\R} f^{(r)}(x+wh) M_r(w) dw.
  \end{aligned}
  \]
  であることと$\mathrm{supp}M_r \subset [0,r]$より
  \[
  \left\lVert \int_{\R} f^{(r)}(\cdot+t)M_r(t,h) dt \right\rVert_{L^\infty(I)} \leq r \lVert f^{(r)} \rVert_{L^\infty(I)}
  \]
  となり成立する. 次に$p < \infty$の場合を考える. Minkowskiの不等式(系\ref{CorOfMinkowskiInequality})より
  \[
  \begin{aligned}
  &\left( \int_I \left\lvert \int_{\R} f^{(r)}(x+t)M_r(t,h) dt  \right\rvert^p dx \right)^{1/p} \\
  &\leq \left( \int_I \left( \int_{\R} \lvert f^{(r)}(t)\rvert M_r(t-x,h) dt  \right)^p dx \right)^{1/p} \\
  &\leq \int_{\R} |f^{(r)}(t)| \left( \int_I M_r(t-x,h)^p dx \right)^{1/p} dt 
  \end{aligned}
  \]
  である. ここで, 
  \[
  \begin{aligned}
  \int_I M_r(t-x,h)^p dx
  &= \int_{-\pi}^{\pi} \frac{1}{h}M_r\left(\frac{t-x}{h} \right)^p dx \\
  &= \int_{\frac{t-\pi}{h}}^{\frac{t+\pi}{h}} M_r(y)^p dy 
  \end{aligned}
  \]
  であり$\mathrm{supp}M_r \subset [0,r]$であるので$t \in (-\infty,-\pi] \cup [\pi+rh,\infty)$のとき
  \[
  \int_I M_r(t-x,h)^p dx = 0
  \]
  である. よって, $h \in (0,2\pi/r)$に注意すれば, 
   \[
  \begin{aligned}
  &\int_{\R} |f^{(r)}(t)| \left( \int_I M_r(t-x,h)^p dx \right)^{1/p} dt \\
  &= \int_{-\pi}^{3\pi} |f^{(r)}(t)| \left( \int_{\frac{t-\pi}{h}}^{\frac{t+\pi}{h}} M_r(y)^p dy  \right)^{1/p} dt \\
  &\leq \int_{-\pi}^{3\pi} |f^{(r)}(t)| \left( \int_0^r 1 dy  \right)^{1/p} dt = 2r^{1/p} \int_{-\pi}^{\pi} |f^{(r)}(t)| dt 
  \end{aligned}
  \]
  となる. これで$p=1$の場合の成立はわかる. $1 < p < \infty$の場合は$1/p+1/q=1$とするとH\"{o}lderの不等式より
  \[
  \int_{-\pi}^{\pi} |f^{(r)}(t)| dt \leq \left(\int_{-\pi}^\pi 1^q dt \right)^{q}  \left(\int_{-\pi}^\pi |f^{(r)}(t)|^p dt \right)^{p} 
  \]
  だから成立する. 
\end{proof}
\begin{cor}
  $I = (-\pi,\pi) \subset \R$とおく. $r \geq 1$, $1 \leq p \leq \infty$に対して定数$C_{r,p} > 0$が存在して, 周期$2\pi$の周期関数$f:\R \rightarrow \R$で任意の有界開区間$J \subset \R$について$f \in W^{r,p}(J)$であるものに対して
  \[
  \omega_{r}(f)_{p,I}\left(\frac{1}{n}\right) \leq \frac{C_{r,p}}{n^r} \lVert f^{(r)} \rVert_{L^p(I)} ~~~~(n=1,2,3,\ldots).
  \]
\end{cor}
\begin{proof}
  $N_{r,p} \in \N$を十分大きく取れば$n \geq N_{r,p}$に対して$1/n \in (0,2\pi/r)$となる. よって, すぐ上の定理より
  \[
  \omega_{r}(f)_{p,I}\left(\frac{1}{n}\right) \leq \frac{C_{r,p}}{n^r} \lVert f^{(r)} \rVert_{L^p(I)} ~~~~(\forall n \geq N_{r,p}).
  \]
  となる. そこで$C_{r,p}$を必要に応じて大きくすれば$n=1,\ldots,N_{r,p}$についても
  \[
  \omega_{r}(f)_{p,I}\left(\frac{1}{n}\right) \leq \frac{C_{r,p}}{n^r} \lVert f^{(r)} \rVert_{L^p(I)}
  \]
  となる. これで示せた. 
\end{proof}

\subsection{Baireのカテゴリー定理}
本節では完備距離空間におけるBaireのカテゴリー定理について述べる. 
\begin{thm}[Baireのカテゴリー定理]　\\
  $(X,d)$を完備距離空間とし, $(V_n)_{n\in \mathbb{N}}$を$X$において稠密な開集合からなる列とする. このとき $\bigcap_{n\in\mathbb{N}}V_n$は$X$において稠密である. 
\end{thm}
\begin{proof}
  任意の$x_0\in X$と$r_0\in (0,\infty)$に対し, 
  \[
  B(x_0,r_0)\cap \bigcap_{n\in \mathbb{N}}V_n\neq\emptyset\quad\quad(*)
  \]
  が成り立つことを示せばよい. $x_0\in X=\overline{V_1}$より$B(x_0,r_0)\cap V_1\neq\emptyset$である. よって, 
  \[
  \overline{B}(x_1,r_1)\subset B(x_0,r_0)\cap V_1,\quad 0<r_1\leq\frac{r_0}{2}
  \]
  なる閉球$\overline{B}(x_1,r_1)=\{x\in X:d(x,x_1)\leq r_1\}$が取れる. $x_1\in X=\overline{V_2}$より$B(x_1,r_1)\cap V_2\neq\emptyset$ であるから,  
  \[
  \overline{B}(x_2,r_2)\subset B(x_1,r_1)\cap V_2,\quad 0<r_2\leq\frac{r_0}{2^2}
  \]
  なる閉球$\overline{B}(x_2,r_2)$が取れる. 同様のことを繰り返せば閉球の列 $(\overline{B}(x_n,r_n))_{n\in \mathbb{N}}$で, 
  \[
  \overline{B}(x_{n},r_{n})\subset B(x_{n-1},r_{n-1})\cap V_{n},\quad 0<r_n\leq\frac{r_0}{2^n}\quad(\forall n\in \mathbb{N})\quad\quad(**)
  \]
  を満たすものが取れる. このとき, 
  \[
  \bigcap_{n\in \mathbb{N}}\overline{B}(x_n,r_n)
  \subset \bigcap_{n \in \N} B(x_{n-1},r_{n-1}) \cap \bigcap_{n\in \mathbb{N}}V_n
  \subset B(x_0,r_0)\cap \bigcap_{n\in \mathbb{N}}V_n
  \]
  である. よって$(*)$を示すためには$\bigcap_{n\in \mathbb{N}}\overline{B}(x_n,r_n)\neq\emptyset$を示せばよい. $m>n$ なる任意の$n,m\in \mathbb{N}$に対して$(**)$より, 
  \[
  d(x_m,x_n)\leq d(x_m,x_{m-1})+\ldots+d(x_{n+1},x_n)\leq r_{m-1}+\ldots+r_n\leq2\frac{r_0}{2^n}
  \]
  である. これより$(x_n)_{n\in \mathbb{N}}$はCauchy列であるので, $X$の完備性よりある$x\in X$に収束する. $x\in \bigcap_{n\in \mathbb{N}}\overline{B}(x_n,r_n)$が成り立つことを示そう. 任意の$n\in \mathbb{N}$と$\epsilon\in (0,\infty)$に対して, $m\geq n$かつ $d(x_m,x)<\epsilon$を満たす$m\in \mathbb{N}$を取れば, $x_m\in B(x,\epsilon)\cap \overline{B}(x_n,r_n)\neq\emptyset$である.
  よって$\epsilon$の任意性より$x\in \overline{B}(x_n,r_n)$であり, $n\in \mathbb{N}$の任意性より$x\in \bigcap_{n\in \mathbb{N}}\overline{B}(x_n,r_n)$である. 
\end{proof}
\begin{cor}\label{BaireCategoryTheorem}
  $X$を空でない完備距離空間, $(F_n)_{n\in \mathbb{N}}$を$X$ の閉集合からなる列とし, $X=\bigcup_{n\in \mathbb{N}}F_n$とする.  このときある$n\in \mathbb{N}$に対して$F_n$は内点を持つ. 
\end{cor}
\begin{proof}
  任意の$n\in \mathbb{N}$に対して$F_n$が内点を持たないと仮定する. 
  このとき, $V_n:=X \setminus F_n$ $(\forall n\in \mathbb{N})$とおけば, $(V_n)_{n\in \mathbb{N}}$は$X$において稠密な開集合からなる列である. 
  実際, ある$x \in X$と$\varepsilon > 0$に対して$B(x,\varepsilon) \cap V_n = \emptyset$とすると$B(x,\varepsilon) \subset F_n$となり$F_n$が内点を持たないことに矛盾する. 
  よって, Baireのカテゴリ定理より$\bigcap_{n\in\mathbb{N}}V_n$は$X$において稠密である. 
  ところが$\bigcap_{n\in \mathbb{N}}V_n=X\setminus \bigcup_{n\in \mathbb{N}}F_n=\emptyset$であるからこれは$X$が空でないことに矛盾する. 
\end{proof}

\subsection{$C^{\infty}$級関数が多項式と一致するための必要十分条件}
本節では開区間上の$C^{\infty}$級関数が$\R$係数多項式と一致するための必要十分条件について述べる.  以下では, ある開区間上で値が一致する二つの$\R$係数多項式は多項式として等しいという事実を断りなく用いる. さて, 簡単な必要十分条件としては, ある整数$N \geq 0$が存在して$N$階微分が恒等的に$0$となるというものがあるが, 実は次の定理が成立する. なお, 以下の証明は\cite{Donoghue}を参考にした. 
\begin{thm}\label{ConditionOfSmoothFunctionIsPoly}
  $f$を開区間$(c,d)$上の$C^{\infty}$級関数とする. このとき, 多項式$p$が存在して$(c,d)$上$f=p$となるための必要十分条件は, 任意の$x \in (c,d)$に対してある整数$N_x \geq 0$が存在して$f^{(N_x)}(x)=0$となることである. 
\end{thm}
\begin{proof}
  必要性は明らかである. 十分性を示す. 
  十分小さいすべての$\varepsilon > 0$についてある多項式$p$が存在して$[c+\varepsilon,d-\varepsilon]$上$f=p$となることを示せばよいので, はじめから$[c,d]$上で$C^{\infty}$級として示せばよい(端点での導値は片側微分の意味で考える). 開集合$G \subset [c,d]$を次のように定義し$F := [c,d] \setminus G$とおく. 
  \[
  G := \{x \in [c,d] \mid x\mbox{の}[c,d]\mbox{における開近傍}U\mbox{と多項式}p\mbox{が存在して}U\mbox{上}f=p \}
  \]
  $F = \emptyset$となることを示せばよい. 実際, そのとき$G = [c,d]$であり, ある多項式$p_1,p_2$と$\delta > 0$があって$[c,c+2\delta)$上$f=p_1$かつ$(d-2\delta,d]$上$f=p_2$となる. そして, $a=c+\delta, b= d - \delta$とおき, $[a,b] \subset (c,d)$の各元$x$に対して$x$を含む開区間$I_x \subset (c,d)$と多項式$p_x$で$I_x$上$f = p_x$なるものを取る. このとき$[a,b] \subset \bigcup_{x \in [a,b]} I_x$であるので$[a,b]$のコンパクト性より有限個の$x_1,\ldots,x_n \in [a,b]$が存在して$[a,b] \subset I_{x_1} \cup \cdots \cup I_{x_n}$となる. そこで$I_0 = [c,c+2\delta),  I_1=I_{x_1}, \ldots, I_n=I_{x_n}, I_{n+1} = (d-2\delta,d]$とおくと$[c,d] = \bigcup_{k=0}^{n+1} I_k$となる. よって$[c,d]$の連結性より多項式として
  \[
  p_1 = p_{x_1}=\cdots=p_{x_n}=p_2
  \]
  がわかり$[c,d]$上$f = p_1$となる. よって, $F = \emptyset$を示せばよい. 以下では$F \neq \emptyset$と仮定して矛盾を導く. \\
  {\bf(Step1)}:$F$が孤立点を持たないことを示す. $x_0 \in F$が孤立点であるとして矛盾を導く. \\
  $(1)$:$x_0 = d$の場合. 
  ある$b \in (c,d)$が存在して$(a,x_0) \subset G$となる. 
  このとき, ある$\delta > 0$と多項式$P$が存在して$(a,a+\delta) \subset (a,x_0)$上$f=P$となる. そこで
  \[
  y_0 := \sup\{y \in (a,x_0] \mid \forall x \in (a,y),~ f=P \}
  \]
  とおくと$y_0 = x_0$となる. 
  実際, $y_0 < x_0$とすると$y_0 \in (a,x_0) \subset G$であるので, ある$\varepsilon>0$と多項式$Q$が存在して$(y_0-\varepsilon,y_0+\varepsilon) \subset (a,x_0]$上で$f=Q$となる.
  一方, $y_0$の定義から$y-\varepsilon<z\leq y_0$で$(a,z)$上で$f=P$となるものが取れる. 
  従って, $(y-\varepsilon,z)$上で$P=f=Q$となるので多項式として$P=Q$である. よって, $(a,y_0+\varepsilon)$上で$f=Q$となるがこれは$y_0$の定義に反する. 以上から$(a,x_0]$上で$f=P$となる. よって, $(a.x_0]$は$x_0$の$[c,d]$における開近傍であるので$x_0 \in G$となる. これは不合理. \\
  $(2)$:$x_0=c$の場合. $(1)$と同様に矛盾が導かれる. \\
  $(3)$:$x_0 \in (c,d)$の場合. 
  ある$a,b \in(c,d)$が存在して$[a,x_0), (x_0,b] \subset G$となる. 
  このとき, $(1)$の証明と同様にある多項式$P$が存在して$(a,x_0]$上で$f=P$となる. また, ある多項式$P'$が存在して$[x_0,b)$上で$f=P'$となる. このとき$P=P'$でなければならない. 
  このことを示すためにまず, $n:=\mathrm{deg}(P) = \mathrm{deg}(P') =: m$を示す. 
  いま$(a,x_0]$上で$f$は$n$次多項式なので$f^{(k)}(x_0)=0 ~(k \geq n+1)$である. 一方, $[x_0,b)$上で$f$は$m$次多項式なので$f^{(m)}(x_0) \neq 0$である. そこでもし$n < m$ならば$0 = f^{(m)}(x_0) \neq 0$となる. これは不合理であるので$n \geq m$である. 同様に$n \leq m$なので$n=m$である. 次に$x_0$中心の$f$のTaylor展開を考えると
  \[
  \begin{aligned}
  &P(x) = f(x) = \sum_{k=0}^n \frac{f^{(k)}(x_0)}{k!}(x-x_0)^k ~~~(x \in (a,x_0]) \\
  &P'(x) = f(x) = \sum_{k=0}^n \frac{f^{(k)}(x_0)}{k!}(x-x_0)^k ~~~(x \in [x_0,b)) 
  \end{aligned}
  \]
  となる. これより
  \[
  \begin{aligned}
  &P(x) = \sum_{k=0}^n \frac{f^{(k)}(x_0)}{k!}(x-x_0)^k ~~~(x \in \R) \\
  &P'(x) = \sum_{k=0}^n \frac{f^{(k)}(x_0)}{k!}(x-x_0)^k ~~~(x \in \R)
  \end{aligned}
  \]
  となるから$P=P'$である. 以上から$(a,b)$上で$f=P$である. これは$x_0 \in G$を意味する. ところが$x_0 \in F$なのでこれは不合理. よって$F$は孤立点を持たない. \\
  {\bf(Step2)}:整数$N \geq 0$と長さ正の閉区間$I \subset (c,d)$が存在して$F \cap I \neq \emptyset$かつ$F \cap I$上$f^{(k)}=0 ~(\forall k \geq N)$となることを示す. 
  整数$n \geq 0$に対して$E_n := \{x \in F \mid f^{(n)}(x) = 0\}$とおく. $E_n$は$F$の閉集合であり, 定理の仮定より$F = \bigcup_{n=0}^{\infty} E_n$である. いま$F \neq \emptyset$は閉区間$[c,d]$の閉集合なので完備距離空間である. よって, Baireのカテゴリー定理(系\ref{BaireCategoryTheorem})よりある整数$N \geq 0$が存在して$E_N$は内点を持つ. ゆえに長さ正の閉区間$I \subset (c,d)$が存在して$\emptyset \neq F \cap I \subset E_N$となる. Step$1$より$F$は孤立点を持たないので$F \cap I$も孤立点を持たない. ゆえに$x \in F \cap I$に対して$x_n \rightarrow x$なる$(F \cap I) \setminus \{x\}$の元からなる列$(x_n)_{n \in \N}$が取れるので, 
  \[
  f^{(N+1)}(x) = \lim_{n \rightarrow \infty} \frac{f^{(N)}(x_n) - f^{(N)}(x)}{x_n - x} = \lim_{n \rightarrow \infty} \frac{0 - 0}{x_n - x} = 0
  \]
  となる. よって帰納的に$F \cap I$上$f^{(k)}=0 ~(\forall k \geq N)$となる. \\
  {\bf (Step3)}:Step2の$N$と$I$について$I$上で$f^{(N)}=0$となることを示す. これが示されれば$I \subset G$であるので$\emptyset \neq F \cap I = \emptyset$なる矛盾に到達する. 
  さて, Step2より$F \cap I$上で$f^{(N)}=0$であるので$G \cap I \neq \emptyset$として$G \cap I$上で$f^{(N)}=0$となることを示せばよい. 
  任意に$x_0 \in G \cap I$をとる. $x_0$を含む開区間$(a,b) \subset G \cap I$と多項式$p$が存在して$(a,b)$上$f=p$となる. ここで
  \[
  a_0 = \inf \{ y \in I \mid \forall x \in (y,b),~ f(x)=p(x) \}
  \]
  とおく. すると, $a_0 \in F \cap I$である. 
  実際, $a_0 \in G \cap I$であるならば, $I \subset (c,d)$より$a_0$は端点ではないので, 多項式$p'$と$\delta>0$が存在して$(a_0-\delta,a_0+\delta) \subset G \cap I$上で$f=p'$となる. すぐわかるように$(a_0,a_0+\delta)$上で$p'=f=p$であるので多項式として$p=p'$となり, 従って, $(a_0-\delta,b)$上で$f=p$となる. これは$a_0$の定義に反する. よって$a_0 \in F \cap I$である. 
  ゆえに, Step2より$f^{(k)}(a_0)=0 ~(\forall k \geq N)$となる. 
  一方, $(a_0,b)$上で$f=p$であるので$\mathrm{deg}(p) = m$とおくと, 
  \[
  f(z) = \sum_{k=0}^m \frac{f^{(k)}(a_0)}{k!}(z-a_0)^k ~~~(z \in (a_0,b))
  \]
  となる. これと$f^{(k)}(a_0)=0 ~(\forall k \geq N)$を合わせると
  \[
  f(z) = \sum_{k=0}^{N-1} \frac{f^{(k)}(a_0)}{k!}(z-a_0)^k ~~~(z \in (a_0,b))
  \]
  となる. よって, $f^{(N)}(x_0) = 0$である. 
\end{proof}

\end{document}